\documentclass[a4paper,11pt]{article}     %
\usepackage[margin=2cm]{geometry}          %
\usepackage{comment}                       %
\usepackage{placeins}                     %
\AddToHook{env/thebibliography/begin}{\footnotesize}
\newif\ifpreprintbuild \preprintbuildtrue  %
\makeatletter
\newcommand{\preprintonly}[1]{\ifpreprintbuild\expandafter\@firstofone\else\expandafter\@gobble\fi{#1}}
\newcommand{\journalonly}[1]{\ifpreprintbuild\expandafter\@gobble\else\expandafter\@firstofone\fi{#1}}
\newcommand{\journalpreprint}[2]{\ifpreprintbuild\expandafter\@secondoftwo\else\expandafter\@firstoftwo\fi{#1}{#2}}
\makeatother
\NewDocumentEnvironment{preprinttext}{+b}{\preprintonly{#1}}{}
\NewDocumentEnvironment{journaltext}{+b}{\journalonly{#1}}{}

\newcommand{\preprinttitle}{\fwname{}: Neuro-Symbolic Technical Intelligence for Auditing Engineering Specifications with an Application to Aircraft}                        %
\newtheorem{proposition}{Proposition}      %
\newenvironment{pf}{\par\noindent\textit{Proof.}\ \ignorespaces}{\hfill$\square$\par}
\makeatletter
\def\preprint@newlabel#1#2{\@newl@bel r{#1}{#2}\@newl@bel r{S-#1}{#2}\@newl@bel r{M-#1}{#2}}
\AddToHook{begindocument/before}{\let\newlabel\preprint@newlabel}
\makeatother

\usepackage[T1]{fontenc}
\usepackage[utf8]{inputenc}
\usepackage[english]{babel}
\ifdefined\pdfgentounicode   %
  \input{glyphtounicode}
\fi

\usepackage{graphicx}
\graphicspath{{figures/}}
\usepackage{adjustbox}
\usepackage{booktabs}
\usepackage{multirow}
\usepackage{array}
\usepackage{tabularx}
\usepackage{threeparttable}

\usepackage{amsmath}
\usepackage{amssymb}
\usepackage{textcomp}

\usepackage[ruled,vlined,linesnumbered]{algorithm2e}

\usepackage{tikz}
\usetikzlibrary{arrows.meta,positioning,fit,backgrounds,calc,shapes.geometric,decorations.pathreplacing}

\usepackage{url}
\usepackage{xcolor}
\usepackage{comment}
\usepackage{cite}                  %
\usepackage[hidelinks]{hyperref}   %

\newif\ifblind  \blindfalse
\newif\ifmarkup \markupfalse

\ifmarkup
  
  \DeclareRobustCommand{\omesso}[1]{\textcolor{red}{\relax#1}}
  
  \newenvironment{omessoblock}{\par\color{red}}{\par}
\else

  \ifblind
    \DeclareRobustCommand{\omesso}[1]{#1}
    
  \else
    \DeclareRobustCommand{\omesso}[1]{}
    \excludecomment{omessoblock}
  \fi
\fi

\newcommand{\fwname}{NeSySpec}   %
\newcolumntype{L}[1]{>{\raggedright\arraybackslash}p{#1}}
\newcolumntype{C}[1]{>{\centering\arraybackslash}p{#1}}
\providecommand{\keywords}[1]{\par\medskip\noindent\textbf{Keywords:} #1}

\makeatletter
\newenvironment{widefigure}[1][t]%
  {\if@twocolumn\def\wide@env{figure*}\else\def\wide@env{figure}\fi
   \begin{\wide@env}[#1]}%
  {\end{\wide@env}}
\newenvironment{widetable}[1][t]%
  {\if@twocolumn\def\wide@env{table*}\else\def\wide@env{table}\fi
   \begin{\wide@env}[#1]}%
  {\end{\wide@env}}
\makeatother

\definecolor{neuralfill}{HTML}{E8F0FE}  \definecolor{neuralline}{HTML}{1A56A8}
\definecolor{symbfill}{HTML}{FDF0E3}    \definecolor{symbline}{HTML}{A85A1A}
\definecolor{iofill}{HTML}{F0F0F2}      \definecolor{ioline}{HTML}{55555C}
\definecolor{outfill}{HTML}{E7F4EC}     \definecolor{outline}{HTML}{1E7A45}
\definecolor{expertfill}{HTML}{F5E8FD}  \definecolor{expertline}{HTML}{7A1AA8}
\tikzset{
  blk/.style    = {draw, rounded corners=2pt, align=center, font=\footnotesize,
                   minimum height=9mm, minimum width=15mm, inner sep=3pt,
                   line width=0.6pt},
  neural/.style = {blk, fill=neuralfill, draw=neuralline},
  symb/.style   = {blk, fill=symbfill,   draw=symbline},
  io/.style     = {blk, fill=iofill,     draw=ioline},
  result/.style = {blk, fill=outfill,    draw=outline},
  expert/.style = {blk, fill=expertfill, draw=expertline},
  fusion/.style = {blk, left color=neuralfill, right color=symbfill, draw=ioline},
  key/.style    = {minimum width=3.5mm, minimum height=2.5mm, inner sep=0pt},
  flow/.style   = {-{Stealth[length=2mm]}, line width=0.6pt, draw=ioline},
  lbl/.style    = {font=\scriptsize\itshape, text=ioline},
}

\newcommand{\feas}{\mathrm{feas}}
\newcommand{\sat}{\mathrm{sat}}
\newcommand{\Loc}{\mathrm{Loc}}
\newcommand{\vis}{\mathrm{vis}}
\newcommand{\pME}{\mathrm{pME}}

\newcommand{\IR}{I_{R}}

\newcommand{\KB}{\mathcal{K}}
\newcommand{\Aset}{\mathcal{A}}
\newcommand{\Rset}{\mathcal{R}}
\newcommand{\AVset}{\mathcal{AV}}
\newcommand{\AXset}{\mathcal{AX}}
\newcommand{\pvis}{p_{\mathrm{vis}}}

\newcommand{\lamaux}{\lambda_{\mathrm{aux}}}

\newcommand{\MTOW}{\mathit{MTOW}}
\newcommand{\OEW}{\mathit{OEW}}
\newcommand{\AR}{\mathit{AR}}

\newcommand{\src}[1]{}

\newcommand{\pzdmark}[1]{{\usefont{U}{pzd}{m}{n}\char#1}}
\newcommand{\capY}{\pzdmark{51}}
\newcommand{\capP}{(\pzdmark{51})}
\newcommand{\capN}{\pzdmark{55}}
\newcommand{\capYb}{\pzdmark{52}}

\usepackage{orcidlink}
\hypersetup{pdftitle={NeSySpec: Neuro-Symbolic Technical Intelligence for Auditing Engineering Specifications with an Application to Aircraft}}
\ifblind
  \hypersetup{pdfauthor={Anonymous Authors}}
\else
  \hypersetup{pdfauthor={Pier Mattia Basciano, Gianvito Antenori, Cristiano Andolfi, Giovanni Farina, Salvatore Monteleone}}
\fi

\begin{document}

\title{\preprinttitle}

\ifblind
  \author{Anonymous Authors\\ \textit{Anonymous Institution}}
\else
  \author{%
    Pier Mattia Basciano\textsuperscript{1,2,$\ast$,$\dagger$}\,\orcidlink{0009-0007-2551-4381}\quad
    Gianvito Antenori\textsuperscript{1,2}\,\orcidlink{0009-0000-4396-8490}\quad
    Cristiano Andolfi\textsuperscript{2}\,\orcidlink{0009-0000-3337-9342}\\
    Giovanni Farina\textsuperscript{1}\,\orcidlink{0000-0002-4792-5305}\quad
    Salvatore Monteleone\textsuperscript{1}\,\orcidlink{0000-0003-0158-2295}\\[0.6ex]
    \small\textsuperscript{1}\textit{Department of Engineering, Niccol\`o Cusano University, Rome, Italy}\\[0.2ex]
    \small\textsuperscript{2}\textit{SOMA Research Division, Rome, Italy}\\[0.4ex]
    \footnotesize\texttt{\{piermattia.basciano, gianvito.antenori, giovanni.farina, salvatore.monteleone\}@unicusano.it}\\
    \footnotesize\texttt{\{piermattia.basciano, gianvito.antenori, cristiano.andolfi\}@somatwin.com}\\[0.4ex]
    \footnotesize\textsuperscript{$\ast$}Corresponding author.\quad
    \textsuperscript{$\dagger$}Lead author, who carried out the large majority of the work.}
\fi
\date{}
\maketitle
\suppressfloats[t]%

\begin{abstract}
Engineers who audit a declared specification of an engineered system need to know which physical constraints it respects, which it violates, which field was altered and how to correct it. \fwname{} is an integrated and interpretable neuro-symbolic audit framework, built on Logic Tensor Networks, that answers these questions in one pipeline. Its axioms encode the physical relations of the domain as differentiable graded predicates whose tolerances are learned from a fleet of real records; from the same degrees, it diagnoses the violated constraint, identifies the altered field by abduction, proposes a minimum-change correction and ties image evidence to the parts it describes, abstaining exactly where evidence is missing, and returns all of it in one verification report. We validate \fwname{} on 107 aircraft types, with 16 test aircraft of 9 families held out of training, on specification records paired with ground photographs and satellite imagery and on synthetic errors: generated corruptions, edits labeled from the edit alone and deliberate falsifications. \fwname{} combines all the assessed audit outputs in one pipeline, which no learned, rule-based, solver-based, neuro-symbolic or language-model comparator does; on the test aircraft none is ahead of it in the macro-F1 of its constraint verdicts, in naming the violated constraint and the altered field or in the repair error, and it is competitive in record and image-pair detection. Its diagnosis names the violated constraint in 94.2\% of corrupted test records, against 79.5\% for classical Fellegi-Holt diagnosis, and on edits labeled from the edit alone, without retraining, abduction finds the altered field in 59.3\% of records, against 30.6\% for minimum-cardinality localization with random tie-breaking. Only the knowledge base, the visual predicates, the corruption generator and the data are specific to the domain.
\end{abstract}


\par\medskip\noindent\textbf{License.} This work is licensed under a Creative Commons Attribution 4.0 International License (CC BY 4.0, \url{https://creativecommons.org/licenses/by/4.0/}). Third-party photographs reproduced in the figures are excluded from this license.

\section{Introduction}
\label{sec:intro}

Certification, maintenance and procurement engineers work with engineered systems they did not design and reverse-engineer them from declared specifications of masses, dimensions, ratings and performance.
A record can be incomplete, inconsistent, valid for another variant or misreported, and public knowledge bases hold wrong numerical values~\cite{wienand2014detecting}; the errors of a record propagate into the performance models it feeds~\cite{sun2020openap}.
Auditing such records falls within the scope of civil technical intelligence, which gathers and assesses information about technologies and technical products~\cite{lichtenthaler2004technology,zhang2016technology}. Technical intelligence increasingly draws on open sources~\cite{glassman2012intelligence,pastorgalindo2020osint}. For aircraft, photographs and satellite imagery can also help verify the visually observable attributes reported in the specifications~\cite{maji2013fine,shermeyer2021rareplanes}.
We use reverse engineering in a first-approach sense: inferring which physical relations a declared specification and, when available, its image respect or violate, which field departs from them and which value restores them.
Table~\ref{tab:acronyms} lists the acronyms.

\begin{table}[t]%
  \centering
  \caption{Acronyms.}
  \label{tab:acronyms}
  \footnotesize
  \begin{tabularx}{\textwidth}{@{}l>{\raggedright\arraybackslash}X@{\hspace{1.5em}}l>{\raggedright\arraybackslash}X@{}}
    \toprule
    Acronym & Meaning & Acronym & Meaning \\
    \midrule
    AUROC & area under the receiver operating characteristic curve & MLP & multilayer perceptron \\
    BCE   & binary cross-entropy & MLP-SHAP & grouped SHAP attributions of the residual MLP \\
    CAD   & computer-aided design & MTOW & maximum take-off mass \\
    CAM   & class activation map & NeSy & neuro-symbolic artificial intelligence \\
    CFD   & computational fluid dynamics & OEW & operating empty mass \\
    FAA   & Federal Aviation Administration & PGD & projected gradient descent \\
    FEM   & finite element method & PSL & probabilistic soft logic \\
    FPR   & false-positive rate & RRR & Right for the Right Reasons \\
    ISA   & International Standard Atmosphere & SAM & Segment Anything Model \\
    LEN   & Logic Explained Networks & SHAP & Shapley additive explanations \\
    LTN   & Logic Tensor Networks & SMT & satisfiability modulo theories \\
    MCAR  & missing completely at random & SVM & support vector machine \\
    MICE  & multivariate imputation by chained equations & TPR & true-positive rate \\
    \bottomrule
  \end{tabularx}
\end{table}

Learned methods, from anomaly detectors to tabular foundation models and language models, score each record~\cite{han2022adbench,hollmann2025accurate,tsai2025anollm}, with field-level attributions~\cite{lundberg2017unified}. Statistical data editing and consistency-based diagnosis identify violated constraints and the fields implicated in those violations, based on bounds set by experts~\cite{fellegi1976systematic,reiter1987theory}.
Neuro-Symbolic Artificial Intelligence (NeSy) integrates learning with explicit knowledge and reasoning~\cite{davilagarcez2023neurosymbolic,vonrueden2021informed}, also surveyed for advanced air mobility~\cite{acharya2025integrating}, and can learn such bounds from data together with the constraints they delimit. To our knowledge, no previous work provides and evaluates together the outputs that such an audit needs, from the violated constraint and the altered field to part-grounded image evidence and a report (Table~\ref{tab:sota}).

We present \fwname{}, an integrated and interpretable neuro-symbolic framework for auditing engineering specifications, built on Logic Tensor Networks (LTN)~\cite{badreddine2022ltn}: trained on fleet data, it determines through its logical axioms which physical constraints a declared record respects and which it violates, and returns a verification report (Figures~\ref{fig:pipeline} and~\ref{fig:report}). %
Its graded predicates measure each physical, definitional or regulatory relation by its residual against bands learned from the training fleet, and cross-modal axioms compare the fields with what a vision module reads.
Masked quantification gives \fwname{} exact abstention on missing evidence: each verdict reduces to the constraints that the available evidence makes checkable, and the report lists every undeclared field with the constraints it leaves unchecked. %
The audit proceeds from detection to diagnosis and repair. It first scores the record using a masked mean of constraint violation degrees. This overall verdict can then be traced back to individual constraints, with the largest violation degree identifying the most violated constraint. An abductive search over single-field changes then localizes the altered field. Finally, a minimum-change repair searches for the smallest adjustment that restores coherence. %
A grounding axiom encourages each visual head to place its class-activation evidence on the part its predicate describes, so that the report ties image evidence to parts. %

\begin{widefigure}[t]
\centering
\begin{adjustbox}{max width=\textwidth}
\begin{tikzpicture}[
  trainflow/.style = {flow, densely dotted},
  alarmflow/.style = {flow, dashed},
]
  \node[io] (rec) at (0.85, 1.85) {Specification\\records};
  \node[io] (pho) at (0.85, 0)    {Ground\\photographs};
  \node[io] (sat) at (0.85,-1.85) {Satellite\\crops};

  \node[io, minimum height=50mm] (dat) at (3.05, 0)
       {Data and\\verification\\instances\\[3pt]{\scriptsize $(x, m, I)$}};
  \node[symb]   (kb)  at (6.05, 1.85) {Knowledge base\\and graded\\predicates};
  \node[expert] (exp) at (6.05, 0)    {Physical relations,\\train-prior bands,\\part definitions};
  \node[neural] (vis) at (6.05,-1.85) {Vision module\\and grounding\\axiom};
  \node[symb]   (ax)  at (8.9,-1.85) {Cross-modal\\axioms};
  \node[fusion] (agg) at (8.9, 0)    {Masked\\stratified\\aggregation};
  \node[fusion, right=1.25cm of agg] (hd)
       {Verdict\\and keys\\[1pt]{\scriptsize masked verdict,}\\{\scriptsize largest-}\\{\scriptsize violation key,}\\{\scriptsize second head}};
  \node[io, dashed] (alm) at ([yshift=3.5mm]hd |- rec) {Optional settings\\[1pt]{\scriptsize rank-then-explain alarm,}\\{\scriptsize critical-field alert,}\\{\scriptsize MICE imputation}};%
  \node[fusion, right=1.2cm of hd] (rep) {Abduction\\and repair};%
  \coordinate (outlane) at (0,-2.45);
  \node[result, anchor=east] (out) at (rep.east |- outlane)
       {\textbf{Verification report}\\[1pt]{\scriptsize constraint verdicts,}\\{\scriptsize violated constraint,}\\{\scriptsize altered field, repair,}\\{\scriptsize coverage, part maps}};

  \draw[flow] (rec) -- (rec -| dat.west);
  \draw[flow] (pho) -- (pho -| dat.west);
  \draw[flow] (sat) -- (sat -| dat.west);
  \draw[flow] (dat.east |- kb)  -- node[lbl, above] {$x$, $m$} (kb);
  \draw[flow] (dat.east |- vis) -- node[lbl, above] {$I$} (vis);
  \draw[flow] (exp) -- node[lbl, right] {axioms} (kb);
  \draw[trainflow] (exp) -- node[lbl, right] {parts} (vis);
  \draw[flow] (vis) -- node[lbl, above] {$\pvis$} (ax);
  \draw[flow] (dat.south) -- ++(0,-0.45) -| node[lbl, below, pos=0.25] {declared fields $x$} (ax.south);
  \draw[flow] (kb) -| node[lbl, above, pos=0.3] {$v_k$, $\alpha_k$} (agg.north);
  \draw[flow] (ax) -- node[lbl, right] {$\sat_k$, $\alpha_k$} (agg);
  \draw[flow] (agg) -- node[lbl, above] {$\feas$} node[lbl, below] {$v_k$, $\alpha_k$} (hd);
  \draw[flow] (hd) -- node[lbl, above] {verdicts} (rep);
  \draw[flow] (rep.south) -- (rep.south |- out.north);
  \draw[trainflow] (dat.north) -- ++(0,0.45) -| node[lbl, above, pos=0.25] {train-prior bands} ([xshift=-6mm]kb.north);
  \draw[alarmflow] ([xshift=6mm]kb.north) -- ++(0,0.8) -| node[lbl, above, pos=0.25] {physical residuals} (alm.north);
  \draw[alarmflow] (alm) -| node[lbl, above, pos=0.25] {flag} (rep.north);

  \begin{scope}[shift={(0,-3.75)}, every node/.append style={font=\scriptsize, text height=1.5ex, text depth=0.25ex}]
    \node[neural, key] (k1) at (0,0)      {}; \node[right=1mm of k1] {neural};
    \node[symb,   key] (k2) at (3.2,0)    {}; \node[right=1mm of k2] {symbolic};
    \node[fusion, key] (k3) at (6.4,0)    {}; \node[right=1mm of k3] {neuro-symbolic};
    \node[expert, key] (k4) at (9.6,0)    {}; \node[right=1mm of k4] {expert knowledge};
    \node[io,     key] (k5) at (0,-0.5)   {}; \node[right=1mm of k5] {data};
    \node[result, key] (k6) at (3.2,-0.5) {}; \node[right=1mm of k6] {output};
    \draw[trainflow] (6.22,-0.5) -- ++(7mm,0) node[right=1mm] {training time};
    \draw[alarmflow] (9.42,-0.5) -- ++(7mm,0) node[right=1mm] {alarm branch};
  \end{scope}
\end{tikzpicture}
\end{adjustbox}
\caption{Overview of the \fwname{} pipeline.}
\label{fig:pipeline}
\end{widefigure}

We validate \fwname{} on 107 aircraft, with specification records, FGVC-Aircraft photographs~\cite{maji2013fine} and RarePlanes satellite imagery~\cite{shermeyer2021rareplanes}, holding 16 test aircraft of 9 families out of training, and evaluate constraint verdicts and diagnosis, field localization and repair, record coherence, missing evidence and absent images, the grounding of the visual axioms and robustness to deliberate falsification, on errors labeled by the corruption generator and on errors labeled from the edit alone (Section~\ref{sec:res}). %
The constraint verdicts reach a macro-F1 of 0.795 on generator labels and, without retraining, 0.575 on intervention labels, above a network trained with the DL2 loss on the same residuals and targets, the like-for-like neuro-symbolic test, and find every labeled violation in 0.934 of generator-labeled test corruptions against 0.522 for that network, with micro precision on par on test and lower on the unseen class (Section~\ref{sec:res-diag}, Appendix~\ref{S-app:results-verdicts}). %
The largest-violation key identifies the violated constraint in 0.942 of generator-labeled test corruptions, outperforming Fellegi--Holt, solver-based and rule-based diagnosis, DL2 applied to the declared fields, and a prompted language model. Abduction then traces the error to an individual field, correctly identifying the altered field in 0.593 of independently labeled errors. By comparison, minimum-cardinality localization reaches 0.306 with random tie-breaking and 0.453 with oracle tie-breaking; a sizing-regression residual also reaches 0.453 (Figure~\ref{fig:diagnosis}). %
The record verdict detects incoherent records at a test AUROC of 0.979 on generator labels, and on satellite pairs the multimodal verifier, which names the contradicted axiom, detects incoherent record-image pairs at 0.994 over all pair types, against 0.748 for a prompted vision-language model (Section~\ref{sec:res-xmodal}). Table~\ref{tab:glance} lines these results up. No comparator returns every output of the audit, and on the test aircraft no comparator is ahead of \fwname{} in the macro-F1 of the constraint verdicts, in naming the violated constraint and the altered field or in the repair error on the altered field; in record and pair detection, learners on the residuals and late fusions are close to it and in some settings rank higher (Section~\ref{sec:res-setup}). %

Among the approaches evaluated here - the DL2 method on the declared fields, the network trained with the DL2 loss on the same residuals read per constraint, a constraint-layer repair ordering, LTN ablations, solver and rule diagnosis and prompted language and vision-language models -, \fwname{} is, to our knowledge, the most suitable neuro-symbolic framework for auditing engineering specifications: it provides and evaluates together the seven outputs of Table~\ref{tab:sota}, from constraint verdicts and field localization to part-grounded image evidence and a verification report. Its logic contributes more to the physical audit than to record classification. Training its knowledge base through \emph{Real Logic} (the differentiable formulation of first-order fuzzy logic underlying LTN~\cite{badreddine2022ltn}) improves detection over the same knowledge base left at its initial values. Detection also improves as more knowledge is encoded, a trend shared by a residual MLP using the same constraints (Section~\ref{sec:res-knowledge}). %
Only the knowledge base, visual predicates, corruption generator and data of \fwname{} are aircraft-specific, so another engineered system enters through these four parts (Appendix~\ref{app:porting}). %

The contributions of this work are:
\begin{itemize}
  \item An integrated and interpretable neuro-symbolic audit framework on Logic Tensor Networks whose graded residual-aware axioms, trained on fleet data, give each physical constraint of a declared engineering specification a held, violated or masked verdict and derive from these verdicts the altered field, a repair proposal, an exact accounting of missing evidence, part-grounded visual evidence and a verification report;
  \item A verification protocol that re-estimates the priors from training records, refits the generator oracle without test or out-of-distribution aircraft, labels errors independently of the generator, reads the held and violated verdicts against per-constraint labels and stress-tests the verdicts against falsification and human part annotations;
  \item A comparative study of rule-based and solver-based methods, models trained with constraint losses, language and vision-language models, and residual learners, all using the same physical knowledge base (Table~\ref{tab:glance}). \fwname{} identifies the violated constraint more often than classical diagnosis, DL2 applied to the declared fields, and a prompted language model. It also identifies the altered field more often than minimum-cardinality localization with random tie-breaking and produces a repair proposal from the same violation degrees (Figure~\ref{fig:diagnosis});
  \item A deployment map (Table~\ref{tab:deploy}) that assigns each audit question to the output of \fwname{} that answers it, including its optional rank-then-explain alarm and critical-field alert.
\end{itemize}

The remainder of this paper is organized as follows. Section~\ref{sec:sota} reviews background and related works. Section~\ref{sec:meth} presents the \fwname{} framework in full. Section~\ref{sec:res} reports experimental results. Section~\ref{sec:end} concludes with directions for future work. Appendix~\ref{app:porting} describes how to port \fwname{} to other engineered systems.
\section{Background and Related Works}
\label{sec:sota}

\subsection{Background}
\label{sec:background}

\paragraph{Real Logic and Logic Tensor Networks.}
LTN~\cite{badreddine2022ltn} implement Real Logic, a first-order language whose symbols are interpreted on data by a parametric grounding $\mathcal{G}_\Theta$: constants are mapped to tensors of real features, variables to sequences of individuals, and predicates to differentiable functions into $[0,1]$.
Fuzzy connectives and aggregators make the truth of every formula of a knowledge base $\KB$ differentiable, so that $\Theta$ is learnable from data.
In \fwname{} an individual is a declared record with its availability mask and an optional image, and the predicates are graded physical predicates and visual attribute heads (Sections~\ref{sec:meth-kb} and~\ref{sec:meth-vision}).
From the product configuration~\cite{badreddine2022ltn,vankrieken2022analyzing}, \fwname{} uses the negation $\neg a=1-a$, the product t-norm $a\wedge b=ab$ and the Reichenbach implication $a\rightarrow b=\IR(a,b)=1-a+ab$.
The abovementioned implication passes its consequent the gradient $\partial\IR/\partial b=a$, where the residuated Goguen and \L{}ukasiewicz implications pass none wherever $a\le b$~\cite{vankrieken2022analyzing}, so the localization pressure of the grounding axiom $P_k\rightarrow\Loc_k$ on a visual head grows with the truth of its own assertion (Section~\ref{sec:meth-vision}).
The universal quantifier is interpreted by the generalized mean error of the truth values $a_1,\dots,a_n$ that a formula takes on $n$ individuals~\cite{badreddine2022ltn},
\begin{equation}
\label{eq:bg-pme}
\pME_p(a_1,\dots,a_n)=1-\Big(\frac{1}{n}\sum_{i=1}^{n}(1-a_i)^{p}\Big)^{1/p},\qquad p\ge 1,
\end{equation}
which is the arithmetic mean for $p=1$ and tends to the minimum as $p\to\infty$.
Guarded quantification, $\forall x\,[g(x)]\colon\varphi(x)$, aggregates only over the individuals that satisfy a crisp condition $g$ and passes no gradient to the others~\cite{badreddine2022ltn}.

\paragraph{Learning from a knowledge base.}
Learning maximizes the satisfaction of the closed formulas $\varphi$ of $\KB$, aggregated by an operator such as $\pME_p$ into $\mathrm{SatAgg}$~\cite{badreddine2022ltn}, that is, it minimizes
\begin{equation}
\label{eq:bg-loss}
\mathcal{L}(\Theta)=1-\mathrm{SatAgg}_{\varphi\in\KB}\,\mathcal{G}_\Theta(\varphi).
\end{equation}
Labels enter as data axioms over the positive and negative examples of a predicate~\cite{badreddine2022ltn}, whose satisfaction, aggregated with the log-product that van Krieken et al.~\cite{vankrieken2022analyzing} recommend, is the negative binary cross-entropy (BCE) up to normalization.
In \fwname{} the verdict of a record is a masked aggregation of its axiom truths that applies this guard per axiom, with availability as the guard (Section~\ref{sec:meth-agg}), and both objectives are BCEs of data axioms, the visual one adding the loss of Eq.~\eqref{eq:bg-loss} over the grounding axioms (Sections~\ref{sec:meth-kb-obj} and~\ref{sec:meth-vision-obj}). %

\paragraph{Real Logic for specification auditing.}
Our formulation of the audit in \emph{Real Logic} has four properties. First, the bands and slopes of graded predicates are learned through their truth values, jointly with the record verdict. Second, guarded quantification uses evidence availability as its guard: constraints whose required evidence is missing are excluded, and the verdict reduces exactly to that of the available constraints (Proposition~\ref{prop:reduction}). Third, physical predicates and visual heads share a common truth scale, allowing cross-modal axioms to enter the same knowledge base. Finally, the grounding axiom uses the Reichenbach implication to scale the localization pressure with the truth of the visual predicate.

In the specification stratum, learned graded predicates are aggregated by a masked mean, with no logical connectives between them. Each axiom therefore retains its own truth value at inference, and the record verdict decomposes into constraint-level verdicts. Crisp rule validation also provides exact abstention (Table~\ref{tab:sota}), and Section~\ref{sec:res-knowledge} examines the contribution of each of the four properties. %

\subsection{Related Works}
\label{sec:related}

Specification auditing draws on several lines of research, from detecting and repairing errors in structured records to combining domain knowledge with visual evidence. We review work on data editing and diagnosis, learned anomaly detection, differentiable logic and neuro-symbolic reasoning, explanation supervision, and multimodal verification. %

\paragraph{Reverse engineering, specification records and data editing.}
Physics-driven correction flags machine-learning estimates of design parameters, such as those of an aero engine, whose accumulated physical error exceeds a threshold, and corrects them by surrogate-assisted optimization against physical evaluations~\cite{kang2023physics}; \fwname{} works from fleet data without a simulator and names the constraint that fails. %
Fellegi--Holt editing localizes errors by finding the minimum-weight set of fields whose modification satisfies every edit~\cite{fellegi1976systematic}. Data validation leaves a rule with a missing value unevaluated~\cite{vanderloo2021data}. Data repair methods, benchmarked by Ni et al.~\cite{ni2024automatic}, correct cells through probabilistic inference under integrity constraints in HoloClean~\cite{rekatsinas2017holoclean} or using a few example repairs in Baran~\cite{mahdavi2020baran}. The satisfiability modulo theories (SMT) solver Z3~\cite{demoura2008z3} returns conflict sets for consistency-based diagnosis~\cite{reiter1987theory} as unsatisfiable cores. It also supports minimum-cardinality localization through MaxSMT~\cite{bjorner2015nuz}. Rule engines and language models check building designs against codes~\cite{eastman2009automatic,iversen2026leveraging}.
Their bounds, where present, are crisp, set by experts or mined from data~\cite{schelter2018automating}, and they take no image evidence; \fwname{} grades every constraint against bands learned from the fleet, ranks the violated constraints by the graded degrees of its own verdict, localizes the altered field by abduction with no oracle to break ties and adds image evidence (Sections~\ref{sec:res-diag} and~\ref{sec:res-local}). %
A reference architecture for auditing mechanical engineering drawings checks the entities read from them against a standards knowledge graph~\cite{chen2026ai}. %
Hybrid fault diagnosis pairs model-based residuals with anomaly classifiers~\cite{jung2018combining}, or residuals learned from fault-free data with consistency-based isolation~\cite{mohammadi2025consistency}, and a neuro-symbolic architecture lets knowledge-based and learned methods support each other in anomaly detection and fault diagnosis~\cite{bohne2025domain}, a division of labor that \fwname{} offers as its optional rank-then-explain alarm, a residual-learner alarm explained by the violated constraint. %

\paragraph{Anomaly detection and learned detectors.}
TabPFN v2~\cite{hollmann2025accurate} and AnoLLM~\cite{tsai2025anollm} score records by in-context learning and language-model likelihood, RegValidator checks the detection ideas of a language model with data-validation rules~\cite{huang2026integrating}, GIDCL detects and corrects errors with a graph network and cleaning rules generated by a language model~\cite{yan2024gidcl}, and ZeroED detects erroneous values without manual labels through language-model reasoning~\cite{ni2025zeroed}. %
\fwname{} returns, for each record, which constraints the record respects and which it violates, the altered field, a repair proposal and the evidence coverage, together with a score; Section~\ref{sec:res-detect} compares the scores in area under the receiver operating characteristic curve - AUROC. %

\begin{widetable}[t]
\centering
\scriptsize
\setlength{\tabcolsep}{4pt}
\newcommand{\capH}[1]{\begin{tabular}[b]{@{}c@{}}#1\end{tabular}}
\begin{threeparttable}
\let\TPTrlap\relax
\caption{Comparison with the state of the art.}
\label{tab:sota}
\begin{tabularx}{\linewidth}{@{}>{\raggedright\arraybackslash}X*{7}{c}@{}}
\toprule
Work & \capH{Violated\\constraint} & \capH{Altered\\field} & \capH{Repair\\proposal} & \capH{Exact\\abstention} & \capH{Learned\\tolerances} & \capH{Part\\grounding} & \capH{Record\\report} \\
\midrule
LTN and its applications~\cite{badreddine2022ltn,donadello2017logic,hogea2024logiclstm,hogea2026rule,mondal2025logic} & \capP & \capP & \capN & \capP & \capP & \capP & \capN \\
DeepProbLog, Scallop, NeurASP, abductive learning\tnote{b}~\cite{manhaeve2021neural,li2023scallop,yang2020neurasp,hu2025efficient} & \capP & \capP & \capP & \capP & \capP & \capN & \capN \\
DL2\tnote{a}, semantic loss, constraint layers\tnote{a}~\cite{fischer2019dl2,xu2018semantic,stoian2024pishield} & \capP & \capN & \capP & \capN & \capN & \capN & \capN \\
KENN, LEN, PSL\tnote{b}~\cite{daniele2019kenn,ciravegna2023logic,bach2017hinge} & \capP & \capP & \capP & \capP & \capP & \capN & \capP \\
Statistical data editing and validation\tnote{a}~\cite{fellegi1976systematic,vanderloo2021data} & \capY & \capY & \capY & \capY & \capP & \capN & \capY \\
SMT diagnosis\tnote{a}, language model with SMT~\cite{reiter1987theory,bjorner2015nuz,hsia2025neurosymbolic} & \capY & \capY & \capY & \capP & \capN & \capN & \capP \\
HoloClean, Deequ~\cite{rekatsinas2017holoclean,schelter2018automating} & \capP & \capY & \capY & \capN & \capP & \capN & \capP \\
Learned and language-model data cleaning~\cite{mahdavi2020baran,yan2024gidcl,ni2025zeroed} & \capP & \capY & \capY & \capN & \capP & \capN & \capP \\
Tabular learners and anomaly scores\tnote{a}~\cite{hollmann2025accurate,tsai2025anollm} & \capN & \capN & \capN & \capP & \capN & \capN & \capN \\
Language models with rules\tnote{a}~\cite{yang2025qwen3,huang2026integrating} & \capP & \capP & \capP & \capP & \capP & \capN & \capP \\
Vision-language models with rules\tnote{a}~\cite{bai2025qwen3vl,doris2025designqa} & \capP & \capN & \capN & \capP & \capN & \capN & \capP \\
Rule checking of building designs~\cite{eastman2009automatic,iversen2026leveraging} & \capY & \capP & \capN & \capN & \capN & \capN & \capY \\
Hybrid fault diagnosis~\cite{jung2018combining,mohammadi2025consistency} & \capY & \capY & \capN & \capN & \capY & \capN & \capP \\
Neuro-symbolic visual compliance checking~\cite{barbara2023neurosymbolic,kim2026visual} & \capY & \capY & \capN & \capN & \capN & \capP & \capP \\
Explanation supervision\tnote{a}~\cite{ross2017right,simpson2019gradmask,gao2022res,rao2023studying,pukdee2023learning} & \capN & \capN & \capN & \capN & \capN & \capP & \capN \\
Tabular-image fusion\tnote{a}~\cite{du2024tip,tang2024autogluon} & \capN & \capN & \capN & \capP & \capN & \capN & \capN \\
\midrule
\textbf{\fwname{} (ours)} & \capYb & \capYb & \capYb & \capYb & \capYb & \capYb & \capYb \\
\bottomrule
\end{tabularx}
\begin{tablenotes}[para,flushleft]
\scriptsize
\item[] \capY{}: provided and evaluated in a cited work; \capP{}: expressible or partially provided; \capN{}: not reported.
\item[a] Benchmarked here (Section~\ref{sec:res}); constraint layers used for repair ordering (Section~\ref{sec:res-local}).
\item[b] Expressible capabilities, not evaluated in the cited works or benchmarked here. %
\end{tablenotes}
\end{threeparttable}
\end{widetable}

\paragraph{Differentiable logic and neuro-symbolic frameworks.}
Neuro-symbolic methods use explicit knowledge to guide learning and reasoning. Frameworks such as LTN, DeepProbLog~\cite{manhaeve2021neural}, Scallop~\cite{li2023scallop} and NeurASP~\cite{yang2020neurasp} combine trainable models with logical reasoning. Knowledge can also enter through training objectives or mechanisms that constrain or refine predictions, as in abductive learning~\cite{hu2025efficient}, DL2~\cite{fischer2019dl2}, the semantic loss~\cite{xu2018semantic}, constraint layers~\cite{stoian2024pishield,stoian2024realistic} and Knowledge Enhanced Neural Networks (KENN)~\cite{daniele2019kenn}. Probabilistic soft logic (PSL) infers truth values for open atoms under weighted rules~\cite{bach2017hinge}. Logic Explained Networks (LEN) extract rules from a trained network~\cite{ciravegna2023logic}. These frameworks can compute a truth, a probability or a distance to satisfaction per formula, from which a violated constraint can be read by an argmax or argmin, can carry a tolerance band as a neural or weighted predicate, and marginalize, infer or leave free what is unobserved (Table~\ref{tab:sota}, note b). Encoding \fwname{}'s knowledge base in DeepProbLog, Scallop or PSL would change its axiomatic and logical structure, with probabilistic facts and proofs in place of graded truths and masked quantification, so that comparison is left to future work (Section~\ref{sec:end}). %
\fwname{} joins these per-constraint truths with tolerance bands learned from the fleet, abstention under missing evidence (Proposition~\ref{prop:reduction}), a search for the altered field with its repair through the same degrees, part-grounded image evidence and a per-record report, provided and evaluated together, which no cited work does; Sections~\ref{sec:res-diag} and~\ref{sec:res-knowledge} compare it with the DL2 method on the declared fields and with a network trained with the DL2 loss on the same residuals, which returns per-constraint probabilities. %
LTN has grounded part-of relations, fault signatures, process traces~\cite{donadello2017logic,radtke2022expert,gaikwad2026neurosymbolic}, fault-diagnosis rules for gearboxes and rotating machinery~\cite{hogea2024logiclstm,hogea2026rule} and medical rules with learnable thresholds~\cite{mondal2025logic}; an answer set program names the missing, misplaced or wrongly connected component of electrical control panels after a separately trained detector~\cite{barbara2023neurosymbolic}, and GuardEn executes compiled safety rules on scene graphs of an image~\cite{kim2026visual}; \fwname{} adds graded constraints against bands learned from the fleet and visual predicates grounded in the parts they describe (Section~\ref{sec:res-grounding}). %

\paragraph{Explanation supervision and part-level evidence.}
Explanation-guided learning adds supervision on model explanations~\cite{gao2024going}. Right for the Right Reasons (RRR) penalizes input gradients in regions annotated as irrelevant~\cite{ross2017right}, while GradMask penalizes them outside the annotated region~\cite{simpson2019gradmask}. RES learns from imprecise human masks~\cite{gao2022res}. The Energy loss of Rao et al.~\cite{rao2023studying} increases the share of positive attribution inside annotated bounding boxes, and Pukdee et al.~\cite{pukdee2023learning} learn from explanation constraints satisfied in expectation.
\fwname{} expresses part-level supervision as an axiom on each head's class-activation map, gated by the head's assertion and guarded by part visibility. Section~\ref{sec:res-grounding} compares this approach with RRR, GradMask and Pukdee et al.\ under a common protocol. %

\paragraph{Multimodal consistency and verifiers.}
TIP pre-trains on images and masked tables for incomplete tabular data~\cite{du2024tip}, and AutoGluon-Multimodal (AutoGluon-MM) fuses tables and images with foundation models~\cite{tang2024autogluon}.
The fused verdict of \fwname{} aggregates its strata without labels and reduces exactly to the available constraints, and Section~\ref{sec:res-xmodal} compares it with these fusions; its cross-modal axioms read the attributes predicted by the visual heads, as concept bottleneck models predict a label from predicted concepts~\cite{koh2020concept}. %
Open vision-language models~\cite{bai2025qwen3vl,wang2025internvl35} can judge whether a record and an image agree and a language model~\cite{yang2025qwen3} whether a record is coherent, and vision-language models have been benchmarked on engineering rules, computer-aided design renderings and drawings, where they still struggle~\cite{doris2025designqa}; \fwname{} returns with its judgment a verdict for every constraint (Section~\ref{sec:res-xmodal}). %

\par\smallskip\indent
Table~\ref{tab:sota} compares the reviewed approaches on seven outputs needed for specification auditing: identification of the violated constraint, localization of the altered field, a repair proposal, exact abstention under missing evidence, learned tolerances, part-grounded image evidence, and a per-record verification report. To our knowledge, none of the reviewed works provides and evaluates all seven outputs together. Statistical data editing and validation come closest, with crisp bounds set by experts or fitted to data and no image evidence. The neuro-symbolic formulation of \fwname{} uses graded constraint violation degrees for the largest-violation key and for abductive field localization, without an oracle to break ties (Sections~\ref{sec:res-diag} and~\ref{sec:res-local}). %
\section{Proposed Methodology}
\label{sec:meth}

\fwname{} audits a declared record in one integrated, interpretable NeSy framework: through its logical axioms it gives a held or violated verdict on each physical relation that binds the record's fields and, when an image is paired with it, on whether the image agrees. The framework is organized into seven modules (Figure~\ref{fig:pipeline}), built on Real Logic (Section~\ref{sec:background}). Guarded quantification allows each axiom to contribute only when its required evidence is available (E1), while graded predicates express how well each constraint is satisfied (E2). A grounding axiom ties visual predictions to the relevant image parts (E3), and stratification preserves a separate verdict for each evidence source (E4). A monotone second head complements the main verdict with an additional score and diagnosis key (E5).
\begin{itemize}
  \item \textbf{Data and Verification Instances} encode records and images, then generate corrupted examples and controls;
  \item \textbf{Knowledge Base and Graded Predicates} grade each constraint by a violation degree, with bands and slopes learned from the fleet (E2);
  \item \textbf{Vision Module and Grounding Axiom} encourage each head to place its evidence on its part (E3);
  \item \textbf{Cross-Modal Axioms} compare visual posteriors with declared fields;
  \item \textbf{Masked Stratified Aggregation} guards each axiom by its availability, so the verdict reduces exactly to the available evidence (E1), with one stratum per evidence source (E4);
  \item \textbf{Verdict and Diagnosis Keys} return the verdict and name the violated constraint, with a monotone second head (E5);
  \item \textbf{Verification Report and Optional Checks} give the altered field, a repair proposal and the report, with three optional settings of the audit.
\end{itemize}
The verdict is the masked mean of the constraint degrees, so it decomposes into constraint-level verdicts, from which the violated constraint, the altered field, the repair, the accounting of missing evidence and the report follow; detection is one of these outputs. %
Algorithm~\ref{S-alg:forward} in Appendix~\ref{S-app:method} transcribes the forward pass from the code. In our framework, we used the LTNtorch library~\cite{carraro2024ltntorch}.
Appendix~\ref{app:repo} describes the organization of the implementation and gives an example of a verification report.

\paragraph{Domain-specific and domain-independent parts.}
The domain enters \fwname{} through the knowledge base, the visual predicates, the corruption generator and the data; masked aggregation, graded predicates, grounding axiom, keys, report and protocol are written to apply unchanged to engineered systems whose declared fields obey physical, definitional or regulatory relations with statable tolerances; fitted parameters and numbers concern aircraft. %
Appendix~\ref{app:porting} lists what a port to another engineered system requires.

\subsection{Data and Verification Instances}
\label{sec:meth-data}

\paragraph{Data acquisition.}
\label{sec:meth-data-acq}
A \emph{specification record} (Appendix~\ref{S-app:data}) for each of 107 aircraft types of 71 families in 7 classes %
declares 16 numeric fields in SI units: %
maximum take-off and operating empty mass $\MTOW$ and $\OEW$, wingspan $b$, wing area $S$, aspect ratio $\AR$, fuselage length $L_{\mathrm{fus}}$, cruise speed $V_{\mathrm{cr}}$ and altitude $h_{\mathrm{cr}}$, range $R$, ceiling $h_{\mathrm{ceil}}$, total static thrust $F_{\mathrm{tot}}$ or shaft power $P_{\mathrm{tot}}$, engine count $n_{\mathrm{eng}}$, usable fuel mass $m_{\mathrm{fuel}}$, stall speed $V_{\mathrm{s}}$ and passenger limit $n_{\mathrm{pax}}$, with the class $c$ and the engine type $e$. %
\emph{Ground photographs} of FGVC-Aircraft~\cite{maji2013fine}, 6,800 paired with the record of their type or family (Appendix~\ref{S-app:data-benchmarks}), carry variant, family and manufacturer labels, from which per-variant attribute labels are assigned. %
The real release of RarePlanes~\cite{shermeyer2021rareplanes} annotates 14,707 aircraft in \emph{satellite} tiles; a diamond through nose, wingtips and tail and the attributes form a record of 9 numeric fields: %
wingspan $b$, length $L$, diamond area $A_{\diamond}$, $n_{\mathrm{eng}}$, vertical-stabilizer count $n_{\mathrm{fin}}$, and codes for the airplane design group of the Federal Aviation Administration (FAA)\footnote{FAA Advisory Circular 150/5300-13B, Airport Design, March 31, 2022.}, canards, wing type $t$ and wing position, with the role $r$ and the propulsion $\pi$. %

\paragraph{Encoding.}
\label{sec:meth-data-enc}
In an instance $(x,m,I)$, the record $x\in(\mathbb R\cup\{\varnothing\})^{n_x}$ holds $n_x=16$ or 9 declared numeric fields and the categorical attributes, $\varnothing$ marking an undeclared field; $m_j=\mathbf 1[x_j\neq\varnothing]$, and $I=\varnothing$ without an image. %
Axiom $k$ is available, $\alpha_k(x,I)=1$, when its fields are declared and the record falls within its scope. %
In Figure~\ref{fig:data}(a), a wide-body turbofan record, outside the scope of the propeller axiom A4P, declares no stall speed, so A6 is masked ($\alpha_k=0$ for both), and an undetected landing gear masks the main-gear axiom AX3; panel (c) is the 23-field frame of the cross-dataset transfer (Appendix~\ref{S-app:data-benchmarks}).

\begin{widefigure}[t]
\centering
\begin{adjustbox}{max width=\textwidth}
\begin{tikzpicture}[
  x=1mm, y=1mm,
  ftxt/.style    = {font=\fontsize{7}{8.4}\selectfont},
  fitx/.style    = {font=\fontsize{7}{8.4}\selectfont\itshape},
  ptitle/.style  = {anchor=north west, inner sep=0pt, text height=2.2mm, font=\fontsize{8}{9.6}\selectfont\bfseries},
  psub/.style    = {anchor=north west, inner sep=0pt, ftxt, text=ioline},
  tcell/.style   = {inner sep=0pt, ftxt, text height=1.7mm, text depth=0.5mm},
  tframe/.style  = {draw=ioline, fill=iofill, line width=0.4pt},
  mbit1/.style   = {draw=ioline, fill=ioline!45, line width=0.4pt, minimum width=3.2mm, minimum height=2.4mm, inner sep=0pt, ftxt},
  mbit0/.style   = {mbit1, fill=white, densely dashed, text=ioline},
  achip/.style   = {draw=symbline, line width=0.4pt, minimum width=7mm, minimum height=2.6mm, inner sep=0pt, ftxt},
  a1/.style      = {achip, fill=symbfill},
  a0/.style      = {achip, fill=white, densely dashed, text=symbline},
  stbox/.style   = {symb, ftxt, align=flush center, inner sep=1mm, minimum height=0pt, minimum width=0pt, execute at begin node={\hyphenpenalty=10000\exhyphenpenalty=10000}},
  note/.style    = {inner sep=0pt, ftxt, text=ioline, align=flush left, execute at begin node={\hyphenpenalty=10000\exhyphenpenalty=10000}},
  silh/.style    = {draw=ioline, fill=white, line width=0.5pt},
  pmask/.style   = {draw=ioline, densely dashed, line width=0.6pt, fill=ioline, fill opacity=0.18},
  partlbl/.style = {inner sep=0.5pt, fitx, text=ioline},
  ucell/.style   = {draw=ioline!70, line width=0.3pt, minimum width=4.4mm, minimum height=2.6mm, inner sep=0pt},
  u2/.style      = {ucell, fill=ioline!45},
  u1/.style      = {ucell, fill=ioline!15},
  u0/.style      = {ucell, fill=white, draw=ioline, dash pattern=on 1pt off 0.7pt},
  dcell/.style   = {draw=ioline!70, line width=0.3pt, minimum width=3mm, minimum height=2.2mm, inner sep=0pt},
  dk0/.style     = {dcell, fill=ioline!15},
  dk1/.style     = {dcell, fill=white, draw=ioline, dash pattern=on 1pt off 0.7pt},
  dk2/.style     = {dcell, fill=ioline!85},
  dk3/.style     = {dcell, fill=ioline!45},
  dk4/.style     = {dcell, fill=white, path picture={\fill[ioline] (path picture bounding box.center) circle (0.5mm);}},
  hitc/.style    = {fill=symbline, draw=symbline, text=white, ftxt, line width=0.4pt, inner xsep=0.6mm, inner ysep=0.3mm, text height=1.7mm, text depth=0.5mm},
  trainflow/.style = {flow, densely dotted},
]
  \node[ptitle] at (0,0) {(a) Specification instance};
  \node[psub] at (0,-4) {record $x$, mask $m$ and photograph $I$};
  \draw[tframe] (0,-7) rectangle (46,-66);
  \node[tcell, anchor=west, text=ioline] at (1,-9) {field};
  \node[tcell, anchor=east, text=ioline] at (38,-9) {declared value $x_j$};
  \node[tcell, text=ioline] at (42,-9) {$m_j$};
  \foreach \i/\s/\v/\m in {
    1/{$\MTOW$}/{242,000 kg}/1,%
    2/{$\OEW$}/{129,400 kg}/1,%
    3/{$b$}/{60.30 m}/1,%
    4/{$S$}/{361.6 m$^{2}$}/1,%
    5/{$\AR$}/{10.06}/1,%
    6/{$L_{\mathrm{fus}}$}/{63.66 m}/1,%
    7/{$V_{\mathrm{cr}}$}/{242 m\,s$^{-1}$}/1,%
    8/{$h_{\mathrm{cr}}$}/{11,278 m}/1,%
    9/{$R$}/{11,750 km}/1,%
    10/{$h_{\mathrm{ceil}}$}/{12,527 m}/1,%
    11/{$F_{\mathrm{tot}}$}/{632.4 kN}/1,%
    12/{$P_{\mathrm{tot}}$}/{undeclared (kW)}/0,%
    13/{$n_{\mathrm{eng}}$}/{2}/1,%
    14/{$m_{\mathrm{fuel}}$}/{78,414 kg}/1,%
    15/{$V_{\mathrm{s}}$}/{undeclared (m\,s$^{-1}$)}/0,%
    16/{$n_{\mathrm{pax}}$}/{440}/1}{%
    \node[tcell, anchor=west] at (1,{-9-3*\i}) {\s};
    \ifnum\m=1
      \node[tcell, anchor=east] at (38,{-9-3*\i}) {\v};
      \node[mbit1] at (42,{-9-3*\i}) {1};
    \else
      \node[tcell, anchor=east, text=ioline] at (38,{-9-3*\i}) {\v};
      \node[mbit0] at (42,{-9-3*\i}) {0};
    \fi
  }
  \draw[ioline, line width=0.3pt] (0,-59) -- (46,-59);
  \foreach \i/\s/\v in {
    17/{class $c$}/{wide-body},%
    18/{engine type $e$}/{turbofan}}{%
    \node[tcell, anchor=west] at (1,{-10-3*\i}) {\s};
    \node[tcell, anchor=east] at (38,{-10-3*\i}) {\v};
  }
  \begin{scope}[shift={(0,-69)}]
    \fill[iofill] (0,0) rectangle (46,-26);
    \begin{scope}[shift={(23,-12)}]
      \draw[silh] (12,2.2) -- (15.3,9.8) -- (18,9.8) -- (17,3) -- cycle;
      \draw[silh] (-18,0) .. controls (-18,2) and (-16,2.2) .. (-14,2.2) -- (13,2.2) -- (18,3.3) -- (19,2)
                  .. controls (14,0.2) and (12,-2) .. (8,-2) -- (-14,-2) .. controls (-16,-2) and (-18,-1.4) .. cycle;
      \draw[silh] (17,3) ellipse (3 and 0.5);
      \draw[ioline, line width=0.6pt, dotted] (-12,0.8) -- (10,0.8);
      \fill[ioline] (-16,0.5) -- (-15,1.2) -- (-13.6,1.2) -- (-13.6,0.4) -- cycle;
      \draw[silh] (-4.2,-1.5) -- (5,-1.5) -- (9,-4.5) -- (5.2,-4.7) -- cycle;
      \draw[silh] (-1,-4) ellipse (3.4 and 1.1);
      \draw[ioline, line width=0.4pt] (-4.2,-4) ellipse (0.35 and 1);
      \draw[pmask] (-1,-4) ellipse (4.6 and 1.9);
      \draw[pmask] (0.4,-1.6) ellipse (6.4 and 1.3);
      \draw[pmask] (11,1.6) -- (14.4,10.7) -- (19,10.7) -- (20,3.4) -- (19,1.8) -- cycle;
      \node[partlbl, anchor=east] at (-6.5,-4) {engine};
      \node[partlbl, anchor=west] at (10,-4) {wing};
      \node[partlbl, anchor=east] at (12,6.5) {tail};
    \end{scope}
    \node[note, anchor=south west] at (1,-25) {landing gear not detected:\\$\vis_k=0$, AX3 masked};
    \draw[ioline, line width=0.4pt] (0,0) rectangle (46,-26);
    \node[note, text width=48mm, anchor=north west] at (0,-27) {photograph $I$ (schematic);\\regions~$\tilde M_k$ from part pseudo-masks};
  \end{scope}

  \begin{scope}[shift={(118,0)}]
    \node[ptitle] at (0,0) {(b) Satellite instance};
    \node[psub] at (0,-4) {record $x$, mask $m$ and crop $I$};
    \draw[tframe] (0,-7) rectangle (45,-45);
    \node[tcell, anchor=west, text=ioline] at (1,-9) {field};
    \node[tcell, anchor=east, text=ioline] at (39,-9) {annotated value $x_j$};
    \node[tcell, text=ioline] at (42,-9) {$m_j$};
    \foreach \i/\s/\v in {
      1/{$b$}/{47.07 m},%
      2/{$L$}/{53.72 m},%
      3/{$A_{\diamond}$}/{1,264.33 m$^{2}$},%
      4/{$n_{\mathrm{eng}}$}/{2},%
      5/{$n_{\mathrm{fin}}$}/{1},%
      6/{design group}/{4 (IV)},%
      7/{canards}/{0 (no)},%
      8/{wing type $t$}/{1 (swept)},%
      9/{wing position}/{0 (mid or low)}}{%
      \node[tcell, anchor=west] at (1,{-9-3*\i}) {\s};
      \node[tcell, anchor=east] at (39,{-9-3*\i}) {\v};
      \node[mbit1] at (42,{-9-3*\i}) {1};
    }
    \draw[ioline, line width=0.3pt] (0,-38) -- (45,-38);
    \foreach \i/\s/\v in {
      10/{role $r$}/{civil large},%
      11/{propulsion $\pi$}/{jet}}{%
      \node[tcell, anchor=west] at (1,{-10-3*\i}) {\s};
      \node[tcell, anchor=east] at (39,{-10-3*\i}) {\v};
    }
    \node[note, text width=45mm, anchor=north west] at (0,-46) {one row per observed aircraft;\\ split by site};
    \begin{scope}[shift={(0,-69)}]
      \fill[ioline!30] (0,0) rectangle (26,-26);
      \begin{scope}[shift={(13,-13)}]
        \draw[silh] (-1.2,2.7) -- (-9.3,-2.6) -- (-9.3,-3.6) -- (-1.2,-1.4) -- cycle;
        \draw[silh] (1.2,2.7) -- (9.3,-2.6) -- (9.3,-3.6) -- (1.2,-1.4) -- cycle;
        \draw[silh] (-1,-7.5) -- (-4.5,-10) -- (-4.5,-10.7) -- (-1,-9.6) -- cycle;
        \draw[silh] (1,-7.5) -- (4.5,-10) -- (4.5,-10.7) -- (1,-9.6) -- cycle;
        \draw[silh] (0,10.4) .. controls (1.1,9.6) and (1.2,8) .. (1.2,6.5) -- (1.2,-8.6) -- (0,-11) -- (-1.2,-8.6) -- (-1.2,6.5)
                    .. controls (-1.2,8) and (-1.1,9.6) .. cycle;
        \draw[silh, fill=ioline!15] (-4.6,0.5) ellipse (0.6 and 1.5);
        \draw[silh, fill=ioline!15] (4.6,0.5) ellipse (0.6 and 1.5);
        \draw[pmask] (0,10.4) -- (9.3,-3.1) -- (0,-11) -- (-9.3,-3.1) -- cycle;
        \draw[ioline, line width=0.4pt, densely dotted] (-9.3,-3.1) -- (9.3,-3.1);
        \draw[ioline, line width=0.4pt, densely dotted] (0,10.4) -- (0,-11);
        \node[partlbl, anchor=north] at (-5,-3.3) {$b$};
        \node[partlbl, anchor=west] at (1.5,-6) {$L$};
      \end{scope}
      \draw[ioline, line width=0.4pt] (0,0) rectangle (26,-26);
      \node[note, text width=17mm, anchor=north west] at (28,0) {crop $I$ (schematic); annotated diamond: region of all 25 heads};%
    \end{scope}
  \end{scope}

  \begin{scope}[shift={(82,0)}]
    \node[anchor=north, inner sep=0pt, text height=2.2mm, font=\fontsize{8}{9.6}\selectfont\bfseries] at (0,0) {Axiom strata of $\KB$};
    \node[stbox, text width=26mm, minimum height=14mm, anchor=north] (sA) at (-15,-5)
         {$\Aset$: specification\\ reads 15 of 16 fields\\ (not $h_{\mathrm{ceil}}$)};
    \node[stbox, text width=26mm, minimum height=14mm, anchor=north] (sR) at (15,-5)
         {$\Rset$: remote-sensing\\ reads 8 of 9 fields\\ (not wing position),\\ role and propulsion};
    \node[ftxt, anchor=north, inner sep=0pt] at (-15,-20) {$\alpha_k$ for (a): 10 of 12};
    \node[ftxt, anchor=north, inner sep=0pt] at (15,-20) {$\alpha_k$ for (b): 10 of 10};
    \foreach \pass in {1,0}{
      \foreach \k/\a [count=\n from 0] in {A1/1,A2/1,A3/1,A4/1,A4P/0,A5/1,A6/0,A7/1,A8/1,A9/1,A10/1,A11/1}{
        \pgfmathtruncatemacro\r{\n/4}
        \pgfmathtruncatemacro\c{\n-4*\r}
        \ifnum\a=\pass
          \ifnum\a=1
            \node[a1] at ({-29+3.5+7*\c},{-24.6-3*\r}) {\k};
          \else
            \node[a0] at ({-29+3.5+7*\c},{-24.6-3*\r}) {\k};
          \fi
        \fi
      }
    }
    \foreach \k [count=\n from 0] in {R1,R2,R3,R4,R5,R6,R7,R8,R9,R10}{
      \pgfmathtruncatemacro\r{\n/4}
      \pgfmathtruncatemacro\c{\n-4*\r}
      \node[a1] at ({1+3.5+7*\c},{-24.6-3*\r}) {\k};
    }
    \node[a1, key] (ka1) at (-25,-34.6) {};
    \node[ftxt, anchor=west, inner sep=1pt] at (ka1.east) {available, $\alpha_k=1$};
    \node[a0, key] (ka0) at (3,-34.6) {};
    \node[ftxt, anchor=west, inner sep=1pt] at (ka0.east) {masked, $\alpha_k=0$};
    \node[stbox, text width=56mm, anchor=north] (sX) at (0,-38)
         {$\AXset$: cross-modal axioms, inference only\\
          photograph: $\pvis(\cdot\mid I)$ against $n_{\mathrm{eng}}$\\ (AX1), engine type and cruise Mach (AX2),\\ $\MTOW$ (AX3); AX4 on the image alone\\
          crop: $\pvis(\cdot\mid I)$ against\\ $n_{\mathrm{eng}}$, propulsion, wing type,\\ role~(AX1~to~AX4) and design group (AX5)};
    \node[note, align=center, text width=52mm, anchor=north] at (0,-62) {a missing field or an invisible part masks every axiom that needs it};
    \node[stbox, text width=56mm] (sV) at (0,-82)
         {$\AVset$: grounding axioms, training only\\
          $P_k\rightarrow\Loc_k$, guard $\vis_k$, region $\tilde M_k$\\ (photograph) or annotated diamond (crop)};
  \end{scope}
  \draw[flow] (46,-12) -- (sA.west |- 0,-12);
  \draw[flow] (118,-12) -- (sR.east |- 0,-12);
  \draw[flow] (46,-48) -- (sX.west |- 0,-48);
  \draw[flow] (118,-42) -- (sX.east |- 0,-42);
  \draw[flow] (46,-69) -- node[lbl, fitx, left] {$I$} (sX.south west);
  \draw[flow] (118,-69) -- node[lbl, fitx, right] {$I$} (sX.south east);
  \draw[trainflow] (46,-82) -- (sV.west);
  \draw[trainflow] (118,-82) -- (sV.east);

  \begin{scope}[shift={(0,-105)}]
    \node[ptitle] at (0,0) {(c) Unified frame: union of the two tabular schemas (23 columns)};
    \foreach \j/\s in {1/{$\MTOW$}, 2/{$\OEW$}, 3/{$b$}, 4/{$S$}, 5/{$\AR$}, 6/{$L_{\mathrm{fus}}$}, 7/{$V_{\mathrm{cr}}$},
        8/{$h_{\mathrm{cr}}$}, 9/{$R$}, 10/{$h_{\mathrm{ceil}}$}, 11/{$F_{\mathrm{tot}}$}, 12/{$P_{\mathrm{tot}}$}, 13/{$n_{\mathrm{eng}}$},
        14/{$m_{\mathrm{fuel}}$}, 15/{$V_{\mathrm{s}}$}, 16/{$n_{\mathrm{pax}}$}, 17/{$L$}, 18/{$A_{\diamond}$}, 19/{$n_{\mathrm{fin}}$},
        20/{design group}, 21/{canards}, 22/{wing type $t$}, 23/{wing position}}{
      \node[ftxt, rotate=90, anchor=west, inner sep=0.5pt] at ({21+4.6*\j-2.3},-21) {\s};
    }
    \node[tcell, anchor=east] at (20,-23) {specification (a)};
    \node[tcell, anchor=east] at (20,-26) {satellite (b)};
    \foreach \t [count=\j] in {2,2,2,2,2,2,2,2,2,2,2,0,2,2,0,2,1,0,0,1,0,0,0}{
      \ifcase\t \node[u0] at ({21+4.6*\j-2.3},-23) {};
      \or       \node[u1] at ({21+4.6*\j-2.3},-23) {};
      \or       \node[u2] at ({21+4.6*\j-2.3},-23) {};
      \fi
    }
    \foreach \t [count=\j] in {0,0,2,0,0,1,0,0,0,0,0,0,2,0,0,0,2,2,2,2,2,2,2}{
      \ifcase\t \node[u0] at ({21+4.6*\j-2.3},-26) {};
      \or       \node[u1] at ({21+4.6*\j-2.3},-26) {};
      \or       \node[u2] at ({21+4.6*\j-2.3},-26) {};
      \fi
    }
    \foreach \j in {3,6,13,17}{
      \node[draw=ioline, line width=0.9pt, minimum width=4.6mm, minimum height=6.4mm, inner sep=0pt] at ({21+4.6*\j-2.3},{-23-1.5}) {};
    }
    \draw[ioline, line width=0.4pt, decorate, decoration={brace, mirror, amplitude=1mm}] ({21+0.2},-28) -- ({21+4.6*16-0.2},-28);
    \node[ftxt, anchor=north, inner sep=0pt, text=ioline] at ({21+4.6*8},-30) {16 specification fields};
    \draw[ioline, line width=0.4pt, decorate, decoration={brace, mirror, amplitude=1mm}] ({21+4.6*16+0.2},-28) -- ({21+4.6*23-0.2},-28);
    \node[ftxt, anchor=north, inner sep=0pt, text=ioline] at ({21+4.6*19.5},-30) {7 further annotation fields};
    \node[u2, key] (kc2) at (132,-5) {};  \node[ftxt, anchor=west, inner sep=1pt] at (kc2.east) {declared or annotated};
    \node[u1, key] (kc1) at (132,-8) {};  \node[ftxt, anchor=west, inner sep=1pt] at (kc1.east) {derived or aliased};
    \node[u0, key] (kc0) at (132,-11) {}; \node[ftxt, anchor=west, inner sep=1pt] at (kc0.east) {empty, masked};
    \node[ucell, key, line width=0.9pt, draw=ioline] (kcs) at (132,-14) {}; \node[ftxt, anchor=west, inner sep=1pt] at (kcs.east) {shared quantity};
  \end{scope}

  \begin{scope}[shift={(0,-142)}]
    \node[ptitle] at (0,0) {(d) Generated instances of the record of (a): one per corruption family and a control};
    \node[psub] at (0,-4) {each grid: the 16 fields of (a) in the same order, eight per row};
    \begin{scope}[shift={(0,-3)}]%
    \foreach \k/\nm/\cells in {
      0/{single value}/{0,2,0,0,0,0,0,0,0,0,0,1,0,0,1,0},%
      1/{geometric}/{0,0,0,2,0,0,0,0,0,0,0,1,0,0,1,0},%
      2/{performance}/{0,0,0,0,0,0,0,0,2,0,0,1,0,0,1,0},%
      3/{cruise}/{0,0,0,0,0,0,0,2,0,0,0,1,0,0,1,0},%
      4/{Frankenstein swap}/{0,0,2,2,2,2,0,0,0,0,0,1,0,0,1,0},%
      5/{noisy-but-coherent\\control}/{3,3,3,3,4,3,3,3,3,3,3,1,0,3,1,0}}{%
      \node[ftxt, anchor=south, align=center, inner sep=0pt, text depth=0.5mm] at ({26*\k+12},-7) {\nm};
      \foreach \t [count=\c from 0] in \cells {
        \pgfmathtruncatemacro\r{\c/8}
        \pgfmathtruncatemacro\q{\c-8*\r}
        \ifcase\t \node[dk0] at ({26*\k+1.5+3*\q},{-9-2.2*\r}) {};
        \or
        \or       \node[dk2] at ({26*\k+1.5+3*\q},{-9-2.2*\r}) {};
        \or       \node[dk3] at ({26*\k+1.5+3*\q},{-9-2.2*\r}) {};
        \or       \node[dk4] at ({26*\k+1.5+3*\q},{-9-2.2*\r}) {};
        \fi
      }
      \foreach \t [count=\c from 0] in \cells {
        \pgfmathtruncatemacro\r{\c/8}
        \pgfmathtruncatemacro\q{\c-8*\r}
        \ifnum\t=1 \node[dk1] at ({26*\k+1.5+3*\q},{-9-2.2*\r}) {};\fi
      }
    }
    \node[ftxt, anchor=north, inner sep=0pt] at (12,-13) {$\OEW/1.40$};%
    \node[hitc] at (12,-22) {A8};
    \node[ftxt, anchor=north, inner sep=0pt] at (38,-13) {$S/1.05$};%
    \node[hitc] at (38,-22) {A1};
    \node[ftxt, anchor=north, inner sep=0pt] at (64,-13) {$R\times1.20$};%
    \node[hitc] at (64,-22) {A5};
    \node[ftxt, anchor=north, inner sep=0pt] at (90,-13) {$h_{\mathrm{cr}}\times1.20$};%
    \node[hitc] at (90,-22) {A2};
    \node[ftxt, anchor=north, inner sep=0pt, align=center] at (116,-13) {geometry block\\from another class};%
    \node[hitc] (f3) at (116,-22) {A3};
    \node[hitc, left=0.6mm of f3] {A2};
    \node[hitc, right=0.6mm of f3] {A11};
    \node[ftxt, anchor=north, inner sep=0pt, align=center] at (142,-13) {relative noise 0.02;\\$\AR$ re-derived};%
    \node[ftxt] at (142,-22) {none, $y=0$};
    \begin{scope}[shift={(0,-27)}]
      \node[dk2, key] (k2) at (2,0) {};  \node[ftxt, anchor=west, inner sep=1pt] at (k2.east) {rescaled or copied};
      \node[dk3, key] (k3) at (31,0) {}; \node[ftxt, anchor=west, inner sep=1pt] at (k3.east) {noised};
      \node[dk4, key] (k4) at (46,0) {}; \node[ftxt, anchor=west, inner sep=1pt] at (k4.east) {re-derived};
      \node[dk0, key] (k0) at (65,0) {}; \node[ftxt, anchor=west, inner sep=1pt] at (k0.east) {unchanged};
      \node[dk1, key] (k1) at (85,0) {}; \node[ftxt, anchor=west, inner sep=1pt] at (k1.east) {undeclared};
      \node[hitc] (kh) at (105,0) {A$k$}; \node[ftxt, anchor=west, inner sep=1pt] at (kh.east) {flagged by the generator oracle};
    \end{scope}
    \end{scope}
  \end{scope}
\end{tikzpicture}
\end{adjustbox}
\caption{Data and verification instances.}
\label{fig:data}
\end{widefigure}

\paragraph{Corruption generator and labels.}
\label{sec:meth-data-gen}
The specification benchmark splits the records by family: the business-jet families form the unseen class, six random families the unseen families, and the others fill training, validation and test at 60/20/20 (Appendix~\ref{S-app:data-benchmarks}). Each clean record yields up to two instances of each of five corruption families: an impossible single value and geometric, performance and cruise inconsistencies multiply or divide one field by $1+\epsilon$, with $\epsilon=0.05$ and $0.10$, and a cross-class block swap (Frankenstein) copies a coherent block of fields from another class in the same split. The generator oracle, a crisp check of the same relations against class bands $[0.88\min,\,1.12\max]$ of each quantity over all 107 clean aircraft of every split, keeps an instance only when it flags the target constraint (any constraint for a swap); otherwise $\epsilon$ doubles, for at most six attempts (Table~\ref{S-tab:magnitudes}; realized instances in Figure~\ref{fig:data}(d)). The verifier is fitted on training-split instances only: the oracle's bands reach it through their labels and per-constraint targets, which thus carry the class extremes of all 107 aircraft, and no field value of another split enters training. %
Two noisy-but-coherent controls per record perturb each continuous field by a Gaussian factor with relative standard deviation 0.02, with $\AR$ rescaled according to the induced change in $b^2/S$; both controls must remain coherent. %
The \emph{intervention benchmark} (2,582 rows, same aircraft and splits) draws the same operators once per magnitude, without escalation or oracle certification, and labels each row from the edit alone. A row is labeled incoherent when a continuous field differs from its clean source by more than three noise standard deviations or an integer field changes. These tolerances are fixed before generation, and the label applies regardless of whether a physical relation is violated. An available constraint receives a positive label when its scope includes such an altered field. Noisy copies and edits within tolerance count as unaltered. These labels evaluate the constraint verdicts, the diagnosis key, abduction and detection. %
The satellite and multimodal benchmarks likewise include corrupted examples and controls. In the multimodal benchmark, each image is paired with its matching record, a mismatched record, or an altered version of its matching record. Mismatches involve aircraft from another class or role, or a sister variant. Occluded and degraded photographs serve as controls.. %

\subsection{Knowledge Base and Graded Predicates}
\label{sec:meth-kb}
\begin{widetable}[t]
\centering
\caption{Axiom catalog.}
\label{tab:axioms}
\scriptsize
\setlength{\tabcolsep}{3pt}
\begin{threeparttable}
\begin{tabular}{@{}L{6mm}L{70mm}L{25mm}L{28mm}L{27mm}@{}}
\toprule
ID & Graded predicate and quantity & Fields read & Available ($\alpha_k=1$) when & Learned parameters\\
\midrule
\multicolumn{5}{@{}l}{\textbf{Specification axioms $\Aset$} (records; 336 trainable scalars, 126 of them latents; $w_k=1$)}\\ %
A1 & graded identity of $(\AR-\hat A)/\hat A$, $\hat A=b^2/S$ & $b$, $S$, $\AR$ & all declared & tolerance (8); \mbox{$\kappa$, $\beta$ (2)}\\
A2 & graded band of $C_L=2\MTOW g/\big(\rho(h_{\mathrm{cr}})V_{\mathrm{cr}}^2S\big)$ & $\MTOW$, $V_{\mathrm{cr}}$, $h_{\mathrm{cr}}$, $S$ & all declared & band (16); \mbox{$\kappa$, $\beta$ (2)}\\
A3 & graded band of $\MTOW/S$ & $\MTOW$, $S$ & all declared & band (16); \mbox{$\kappa$, $\beta$ (2)}\\
A4 & graded band of $F_{\mathrm{tot}}/(\MTOW g)$ & $F_{\mathrm{tot}}$, $\MTOW$ & all declared; turbofan & band (16); \mbox{$\kappa$, $\beta$ (2)}\\
A4P & graded band of $\MTOW/P_{\mathrm{tot}}$ & $P_{\mathrm{tot}}$, $\MTOW$ & all declared; propeller & band (16); \mbox{$\kappa$, $\beta$ (2)}\\
A5 & graded identity of $\operatorname{clip}_3\ln(R/\hat R)$, $\hat R=\frac{V_{\mathrm{cr}}}{g\,c_T}\frac{L}{D}\ln\frac{\MTOW}{W_1}$ for a turbofan, $\hat R=\frac{\eta_p}{g\,c_P}\frac{L}{D}\ln\frac{\MTOW}{W_1}$ for a propeller aircraft; $W_1=\max(\MTOW-0.94\,m_{\mathrm{fuel}},\,0.05\,\MTOW)$ & $R$, $m_{\mathrm{fuel}}$, $\MTOW$; $V_{\mathrm{cr}}$ (turbofan) & $R$, $V_{\mathrm{cr}}$, $m_{\mathrm{fuel}}$, $\MTOW$ and $\AR$ declared & tolerance (8); \mbox{$\kappa$, $\beta$ (2)}; $L/D$, $c_T$, $c_P$, $\eta_p$ \mbox{per pair (84)}\\
A6 & graded identity of $\operatorname{clip}_2\ln(V_{\mathrm{s}}/\hat V_{\mathrm{s}})$, $\hat V_{\mathrm{s}}=\sqrt{2\MTOW g/(\rho_0SC_{L,\max})}$ & $V_{\mathrm{s}}$, $\MTOW$, $S$ & all declared & tolerance (8); \mbox{$\kappa$, $\beta$ (2)}; $C_{L,\max}$ \mbox{per pair (21)}\\
A7 & lower hinge of the margin $(\MTOW-\OEW-f_u\,m_{\mathrm{fuel}}-95\,n_{\mathrm{pax}})/\MTOW$ & $\OEW$, $m_{\mathrm{fuel}}$, $n_{\mathrm{pax}}$, $\MTOW$ & all declared & edge offsets (7); \mbox{$\kappa$ (1)}; $f_u$ \mbox{per pair (21)}\\
A8 & graded band of $\OEW/\MTOW$ & $\OEW$, $\MTOW$ & all declared & band (16); \mbox{$\kappa$, $\beta$ (2)}\\
A9 & graded band of $F_{\mathrm{tot}}/n_{\mathrm{eng}}$ if $F_{\mathrm{tot}}$ is declared, else of $P_{\mathrm{tot}}/n_{\mathrm{eng}}$ & $n_{\mathrm{eng}}$; $F_{\mathrm{tot}}$ or $P_{\mathrm{tot}}$ & $n_{\mathrm{eng}}$, and $F_{\mathrm{tot}}$ or $P_{\mathrm{tot}}$, declared & 2 bands (32); \mbox{$\kappa$, $\beta$ (4)}\\
A10 & graded band of $M=V_{\mathrm{cr}}/a(h_{\mathrm{cr}})$, per engine type $e$ & $V_{\mathrm{cr}}$, $h_{\mathrm{cr}}$ & all declared & band per $e$ (8); \mbox{$\kappa$, $\beta$ (2)}\\
A11 & sum of the graded bands of $L_{\mathrm{fus}}/b$ and $n_{\mathrm{pax}}/L_{\mathrm{fus}}$ & $L_{\mathrm{fus}}$, $b$, $n_{\mathrm{pax}}$ & all declared & 2 bands (32); \mbox{$\kappa$, $\beta$ (4)}\\
\midrule
\multicolumn{5}{@{}l}{\textbf{Remote-sensing axioms $\Rset$} (annotations; 57 trainable scalars; $w_k=1$ except $w_{\mathrm{R9}}=0.5$)}\\ %
R1 & graded band of $b$ for the declared design group, edges fixed by regulation & $b$, design group & all declared & \mbox{$\kappa$, $\beta$ (2)}\\
R2 & graded band of $b$ per role $r$; civil bands fixed by definition, military bands learned & $b$, $r$ & all declared & military bands (8); \mbox{$\kappa$, $\beta$ (4)}\\
R3 & graded band of $L/b$ per $r$ & $L$, $b$, $r$ & all declared & band per $r$ (14); \mbox{$\kappa$, $\beta$ (2)}\\
R4 & graded band of $n_{\mathrm{eng}}$ per $r$ & $n_{\mathrm{eng}}$, $r$ & all declared & band per $r$ (14); \mbox{$\kappa$, $\beta$ (2)}\\
R5 & $-\ln$ of $[\pi=\text{unpowered}]\,\zeta+[\pi\neq\text{unpowered}]\,(1-\zeta)$, $\zeta=\exp\big(-(n_{\mathrm{eng}}/\tau)^2\big)$ & $n_{\mathrm{eng}}$, $\pi$ & all declared & $\tau$ (1)\\
R6 & $-\ln$ of $p_r$, $1-p_r$ or $\tfrac12$ for a jet, propeller or unpowered $\pi$; $p_r=\sigma(\theta_r)$ & $\pi$, $r$ & all declared & $\theta_r$ per $r$ (7)\\
R7 & $-\ln$ of $\IR\big([t\in\{\text{delta, variable-swept}\}],\,[r\text{ military}]\big)$ & $t$, $r$ & all declared & none\\
R8 & $-\ln$ of $\IR\big([\text{canards}],\,[r\in\{\text{bomber, fighter}\}]\big)$ & canards, $r$ & all declared & none\\
R9 & $-\ln$ of $\IR\big([n_{\mathrm{fin}}\ge2],\,[r=\text{fighter}]\big)$ & $n_{\mathrm{fin}}$, $r$ & all declared & none\\
R10 & graded identity of $\ln A_\diamond-\ln(bL/2)$ & $A_\diamond$, $b$, $L$ & all declared & tolerance (1); \mbox{$\kappa$, $\beta$ (2)}\\

\bottomrule
\end{tabular}
\begin{tablenotes}[flushleft]
\scriptsize
\item Field symbols: Section~\ref{sec:meth-data-acq}; latents and graded predicates: Section~\ref{sec:meth-kb}, Eq.~\eqref{eq:soft}. Formulas in SI units. Graded band: two edge softplus hinges plus the in-band term; identity: same, with symmetric tolerance around zero; lower hinge: lower edge only. Degrees capped at $10^4$; logical axioms use $-\log\max(t,10^{-6})$, with truth $t$. Specification banks: fleet value plus class offset scaled by $n_g/(n_g+5)$ (engine type for A10); remote-sensing banks: flat. $\rho(h),a(h)$: International Standard Atmosphere density and sound speed; $\rho_0$: sea-level density; $g$: standard gravity; pair: (class, engine type); $[\cdot]$: indicator; $\operatorname{clip}_a$: clipping to $[-a,a]$; propeller: turboprop or piston. Initial values and bounds: Appendix~\ref{S-app:physics}..
\end{tablenotes}
\end{threeparttable}
\end{widetable}
\begin{widefigure}[t]
\centering
\begin{adjustbox}{max width=\textwidth}
\begin{tikzpicture}[
  zoomframe/.style = {draw=#1, dashed, rounded corners=4pt, line width=0.6pt, inner sep=3.5mm},
  tab/.style       = {minimum height=0pt, minimum width=0pt, inner sep=2pt, font=\scriptsize\bfseries},
  trainflow/.style = {flow, densely dotted},
  alarmflow/.style = {flow, dashed},
]
  \node[io]     (dat) at (-0.9, 0.5) {Data and\\verification\\instances};
  \node[expert] (phy) at (-0.4, 2.5) {Physical relations\\(ISA, Breguet, ratios)};
  \node[expert] (tol) at (5.95, 2.5) {Latent bounds\\(design ranges)};
  \node[fusion] (agg) at (13.14, 0.5) {Masked\\stratified\\aggregation};

  \node[symb] (der) at (2.4, 0.5) {Derived\\quantities $q$,\\coordinates $z$};
  \node[symb] (prd) at (5.95, 0.5)
       {Graded predicates and\\bounded latents\\[1pt]{\scriptsize degree $v_k(z;\ell,h,\kappa,\beta)$,}\\{\scriptsize truth $\sat_k=\exp(-v_k)$}};
  \node[symb] (str) at (9.3, 0.5)
       {Axiom strata\\[1pt]{\scriptsize $\Aset$: 12 axioms,}\\{\scriptsize $\Rset$: 10 axioms}}; %
  \node[symb] (obj) at (8.95,-1.4)
       {Tabular objective\\[1pt]{\scriptsize $\mathrm{BCE}(e^{-V_{\mathcal{T}}},1-y)$}\\{\scriptsize $+\,\lamaux\,\mathrm{BCE}(\sat_k,1-y_k)$}};

  \begin{scope}[on background layer]
    \node[zoomframe=symbline, fit={(der)(str)(obj)([yshift=-1.2mm]prd.north west)(prd.south east)}] (frame) {};
  \end{scope}
  \node[symb, tab, anchor=west] at ([xshift=4mm]frame.south west) {Knowledge base and graded predicates};

  \draw[flow] (dat) -- node[lbl, above, pos=0.4] {$x$, $m$} (der);
  \draw[flow] (phy) -| ([xshift=-4mm]der.north);
  \draw[flow] (der) -- (prd);
  \draw[flow] (prd) -- (str);
  \draw[flow] (str) -- (agg);
  \node[lbl, anchor=south west, inner xsep=0pt] at ([xshift=1mm]frame.east |- str.east) {$v_k$, $\alpha_k$};
  \draw[flow] (tol) -- node[lbl, right, pos=0.34] {initial values} (prd);
  \coordinate (env) at (0,-0.95);
  \draw[trainflow] (dat.south) |- ([xshift=-9mm]prd.south |- env) -- ([xshift=-9mm]prd.south);
  \node[lbl, below] at (3.0,-0.95) {train-prior bands};
  \draw[trainflow] ([xshift=-3mm]str.south) -- node[lbl, right] {$V_{\mathcal{T}}$, $v_k$} ([xshift=-3mm]str.south |- obj.north);
  \draw[trainflow] (obj.west) -| node[lbl, left, pos=0.76] {gradients} ([xshift=6mm]prd.south);
\end{tikzpicture}
\end{adjustbox}
\caption{Knowledge base and graded predicates.}
\label{fig:kb}
\end{widefigure}
The knowledge base $\KB$ has four strata (Table~\ref{tab:axioms}; visual part in Table~\ref{S-tab:axioms-visual}): the specification axioms $\Aset$, 12 relations on records; the remote-sensing axioms $\Rset$, 10 relations on annotations; %
the grounding axioms $\AVset$, in training only; and the cross-modal axioms $\AXset$, at inference only. The tabular strata contain no neural network (Figure~\ref{fig:kb}): each axiom computes an exact, differentiable quantity $q$ of the declared fields, with atmospheric properties from the International Standard Atmosphere (ISA)\footnote{ISO 2533:1975, Standard Atmosphere.} (Appendix~\ref{S-app:physics}), reads it in units of the scatter of clean training records, $z=(\log_{10}q-m_g)/\sigma_g$ for a band and $z=\xi/\sigma_g$ for the residual $\xi$ of an identity or the mass-balance margin, and grades it with a residual-aware predicate (E2), %
\begin{equation}\label{eq:soft}
v=\operatorname{softplus}\big(\kappa(\ell_g-z)\big)+\operatorname{softplus}\big(\kappa(z-h_g)\big)+\beta\Big(\sqrt{1+(z-c_g)^2}-1\Big),\qquad \sat=\exp(-v),
\end{equation}
a violation degree $v$ and its truth. The mean $m_g$ and scatter $\sigma_g\ge10^{-3}$ of group $g$ (class; engine type for A10; fleet for A1 and A2) come from its $n_g$ clean training rows (132 for records: 44 clean aircraft and 88 noisy copies), shrunk toward the fleet with weight $n_g/(n_g+5)$, $\sigma_g$ being a root-mean-square around 0 for an identity or the margin; $[\ell_g,h_g]$ is the band, $[-t_g,t_g]$ for an identity, with center $c_g$; each derived quantity has a slope $\kappa$ and a weight $\beta\ge0$ of the in-band term, which grades a value inside the band by its distance from the center. %
The mass balance A7 keeps its lower hinge only, A9 grades thrust or power per engine as declared, A11 sums its two quantities, every degree is capped at $10^4$, and $\sat_k=\exp(-v_k)$ is never clamped. %
Bounded latents supply quantities that records rarely declare: cruise lift-to-drag ratio $L/D$, thrust- and power-specific fuel consumptions $c_T$ and $c_P$, propeller efficiency $\eta_p$, maximum lift coefficient $C_{L,\max}$ and fuel utilization $f_u$, each a sigmoid of one parameter per pair of class $c$ and engine type $e$ (21 pairs), rescaled to conceptual-design ranges~\cite{raymer2018aircraft}; shared by a class and engine type, they check a record against its class. %

The train prior fits every data-dependent band on the 44 clean training aircraft and widens it to $[0.88\min,\,1.12\max]$, as the oracle does. %
The bands start from the train prior and are learned with their slopes from the fleet and its labeled corruptions, so they act as learned tolerances: 45.7\% of them keep both edges inside the training fleet range widened in the same way, and the others, all wing-loading bands among them, extend beyond it after training (Appendix~\ref{S-app:results-refine}). The learned slopes act as implicit per-constraint weights. %
In hierarchical banks each band or tolerance parameter of a class is a fleet value plus a class offset with the same weight, so a class without training aircraft, such as the business jets, is read with fleet statistics and parameters. %
The remote-sensing axioms follow the same recipe: six graded slots (R1, R2 for civil and for military roles, R3, R4 and R10) read the residual in standard units of the clean training rows, and the logical axioms R5 to R9 enter as the negative logarithm of their truth. %
Each tabular stratum's settings (the power $p$ of a masked power mean of the degrees, the in-band term, flat or hierarchical banks) are chosen on validation alone among eight candidates trained with eight seeds each; both strata select $p=1$, the masked mean of Eq.~\eqref{eq:pme}, and on generator-oracle labels the specification stratum keeps the in-band term and hierarchical banks. %

\paragraph{Training objective.}
\label{sec:meth-kb-obj}
With instance labels $y_i$, $y=1$ for incoherence, and per-constraint targets $y_{ik}$, the generator oracle's verdict on constraint $k$ for a corruption and zero otherwise, the tabular stratum $\mathcal T\in\{\Aset,\Rset\}$ minimizes over a mini-batch $\mathcal B$
\begin{equation}\label{eq:loss-tab}
\mathcal L_{\mathrm{tab}}=\frac{1}{|\mathcal B|}\sum_{i\in\mathcal B}\mathrm{BCE}\big(e^{-V_{\mathcal T}(x_i)},\,1-y_i\big)+\lamaux\,\big\langle\mathrm{BCE}\big(\sat_k(x_i),\,1-y_{ik}\big)\big\rangle_{\alpha_{ik}},\qquad
\langle q\rangle_{\chi}=\frac{\sum_j \chi_j\,q_j}{\max\big(1,\sum_j \chi_j\big)},
\end{equation}
with $V_{\mathcal T}$ of Eq.~\eqref{eq:pme}, both cross-entropies computed in log space from the unclamped degrees, $\lamaux=0.5$ and $\langle q\rangle_{\chi}$ the mean of $q$ where $\chi_j=1$. The residual learners of the comparison receive no per-constraint target. AdamW~\cite{loshchilov2017decoupled} trains bands, slopes and latents for at most 100 epochs, with early stopping on validation AUROC (Table~\ref{S-tab:config}). %

\subsection{Vision Module and Grounding Axiom}
\label{sec:meth-vision}
For each image, the vision module predicts the attributes that the cross-modal axioms compare with the declared record. To let an analyst inspect the evidence behind these predictions, a grounding axiom encourages each head to concentrate its evidence on the part it describes, and the report displays the corresponding activation map (Section~\ref{sec:meth-xmodal}). %

The module uses a backbone with one sigmoid head per class of each visual predicate: 25 heads covering eight predicates for photographs, and 25 covering seven predicates for satellite crops (Figure~\ref{S-fig:grounding} in Appendix~\ref{S-app:method}). %
On photographs a ResNet-50~\cite{he2016resnet} pretrained on ImageNet~\cite{deng2009imagenet} is read at its third stage, a $28\times28$ grid at 448~pixels. Head $k$ returns $P_k(I)=\sigma(s_k(I))$, with $s_k$ linear in the spatially averaged feature maps, and its class activation map (CAM)~\cite{zhou2016learning} normalized to unit mass,
\begin{equation}\label{eq:cam}
H_k(I)=\operatorname{ReLU}\Big(\textstyle\sum_c w_{kc}\,A_c(I)\Big),\qquad
\hat H_k(I)=\frac{H_k(I)}{\sum_{u,v}H_k(I)[u,v]+\varepsilon},
\end{equation}
where $A_c(I)$ is channel $c$ of the feature map, $w_{kc}$ the weight of head $k$ on it, $[u,v]$ a grid cell and $\varepsilon=10^{-6}$. %

On photographs the region $\tilde M_k(I)$ of head $k$ is the union of the automatic pseudo-masks of its parts, OWLv2~\cite{minderer2023scaling} boxes of confidence at least 0.15 segmented by the Segment Anything Model (SAM)~\cite{kirillov2023segment}, dilated by 2~native pixels and averaged over the grid cells, none annotated or verified by hand. The head is visible, $\vis_k(I)=1$, when a part is detected with a confidence of at least 0.20 (Appendix~\ref{S-app:method}). %
On satellite crops every region is the annotated diamond. The in-mask share $\rho_k$ and the localization predicate $\Loc_k$,
\begin{equation}\label{eq:rho-loc}
\rho_k(I)=\frac{\sum_{u,v}\hat H_k(I)[u,v]\,\tilde M_k(I)[u,v]}{\sum_{u,v}\hat H_k(I)[u,v]+\varepsilon},\qquad
\Loc_k(I)=\sigma\big(\kappa_k(\rho_k(I)-\ell_k)\big),\quad \ell_k=\sigma(\tilde\ell_k),\quad \kappa_k=\exp(\tilde\kappa_k),
\end{equation}
read ``at least a fraction $\ell_k$ of the attribution lies in the region'', with a learned threshold $\ell_k$ and slope $\kappa_k$ per head. The grounding axiom of head $k$ (E3) is the guarded Reichenbach implication
\begin{equation}\label{eq:av}
\mathcal{AV}_k:\ \forall I\,[\vis_k(I)]:\ P_k(I)\rightarrow\Loc_k(I),\qquad
\sat_{\mathcal{AV}_k}(I)=\IR\big(P_k(I),\Loc_k(I)\big)=1-P_k(1-\Loc_k).
\end{equation}
Its \emph{gate} $P_k$ asks a head that asserts its class to localize its evidence, with no label, and its \emph{guard} masks a conjunct whose part is not visible. %

\paragraph{Training objective.}
\label{sec:meth-vision-obj}
With per-head labels $t_{ik}$, the one-hot attribute labels, and their availability $\omega_{ik}$, the vision module minimizes
\begin{equation}\label{eq:loss-vis}
\begin{gathered}
\mathcal L_{\mathrm{vis}}=\big\langle\mathrm{BCE}\big(P_k(I_i),\,t_{ik}\big)\big\rangle_{\omega_{ik}}+\lambda_{\AVset}\Big(1-\big\langle\sat_{\mathcal{AV}_k}(I_i)\big\rangle_{\vis_k(I_i)}\Big)+\lambda_{\mathrm{fl}}\,\Phi,\\
\Phi=\Big\langle\operatorname{sg}\big[P_k(I_i)\big]\max\big(0,1-S_k(I_i)/S_0\big)\Big\rangle,
\end{gathered}
\end{equation}
with $\lambda_{\AVset}=1$ and $\lambda_{\mathrm{fl}}=0.1$. The \emph{floor} $\Phi$ ($S_k$ the total activation of $H_k$, $S_0=1$, $\operatorname{sg}$ a stop-gradient) stops a head from asserting its class through its bias with a negligible map; the cross-modal axioms read the frozen posteriors (Table~\ref{S-tab:config}).
Explanation supervision~\cite{ross2017right,gao2022res,pukdee2023learning} suggests instead an auxiliary loss, the Energy loss~\cite{rao2023studying}, here the mean of $1-\rho_k$ over the visible label-positive heads, with about 4.3 times the axiom's localization weight per conjunct at the same loss weight (Section~\ref{sec:res-grounding}). %

\par\smallskip\indent
We train and compare multiple visual backbones. Satellite experiments use ResNet-50, ConvNeXt-T and ViT-B/16, while photographic grounding is also evaluated across multiple backbones (Table~\ref{S-tab:config}; Appendix~\ref{S-app:results-grounding}).

\subsection{Cross-Modal Axioms}
\label{sec:meth-xmodal}
The cross-modal axioms (Table~\ref{S-tab:axioms-visual}; center of Figure~\ref{fig:data}) compare declared attributes with the visual posteriors $\pvis(\cdot\mid I)$ predicted from the image. For each predicate, the posterior is obtained by a softmax over its heads, or a sigmoid for a binary predicate. Axiom truths are computed from these posteriors, expected truths of implications, or products of such terms, without additional trainable parameters. An axiom is available when its declared value is covered by the corresponding visual head and the relevant parts are visible.

The report displays each grounded head's activation map and in-mask share $\rho_k$ alongside the cross-modal verdicts (Figure~\ref{S-fig:rankexplain}), allowing an analyst to inspect which image regions support a reported contradiction. %

\subsection{Masked Stratified Aggregation}
\label{sec:meth-agg}
A tabular stratum $\mathcal T\in\{\Aset,\Rset\}$ averages the degrees of its available axioms (E1), and the cross-modal stratum aggregates their truths by the masked $p$-mean error,
\begin{equation}\label{eq:pme}
V_{\mathcal T}(x)=\frac{\sum_{k\in\mathcal T}w_k\,\alpha_k\,v_k}{\sum_{k\in\mathcal T}w_k\,\alpha_k},\qquad
\feas_{\AXset}(x,I)=1-\left(\frac{\sum_{k\in\AXset}w_k\,\alpha_k\,(1-\sat_k)^{4}}{\sum_{k\in\AXset}w_k\,\alpha_k}\right)^{1/4},
\end{equation}
each where its denominator is positive, with $w_k=1$ except 0.5 for R9 and the satellite AX5. The feasibility $\feas_{\mathcal T}=\exp(-V_{\mathcal T})$ is the weighted geometric mean of the available truths, and $V_{\mathcal T}=\ln 2$ (feasibility 0.5) without available axioms; with unit weights $\feas_{\AXset}$ is the masked aggregator of LTNtorch, and without available axioms the cross-modal stratum drops out. The strata present combine by a weighted mean (E4), %
\begin{equation}\label{eq:strata}
\feas(x,I)=\frac{\lambda_{\mathcal T}\,\feas_{\mathcal T}(x)+\lambda_{\AXset}\,\gamma_{\AXset}(x,I)\,\feas_{\AXset}(x,I)}{\lambda_{\mathcal T}+\lambda_{\AXset}\,\gamma_{\AXset}(x,I)},\qquad
\gamma_{\AXset}=\Big[\textstyle\sum_{k\in\AXset}w_k\alpha_k>0\Big],
\end{equation}
with $\lambda_{\mathcal T}=\lambda_{\AXset}=1$; when $\gamma_{\AXset}=0$ the implementation returns $\feas_{\mathcal T}$ itself. Photographic pairs are scored by the cross-modal stratum alone, which reads no specification truth. %

\begin{proposition}[Exact reduction]
\label{prop:reduction}
Let the verdict be computed with Eqs.~\eqref{eq:pme} and~\eqref{eq:strata}.
(a)~For every stratum $\mathcal S$, instance $(x,I)$ and fixed availability pattern $\alpha$ (which, for cross-modal axioms, depends on the declared value and on visibility), let $\mathcal S'=\{k\in\mathcal S:\alpha_k(x,I)=1\}$. If $\mathcal S'$ is nonempty, the stratum verdict, $V_{\mathcal S}$ or $\feas_{\AXset}$, coincides with Eq.~\eqref{eq:pme} on $\mathcal S'$ with every axiom available; otherwise a tabular stratum takes $V_{\mathcal S}=\ln 2$ and the cross-modal stratum is absent ($\gamma_{\AXset}=0$). Either way it depends neither on the degree or truth of a masked axiom nor on a field read only by masked axioms, and has zero derivative with respect to them.
(b)~With a tabular stratum ($\lambda_{\mathcal T}>0$), if $I=\varnothing$, or every cross-modal axiom is masked, then $\feas(x,I)=\exp\big(-V_{\mathcal T}(x)\big)$ exactly; otherwise $\feas(x,I)$ is the verdict of the knowledge base restricted to the strata present (including the tabular stratum at feasibility $1/2$ when none of its axioms is available), with their weights renormalized.
\end{proposition}
The proof is in Appendix~\ref{S-app:method-proof}. The reduction holds because only available axioms enter the mean, a structural property of E1 that also holds in training, where a masked axiom receives zero gradient; its practical value is an exact abstention, which crisp rule validation shares (Table~\ref{tab:sota}) and learned heads do not keep (Section~\ref{sec:res-missing}); stratification (E4) keeps one verdict per evidence source (Section~\ref{sec:res-knowledge}). %

\subsection{Verdict and Diagnosis Keys}
\label{sec:meth-heads}
\fwname{}'s verdict is $V=-\ln\feas(x,I)$, the masked mean degree $V_{\mathcal T}$ for a record without image. Bands and slopes are learned for each derived quantity, with no additional weights learned for aggregation. The resulting verdict serves as the record score. \fwname{} converts this score into a binary decision using the threshold $t^{\star}$ that maximizes Youden's index on the validation set~\cite{youden1950index}. %
A specification record with no available constraint scores $V=\ln 2$, above the validation thresholds of the specification verifier (0.310 to 0.558 across seeds), so it is flagged for lack of evidence, with an evidence coverage of 0 in its report (Section~\ref{sec:meth-report}). %
For a flagged specification record with at least one available constraint, \emph{constraint diagnosis} identifies the constraint with the largest violation degree. This gives the \emph{largest-violation key}: the argmax of the degrees entering the verdict, with no additional parameter to fit. We also retain a \emph{secondary key}, identifying the constraint that contributes most to the second head's score. %
The second head (E5) is a logistic regression with non-negative weights on the degrees $e_k=\alpha_k v_k$ of the available axioms,
\begin{equation}\label{eq:mono}
\hat p_{\mathrm{m}}(x,I)=\sigma\Big(\phi_0+\textstyle\sum_k\big(\phi_k\,e_k+\phi^{\mathrm a}_k\,\alpha_k\big)+\phi_{\mathrm s}\,V\Big),\qquad \phi_k,\ \phi_{\mathrm s}\ge0,
\end{equation}
with availability terms $\phi^{\mathrm a}_k$ free in sign, fitted by $L_2$-regularized maximum likelihood with a bound-constrained quasi-Newton method~\cite{byrd1995limited} and selected on validation AUROC; on satellite pairs it reads $V_{\Rset}$ in place of $V$ and also the cross-modal truths, with non-positive weights, and the cross-modal score $1-\feas_{\AXset}$ (0 without an image), with a non-negative weight. The contribution $c_k=\operatorname{logit}\hat p_{\mathrm{m}}(x,I)-\operatorname{logit}\hat p_{\mathrm{m}}(x,I\mid v_k{:=}0)$, with $V$ (or $V_{\Rset}$) recomputed, gives the secondary key, a learned second ranking of the same degrees that the report prints beside the largest-violation key (Section~\ref{sec:res-diag}, Figure~\ref{fig:diagnosis}b). %

\subsection{Verification Report and Optional Checks}
\label{sec:meth-report}
Every instance receives a verification report (Figure~\ref{S-fig:rankexplain} in Appendix~\ref{S-app:method}; example in Figure~\ref{fig:report}) whose per-constraint degrees state, by one rule, which constraints the record violates and which it respects: an available constraint is \emph{violated} when its truth $\exp(-v_k)$ is below 0.5, that is $v_k>\ln 2$, reached at the learned band edge, where a hinge of Eq.~\eqref{eq:soft} alone equals $\ln 2$, or inside the band where the learned in-band term adds to it (A1, A2); it is \emph{held} otherwise, and a masked constraint is not evaluated. The rule has no tuned parameter: it reads at 0.5 the truths that the per-constraint term of Eq.~\eqref{eq:loss-tab} trains toward 1 on constraints the generator leaves held and toward 0 on those it violates. Its evidence coverage counts the available constraints among those whose scope contains the record, so a masked axiom is an abstention, and the report gives the class of the nearest clean training record and its fleet similarity $\exp(-d/\tau)$, with $d$ the root-mean-square difference from it over the fields both records declare, in log-scaled standardized coordinates, and $\tau$ the median leave-one-out nearest-neighbor $d$ in the clean training fleet, so that a value below $e^{-1}$ marks a record farther from the fleet than a typical clean training record. %

For a flagged specification record, \emph{field localization} names the altered field by abduction, a search over single-field changes for the one that restores consistency; a field whose change alone restores it is a single-fault diagnosis in the sense of consistency-based diagnosis~\cite{reiter1987theory}. For every field $f$ read by an available constraint, the verifier evaluates the verdict when $f$ alone moves by $d$ in log space,
\begin{equation}\label{eq:abduction}
\psi_f(d)=V_{\Aset}\big(x^{(f,d)}\big),\qquad x^{(f,d)}_f=x_f\exp(d),\quad x^{(f,d)}_j=x_j\ (j\neq f),\qquad d\in\{-3,-2.995,\dots,3\},
\end{equation}
and ranks first the fields whose change alone can bring the verdict to $V\le-\ln 0.9\approx0.1054$, by the smallest such change, then the others by the lowest verdict they reach, a rule and a threshold selected on validation and independent of the diagnosis. %
\emph{Repair proposals} then solve
\begin{equation}\label{eq:repair}
d^{\star}=\arg\min_{d}\ \|d\|_2^2\quad\text{s.t.}\quad\feas_{\Aset}\big(x\odot\exp(d)\big)\ge\tau_r,\qquad d_j=0\ \text{for } j\notin F_{k_r}(x),
\end{equation}
the smallest log-space change of the $k_r$ top-ranked fields $F_{k_r}(x)$ that brings the verdict to the threshold, a counterfactual over the declared fields~\cite{guidotti2024counterfactual}, by an augmented-Lagrangian method~\cite{nocedal2006numerical} with Adam~\cite{kingma2014adam} steps, keeping its best-effort change when the threshold is not reached; validation chose $\tau_r = 0.9$ ($V\le-\ln 0.9$), a target stricter than the decision threshold, and $k_r = 1$ for the repair error (Section~\ref{sec:res-local}). %

\fwname{}'s audit has three optional settings. With \emph{rank-then-explain}, the residual MLP, an MLP on the 14 physical quantities of the oracle (Table~\ref{S-tab:comparators}), raises the alarm at its validation Youden threshold and the largest-violation key explains each flagged record; the flagged-and-diagnosed rate, the share of corrupted records flagged with a violated top-ranked constraint, scores this setting. %
The \emph{critical-field alert} flags a record that leaves undeclared a field required by an applicable axiom and declared by at least 95\% of the training records of its class - or of its engine type for classes with fewer than five training records -, reading only the undeclared fields, the class, the engine type and a table of the training records; with it enabled, \fwname{} flags a record when its verdict or the alert fires. %
With \emph{MICE imputation} of the missing fields enabled, \fwname{} scores imputed values in place of abstaining on them; Section~\ref{sec:res-missing} compares this setting with learners given the same imputer. %
\providecommand{\twoci}[4]{\ensuremath{[#1,\penalty 200\ #2]\cdot\allowbreak[#3,\penalty 200\ #4]}}
\providecommand{\oneci}[2]{\ensuremath{[#1,\penalty 200\ #2]}}
\ifdefined\bodyfigwidth\else\newlength{\bodyfigwidth}\fi
\setlength{\bodyfigwidth}{16.45cm}%

\section{Evaluation and Results}
\label{sec:res}

\subsection{Experimental Setup}
\label{sec:res-setup}\suppressfloats[t]%

The vision module and the tabular verifier are fitted separately, as described in Section~\ref{sec:meth} (Table~\ref{S-tab:config}, Appendix~\ref{S-app:protocol}); where a result uses one of \fwname{}'s optional settings, rank-then-explain, the critical-field alert or MICE imputation, the text names it. %
The configuration was selected on validation data and held-out training classes. The reported results use a visibility threshold calibrated on a subset including test and out-of-distribution images, and generator labels from an oracle fitted on all 107 aircraft (Appendix~\ref{S-app:protocol}). Sensitivity analyses at the end of this section repeat threshold selection and oracle fitting without test or out-of-distribution data. %

\paragraph{Benchmarks and labels.}
Table~\ref{tab:evalmap} maps each benchmark to its labels and to the sections that use it; every error in this study is synthetic, an edit or a constructed mismatch of public records and images. %
Two label families recur. \emph{Generator labels} come from the generator oracle of Section~\ref{sec:meth-data-gen}, which judges the same residual quantities as the knowledge base against class bands fitted on all 107 aircraft; its single-field corruptions are escalated from 5 or 10\% until the oracle flags them, mostly to 20 to 160\% (geometric ones stay at 5 or 10\%; Table~\ref{S-tab:magnitudes}). %
\emph{Intervention labels} cover single-field edits drawn once at 5 to 80\% and cross-class block swaps: a field counts as altered beyond three noise standard deviations (0.06 in natural-log units for continuous fields, 0.134 for the aspect ratio, any change of an integer field), so edits within this tolerance, e.g., the 5\% edits of continuous fields, count as unaltered, and a constraint is labeled when its scope contains an altered field. Such a constraint is involved in the edit, which need not violate it: an edit can keep the record inside the physical envelope of the fleet, and 613 of the 1,810 altered rows, about a third, break no constraint even under bands fitted on all 107 aircraft, so that no physical check can flag them (Appendix~\ref{S-app:results-detect}). %
Generator labels follow the same relations as the axioms and were \fwname{}'s per-constraint training targets; intervention labels, set from the edit alone, are independent of them. %
Minimum-cardinality localization (Fellegi-Holt and MaxSMT~\cite{bjorner2015nuz}, all optimal solutions) breaks ties at random unless the oracle tie-break is named. %

\begin{widetable}[t]
  \centering
  \caption{Benchmarks and labels.}
  \label{tab:evalmap}\vspace{3pt}
  \scriptsize
  \setlength{\tabcolsep}{4pt}
  \begin{tabularx}{\textwidth}{@{}L{2.55cm}>{\raggedright\arraybackslash}XC{1.6cm}L{1.3cm}L{4.1cm}@{}}
    \toprule
    Benchmark & Labels & Independent of the oracle & Sections & Size \\
    \midrule
    Specification records & generator labels: corruptions kept when the oracle flags them, with per-constraint labels; noisy-but-coherent controls; fields hidden at random & no & \ref{sec:res-diag} to \ref{sec:res-missing}, \ref{sec:res-knowledge} & 107 aircraft of 71 families; test: 16 aircraft of 9 families, 204 rows with 32 controls \\ %
    Intervention-labeled records & from the edit alone, with a tolerance fixed before generation & yes & \ref{sec:res-diag} to \ref{sec:res-detect}, \ref{sec:res-knowledge} & 2,582 rows, 1,810 altered; same aircraft and splits \\ %
    Falsification & the attacker's edits of genuine records & yes & \ref{sec:res-falsif} & 736 attack problems on the 46 genuine held-out records \\ %
    Satellite pairs (RarePlanes) & generator labels; image swaps & no & \ref{sec:res-detect}, \ref{sec:res-missing} & 5,250 pairs; 10 test and 8 unseen sites \\ %
    Photographic pairs (FGVC-Aircraft) & constructed mismatches, attribute labels, automatic pseudo-masks & yes & \ref{sec:res-detect}, \ref{sec:res-grounding} & 4,543 pairs; 800 test photographs \\ %
    Human part masks & PASCAL-Part~\cite{chen2014detect}, PartImageNet~\cite{he2022partimagenet} and UDA-Part~\cite{liu2022udapart} annotations, for evaluation only & yes & \ref{sec:res-grounding} & 535 aircraft in 506 PASCAL-Part images; 40 UDA-Part items \\ %
    \bottomrule
  \end{tabularx}
\end{widetable}%

\paragraph{Statistics.}
Intervals are 95\% two-level cluster-bootstrap percentile intervals~\cite{field2007bootstrapping}, reported as [aircraft]~$\cdot$~[family] for records. The $\pm$ notation denotes standard deviation across seeds. Residual MLP and \fwname{} runs are paired by seed. %
A method is higher or lower than another, ahead or below, when both intervals of their difference exclude zero, and on par with it when both include zero; when only one excludes zero, when a bound prints 0.000 or lies within 0.003 of zero, or when the reading changes in one of 20 alternative bootstrap streams, the text gives the levels in plain words (close, comparable, slightly). %
The text gives levels, Tables~\ref{S-tab:contrasts} to~\ref{S-tab:contrasts-4} and Appendix~\ref{S-app:results-verdicts} the differences and intervals, and Appendix~\ref{S-app:results} the spread over seeds and the readings across the alternative streams; no multiplicity correction is applied.

\paragraph{Comparators, metrics and cost.}
Learners on the raw fields or on the residuals, the 14 physical quantities that the oracle checks, score records (the residual MLP, TabPFN v2~\cite{hollmann2025accurate}, LightGBM~\cite{ke2017lightgbm}, a one-class support vector machine, SVM); class standardization and the grouped Shapley additive explanations (SHAP)~\cite{lundberg2017unified} of a second model (MLP-SHAP, LightGBM-SHAP) rank the residuals; Fellegi-Holt, Z3 and rule checkers return crisp violated sets, a network trained with the DL2 loss~\cite{fischer2019dl2} on the same residuals a probability per constraint, and language and vision-language models given the rules a judgment (Table~\ref{S-tab:comparators}). %
We call $V$ the record verdict and the held or violated outputs per constraint the constraint verdicts. The constraint verdicts use exact-set agreement per record, macro-F1 over the constraints with at least five labeled violations and the false-violation rate on clean records. Diagnosis and localization use top-1 rates, repair the log-space error and grounding the pointing game~\cite{zhang2018top}; detection scores the record verdict by AUROC and the false-positive rate (FPR) on the 32 controls at a true-positive rate (TPR) of 0.90. %
Experiments ran on an AMD Ryzen 9 9950X3D with 192~GB of RAM, an NVIDIA RTX 5090 and an RTX 4090. Per record, the verdict and the key take one pass through the knowledge base, abduction evaluates the verdict at 1,201 log-space steps for each field read by an available constraint, and the repair solves one augmented-Lagrangian problem; training curves are in Appendix~\ref{S-app:results-curves}.\label{sec:res-dyn} %

\noindent Figure~\ref{fig:audit} shows, for each test record, the violation $1-\sat_k=1-\exp(-v_k)$ of each constraint (mean of 8 seeds) with the rule of Section~\ref{sec:meth-report} drawn as a step at 0.5: dark cells are violated constraints, gray cells masked ones, framed cells the oracle's flags. %

\begin{widefigure}[t]
  \centering
  \includegraphics[width=\bodyfigwidth]{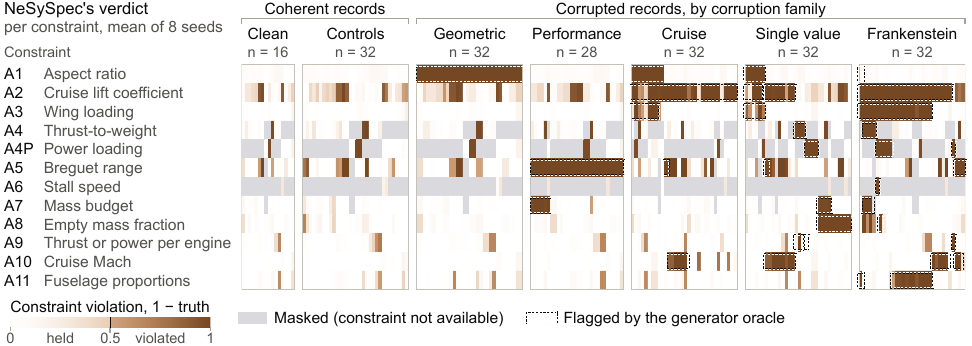}%
  \caption{Constraint-level verdicts on test records.}
  \label{fig:audit}
\end{widefigure}

\paragraph{The audit at a glance.}
Table~\ref{tab:glance} lines up \fwname{} with its comparators on every output of the audit. Learners on the physical residuals return a record score and need a second model, a head per constraint or attributions, for a verdict or a diagnosis; solvers return a violated set, an edit or a repair; language and vision-language models return a judgment. \fwname{} is the only method that returns every output, and without an image its fused verdict alone returns the record verdict exactly. On the test aircraft no comparator is ahead of it in the macro-F1 of the constraint verdicts, in naming the violated constraint and the altered field or in the repair error on the altered field; the network trained with the DL2 loss keeps fewer false violations (Section~\ref{sec:res-diag}). In record and pair detection, learners on the residuals and late fusions are close to \fwname{} and in some settings rank records or pairs higher (Sections~\ref{sec:res-detect} to~\ref{sec:res-missing}). Learners on the residuals do so on the unseen class, after retraining on intervention labels (Table~\ref{S-tab:contrasts-3}) and with 10\% of the fields hidden, the MLP trained with masks on test and the residual MLP on unseen families; late fusions do so on satellite pairs after learning from labeled pairs, as does Qwen3-VL-32B on photographs of unseen families. With the axiom's 8 seeds, the auxiliary loss points higher on test pseudo-masks (Section~\ref{sec:res-grounding}, Appendix~\ref{S-app:protocol}). %

\begin{widetable}[t]
  \centering
  \caption{The audit at a glance: \fwname{} and its comparators on each output.}
  \label{tab:glance}\vspace{3pt}
  \scriptsize
  \setlength{\tabcolsep}{2pt}
  \begin{tabularx}{\textwidth}{@{}>{\raggedright\arraybackslash}XC{1.75cm}C{1.3cm}C{1.75cm}C{0.85cm}C{1.05cm}C{1.3cm}C{1.1cm}C{1.1cm}C{1.05cm}@{}}
    \toprule
    Method & Verdicts macro-F1 & Violated constraint & Altered field & Repair error & Exact, no image & Part pointing, PASCAL-Part & Records AUROC & Half fields hidden, unseen class & Satellite pairs \\
    \midrule
    \multicolumn{10}{@{}l}{\textit{Learners on the physical residuals}} \\
    Residual MLP & \textperiodcentered & \textperiodcentered & \textperiodcentered & \textperiodcentered & \textperiodcentered & \textperiodcentered & 0.990 & 0.536 & \textperiodcentered \\ %
    TabPFN~v2 with native missing values & \textperiodcentered & \textperiodcentered & \textperiodcentered & \textperiodcentered & \textperiodcentered & \textperiodcentered & 0.957 & 0.539 & \textperiodcentered \\ %
    LightGBM trained with masks & \textperiodcentered & \textperiodcentered & \textperiodcentered & \textperiodcentered & \textperiodcentered & \textperiodcentered & - & 0.566 & \textperiodcentered \\ %
    MLP-SHAP & 0.693/0.479\textsuperscript{a} & 0.942 & \textperiodcentered & \textperiodcentered & \textperiodcentered & \textperiodcentered & 0.990 & - & \textperiodcentered \\ %
    LightGBM-SHAP & 0.595/0.418\textsuperscript{a} & 0.871 & \textperiodcentered & \textperiodcentered & \textperiodcentered & \textperiodcentered & 0.957 & - & \textperiodcentered \\ %
    Network with the DL2 loss, one head per constraint & 0.661/0.426\textsuperscript{b} & 0.752 & \textperiodcentered & \textperiodcentered & \textperiodcentered & \textperiodcentered & 0.965 & - & \textperiodcentered \\ %
    Class-standardized residual & 0.714/0.480\textsuperscript{b} & 0.929 & 0.403/- & \textperiodcentered & \textperiodcentered & \textperiodcentered & - & - & \textperiodcentered \\ %
    \multicolumn{10}{@{}l}{\textit{Solvers, statistics and imputation}} \\
    Fellegi-Holt & 0.829/0.505 & 0.795 & - & 0.241 & \textperiodcentered & \textperiodcentered & 0.856 & - & \textperiodcentered \\ %
    Z3 & 0.829/0.505 & 0.825 & - & - & \textperiodcentered & \textperiodcentered & 0.867 & - & \textperiodcentered \\ %
    Minimum-cardinality localization (Fellegi-Holt, MaxSMT), random tie-break & - & - & 0.367/0.306 & - & \textperiodcentered & \textperiodcentered & - & - & \textperiodcentered \\ %
    Same localization, oracle tie-break & - & - & 0.629/0.453 & - & \textperiodcentered & \textperiodcentered & - & - & \textperiodcentered \\ %
    Class I sizing regressions & \textperiodcentered & \textperiodcentered & 0.597/0.453 & 0.160 & \textperiodcentered & \textperiodcentered & - & - & \textperiodcentered \\ %
    MICE at the Class I field & \textperiodcentered & \textperiodcentered & \textperiodcentered & 0.120 & \textperiodcentered & \textperiodcentered & \textperiodcentered & \textperiodcentered & \textperiodcentered \\ %
    Constraint-layer ordering (PiShield style) & \textperiodcentered & \textperiodcentered & \textperiodcentered & 0.218 & \textperiodcentered & \textperiodcentered & \textperiodcentered & \textperiodcentered & \textperiodcentered \\ %
    \multicolumn{10}{@{}l}{\textit{Knowledge given as a loss or in words}} \\
    DL2 method on the declared fields & 0.425/0.259 & 0.443 & \textperiodcentered & \textperiodcentered & \textperiodcentered & \textperiodcentered & 0.703 & - & \textperiodcentered \\ %
    Language model given the constraints in words & 0.153/- & 0.128 & \textperiodcentered & \textperiodcentered & \textperiodcentered & \textperiodcentered & 0.620 & - & \textperiodcentered \\ %
    \multicolumn{10}{@{}l}{\textit{Images}} \\
    Auxiliary localization loss & \textperiodcentered & \textperiodcentered & \textperiodcentered & \textperiodcentered & \textperiodcentered & 0.739 & \textperiodcentered & \textperiodcentered & \textperiodcentered \\ %
    RRR, GradMask, Pukdee et al. & \textperiodcentered & \textperiodcentered & \textperiodcentered & \textperiodcentered & \textperiodcentered & lower & \textperiodcentered & \textperiodcentered & \textperiodcentered \\ %
    End-to-end fusions trained on labeled pairs & \textperiodcentered & \textperiodcentered & \textperiodcentered & \textperiodcentered & \textperiodcentered & \textperiodcentered & \textperiodcentered & \textperiodcentered & lower\textsuperscript{c} \\ %
    Learned fusions and heads on satellite pairs, \fwname{}'s second head included & \textperiodcentered & \textperiodcentered & \textperiodcentered & \textperiodcentered & 0 of 207 & \textperiodcentered & \textperiodcentered & \textperiodcentered & close\textsuperscript{d} \\ %
    Qwen3-VL-8B prompted with the rules & \textperiodcentered & \textperiodcentered & \textperiodcentered & \textperiodcentered & \textperiodcentered & \textperiodcentered & \textperiodcentered & \textperiodcentered & 0.748 \\ %
    \midrule
    \textbf{\fwname{} (ours)} & \textbf{0.795/0.575} & \textbf{0.942} & \textbf{0.742/0.593} & \textbf{0.133} & \textbf{9 of 9} & \textbf{0.737} & \textbf{0.979} & \textbf{0.577} & \textbf{0.994} \\ %
    \bottomrule
  \end{tabularx}
  \par\vspace{2pt}\begin{minipage}{\textwidth}\scriptsize
  Test unless stated otherwise; slash: generator/intervention labels. Diagnosis: top-1 accuracy; repair: log-space error on the altered field (lower is better); detection: AUROC (Z3 score: violated-edit count). Field masking is random; satellite AUROC pools all pair types. Exactness: satellite cells (test and unseen sites) whose no-image verdict matches the record verdict bit for bit. Pointing: 3 runs per method, loss weights 1/0.3 (ours/auxiliary). Fellegi--Holt/Z3 verdicts, MICE/Class I repairs and auxiliary-loss pointing are on par with \fwname{}. \textperiodcentered: output not provided or no image input; -: not reported.
\textsuperscript{a}Corrupted records only; corresponding \fwname{} macro-F1: 0.832/0.601.
\textsuperscript{b}Validation-tuned thresholds: network 0.746/0.541; residual 0.560 on intervention labels.
\textsuperscript{c}ResNet-50 for both; test and unseen sites.
\textsuperscript{d}Some fusions on knowledge-base truths, trained on labeled pairs, achieve higher AUROC on test swaps, unseen sites or with 30\% of images missing (Section~\ref{sec:res-xmodal}; Appendix~\ref{S-app:results-xmodal}).
Comparisons and intervals: Sections~\ref{sec:res-diag} to~\ref{sec:res-grounding}; Appendix~\ref{S-app:expsup}. %
  \end{minipage}
\end{widetable}

\paragraph{Robustness to the label oracle and the visibility threshold.}
Two further oracles exclude test and out-of-distribution aircraft: the first is fitted on the 44 training aircraft with the paper's tolerances, the second on the 61 training and validation aircraft with tolerances chosen by leaving one family out. On intervention labels, under both oracles, \fwname{} keeps its leads over Fellegi-Holt and Z3 in diagnosis and over minimum-cardinality localization with random tie-break on every split. It keeps its macro-F1 lead over the DL2-loss network on unseen families under both oracles, on test under the first and slightly under the second, and on the unseen class under the first only. Its narrower leads over the oracle tie-break, the crisp edits and the Class I residual become on par or lower on some splits. Its own levels on test are lower (macro-F1 0.488 and 0.537, abduction 0.532 and 0.491, key 0.732 and, slightly, 0.738, against 0.575, 0.593 and 0.786), and its false violations rise only under the paper's tolerances, whose bands flag 38 of the 63 clean held-out aircraft and narrow the bands it learns (Appendix~\ref{S-app:results-refine}). %
Re-selected on validation alone, the part detector's visibility threshold becomes 0.25. Retrained at 0.25, the photograph module keeps its comparisons against the networks without a localization term or with permuted masks and against RRR, GradMask and Pukdee et al.\ on PASCAL-Part, all kept at 0.20, against the auxiliary loss also retrained at 0.25 and against the vision-language models. Its pointing is slightly lower on test pseudo-masks (0.774 against 0.780) and on PASCAL-Part human masks (0.711 against 0.737), where the auxiliary loss falls as well (0.708 against 0.739), and its cross-modal AUROC is slightly lower on unseen families and the unseen class (Appendix~\ref{S-app:protocol}). %

\subsection{Constraint Verdicts and Diagnosis}
\label{sec:res-diag}

For every available constraint of a record, \fwname{} returns a graded degree and, by the fixed rule of Section~\ref{sec:meth-report}, a held or violated constraint verdict (Figures~\ref{fig:audit} and~\ref{fig:report}); these verdicts reach a macro-F1 of 0.795 on generator-label test and, without retraining, 0.575 on intervention-label test. %
Against a network trained with the DL2 loss on the same residuals and per-constraint targets, which returns a probability per constraint, they reach a higher macro-F1 on every split and both label families (the network at 0.661 and 0.426 on the two tests; Figure~\ref{fig:diagnosis}a), and the same key read from the network ranks the violated constraint first less often on every generator-label split (0.752 against 0.942). %
They also find every labeled violation of a corrupted record more often than the network on every split and both label families (0.934 against 0.522 on generator-label test; micro recall 0.959 against 0.635), with micro precision on par on test and lower on the unseen class (0.458 against 0.577 on generator and 0.517 against 0.639 on intervention labels; Appendix~\ref{S-app:results-verdicts}). %
With the network's thresholds tuned on validation the macro-F1 lead holds on generator-label test and is slight on intervention-label test (0.575 against 0.541; Table~\ref{S-tab:verdicts}). %
The network reports fewer false violations on clean records (0.020 against 0.074 on generator-label test, also on intervention labels), and \fwname{} also returns the altered field and a repair proposal from the same degrees. %

\begin{widefigure}[t]
\centering
{\scriptsize%
\begin{adjustbox}{max width=\textwidth}
\begin{tikzpicture}[
    ttl/.style  = {anchor=south west, inner sep=0pt, font=\bfseries},
    sub/.style  = {anchor=west, inner sep=0pt, text height=1.6ex, text depth=0.45ex, text=ioline, font=\itshape},
    cid/.style  = {anchor=west, inner sep=0pt, text height=1.6ex, text depth=0.45ex},
    val/.style  = {anchor=west, inner sep=0pt, text height=1.6ex, text depth=0.45ex},
    vio/.style  = {anchor=west, inner sep=0pt, text height=1.6ex, text depth=0.45ex, font=\bfseries},
    msk/.style  = {anchor=west, inner sep=0pt, text height=1.6ex, text depth=0.45ex, text=ioline, font=\itshape},
    txt/.style  = {anchor=north west, inner sep=0pt, text height=1.6ex, text depth=0.45ex},
    line/.style = {anchor=north west, inner sep=0pt, align=left},
    bar/.style  = {fill=symbfill, draw=symbline, line width=0.4pt},
    hit/.style  = {fill=symbline, draw=symbline, line width=0.4pt},
    cut/.style  = {fill=white, draw=symbline, line width=0.4pt},
  ]
  \node[ttl] at (0,0.12) {(a) Declared record (test split)};
  \node[line] at (0,0) {%
    \begin{tabular}{@{}l@{\hspace{4pt}}l@{\hspace{9pt}}l@{\hspace{4pt}}r@{\hspace{9pt}}l@{\hspace{4pt}}r@{\hspace{9pt}}l@{\hspace{4pt}}r@{}}
      class; engine type & narrow-body; turbofan & $\MTOW$ (kg) & 68,000 & wingspan (m) & 34.1 & cruise speed (m/s) & 230 \\ %
      engine count & 2 & $\OEW$ (kg) & 39,500 & wing area (m$^2$) & 122.6 & cruise altitude (m) & 11,278 \\ %
      total thrust (kN) & \textbf{160.17} \textit{\color{ioline}(clean 192.2)} & usable fuel (kg) & 19,465 & aspect ratio & 9.48 & range (km) & 5,750 \\ %
      \multicolumn{2}{@{}l}{not declared: shaft power, stall speed} & passenger limit & 132 & fuselage length (m) & 31.44 & ceiling (m) & 12,131 \\ %
    \end{tabular}};
  \begin{scope}[shift={(0,-1.87)}]
    \node[ttl] at (0,0.12) {(b) Constraint-level verdicts};
    \node[sub] at (0,-0.13) {constraint};
    \node[sub] at (3.5,-0.13) {verdict};
    \node[sub, anchor=east] at (5.7,-0.13) {$v_k$};
    \node[sub] at (5.8,-0.13) {degree};
    \foreach \id/\nm/\row in {A1/aspect ratio/1, A2/lift coefficient/2, A3/wing loading/3, A4/thrust to weight/4,
        A4P/power loading/5, A5/Breguet range/6, A6/stall speed/7, A7/mass balance/8, A8/empty-mass fraction/9,
        A9/thrust per engine/10, A10/cruise Mach number/11, A11/fuselage ratios/12} {
      \node[cid] at (0,{-0.2-0.27*\row}) {\id};
      \node[cid] at (0.72,{-0.2-0.27*\row}) {\nm};
    }
    \foreach \row in {4,10} {\node[cid] at (0.36,{-0.2-0.27*\row}) {$^{\scriptscriptstyle\square}$};}%
    \node[vio] at (3.5,{-0.2-0.27*4}) {violated};
    \node[vio] at (3.5,{-0.2-0.27*6}) {violated\raisebox{-0.45ex}{$^{\ast}$}};%
    \foreach \row in {1,2,3,8,9,10,11,12} {\node[val] at (3.5,{-0.2-0.27*\row}) {held};}
    \node[msk] at (3.5,{-0.2-0.27*5}) {out of scope};
    \node[msk] at (3.5,{-0.2-0.27*7}) {masked};
    \foreach \pv/\st/\row in {0.082/bar/1, 0.096/bar/2, 0.053/bar/3, 0.094/bar/8, 0.005/bar/11, 0.013/bar/12} {%
      \node[val, anchor=east] at (5.7,{-0.2-0.27*\row}) {\pv};
      \draw[\st] (5.8,{-0.285-0.27*\row}) rectangle ({5.8+0.85*\pv},{-0.115-0.27*\row});
    }
    \foreach \pv/\st/\row in {1.053/bar/6, 0.386/bar/9, 0.384/bar/10} {%
      \node[val, anchor=east] at (5.7,{-0.2-0.27*\row}) {\pv};
      \draw[\st] (5.8,{-0.285-0.27*\row}) rectangle ({5.8+0.85*\pv},{-0.115-0.27*\row});
    }
    \node[val, anchor=east] at (5.7,{-0.2-0.27*4}) {10.001};%
    \draw[bar] (5.8,{-0.285-0.27*4}) rectangle (7.75,{-0.115-0.27*4});
    \draw[cut, shift={(7.35,{-0.2-0.27*4})}, xslant=0.31] (-0.06,-0.13) rectangle (0.06,0.13);
    \node[msk] at (5.8,{-0.2-0.27*5}) {propeller axiom};
    \node[msk] at (5.8,{-0.2-0.27*7}) {no stall speed};
    \node[msk] at (6.78,{-0.2-0.27*6}) {class latents};%
    \foreach \ya/\yb in {0.5/4.5, 5.5/6.5, 7.5/12.6} {%
      \draw[black, densely dashed, line width=0.6pt] ({5.8+0.85*0.6931},{-0.2-0.27*\ya}) to ({5.8+0.85*0.6931},{-0.2-0.27*\yb});
    }
    \node[val, anchor=north] at ({5.8+0.85*0.6931},{-0.2-0.27*12.6}) {$\ln 2$};
    \draw[ioline, line width=0.3pt] (0,-3.9) to (7.75,-3.9);
    \node[cid] at (0,-4.14) {verdict $V$, mean degree};
    \node[val, anchor=east] at (5.7,-4.14) {1.217};%
    \draw[bar, line width=0.6pt] (5.8,-4.23) rectangle ({5.8+0.85*1.217},-4.05);
    \draw[black, densely dashed, line width=0.6pt] ({5.8+0.85*0.310},-4.31) to ({5.8+0.85*0.310},-3.97);
    \node[txt] at (0,-4.4) {decision: flagged (validation threshold 0.310, dashed)};%
    \node[txt] at (0,-4.72) {clean record: $V$ 0.191, not flagged; $^{\ast}$A5 violated there too};%
    \node[txt] at (0,-5.04) {largest-violation key (seed 0): A4, then A5};%
    \node[txt] at (0,-5.36) {evidence coverage: 10 of the 11 constraints in scope};%
    \node[txt] at (0,-5.68) {nearest clean record: narrow-body, fleet similarity 0.274,};
    \node[txt] at (0,-6.01) {below $e^{-1}$: farther from the fleet than typical};%
    \node[val] at (0,-6.5) {$^{\scriptscriptstyle\square}$\,oracle flag (benchmark only)};
    \draw[bar] (5.1,-6.59) rectangle (5.6,-6.41);
    \draw[cut, shift={(5.35,-6.5)}, xslant=0.31] (-0.06,-0.13) rectangle (0.06,0.13);
    \node[val] at (5.69,-6.5) {value off scale};
  \end{scope}
  \begin{scope}[shift={(8.95,-1.87)}]%
    \node[ttl] at (0,0.12) {(c) Abductive field search and repair};
    \node[line] at (0,-0.08) {%
      \begin{tabular}{@{}l@{\hspace{6pt}}r@{\hspace{6pt}}r@{}}
        \textit{\color{ioline}field changed alone} & \textit{\color{ioline}lowest $V$} & \textit{\color{ioline}log change} \\
        total thrust & 0.181 & $+0.245$ \\ %
        usable fuel & 1.103 & $-0.200$ \\ %
        range & 1.112 & $+0.200$ \\ %
        as declared & 1.217 & \\ %
      \end{tabular}};
    \node[line] at (0,-1.98) {%
      \begin{tabular}{@{}l@{}}
        \textbf{altered field}: total thrust; alone it brings $V$ \\
        below the decision threshold (0.181 against 0.310) \\ %
        \textbf{violated bands} (seed 0): A4 $T/W=0.240$ \\ %
        \textit{\color{ioline}(clean 0.288)}, below the learned band 0.278 to 0.332; \\ %
        A5 range 5,750~km, below the learned edge 5,830~km; \\ %
        \textbf{repair} of one field, stricter target \\ %
        $V\leq -\ln 0.9\approx 0.1054$: not reached: A5, which \\ %
        does not read the thrust, and the in-band degrees \\ %
        keep $V$ above it; proposal 204.59~kN \\
        \textit{\color{ioline}(clean 192.2~kN)}, a value to cross-check \\ %
      \end{tabular}};
  \end{scope}
\end{tikzpicture}
\end{adjustbox}\par}
\caption{Verification report of one declared record.}
\label{fig:report}
\end{widefigure}

Crisp consistency edits on validation-selected bands (Fellegi-Holt and Z3) are on par with the constraint verdicts in macro-F1 on generator labels, and on intervention labels the constraint verdicts lead them on unseen families (0.694 against 0.569; Appendix~\ref{S-app:results-verdicts}); the graded degree ranks the violations, and the largest-violation key names the violated constraint first more often than their complete violated set (below). %
With thresholds from the specification benchmark's validation, the constraint verdicts are ahead of the class-standardized residual, MLP-SHAP and LightGBM-SHAP in macro-F1 on test with both label families; with the residual's thresholds chosen on the intervention benchmark's validation, the residual is on par with them in macro-F1 on intervention-label test (0.560 against 0.575) and below them in exact-set agreement, with more false violations (0.174 against 0.084). %
The DL2 method on the declared fields and a language model given the constraints in words are below the constraint verdicts in macro-F1 on every split and label family (the language model on the specification benchmark only). %

\begin{widefigure}[t]
  \centering
  \includegraphics[width=\bodyfigwidth]{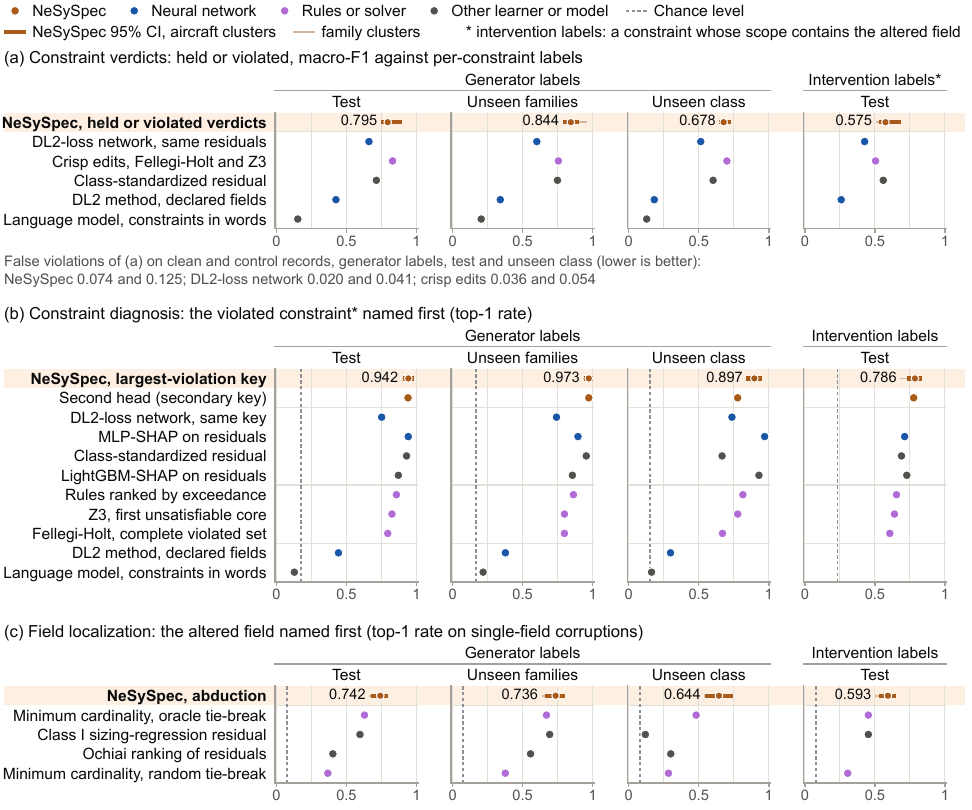}%
  \caption{Constraint verdicts, diagnosis and field localization.}
  \label{fig:diagnosis}
\end{widefigure}

On clean and control test records the constraint verdicts report a violation in 0.074 of the checks (Figure~\ref{fig:audit}), mostly on A2 and A5 (0.255 and 0.208 of their checks), where A2 evaluates the cruise lift coefficient at maximum take-off mass and A5 uses class-level latents; the report shows which relation the declared values strain and by how far (Figure~\ref{fig:report}c), and the key and abduction rank the degrees without this rule. %

The largest degree, the largest-violation key, names the violated constraint in 0.942 of corrupted test records (generator labels; on each of 8 seeds), above Fellegi-Holt, Z3 and rule diagnosis (0.795 to 0.857; Figure~\ref{fig:diagnosis}b); 49 to 85\% of these margins come from the 16 test records that the crisp edits leave unflagged (key 0.812, chance 0.106), and on the 140 they flag it is on par with them. %
On the 81 test records with a single violated constraint the key reaches 0.963, above the class-standardized residual, a z-score of the same residuals, and MLP-SHAP, which are on par with it over all corrupted test records; it comes with a held or violated verdict against learned tolerances, the altered field with its repair and exact abstention under missing evidence, all read from \fwname{}'s own degrees, and the residual's field ranking is below abduction (Section~\ref{sec:res-local}). %
On unseen families the key reaches 0.973 on each of 8 seeds, above Fellegi-Holt, Z3, MLP-SHAP and LightGBM-SHAP and on par with the class-standardized residual, with the rules at 0.863 (cells without an alternative-stream check; Appendix~\ref{S-app:results-diag}), and on the unseen class $0.897 \pm 0.007$, above Fellegi-Holt, Z3, the rules and the class-standardized residual and on par with LightGBM-SHAP, with MLP-SHAP, a second model on the residuals, at 0.970 (Table~\ref{S-tab:diag}). %
On intervention labels without retraining, the key ranks first a constraint whose scope contains the altered field in $0.786 \pm 0.006$ of corrupted test records (chance 0.232), above Fellegi-Holt, Z3 and rule diagnosis, with most of these margins from the 76 records that the crisp edits leave unflagged, and on par with LightGBM-SHAP, the strongest attribution on these labels; on unseen families \fwname{}'s second head, its secondary key, reaches 0.848 (largest-violation key 0.828; Appendix~\ref{S-app:results-diag}). %
With the optional rank-then-explain setting, \fwname{} flags and correctly diagnoses $0.901 \pm 0.013$ of corrupted test records and $0.849 \pm 0.043$ on the unseen class (5 seed pairs), above Fellegi-Holt and Z3 behind the same alarm on test and on par with MLP-SHAP, with the violated constraint, the altered field and a repair from its own degrees; with its own alarm, the record verdict and the key reach 0.929 on test (Appendix~\ref{S-app:results-diag}). %

Figure~\ref{fig:report} shows the report of the first single-field corruption of the test split whose altered field an available constraint reads, chosen by a fixed rule. Its total thrust is altered: the report names A4, the thrust-to-weight ratio below the band learned from the fleet, as the largest violation, abduction finds the thrust, and changing the thrust alone brings $V$ below the decision threshold (0.181 against 0.310). A9, flagged by the oracle's crisp band, is held, its thrust per engine inside the wider learned band; A5, which does not read the thrust, is violated as on the clean record (asterisk in panel b), so the repair, short of its stricter target, is a best-effort value to cross-check. %

\subsection{Field Localization and Repair}
\label{sec:res-local}

Abduction names the altered field more often than minimum-cardinality localization on every split, on generated and on independently labeled errors, and needs no oracle to break ties (Figure~\ref{fig:diagnosis}c). %
On single-field test corruptions (generator labels) it finds the altered field in $0.742 \pm 0.011$, above minimum-cardinality localization (0.367) and the Ochiai ranking of the class-standardized residuals; with the oracle tie-break, the setting most favorable to it, minimum-cardinality localization reaches 0.629, on par with abduction. %
On single-field records that no oracle selected or certified, abduction through the verifier fitted on generator labels, without retraining, finds the field moved beyond tolerance in 0.593 of test records, above minimum-cardinality localization (0.306) and its oracle tie-break (0.453); 78\% of the margin over the oracle tie-break comes from the 70 records that the crisp edits leave unflagged, and on the 175 of 245 that they flag abduction stays above random tie-breaking and on par with the oracle tie-break. %
With the selected abduction rule (Appendix~\ref{S-app:results-refine}), abduction is also ahead of the Class I sizing-regression residual, the field farthest from the conceptual-design (Class I) weight regressions on maximum take-off mass fitted on the training aircraft~\cite{raymer2018aircraft}, on test (0.597), on the unseen class (0.119) and on intervention labels (0.453, the level of the oracle tie-break as well), and on par with it on unseen families (0.693), with the violated constraint and a repair from the same verdict. %

The repair proposal of Eq.~\eqref{eq:repair} comes from \fwname{}'s own abduction and verdict, with the field and the violated constraint in the same report. %
It cuts the log-space error on the altered field from 0.366 to 0.133 and on the whole record to 0.200, and is closer on the record than the Fellegi-Holt, MaxSMT, nearest-neighbor and regression repairs and on the altered field than a constraint-layer repair ordering in the style of PiShield~\cite{stoian2024pishield}; at the field of the Class I residual, MICE~\cite{vanbuuren2011mice} (0.120 and 0.161) and missForest are on par with it on both errors, as is the regression on the altered field (Table~\ref{S-tab:diag}). %
The solver reaches its target, a verdict $V \le -\ln 0.9 \approx 0.1054$ stricter than the decision threshold, on 0.555 of the single-field test records and otherwise returns its best-effort value (Figure~\ref{fig:report}c); the report issues a proposal for each flagged record, as a value to cross-check. %

\subsection{Record Coherence}
\label{sec:res-detect}

On generator labels, with 16 test aircraft of 9 families held out of training, the record verdict tells coherent from incoherent records at a test AUROC of $0.979 \pm 0.002$, above learners on the raw fields (0.748 to 0.792), and each of its flags arrives with the violated constraint, the altered field and a repair proposal, read from the degrees of Figure~\ref{fig:audit}. %
In distribution TabPFN v2 is on par with it and the other learners on the residuals are close to it, the residual MLP at $0.990 \pm 0.003$ and TabPFN v2, LightGBM and the one-class SVM at 0.957 to 0.958 (Appendix~\ref{S-app:results-detect}). At a TPR of 0.90 the record verdict's false alarms on the 32 controls (0.086) are fewer than those of the Mahalanobis distance (0.656), each naming the constraint that raised it (Table~\ref{S-tab:contrasts-3}). %

On the unseen class of business jets, the key and abduction lead classical diagnosis and minimum-cardinality localization (Sections~\ref{sec:res-diag} and~\ref{sec:res-local}). In detection there the record verdict reaches 0.934 and five learners on the residuals rank higher (residual MLP 0.978, TabPFN v2 up to 0.984), and \fwname{}'s optional rank-then-explain setting lets the residual MLP raise the alarm and \fwname{} explain it, with false alarms on 0.205 of the 38 controls and 0.168 of the 19 pristine records at validation thresholds. %
On intervention labels, whose altered single fields move beyond the tolerance, by up to 80\%, the key and abduction lead the classical comparators without retraining (Sections~\ref{sec:res-diag} and~\ref{sec:res-local}); in detection there, the residual MLP retrained on these labels and TabPFN v2, whose statistics include the scored rows, rank test records higher (Table~\ref{S-tab:contrasts-3}). %

\label{sec:res-xmodal}%
\paragraph{Records with images.}
With a paired image, \fwname{} also checks the record against the image and names the cross-modal axiom that the image contradicts (Section~\ref{sec:meth-xmodal}). %
On satellite pairs, the multimodal verifier of Eq.~\eqref{eq:strata} achieves an AUROC of $0.994 \pm 0.002$ over all pair types with a ConvNeXt-T backbone~\cite{liu2022convnet} (8 seeds). Qwen3-VL-8B~\cite{bai2025qwen3vl}, prompted with rules R1-R10, reaches 0.748; the verifier's advantage holds across all evaluated backbones.

With ResNet-50 used for both, the verifier also outperforms end-to-end fusions trained on labeled pairs on test and unseen sites. Some learned fusions achieve higher AUROC on image swaps and unseen-site pairs. Without an image, the verifier returns the record verdict bit for bit. Appendix~\ref{S-app:results-xmodal} details the comparison protocols and their asymmetries. %
On photographs the cross-modal stratum is on par with Qwen3-VL-8B and Qwen3-VL-32B prompted with the same axioms on test (0.692 against 0.668 and 0.712) and prints the part map of its evidence beside the contradicted cross-modal axiom; on the families held out of \fwname{}'s training, most of which the 32-billion-parameter model names from the photograph alone, that model reaches 0.827 and the stratum 0.733. This advantage may partly reflect the visual and semantic priors acquired through extensive multimodal pretraining~\cite{bai2025qwen3vl}. These families are held out from our task-specific training; pretraining exposure is not controlled across the compared systems. Further comparisons should distinguish prior recognition from reasoning over the supplied physical constraints.%

\subsection{Missing Evidence and Absent Images}
\label{sec:res-missing}

\begin{widefigure}[t]
  \centering
  \includegraphics[width=\textwidth]{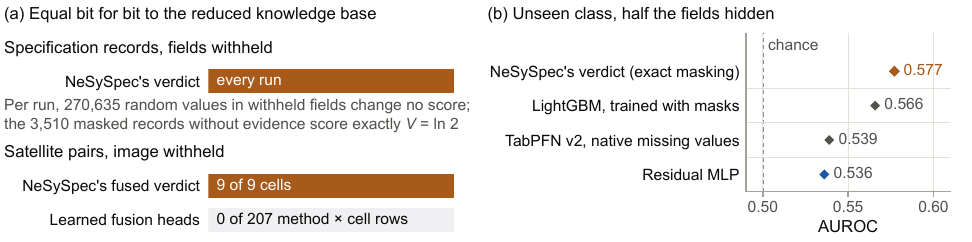}
  \caption{Verdicts under missing evidence.}
  \label{fig:missing}
\end{widefigure}

By Proposition~\ref{prop:reduction}, masked quantification makes every verdict a function of the available evidence only, and every run keeps this property (Figure~\ref{fig:missing}a): every masked record scores bit for bit as its sub-knowledge-base, random values written into withheld fields change nothing, and masked records without evidence score exactly $V = \ln 2$; without an image, \fwname{}'s fused verdict returns the tabular verdict bit for bit in 9 of 9 satellite cells, against 0 of 207 learned method $\times$ cell rows. %

With fields hidden completely at random (MCAR)~\cite{rubin1976inference}, a residual with a hidden input is missing for the learners; with half the fields hidden on the unseen class, where 26.4\% of the records keep no available constraint, \fwname{} reaches 0.577, above the residual MLP (0.536) and TabPFN v2 with native missing values (0.539) and on par with LightGBM trained with masks (0.566; Figure~\ref{fig:missing}b), while MICE imputation at scoring yields a slightly higher AUROC ($+0.033$). %
The masked verdict reads no filled value, and its report states which constraints remain checkable. %
With 10\% of the fields hidden in distribution it reaches 0.848 on test and 0.824 on unseen families, where the MLP trained with masks (0.891 on test), the residual MLP (0.916 on unseen families) and, with 10 to 30\% hidden on unseen families, MICE imputation of every missing field at scoring rank higher; with no field hidden, the masked verdict is ahead of MICE at scoring (Appendix~\ref{S-app:results-missing}). %

With MICE imputation enabled, \fwname{} keeps its constraint verdicts and largest-violation key and ranks records at 0.914 with 10\% of the fields hidden on test, on par with the MLP trained with masks given the same imputer (0.905); at its threshold it flags 0.210 of the clean-labeled records, against 0.137 for the masked verdict, \fwname{}'s default. %

With 30\% of satellite images hidden, the fused verdict keeps the exact reduction that no learned head keeps and ranks test pairs at 0.947, and heads trained with modality dropout~\cite{neverova2016moddrop}, such as CatBoost on knowledge-base truths (0.978), rank them higher; across sensors without labels the record verdict is the fallback (Appendix~\ref{S-app:results-xmodal}). %

\subsection{Grounding of the Visual Axioms}
\label{sec:res-grounding}

When a head asserts an attribute, the grounding axiom encourages it to place its evidence on the corresponding part. The Reichenbach gate scales this soft penalty with the strength of the head's assertion, without using attribute labels. The report then displays the evidence map for inspection. %

We compare methods using per-example penalties, keeping the batch's normalization constants fixed and stopping gradients through those statistics. Loss weights are selected on validation using one seed per candidate value (Appendix~\ref{S-app:expsup}). Under this common protocol, the network trained with the grounding axiom localizes the annotated parts substantially better than RRR~\cite{ross2017right} and GradMask~\cite{simpson2019gradmask}. Every cluster-bootstrap interval of its difference from these two and from Pukdee et al.~\cite{pukdee2023learning}, run both with the calibration of the protocol, which flattened their teacher toward the uniform prediction, and with a recalibrated arm, lies above zero on the FGVC pseudo-masks and on the human part masks of PASCAL-Part, PartImageNet~\cite{he2022partimagenet} and UDA-Part~\cite{liu2022udapart} (Appendix~\ref{S-app:expsup}). %
On test pseudo-masks it points to the part in $0.780 \pm 0.017$ of maps (8 seeds), against $0.360 \pm 0.066$ with permuted masks and $0.103 \pm 0.017$ without a localization term (3 runs each), so localization depends on anatomically correct masks, and it points to the annotated part in $0.737 \pm 0.015$ (3 runs at weight 1) of PASCAL-Part human masks outside the FGVC domain (Figure~\ref{S-fig:partgrounding} in Appendix~\ref{S-app:results-grounding}). %

Across loss weights 0.3, 1 and 3, the grounding axiom keeps at most one collapsed head, compared with up to 5.7 on average for the auxiliary localization loss. Its pointing accuracy rises from 0.512 at weight 0.3 to 0.790 at 1 and 0.803 at 3, with three runs per weight. Localization also holds across all evaluated backbones with per-head maps (Appendix~\ref{S-app:results-grounding}). %
At the validation-selected weights (1 for the axiom and 0.3 for the auxiliary loss), the axiom has a small accuracy advantage out of distribution (Appendix~\ref{S-app:expsup}). Localization is on par on test, out of distribution and on PASCAL-Part when comparing three runs per method. Including all eight axiom runs gives the auxiliary loss a higher pointing accuracy on test pseudo-masks (Appendix~\ref{S-app:protocol}). %

Figure~\ref{fig:heatmaps} shows maps of \fwname{}'s visual module with the grounding axiom and without a localization term: on ground photographs (a) each map prints its in-mask share $\rho$ of Eq.~\eqref{eq:rho-loc}; on satellite crops (b), outlined by the annotated diamond on which the axiom trains, the maps show where the heads look. %
Each row names the cross-modal axiom its head feeds, and the report prints the map beside that axiom's verdict (part maps in Figure~\ref{fig:pipeline}; layout in Figure~\ref{S-fig:rankexplain}), so the analyst sees on which part the evidence of a contradicted axiom lies. %

\begin{widefigure}[t]
  \centering
  \includegraphics[width=0.9\bodyfigwidth]{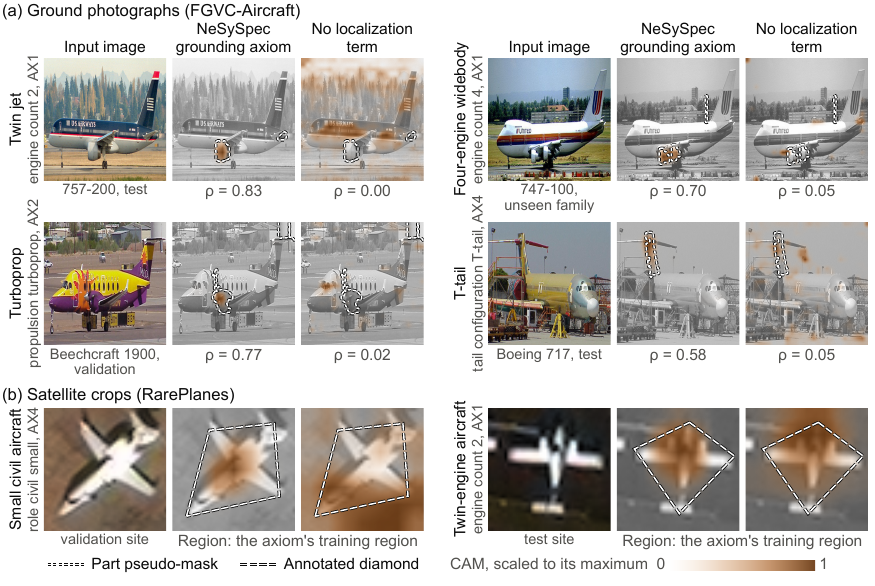}%
  \caption{Evidence heatmaps.}
  \label{fig:heatmaps}
\end{widefigure}

\subsection{Robustness to Deliberate Falsification}
\label{sec:res-falsif}\label{sec:res-indep}%

Under a black-box falsification protocol of bare lies, cover-ups that edit up to five further fields, and omissions (Table~\ref{tab:evalmap} and Appendix~\ref{S-app:results-indep}), a lie flagged by the record verdict is covered less often than one flagged by the residual MLP: with up to five further edits 0.936 of cover-ups pass against 1.000, and 0.247 of its flagged bare lies cannot be covered, against 0.000 for the MLP; a minimal cover of its flagged bare lies takes 2.69 further edited fields on average, against 1.38, 1.37 and 1.63 for the MLP, LightGBM and TabPFN v2 on the residuals (point estimates, censored at six; Appendix~\ref{S-app:results-indep}). %
At that budget the record verdict is also evaded less often than TabPFN v2, the one-class SVM, LightGBM on the raw fields and an isolation forest, while the share of its flagged bare lies beyond cover is 0.247, against 0.494 and 0.325 for the rules with hierarchical pooling and a Mahalanobis distance on the declared fields, whose minimal covers take more edits (3.62 and 3.45); each of its flags comes with the field found by abduction and a repair proposal. %
On bare lies the record verdict (evasion $0.740 \pm 0.038$) is evaded less often than the rules with hierarchical pooling and LightGBM on the raw fields and on par with the MLP (0.754); LightGBM on the residuals (0.685) and TabPFN v2, evaded less often on bare lies, leave almost none of their flagged lies beyond cover (0.004 and 0.024, against 0.247). %

By Proposition~\ref{prop:reduction}, a withheld field masks the axioms that require it, and the report lists every undeclared field with the constraints it leaves unchecked; withholding up to two fields lets 0.964 of lies pass the record verdict alone. %
\fwname{}'s optional critical-field alert addresses this channel, and each of its flags names the undeclared critical field or the violated constraint: against an attacker who knows the alert and withholds up to two fields, 0.729 of lies pass \fwname{} with the alert, on par with the residual MLP given the same alert (0.756) with all splits pooled. %
With the alert, false alarms on the 32 test controls are 0.148, below the Mahalanobis distance (0.656). %

\subsection{What the Logic Adds: Other Encodings of the Same Knowledge}
\label{sec:res-knowledge}\label{sec:res-ablation}%

The logical formulation provides the constraint-level verdicts used for diagnosis and repair. For record-level detection, learners on the same residuals perform similarly to \fwname{} in distribution and better on the unseen class (Section~\ref{sec:res-detect}). %
Against the network trained with the DL2 loss on the same residuals and targets, which learns one head per constraint in place of the graded predicates and is their like-for-like test, \fwname{}'s graded predicates trained through Real Logic give constraint verdicts with a higher macro-F1 and a better key, with the field and a repair from the same degrees (Section~\ref{sec:res-diag}). Of the other Real Logic properties of Section~\ref{sec:background}, guarded quantification gives the exact abstention that none of 207 learned method and cell rows keeps without an image (Section~\ref{sec:res-missing}), the shared truth scale the fused verdict (Section~\ref{sec:res-xmodal}) and the gate grounding stable across loss weights (Section~\ref{sec:res-grounding}). Stratification (E4) keeps one verdict per evidence source, so the report says whether the record or the image raised the flag (Appendix~\ref{S-app:results-ablation}). %
Training the knowledge base through the differentiable truths of Real Logic improves detection over keeping it at its initial values: AUROC rises from 0.907 to 0.979 on test and from 0.845 to 0.934 on the unseen class.
On generator labels, the full model outperforms the flat-bank ablation on every split. Removing in-band grading lowers detection on the unseen class and slightly on test and unseen families. The ablation using a generalized mean at $p=4$ is on par with the full model (Table~\ref{S-tab:ablation}). %
In detection on test the record verdict (0.979) is slightly above the DL2-loss network (0.965) and a per-constraint monotone model on the same residual magnitudes without LTN (0.966), and on par with the network out of distribution; \fwname{} reads these magnitudes as graded predicates in physical units (Figure~\ref{fig:report}c), with the key, abduction and repair from the same degrees. %
The generator's per-constraint targets (Section~\ref{sec:res-diag}) and the graded predicates act together: the ablation trained with the binary record label only equals the generator labels in 0.081 of test records and is below the crisp edits on all three measures. %
Given as a constraint loss on the declared fields or as rules in words, the same knowledge detects less than \fwname{}'s verdict: the DL2 method on the declared fields reaches $0.703 \pm 0.030$ and a language model prompted with the constraints A1-A11 in words (Qwen3-8B~\cite{yang2025qwen3}, anonymized records) 0.620, below the record verdict on every split (Table~\ref{S-tab:spec}), and both are below its key and verdicts (Section~\ref{sec:res-diag}). %

\begin{widefigure}[t]
  \centering
  \includegraphics[width=\bodyfigwidth]{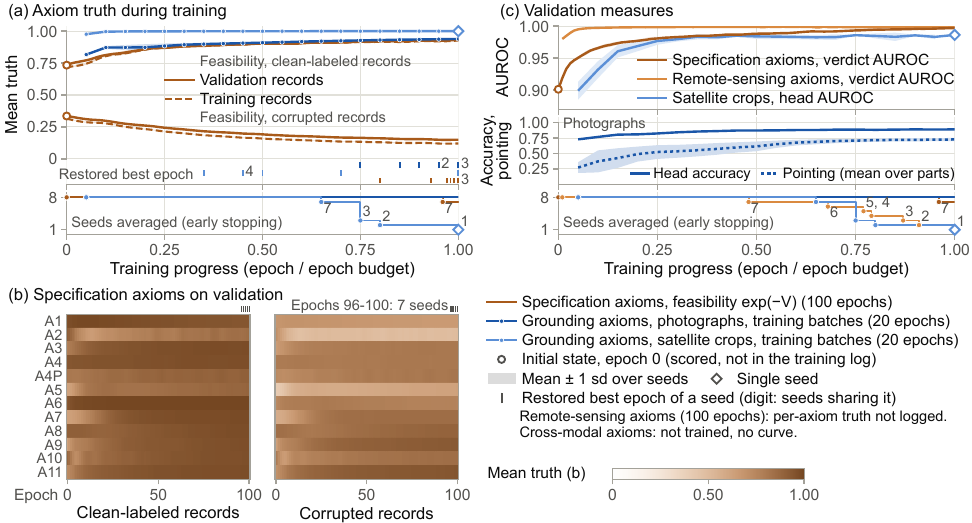}%
  \caption{Training dynamics of the axioms.}%
  \label{fig:dynamics}
\end{widefigure}

Figure~\ref{fig:dynamics} traces the training of the axioms over eight seeds per stratum. %
For the specification axioms, a re-run of \fwname{}'s training whose only addition is a copy of the state at each epoch reproduces every frozen checkpoint and its predictions bit for bit, together with every logged loss and validation AUROC; each state is scored at the end of its epoch on the training and validation records. %
From the initial state, shared by all seeds, to the restored best epoch, the mean feasibility $\exp(-V)$ on validation rises from 0.734 to 0.930 on clean-labeled records and falls from 0.335 to 0.151 on corrupted ones, while the verdict's validation AUROC rises from 0.901 to 0.997. %
Panel b averages each axiom's truth $\exp(-v_k)$, which the constraint verdicts compare with 0.5, over the records where the axiom is available; the feasibility of Eq.~\eqref{eq:pme} in panel a is the geometric mean of a record's available truths. %
The satisfaction of the grounding axioms, averaged over the training batches of each epoch while the weights move, rises from 0.816 at the first epoch to 0.936 at the restored best epoch on photographs and from 0.976 to 0.999 on satellite crops. %
The training logs of the remote-sensing axioms hold no per-axiom truth, so panel c shows their verdict's validation AUROC beside the photographs' head accuracy and pointing; the cross-modal axioms have no trainable parameter, and late epochs average fewer seeds where early stopping ended a run. %

Detection also improves with the knowledge the base encodes (Figure~\ref{fig:knowledge}). Retrained with subsets of its 12 specification constraints, the record verdict's intervention-label AUROC, averaged over the three held-out splits, rises from 0.535 with one constraint to 0.723 with all twelve (Spearman 0.93). It rises with diminishing returns in the number of constraints (panel a) and linearly with the share of edits that alter a field some of them read ($R^2$ 0.88, panel b); the residual MLP on the same subsets gains as well (0.540 to 0.706), with no trend in their gap. Removing one constraint (panel c) lowers detection for A5, the Breguet range, and A2, the cruise lift coefficient, and slightly for A4, A7 and A8, the thrust-to-weight, mass-balance and empty-weight-fraction checks, and raises it for none; it also lowers abduction for 7 of the 12 constraints and slightly for 3 more (Appendix~\ref{S-app:results-ablation}). %

\begin{widefigure}[t]
  \centering
  \includegraphics{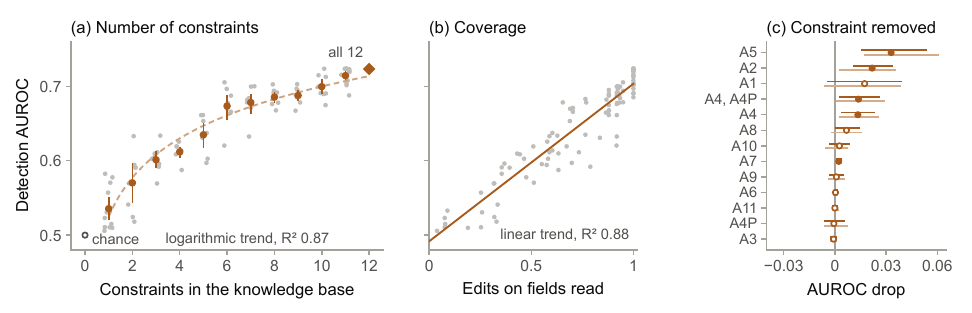}
  \caption{Detection and the knowledge encoded.}
  \label{fig:knowledge}
\end{widefigure}

\subsection{Deployment Map}
\label{sec:res-deploy}

\begin{widetable}[t]
  \centering
  \caption{Deployment map.}
  \label{tab:deploy}
  \scriptsize%
  \renewcommand{\arraystretch}{0.84}%
  \setlength{\tabcolsep}{3pt}
  \begin{threeparttable}
  \begin{tabularx}{\textwidth}{@{}>{\raggedright\arraybackslash}p{2.3cm}>{\raggedright\arraybackslash\hsize=1.16\hsize}X>{\raggedright\arraybackslash\hsize=0.84\hsize}X@{}}
    \toprule
    Regime & What \fwname{} provides & Strongest comparator \\
    \midrule
    \multicolumn{3}{@{}p{\textwidth}@{}}{\raggedright\textit{Workflow:} run \fwname{} first and read its report; on a class without training aircraft, enable its optional rank-then-explain setting, in which a residual learner raises the alarm and \fwname{} explains it; cross-check a repair value against a statistical imputer such as MICE; against omission, enable the optional critical-field alert.} \\ %
    \midrule
    \multicolumn{3}{@{}l}{\textbf{Constraint verdicts and diagnosis} (Section~\ref{sec:res-diag})} \\
    \textit{classical diagnosis and other encodings} & the largest-violation key, from the verdict's own degrees, above Fellegi-Holt, Z3 and the rules (also on intervention labels), the DL2-loss network's key, the DL2 method on the declared fields and a language model; 0.812 where the crisp edits flag nothing & the crisp edits, on par where they flag the record (generator labels) \\ %
    \specialrule{0.15pt}{0.3pt}{1.1pt}
    \textit{held or violated, test and unseen families} & the constraint verdicts, macro-F1 0.795 on generator-label test, above the DL2-loss network, with the key, the field and a repair & crisp edits, on par in macro-F1; the DL2-loss network, with fewer false violations (0.020 against 0.074) \\ %
    \specialrule{0.15pt}{0.3pt}{1.1pt}
    \textit{residual rankings} & the key, above both on 81 single-violation records and above the residual on the unseen class, with a held or violated verdict, a repair and exact abstention & the class-standardized residual and MLP-SHAP, on par in distribution; MLP-SHAP on the unseen class \\ %
    \midrule
    \multicolumn{3}{@{}l}{\textbf{Field localization and repair} (Section~\ref{sec:res-local})} \\
    \textit{field} & abduction, above random tie-breaking on every split and label family and above the Class I sizing-regression residual on the unseen class and intervention labels, with no oracle to break ties & an oracle tie-break, on par in distribution \\ %
    \specialrule{0.15pt}{0.3pt}{1.1pt}
    \textit{value, in distribution} & the repair proposal, with its field and violated constraint in the same report, issued for flagged records (the proposals change 0.562 of clean and control records); the solver reaches its target on 0.555 of single-field test records & MICE at the field of the Class I sizing-regression residual, on par, the recommended value-recovery step \\ %
    \midrule
    \multicolumn{3}{@{}l}{\textbf{Record coherence} (Section~\ref{sec:res-detect})} \\
    \textit{in distribution} & the record verdict, which with the key flags and correctly diagnoses 0.929 of corrupted test records, with the constraint, field and repair of each flag & the residual MLP, close in detection, and TabPFN v2, on par \\ %
    \specialrule{0.15pt}{0.3pt}{1.1pt}
    \textit{class without training aircraft} & rank-then-explain (optional): false alarms on 0.205 of controls and 0.168 of 19 pristine records, with the constraint, field and repair behind each alarm & the residual MLP and TabPFN v2, in detection \\ %
    \specialrule{0.15pt}{0.3pt}{1.1pt}
    \textit{intervention labels} & the key and abduction, without retraining, naming the constraint and the field & TabPFN v2 (statistics from the scored rows) or the MLP retrained on them, in detection \\ %
    \specialrule{0.15pt}{0.3pt}{1.1pt}
    \textit{record with an image} & the multimodal verifier (satellite), above a vision-language model given the rules and end-to-end fusions, and the cross-modal stratum (photographs), each naming the contradicted axiom, with the exact reduction & vision-language models given the axioms, on par on test photographs; learned fusions on image swaps and unseen-site pairs; the largest vision-language model on unseen families \\ %
    \midrule
    \multicolumn{3}{@{}l}{\textbf{Missing evidence and absent images} (Section~\ref{sec:res-missing})} \\
    \textit{10\% of fields hidden, test or unseen families} & the coverage, for the audit; \fwname{} with MICE imputation enabled, to detect & the MLP trained with masks (test) and the residual MLP (unseen families), against the masked verdict; the MLP with masks and the same imputer, on par with the MICE setting (test) \\ %
    \specialrule{0.15pt}{0.3pt}{1.1pt}
    \textit{half hidden, unseen class} & the masked verdict, above the residual MLP and TabPFN v2, with the checkable constraints and an exact verdict & LightGBM trained with masks, on par; MICE imputation at scoring, close \\ %
    \specialrule{0.15pt}{0.3pt}{1.1pt}
    \textit{image absent or partly missing} & the fused verdict, which reduces exactly to the record verdict & on labeled pairs, heads trained with modality dropout, such as CatBoost on knowledge-base truths, to rank \\ %
    \midrule
    \multicolumn{3}{@{}l}{\textbf{Grounding of the visual axioms} (Section~\ref{sec:res-grounding})} \\
    \textit{any backbone} & the grounding axiom, which localizes the annotated parts above RRR and GradMask under one protocol, with at most one collapsed head across loss weights & an auxiliary localization loss at its validation weight, on par with 3 runs each and higher on test pseudo-masks against the axiom's 8 seeds \\ %
    \midrule
    \multicolumn{3}{@{}l}{\textbf{Robustness to deliberate falsification} (Section~\ref{sec:res-falsif})} \\
    \textit{cover-ups, up to five further edits} & the record verdict, which a cover-up passes less often than the residual MLP, with the altered field and a repair for each flag & rules with hierarchical pooling or a Mahalanobis distance, on par \\ %
    \specialrule{0.15pt}{0.3pt}{1.1pt}
    \textit{omission of up to two fields, alert-aware} & the optional critical-field alert, naming the undeclared field or violated constraint of each flag & the residual MLP with the same alert, on par with all splits pooled \\ %
    \bottomrule
  \end{tabularx}
  \begin{tablenotes}[flushleft]\scriptsize\setlength\labelsep{0pt}
    \item[] On intervention labels without retraining, on unseen families, \fwname{}'s second head, its secondary key, reaches 0.848 (largest-violation key 0.828). Satellite cells: comparators not retrained or retuned, and fusions on knowledge-base truths read the truths of the saturating-envelope variant of \fwname{}'s remote-sensing stratum they were fitted on; asymmetries in Appendix~\ref{S-app:results-xmodal}. %
  \end{tablenotes}
  \end{threeparttable}
\end{widetable}

Table~\ref{tab:deploy} assigns each regime to the output of \fwname{} that serves it, with the workflow of an audit and the strongest comparator of each regime. %

\subsection{Limitations}
\label{sec:res-limits}

Each limitation is bounded by the protocol and opens a line of future work (Section~\ref{sec:end}). %
\emph{Domain and use.} \fwname{} is validated on aircraft, whose public civil records and relations a reader can verify; only its knowledge base, visual predicates, corruption generator and data are domain-specific (Appendix~\ref{app:porting}), and its verdicts support a human analyst, with aircraft detection, tracking and identification outside its scope because of the dual-use risks of such capabilities~\cite{brundage2018malicious}. %
\emph{Synthetic errors and scale.} Every error is generated or edited, so its type, magnitude and field are known and each verdict is checked against them, and the intervention labels are independent of the generator; they mark the constraints an edit involves, about a third of the altered rows break none of them physically, and expert per-constraint physical labels will score the verdicts against violations (Section~\ref{sec:end}); the 107 aircraft and one published split hold every test family and satellite site out of training. Naturally occurring errors, expert red teams, larger fleets and a replication on a new fleet will extend this evidence. %
\emph{Class-level relations.} A2 evaluates the cruise lift coefficient at maximum take-off mass and A5 uses latents shared by a class and engine type, estimable from 44 training aircraft; most violations reported on clean records come from these two relations (Section~\ref{sec:res-diag}), which per-type latents and a cruise-weight form of A2 will refine. %
\emph{Classes without training aircraft.} There, residual learners rank higher in record-level detection (Section~\ref{sec:res-detect}); the optional rank-then-explain setting lets such a learner raise the alarm while \fwname{} names the constraint, the field and a repair, and larger fleets will reduce the classes read with fleet statistics alone. %
\emph{Omission.} A withheld field masks the axioms that read it, the exact abstention of Proposition~\ref{prop:reduction}; the optional critical-field alert flags the omission, and class-level axioms will make the omission of a critical field a violation. %
\emph{Masks, statistics and encodings.} Automatic pseudo-masks supervise the grounding axiom, so the protocol is reproducible from public data, and hand-annotated part masks will extend it; no multiplicity correction is applied, which the new-fleet replication will address; encoding the knowledge base in DeepProbLog, Scallop or PSL, with probabilistic facts and proofs, changes its axiomatic and logical structure and is left to a dedicated comparison. %
\section{Conclusion and Future Works}
\label{sec:end}

\fwname{} is an integrated and interpretable neuro-symbolic framework for auditing engineering specifications: its graded axioms on Logic Tensor Networks, trained on fleet data, determine which physical constraints a declared specification respects and which it violates, and from the degree behind each verdict it derives the altered field, a repair proposal, exact abstention on missing evidence, part-grounded image evidence and a verification report. %
Validated on 107 aircraft with generated, intervention-labeled and falsified errors, it provides and evaluates these outputs together. %

Using a fixed rule to classify each constraint as held or violated, \fwname{} reaches a test macro-F1 of 0.795 with generator labels and, without retraining, 0.575 with intervention labels. On every split evaluated with generator labels, it outperforms three comparators in macro-F1: a network trained with the DL2 loss on the same residuals; DL2 applied to the declared fields; and a language model given the constraints in words. %

\noindent
The largest-violation key identifies the violated constraint in 0.942 of corrupted test records with generator labels, outperforming Fellegi--Holt, solver-based diagnosis and rule-based diagnosis. On independently labeled errors, abduction identifies the altered field in 0.593 of records without retraining. Minimum-cardinality localization reaches 0.306 with random tie-breaking and 0.453 with oracle tie-breaking. The Class I sizing-regression residual also reaches 0.453. %
The record verdict detects incoherent records at a test AUROC of 0.979 on generator labels, and a missing field or image leaves each verdict exactly that of the available constraints. %

\noindent
An image lets the audit check the declared record against visible attributes and identify the axiom behind a disagreement. On satellite pairs, the multimodal verifier achieves an AUROC of 0.994 across all pair types, compared with 0.748 for a prompted vision-language model.

\noindent
The report also shows where each visual head places its evidence. With the grounding axiom, these maps point to the annotated part in 0.737 of PASCAL-Part human masks. Localization is substantially better than RRR and GradMask on human masks from PASCAL-Part, PartImageNet and UDA-Part, with wider intervals on the 40 UDA-Part items (Appendix~\ref{S-app:expsup}).

Among the approaches evaluated here (the DL2 method on the declared fields, the network trained with the DL2 loss on the same residuals read per constraint, a constraint-layer repair ordering, LTN ablations, solver and rule diagnosis and prompted language and vision-language models), \fwname{} is, to our knowledge, the most suitable neuro-symbolic framework for auditing engineering specifications: it provides and evaluates together the outputs of Table~\ref{tab:sota}, and its key and abduction lead classical diagnosis, the DL2 method, a prompted language model and minimum-cardinality localization with random tie-breaking. No comparator returns all the outputs of the audit, and on the test aircraft none is ahead of \fwname{} in the macro-F1 of the constraint verdicts, in naming the violated constraint and the altered field or in the repair error on the altered field (Table~\ref{tab:glance}). With oracles fitted without test or out-of-distribution aircraft, its leads on intervention labels over Fellegi-Holt and Z3 in diagnosis and over minimum-cardinality localization with random tie-breaking hold on every split, at lower levels (Section~\ref{sec:res-setup}). Its logic serves the physical audit more than record classification: in record-level detection, learners on the same residuals are close to it and in some settings rank records higher (Section~\ref{sec:res-detect}). %
Its detection grows with the knowledge it encodes, following the share of edits that alter a field its constraints read, with diminishing returns in their number. 

\par\smallskip\indent
This study is a first approach to specification auditing, evaluated on aircraft and constructed errors. Broader physical knowledge, richer input evidence and larger, more diverse datasets may improve generalization to unseen systems, to be tested through independent replications. The logical formulation helps interpret anomalies by exposing which relations fail and which fields are implicated.
Future work will extend \fwname{} along the following directions:
\begin{itemize}
  \item \emph{Other neuro-symbolic encodings.} A version of the knowledge base in DeepProbLog, Scallop and PSL, with probabilistic facts and proofs in place of graded truths and masked quantification, will be compared with \fwname{}, together with the semantic loss and the monotone model read per constraint. %
  \item \emph{Aircraft-level latents and refined relations.} Per-type latents in A5 and a cruise-weight formulation of A2 will refine the constraint verdicts. Per-constraint thresholds selected on validation will calibrate these verdicts, and expert physical labels for each constraint will be used to evaluate them. The report will display every learned band in physical units. A new relation involving the ceiling will cover the only declared field not currently checked by any constraint. Additional relations can improve detection by checking different physical relationships among fields already covered elsewhere, as A2 does. %
\item \emph{Other engineering applications.} Visual recognition of components and geometry could extend the audit to mechanical assemblies, industrial equipment and civil structures, with adapted knowledge bases, visual predicates and corruption generators.
 \item \emph{Poorly characterized observations.} The framework could also support the analysis of observations whose identity or physical behaviour remains uncertain. Physical constraints could help compare candidate explanations and identify the additional measurements needed to distinguish them.

\item \emph{Generative design and simulation.} Exporting structured parameter vectors, units and constraints could connect the audit to agentic pipelines that generate parametric CAD models from text and images. Candidate models could then undergo finite element method (FEM) simulations, computational fluid dynamics (CFD) and other numerical analyses, feeding their results back into the audit.

\item \emph{Training and hyperparameters.} Further ablations will assess how supervision, loss balancing and joint versus separate module training affect diagnosis and generalization. Hyperparameter sensitivity will be evaluated on validation data.

\item \emph{Inference efficiency.} Profiling will measure the inference cost of each module. Batching, caching and adaptive search budgets will be evaluated in terms of latency, diagnostic accuracy and repair quality.
  \item \emph{Naturally occurring errors.} Cross-source discrepancies, variant mix-ups, with block abduction, and uncorroborated fields will bring the audit to real errors, and expert red teams will extend the falsification protocol. %
  \item \emph{Evidence, images and declarations.} Imputation channels chosen with block-wise hiding will extend \fwname{} with MICE imputation enabled, hand-annotated FGVC-Aircraft part masks and visual strata adapted to new sensors will extend the grounding axiom, and class-level axioms will make the omission of a critical field a violation. %
\end{itemize}

\section*{Declaration of Generative AI and AI-Assisted Technologies\\ in the Manuscript Preparation Process}
During the preparation of this work, the authors used AI tools to improve the language, readability, and grammar of the manuscript, as well as to assist with debugging and bug hunting.

\section*{Data Licenses}
FGVC-Aircraft~\cite{maji2013fine} is provided for non-commercial research only and its photographs are not redistributed, apart from the four photographs of Figure~\ref{fig:heatmaps}, reproduced for illustration with credit to their photographers (airliners.net); RarePlanes~\cite{shermeyer2021rareplanes} is distributed under the Creative Commons Attribution-ShareAlike 4.0 International license; the two satellite crops of Figure~\ref{fig:heatmaps} are adapted from it (cropped, resized and overlaid with activation maps) and shared under the same license, under which the tables and crop indices derived from it will be released. The human part masks used for evaluation are not redistributed.

\appendix
\counterwithin{table}{section}
\counterwithin{figure}{section}
\counterwithin{equation}{section}
\counterwithin{algocf}{section}   %
\emergencystretch=2em %

  \section{Porting to Other Engineered Systems}
\label{app:porting}

This appendix guides the port of \fwname{} to the declared specifications of another engineered system whose fields are linked by physical, definitional or regulatory relations with statable tolerances; Table~\ref{tab:porting} lists what the new system supplies, what is fitted again and what is reused.

\begin{widetable}[!htbp]
\centering
\caption{Porting to a new engineered system.}
\label{tab:porting}
\scriptsize%
\begin{tabularx}{\linewidth}{@{}L{0.15\linewidth}>{\raggedright\arraybackslash}XL{0.31\linewidth}@{}}
\toprule
Component & Content & Aircraft study \\
\midrule
\multicolumn{3}{@{}l}{\textit{Supplied by the new system}}\\
Knowledge base & one exact, differentiable quantity per relation, each with a graded predicate, an availability condition and a scope; bounded latents for rarely declared quantities; initial bands fitted on clean training records; regulatory bands frozen & 12 constraints on records and 10 on satellite annotations (Table~\ref{M-tab:axioms}); 336 trainable parameters on records, 126 of them latents, and 57 on satellite annotations; no neural network \\ %
Visual predicates & one sigmoid head per class, with the parts where its evidence must lie; cross-modal axioms with the same form & 25 heads per dataset (Table~\ref{tab:axioms-visual}) \\ %
Generator and data & corruptions that break one relation or swap a block of fields; noisy-but-coherent controls; an oracle with per-constraint labels; an intervention tolerance fixed before generation; sourced records with units, split by family or site & five corruption families, controls and intervention labels; 107 records of 71 families (Appendix~\ref{app:data}) \\ %
\midrule
\multicolumn{3}{@{}l}{\textit{Fitted again}}\\
Bands and visual modules & tolerances learned from the training fleet; visual heads specific to their sensor & 45.7\% of the learned bands with both edges inside the widened training fleet range (Appendix~\ref{app:results-refine}); with a visual stratum trained on photographs, the fused verdict reaches an AUROC of 0.724 on satellite pairs, against 0.796 for the record verdict alone (Appendix~\ref{app:results-xmodal}) \\ %
\midrule
\multicolumn{3}{@{}l}{\textit{Reused as Section~\ref{M-sec:meth} defines it}}\\
Verdict and report & masked verdict with slopes and bands learned per derived quantity, and its exact reduction for any knowledge base; a diagnosis key read from its degrees and a monotone second head; abduction, repair proposals and the report; stratified aggregation; grounding axiom with gate, guard and floor & Sections~\ref{M-sec:meth-kb} to~\ref{M-sec:meth-report}; Appendix~\ref{app:repo} \\ %
Protocol & priors built without held-out aircraft, intervention labels, black-box falsification within an edit budget, cluster-bootstrap intervals, comparators grouped by the knowledge they receive (Table~\ref{tab:comparators}), statistical imputers among them & Section~\ref{M-sec:res-setup}; Appendix~\ref{app:protocol} \\ %
\bottomrule
\end{tabularx}
\end{widetable}

The aircraft study suggests this ordered checklist.
\begin{enumerate}
\item \emph{Scope.} \fwname{} audits the coherence of a declaration with physics and with its evidence, and its report supports human review (Section~\ref{M-sec:res-limits}).
\item \emph{Fields and relations.} List the fields with their units and one relation per constraint, and let every constraint that reads an undeclared field abstain: on the records as declared, filling the undeclared fields at scoring lowers detection, with global means on every split and with MICE on test and unseen families, while with fields hidden at random MICE at scoring ranks unseen-family records higher (Appendix~\ref{app:results-missing}). %
\item \emph{Priors and labels.} Fit initial bands on clean training records only, and fix an intervention tolerance before generation, so that the intervention labels also reach errors that the generator oracle misses: on aircraft, the oracle flags no constraint on 613 of the 1,810 rows labeled as altered (33.9\%; Appendix~\ref{app:results-detect}). %
\item \emph{Masks and sensors.} Measure automatic part masks against public human masks, and train visual heads on images of the sensor in use (Sections~\ref{M-sec:res-missing} and~\ref{M-sec:res-grounding}; Appendix~\ref{app:results-xmodal}).
\item \emph{Comparators.} Include learners on the same physical residuals: on aircraft records a multilayer perceptron on them serves as the alarm of the optional rank-then-explain setting (Section~\ref{M-sec:res-detect}). %
\item \emph{Critical fields.} \fwname{} abstains on every constraint that reads an undeclared field; on aircraft, \fwname{}'s optional critical-field alert brings the evasion of lies that withhold fields back to the level of bare lies, which withhold none (Section~\ref{M-sec:res-falsif}). %
\end{enumerate}

\section{Implementation and Verification Report}
\label{app:repo}
\label{app:repo-files}%

This appendix summarizes how the implementation of \fwname{} is organized and shows an example of a verification report.
The implementation follows the seven modules of Section~\ref{M-sec:meth}. Every prediction file, of \fwname{} and of each comparator, has one format (instance identifier, split and score, with per-constraint outputs where the method returns them), so a single evaluation script scores every method in the same way.

\subsection{Example of a Verification Report}
\label{app:repo-report}

Figure~\ref{M-fig:report} (Section~\ref{M-sec:res-diag}) shows a report assembled from the saved outputs of \fwname{} (train prior, seed 0) for a record chosen by a rule fixed before any value was read: the first single-field corruption of the test split, in file order, whose altered field an available constraint reads.
The corruption divides the total thrust of a narrow-body turbofan by 1.2, and the generator oracle flags A4 and A9. %
The verdict flags the record ($V=1.217$, above the validation threshold $t^{\star}=0.310$ of this seed), and the report lists the degree of every available constraint, the largest being A4 (10.001), A5 (1.053), A8 (0.386) and A9 (0.384). By the rule of Section~\ref{M-sec:meth-report}, A4 and A5 exceed $\ln 2$ and are violated, A5 also on the clean record, while A8 and A9 are held. %
The largest-violation key names A4, which the oracle flags; it ranks A4 first in every seed, followed by A5 in this seed and by A8 in 5 of the 8 seeds. %
 %
 %
 %

\section{Physical Relations and Constants}
\label{app:physics}

The specification axioms evaluate closed-form relations of conceptual aircraft design~\cite{raymer2018aircraft} (Table~\ref{M-tab:axioms}).
Table~\ref{tab:constants} gives the constants of the relations and of the International Standard Atmosphere (ISO 2533:1975) with the initial values of the graded degrees, and Table~\ref{tab:rs} the remote-sensing constants with the Airplane Design Groups of the Federal Aviation Administration (FAA)\footnote{FAA Advisory Circular 150/5300-13B, Airport Design, March 31, 2022.}.
The train prior widens each per-class range of a derived quantity over the 44 clean training aircraft to $[0.88\min,\,1.12\max]$, as the oracle does; these initial bands are then learned with their slopes (Section~\ref{M-sec:meth-kb}), and business jets, absent from training, start from the band of all training classes. %
The learned bands act as tolerances learned from the fleet. On specification records, 45.7\% of them keep both edges inside the training fleet range widened in the same way, and the others extend beyond it (Appendix~\ref{app:results-refine}). In the remote-sensing stratum, 103 of the 144 learned group bands (the 18 bands of R2 for military roles, R3 and R4 in each of the 8 seeds) extend beyond the clean training range widened on each side by 12\% of its width. The initial bands of R2 for military roles, R3 and R4 (Table~\ref{tab:rs}) were set from the annotations and from knowledge of the aircraft types to contain every clean annotation, those of test and unseen sites included, and the generator oracle of the satellite benchmark was fitted on the clean annotations of every site, so both may carry held-out information; the satellite benchmark was not regenerated, nor the remote-sensing stratum retrained, without them. %

\begin{widetable}[!htbp]
\centering
\caption{Constants of the physical relations.}
\label{tab:constants}
\scriptsize%
\begin{tabularx}{\linewidth}{@{}l>{\raggedright\arraybackslash}p{0.3\linewidth}@{\hspace{12pt}}l>{\raggedright\arraybackslash}X@{}}
\toprule
\multicolumn{2}{@{}l}{Standard atmosphere} & \multicolumn{2}{l@{}}{Relations and graded degrees} \\
\midrule
$g$ & 9.80665~m\,s$^{-2}$ & units & SI, except $R$ (km), $F_{\mathrm{tot}}$ (kN) and $P_{\mathrm{tot}}$ (kW) \\
$R_{\mathrm{air}}$; $\gamma_{\mathrm{air}}$ & 287.05287~J\,kg$^{-1}$\,K$^{-1}$; 1.4 & A5 & $W_1=\max\{\MTOW-0.94\,m_{\mathrm{fuel}},\,0.05\,\MTOW\}$, a 6\% fuel reserve; log residual clipped to $[-3,3]$; available only when $\AR$ is also declared and, for propeller aircraft, $V_{\mathrm{cr}}$ \\
$T_0$; $\rho_0$ & 288.15~K; 1.225~kg\,m$^{-3}$ & A6 & log residual clipped to $[-2,2]$; wide $C_{L,\max}$ bounds, because published stall speeds mix clean and landing configurations \\
$\Gamma$ & $6.5\times10^{-3}$~K\,m$^{-1}$ (troposphere) & A7 & 95~kg per passenger with baggage; $f_u\in[0.3,1]$ \\
$h_{\mathrm{tr}}$; $T_{\mathrm{tr}}$ & 11{,}000~m; 216.65~K, isothermal above; density continuous & unused & an estimate of maximum $L/D$ from $\AR$ (Oswald factor 0.80, zero-lift drag coefficient 0.020) is implemented but read by no axiom; A5 takes $L/D$ from its latent \\
$\bar h$ & $h_{\mathrm{cr}}$ clamped to $[0,\,20{,}000]$~m & coordinates & $z=(\log_{10}q-m_g)/\sigma_g$ for a band, $z=\xi/\sigma_g$ for the residual $\xi$ of A1, A5, A6 and A7; $m_g$, $\sigma_g\ge10^{-3}$ from the clean training rows of group $g$ ($\sigma_g$ a root-mean-square around 0 for $\xi$), shrunk toward the fleet with weight $n_g/(n_g+5)$, the scatter through its square; a group without rows takes the fleet values \\
$a(\bar h)$ & $\sqrt{\gamma_{\mathrm{air}}R_{\mathrm{air}}T(\bar h)}$ & initial values & tolerances $\tau/\sigma_g$ with $\tau=0.05$ (A1), 0.25 (A5) and 0.15 (A6); slope $\kappa=8/\tilde w$, with $\tilde w$ the median initial band width in standard units (A7: $\kappa=25\,\sigma$, with $\sigma$ the fleet scatter of its margin, and lower edge at 0); in-band weight 0.05 \\
 & & latent bounds & sigmoid of one parameter per pair, rescaled to these bounds and starting at their midpoint; $L/D$: 15-21 (wide-body), 14-19 (narrow-body), 12-17 (regional jet), 11-17 (regional turboprop, business jet), 9-15 (general-aviation turboprop), 7-14 (general-aviation piston); $C_{L,\max}$: 2.2-3.2 (wide-body, narrow-body), 2.0-3.0 (regional jet), 1.9-3.0 (regional turboprop), 1.6-2.8 (business jet), 1.4-2.4 (general-aviation turboprop), 1.2-2.2 (general-aviation piston); $c_T$ (turbofan): $1.2\times10^{-5}$ to $2.3\times10^{-5}$~kg\,N$^{-1}$\,s$^{-1}$; $c_P$: $7.5\times10^{-8}$ to $1.1\times10^{-7}$ (turboprop), $7.0\times10^{-8}$ to $1.0\times10^{-7}$ (piston)~kg\,W$^{-1}$\,s$^{-1}$; $\eta_p$: 0.75-0.88 \\
 & & degrees & capped at $10^{4}$; truth $\exp(-v)$ never clamped; a masked axiom contributes no value and no gradient; $V=\ln 2$ without an available axiom \\
\midrule
\multicolumn{4}{@{}>{\raggedright\arraybackslash}p{\linewidth}@{}}{$g$: standard gravity; $R_{\mathrm{air}}$, $\gamma_{\mathrm{air}}$: gas constant and ratio of specific heats of air; $T_0$, $\rho_0$: sea-level temperature and density; $\Gamma$: temperature lapse rate; $h_{\mathrm{tr}}$, $T_{\mathrm{tr}}$: altitude and temperature of the tropopause; $\bar h$: altitude at which the atmosphere is evaluated; $a$: speed of sound; $q$: derived quantity of an axiom; $n_g$: number of clean training rows of group $g$.} \\
\bottomrule
\end{tabularx}
\end{widetable}

\begin{widetable}[!htbp]
\centering
\caption{Constants of the remote-sensing axioms.}
\label{tab:rs}
\scriptsize%
\setlength{\tabcolsep}{3pt}
\begin{tabular}[t]{@{}lc@{}}
\toprule
Design group & Wingspan (m) \\
\midrule
I & below 14.9352 \\
II & 14.9352-24.0792 \\
III & 24.0792-35.9664 \\
IV & 35.9664-52.1208 \\
V & 52.1208-65.2272 \\
VI & 65.2272-79.8576 \\
\bottomrule
\end{tabular}\hfill
\begin{tabular}[t]{@{}lcccc@{}}
\toprule
Role & R2 span (m), initial & R3 $L/b$, initial & R4 $n_{\mathrm{eng}}$, initial & R6 $P(\text{jet})$, initial \\
\midrule
Civil small & 0-15.005 (frozen) & 0.55-1.70 & 0-2 & 0.25 \\
Civil medium & 15.005-36.0 (frozen) & 0.70-1.45 & 1-4 & 0.55 \\
Civil large & 36.0-90.0 (frozen) & 0.75-1.35 & 2-4 & 0.98 \\
Military transport & 7-75 & 0.70-1.60 & 1-4 & 0.55 \\
Military bomber & 7-60 & 0.60-2.20 & 2-4 & 0.90 \\
Military fighter & 5-20 & 0.90-2.70 & 1-2 & 0.98 \\
Military trainer & 5-15 & 0.80-2.00 & 1-2 & 0.60 \\
\midrule
\multicolumn{5}{@{}l}{Initial tolerance 0.3 (R5) and 0.10 (R10); weight of R9 0.5; frozen: R1 and civil R2 bands} \\
\multicolumn{5}{@{}l}{Graded slots R1, R2 (civil and military), R3, R4 and R10; 57 trained parameters} \\
\bottomrule
\end{tabular}
\end{widetable}

\section{Data and Benchmark Details}
\label{app:data}

This appendix describes the specification records and the benchmarks built from the records and images of Section~\ref{M-sec:meth-data}: the corruption operators with their realized magnitudes (Table~\ref{tab:magnitudes}), the realized instances per type and split (Appendix~\ref{app:data-benchmarks}), the frame of the cross-dataset runs and the rule that chose the examples of Figure~\ref{M-fig:heatmaps}. Every generator is deterministic given its seed.
The 107 specification records cover seven classes: wide-body (17 aircraft of 8 families), narrow-body (18 of 6), regional jet (13 of 6), regional turboprop (11 of 8), business jet (19 of 17, the unseen class), general-aviation turboprop (12 of 11) and general-aviation piston (17 of 15). They were compiled from the type-certificate data sheets of the aviation authorities of the United States, the European Union, Canada, the United Kingdom and France and from manufacturer documents (airport-planning manuals, specification sheets and flight-manual excerpts), with other public sources used for cross-checks. The 67 jets declare thrust and the 40 propeller aircraft shaft power, and 18 records, all general-aviation aircraft, declare the stall speed; every record passed a loose screening with the verifier's physics module, so the clean records were checked against the relations that the verifier evaluates; the intervention labels, set from the edit alone, bypass this dependence. %
Table~\ref{tab:magnitudes} lists the corruption operators. A candidate that the oracle does not flag is generated again with its magnitude doubled, for at most six attempts (Section~\ref{M-sec:meth-data-gen}), so the realized magnitudes reflect the oracle tolerances and extend beyond the declared grid of 5 and 10\%. They extend furthest for performance corruptions, whose Breguet tolerance must absorb the largest residual among the clean records. %

\begin{widetable}[!htbp]
\centering
\caption{Corruption types and magnitudes.}
\label{tab:magnitudes}
\scriptsize%
\begin{tabularx}{\linewidth}{@{}L{0.14\linewidth}>{\raggedright\arraybackslash}XL{0.17\linewidth}*{9}{@{\hspace{5pt}}r}@{}}
\toprule
 & & & & & \multicolumn{7}{c@{}}{Realized magnitude $\epsilon$}\\
\cmidrule(l){6-12}
Operator & Fields altered & Kept if the oracle flags & Kept & Disc. & 5\% & 10\% & 20\% & 40\% & 80\% & 160\% & 320\% \\
\midrule
Single value & one of $S$, $F_{\mathrm{tot}}$ or $P_{\mathrm{tot}}$, $\OEW$, $V_{\mathrm{cr}}$ & the targeted envelope (A3, A4 or A4P, A8, A10) & 212 & 2 & 0 & 0 & 53 & 105 & 46 & 8 & 0 \\
Geometric & one of $b$, $S$, $\AR$ & A1 & 208 & 6 & 104 & 104 & 0 & 0 & 0 & 0 & 0 \\
Performance & $R$ or $m_{\mathrm{fuel}}$ & A5 & 189 & 25 & 1 & 2 & 13 & 35 & 55 & 66 & 17 \\
Cruise & one of $V_{\mathrm{cr}}$, $h_{\mathrm{cr}}$, $S$ & A2 & 206 & 8 & 0 & 8 & 59 & 85 & 46 & 7 & 1 \\
Frankenstein & one block of a record of another class & any constraint & 176 & 38 & \multicolumn{7}{c@{}}{-} \\
Noisy-but-coherent control & every declared non-integer field (relative standard deviation 0.02); $\AR$ re-derived & no constraint ($y=0$) & 214 & 0 & \multicolumn{7}{c@{}}{-} \\
\bottomrule
\end{tabularx}

\medskip
\begin{tabular}{@{}lrrrrr@{}}
\toprule
Constraints violated per corrupted instance & 1 & 2 & 3 & 4-6 & Total \\
\midrule
Corrupted instances & 553 & 253 & 125 & 60 & 991 \\
\bottomrule
\end{tabular}

\smallskip
\begin{minipage}{\linewidth}\scriptsize\raggedright
Kept: instances kept; Disc.: candidates discarded, because the operator does not apply to the record (a missing field, or for a Frankenstein swap no record of another class in the same split) or because the condition of the third column fails in all six attempts. Each operator draws two candidates per clean record, 214 in all; the 38 discarded Frankenstein swaps are those of the unseen class, a single class. Constraints violated: flagged by the generator oracle. Disc.\ has the same meaning in Table~\ref{tab:benchmarks}.
\end{minipage}
\end{widetable}

\subsection{Benchmarks}
\label{app:data-benchmarks}

Table~\ref{tab:benchmarks} gives the realized instances per type and split of four benchmarks: the specification records, the photographic mismatch pairs, the satellite annotations and the satellite multimodal pairs. Of the 100 FGVC-Aircraft variants, 68 are paired with a specification record, 39 with the record of their type and 29 with that of their family; their 6,800 photographs follow the split of the paired family (2,500 training, 1,300 validation, 800 test, 1,500 unseen-family and 700 unseen-class photographs), and the 32 unpaired variants enter no split. The satellite annotations start from 3,000 clean RarePlanes annotations, a sample stratified by role with a floor of 200 annotations per role (all annotations of a rarer role), and each operator draws one candidate from each of them, so the kept and discarded candidates of a type sum to 3,000.
In the cross-dataset runs of Appendix~\ref{app:results-xmodal}, a knowledge base fitted on one source scores the other through a 23-field unified frame (Figure~\ref{M-fig:data}(c)) with only the axioms R2 to R4, which both sources can evaluate; a specification record takes its role and design group from its clean wingspan. %
On the records of each split, the report prints the class of the nearest clean training record (Section~\ref{M-sec:meth-report}; example in Figure~\ref{M-fig:report}): this proximity names the class of a record from an unseen family and attributes each business jet, a class absent from training, to its nearest training class. %
\paragraph{Examples of Figure~\ref{M-fig:heatmaps}.} The figure shows, for seed 0, the network trained with the grounding axiom and the network trained without a localization term. A rule fixed in advance chose its examples: each is the first candidate on which both networks give a non-degenerate map (maximum above 0) and the axiom-trained network predicts the target class with an in-mask share $\rho$ of at least 0.15. Photographs for engine count, propulsion and tail configuration are taken in identifier order from validation and test, and from unseen families only when validation and test offer no such candidate. Satellite crops of civil aircraft, for the role small civil aircraft and for an engine count of 2, are taken in identifier order from validation and test sites, the second from a site other than the first's, with the annotated diamond as region. Both networks are drawn at their selected checkpoints; on photographs the network without a localization term is at its first epoch. Pointing is higher with the axiom in every seed pair at the selected and at the last checkpoints, and the automatic pseudo-masks include some background or off-part regions, visible in the figure. The examples illustrate the maps; pointing measures localization (Section~\ref{M-sec:res-grounding}). %

\begin{widetable}[!htbp]
\centering
\caption{Benchmarks.}
\label{tab:benchmarks}\vspace{3pt}
\scriptsize%
\begin{tabular}{@{}llrrrrrr@{}}
\toprule
Type & Label & Train & Val. & Test & Unseen fam. & Unseen class & Total \\
\midrule
\multicolumn{8}{@{}l}{\textit{Specification records}}\\
Clean & coherent & 44 & 17 & 16 & 11 & 19 & 107 \\
Single value & incoherent & 86 & 34 & 32 & 22 & 38 & 212 \\
Geometric & incoherent & 88 & 32 & 32 & 22 & 34 & 208 \\
Performance & incoherent & 81 & 29 & 28 & 22 & 29 & 189 \\
Cruise & incoherent & 86 & 32 & 32 & 22 & 34 & 206 \\
Frankenstein & incoherent & 88 & 34 & 32 & 22 & 0 & 176 \\
Noisy-but-coherent control & coherent & 88 & 34 & 32 & 22 & 38 & 214 \\
Total & & 561 & 212 & 204 & 143 & 192 & 1,312 \\
\midrule
\multicolumn{8}{@{}l}{\textit{Photographic mismatch pairs}}\\
Matched & coherent & 788 & 499 & 412 & 250 & 51 & 2,000 \\
Inter-class & mismatched & 227 & 108 & 70 & 138 & 0 & 543 \\
Intra-family & mismatched & 227 & 168 & 137 & 68 & 0 & 600 \\
Attribute & mismatched & 232 & 158 & 132 & 63 & 15 & 600 \\
Occlusion & coherent control & 161 & 115 & 69 & 52 & 3 & 400 \\
Degradation & coherent control & 163 & 87 & 88 & 50 & 12 & 400 \\
Total & & 1,798 & 1,135 & 908 & 621 & 81 & 4,543 \\
\bottomrule
\end{tabular}

\medskip
\begin{tabular}{@{}llrrrrr@{}}
\toprule
Type & Label & Train & Val. & Test & Unseen sites & Disc. \\
\midrule
\multicolumn{7}{@{}l}{\textit{Satellite annotations}}\\
Clean & coherent & 1,903 & 262 & 357 & 478 & - \\
Wingspan rescaling & incoherent & 1,903 & 262 & 357 & 478 & 0 \\
Length rescaling & incoherent & 1,139 & 149 & 224 & 306 & 1,182 \\
Design-group change & incoherent & 1,903 & 262 & 357 & 478 & 0 \\
Role change & incoherent & 1,903 & 262 & 357 & 478 & 0 \\
Engine-count swap & incoherent & 1,233 & 170 & 254 & 415 & 928 \\
Propulsion change & incoherent & 1,890 & 256 & 357 & 478 & 19 \\
Frankenstein & incoherent & 1,899 & 262 & 356 & 478 & 5 \\
Noisy-but-coherent control & coherent & 1,903 & 262 & 357 & 478 & 0 \\
Total & & 15,676 & 2,147 & 2,976 & 4,067 & 2,134 \\
\midrule
\multicolumn{6}{@{}l}{\textit{Satellite multimodal pairs}} & Total \\
Clean & coherent & 950 & 144 & 171 & 235 & 1,500 \\
Tabular corruption & incoherent & 956 & 152 & 162 & 230 & 1,500 \\
Image swap & incoherent & 971 & 134 & 170 & 225 & 1,500 \\
Geometric corruption & incoherent & 483 & 75 & 80 & 112 & 750 \\
Total & & 3,360 & 505 & 583 & 802 & 5,250 \\
\bottomrule
\end{tabular}

\smallskip
\begin{minipage}{\linewidth}\scriptsize\raggedright
Photographic pairs: matched, a photograph with its own record; inter-class, with the record of an aircraft of another class; intra-family, with the record of another variant of its family; attribute, with its own record altered in engine count, engine type or maximum take-off mass; occlusion and degradation, with its own record and a part of the photograph occluded or the photograph degraded (blur, compression or low light). Satellite multimodal pairs: clean, an annotation with its own crop; tabular corruption, the annotation altered in engine count, propulsion or role, with its own crop; image swap, the clean annotation with the crop of an aircraft of another role from the same split; geometric corruption, the annotation with its wingspan or length rescaled. Disc.: as in Table~\ref{tab:magnitudes}; in the multimodal panel the last column gives the totals.
\end{minipage}
\end{widetable}

\FloatBarrier
\section{Method Details}
\label{app:method}

This appendix completes Section~\ref{M-sec:meth}. Table~\ref{tab:axioms-visual} adds the visual predicates and the cross-modal axioms to the catalog of Table~\ref{M-tab:axioms}; Appendix~\ref{app:method-proof} proves Proposition~\ref{M-prop:reduction} and reports its numerical check; Figures~\ref{fig:rankexplain} and~\ref{fig:grounding} zoom into the verification report and the vision module; Appendix~\ref{app:method-forms} defines the localization terms of the matched grounding study; and Algorithm~\ref{alg:forward} transcribes the forward pass for one declared record.

\begin{widetable}[!htbp]
\centering
\caption{Axiom catalog: visual predicates and cross-modal axioms.}
\label{tab:axioms-visual}
\scriptsize
\setlength{\tabcolsep}{3pt}
\begin{threeparttable}
\begin{tabular}{@{}L{30mm}L{58mm}L{38mm}L{30mm}@{}}
\toprule
Predicate & Classes (one sigmoid head each) & Part region & Read by\\
\midrule
\multicolumn{4}{@{}l}{\textbf{Photographs} (8 predicates, 25 heads; regions from the pseudo-masks)}\\ %
EngineCount & 1, 2, 3, 4 & engine & AX1\\
EnginePlacement & underwing, rear, overwing & engine & AX4\\
Propulsion & jet, turboprop, piston & propeller and engine & AX2\\
WingPosition & high, mid, low & wing & AX5 (not implemented)\\
TailConfig & conventional, T-tail, cruciform, V-tail & tail & AX4\\
WingletPresence & present (binary) & none (conjuncts masked) & none\\
MainGearWheelCount & 2, 4, 6 to 10, 12 or more (no attribute labels) & landing gear & AX3\\
FuselageSlenderness & stubby, standard, slender (no attribute labels) & none (conjuncts masked) & none\\
\midrule
\multicolumn{4}{@{}l}{\textbf{Satellite crops} (7 predicates, 25 heads; region of every head: the annotated diamond)}\\ %
role & civil small, medium, large; military transport, bomber, fighter, trainer & diamond & AX4, AX5\\
wing type & straight, swept, delta, variable-swept & diamond & AX3\\
wing position & mid or low, high & diamond & none\\
propulsion & jet, propeller, unpowered & diamond & AX2\\
engine count & 0, 1, 2, 3, 4 & diamond & AX1\\
canards & no, yes & diamond & none\\
tail fins & 1, 2 & diamond & none\\
\bottomrule
\end{tabular}

\medskip
\begin{tabular}{@{}L{9mm}L{80mm}L{67mm}@{}}
\toprule
ID & Photographs: truth (availability) & Satellite crops: truth (availability)\\
\midrule
AX1 & $\pvis(n_{\mathrm{eng}}\mid I)$ (declared $n_{\mathrm{eng}}\in\{1,\dots,4\}$; engine visible) & $\pvis(n_{\mathrm{eng}}\mid I)$ (declared $n_{\mathrm{eng}}\in\{0,\dots,4\}$)\\ %
AX2 & $\pvis(\hat e\mid I)\,\big(1-\pvis(\text{propeller}\mid I)\,\sigma(25(M-0.62))\big)$, each factor entering only when its declared fields are present (engine type, or $V_{\mathrm{cr}}$ and $h_{\mathrm{cr}}$, declared; propeller or engine visible) & $\pvis(\pi\mid I)$ (declared $\pi$)\\ %
AX3 & $\sum_j\pvis(j\mid I)\,\mu_j(\log_{10}\MTOW)$, band edges at $\log_{10}\MTOW=4$, 5.05 and 5.5, slope 16 ($\MTOW$ declared; landing gear visible) & $\pvis(t\mid I)$ (declared $t$)\\ %
AX4 & $\IR\big(\pvis(\text{rear}\mid I),\,\pvis(\text{T-tail}\mid I)+\pvis(\text{cruciform}\mid I)\big)$ (engine and tail visible; no declared field) & $\pvis(r\mid I)$ (declared $r$)\\
AX5 & wing position against class: planned, not implemented & $1-\sum_{r'}\pvis(r'\mid I)\,\iota_{r'}$, where $\iota_{r'}=1$ when $r'$ is a civil role whose wingspan band does not overlap the band of the declared design group; weight 0.5 (declared civil $r$ and design group)\\ %
\bottomrule
\end{tabular}
\begin{tablenotes}[flushleft]
\scriptsize
\item $\pvis(\cdot\mid I)$: posterior of a predicate, the softmax over its heads (a sigmoid for a binary predicate); $\hat e$: visual class of the declared engine type (turbofan to jet); propeller: turboprop or piston; $M$: cruise Mach number $V_{\mathrm{cr}}/a(h_{\mathrm{cr}})$ of the declared fields, as in A10; $\mu_j$: soft membership of the declared mass in the wheel band $j$; $\IR$: Reichenbach implication; $[\cdot]$: indicator. Every head of a satellite crop shares the diamond region, visible whenever it is annotated. The photographic score reads the cross-modal axioms alone; the satellite verdict combines them with $\Rset$ (Section~\ref{M-sec:meth-agg}). Settings: Table~\ref{tab:config}.
\end{tablenotes}
\end{threeparttable}
\end{widetable}

\subsection{Proof of the Exact Reduction}
\label{app:method-proof}

\begin{pf}
(a)~Fix $\alpha$. A masked axiom has $\alpha_k=0$, and the implementation sets its degree (for a cross-modal axiom, its truth) to 0 by selection before any sum; its term and its weight therefore vanish from the numerator and from the denominator of Eq.~\eqref{M-eq:pme}, which reduce to the sums over $\mathcal S'$. If $\mathcal S'$ is empty, a tabular stratum takes $V_{\mathcal S}=\ln 2$ by definition and the cross-modal stratum has $\gamma_{\AXset}=0$. Every degree and its derivatives are finite, because undeclared fields carry a finite placeholder and degrees are capped at $10^4$, so the selection passes no value and no gradient from the degree $v_k$, or the truth $\sat_k$, of a masked axiom. The stratum verdict therefore has zero derivative with respect to them, and a field read only by masked axioms does not affect it.
(b)~Without an image, no cross-modal axiom is available. Then, as when every cross-modal axiom is masked, $\gamma_{\AXset}=0$, and since $\lambda_{\mathcal T}>0$, Eq.~\eqref{M-eq:strata} reduces to $\feas_{\mathcal T}=\exp(-V_{\mathcal T})$; the implementation returns $\feas_{\mathcal T}$ itself in this case, so the equality holds bit for bit. Otherwise, the cross-modal stratum enters both sums of Eq.~\eqref{M-eq:strata} with weight $\lambda_{\AXset}\gamma_{\AXset}$ and the tabular stratum keeps $\lambda_{\mathcal T}$, which is the stated renormalization over the strata present.
\end{pf}

In float32, the masked verdict equals the verdict of the knowledge base restricted to the available constraints bit for bit in 16 of 16 runs (Figure~\ref{M-fig:missing}a), eight seeds under each of the two tolerance priors: the train prior and the textbook prior, which takes handbook envelopes in place of the bands fitted on the training fleet (Table~\ref{tab:spec}, note f). It agrees with a float64 recomputation up to float32 rounding, within $2.2\times10^{-7}$ relative, and the reduction without an image also holds in all 180 runs scored with newly drawn masks that hide a nominal 30\% of the satellite images. %

\subsection{Method Components}
\label{app:method-forms}%
Figure~\ref{fig:rankexplain} zooms into the module that assembles the verification report and the optional settings of Section~\ref{M-sec:meth-report}.

\begin{widefigure}[t]
\centering
\begin{adjustbox}{max width=\textwidth}
\begin{tikzpicture}[
  zoomframe/.style = {draw=#1, dashed, rounded corners=4pt, line width=0.6pt, inner sep=3.5mm},
  tab/.style       = {minimum height=0pt, minimum width=0pt, inner sep=2pt, font=\scriptsize\bfseries},
  alarmflow/.style = {flow, dashed},
  dot/.style       = {circle, fill=ioline, inner sep=0pt, minimum size=1.6mm},
]
  \node[symb]   (kb)  at (0, 1.9)    {Knowledge base\\and graded\\predicates};
  \node[fusion] (hd)  at (0,-0.25)
       {Verdict and keys\\[1pt]{\scriptsize \fwname{}'s verdict,}\\{\scriptsize second head}};
  \node[neural] (vm)  at (0,-1.9)    {Vision module\\and grounding\\axiom};
  \node[neural, anchor=west] (alm) at ([xshift=3mm]kb.east) {Residual-\\learner alarm\\(optional)};
  \node[result] (out) at (14.05, 0.4)
       {Constraint verdicts,\\violated constraint,\\altered field, repair\\proposal, coverage};

  \node[symb] (dg) at (5.5, 0.95) {Constraint\\diagnosis};
  \node[symb] (fl) at (8.45, 1.9)  {Field\\localization};
  \node[symb, right=0.68cm of fl] (rp) {Repair\\proposals};
  \node[lbl, align=center, above=1mm of fl] (nfl) {knowledge-base\\abduction};
  \node[lbl, align=center, above=1mm of rp] (nrp) {minimum change\\to $\exp(-V)\geq\tau_r$};
  \node[fusion] (vr) at (8.45,-0.25)
       {Verification report\\[1pt]{\scriptsize verdict $V$, degrees $v_k$,}\\{\scriptsize evidence coverage,}\\{\scriptsize maps, $\rho_k$, nearest class}};
  \coordinate (fb) at (dg |- fl);
  \node[symb, dashed] (alr) at (5.5, 2.95) {Critical-field\\alert (optional)};

  \begin{scope}[on background layer]
    \node[zoomframe=ioline, fit=(dg)(fl)(rp)(nfl)(nrp)(vr)(alr)] (frame) {};
  \end{scope}
  \node[fusion, tab, anchor=west] at ([xshift=8mm]frame.south west) {Verification report and optional settings};

  \draw[alarmflow] (kb.north) -- ++(0,0.55) -| node[lbl, above, pos=0.25] {physical residuals} (alm.north);
  \draw[alarmflow, -] (alm) -- node[lbl, above, pos=0.62] {flag} (fb);
  \coordinate (fj) at ([xshift=-11.5mm]fb);
  \draw[flow, -] (hd.north) -- ++(0,0.45) -| node[lbl, above, pos=0.25] {flag, $V\ge t^{\star}$} (fj);
  \draw[flow, -] (fj) -- (fb);
  \node[dot] at (fj) {};
  \draw[flow] (fb) -- (fl);
  \draw[flow] (fb) -- (dg);
  \node[dot] at (fb) {};
  \draw[alarmflow] (alr) -- (fb);
  \draw[flow] (fl) -- (rp);
  \draw[flow] ([yshift=3mm]hd.east) -| node[lbl, above, pos=0.18] {degrees $v_k$} (dg.south);
  \draw[flow] (hd.east) -- node[lbl, below, pos=0.22] {verdict, degrees} (hd.east -| vr.west);
  \draw[flow] (vm.east) -- node[lbl, above] {maps, $\rho_k$} ([xshift=4mm]frame.west |- vm.east) |- ([yshift=-4mm]vr.west);
  \draw[flow] (dg.east) -- node[lbl, above, pos=0.23] {largest-violation key} (dg.east -| out.west);
  \draw[flow] (rp) -| node[lbl, above, pos=0.43] {flagged records} (out.north);
  \draw[flow] (vr.east |- out.west) -- node[lbl, above, pos=0.4] {every record} (out.west);
\end{tikzpicture}
\end{adjustbox}
\caption{Verification report and optional settings.}
\label{fig:rankexplain}
\end{widefigure}

Figure~\ref{fig:grounding} zooms on the vision module and the grounding axiom of Section~\ref{M-sec:meth-vision}. %
\begin{widefigure}[t]
\centering
\begin{adjustbox}{max width=\textwidth}
\begin{tikzpicture}[
  zoomframe/.style = {draw=#1, dashed, rounded corners=4pt, line width=0.6pt, inner sep=3.5mm},
  tab/.style       = {minimum height=0pt, minimum width=0pt, inner sep=2pt, font=\scriptsize\bfseries},
  dot/.style       = {circle, fill=ioline, inner sep=0pt, minimum size=1.6mm},
]
  \node[io]     (dat) at (-0.7, 1.3)  {Data and\\verification\\instances};
  \coordinate (prtrow) at (0,-2.1);
  \node[expert, anchor=west] (prt) at (dat.west |- prtrow) {Part\\definitions};
  \node[io]     (msk) at (4.35,-2.1)
       {Part masks of $I$\\[1pt]{\scriptsize OWLv2 + SAM (photographs),}\\{\scriptsize annotated diamond (satellite)}};
  \node[io]     (reg) at (8.2, -2.1) {Part region $\tilde M_k$,\\visibility $\vis_k$};
  \node[symb]   (ax)  at (13.85, 1.3) {Cross-modal\\axioms};

  \node[neural] (bb)  at (2.6, 1.3)   {Backbone};
  \node[neural] (hd)  at (5.2, 1.3)   {Sigmoid heads\\$P_k(I)$};
  \node[neural] (map) at (5.2, 0)     {Activation\\maps $\hat H_k(I)$};
  \node[symb]   (rho) at (7.9, 0)     {In-mask\\share $\rho_k$};
  \node[symb]   (av)  at (10.75, 0)
       {Grounding axiom\\$P_k\rightarrow\Loc_k(\rho_k)$\\[1pt]{\scriptsize training only}};
  \coordinate (br) at (10.75, 1.3);

  \begin{scope}[on background layer]
    \node[zoomframe=neuralline, fit=(bb)(hd)(map)(rho)(av)] (frame) {};
  \end{scope}
  \node[neural, tab, anchor=west] at ([xshift=4mm]frame.south west) {Vision module and grounding axiom};

  \draw[flow] (dat) -- node[lbl, above, pos=0.4] {image $I$} (bb);
  \draw[flow] (bb) -- (hd);
  \draw[flow] (bb.south) |- node[lbl, above, pos=0.75] {features} (map.west);
  \draw[flow] (hd) -- node[lbl, right] {weights} (map);
  \draw[flow] (map) -- (rho);
  \draw[flow] (rho) -- node[lbl, above] {$\rho_k$} (av);
  \draw[flow] (hd) -- (br) -- node[lbl, above, pos=0.35] {$\pvis$} (ax);
  \node[dot] at (br) {};
  \draw[flow] (br) -- node[lbl, right] {$P_k$} (av);
  \draw[flow] (prt) -- node[lbl, above] {prompts} node[lbl, below] {(photographs)} (msk);
  \draw[flow] (msk) -- (reg);
  \draw[flow] (rho |- reg.north) -- node[lbl, right, pos=0.78] {$\tilde M_k$} (rho);
  \coordinate (vb) at (av.south |- reg.east);
  \draw[flow, -] (reg.east) -- node[lbl, above] {$\vis_k$} (vb);
  \draw[flow] (vb) -- node[lbl, left, pos=0.55] {guard} (av.south);
  \draw[flow] (vb) -| node[lbl, above, pos=0.25] {availability} (ax.south);
  \node[dot] at (vb) {};
\end{tikzpicture}
\end{adjustbox}
\caption{Vision module and grounding axiom.}
\label{fig:grounding}
\end{widefigure}

\label{app:method-latents}
\label{app:method-xmodal}
\label{app:method-axiomaux}
\label{app:method-heads}
\label{app:method-readout}

Table~\ref{tab:config} gives the settings of the components. The matched grounding study isolates the gate of the grounding axiom: on photographs, it trains localization terms $\lambda\langle g_{ik}(1-\Loc_k(I_i))\rangle_{\vis}$ with different gates $g_{ik}$, at the loss weights $\lambda$ of 0.3, 1 and 3 used in Section~\ref{M-sec:res-grounding}, with three runs per weight. With $g_{ik}=P_k(I_i)$ the term is the grounding axiom; with $g_{ik}$ equal to the positive label of a supervised head and the normalizer of the axiom, only the gate changes (gate only); with $g_{ik}=1$ the gate is removed (ungated term). A fourth variant gates by the label and averages over the visible label-positive heads, which changes the normalizer together with the gate. The auxiliary loss of Section~\ref{M-sec:meth-vision} is trained at the same weights as a reference; Section~\ref{M-sec:res-grounding} and Appendix~\ref{app:expsup} report its comparison with the axiom.

On photographs, the engine-count and engine-type axioms AX1 and AX2 carry the instance-specific signal, since the gear heads read by AX3 have no attribute label and AX4 reads no declared field; a pair without an available axiom scores 0. On satellite crops, AX5 is available on 95.2\% of the pairs. %

Algorithm~\ref{alg:forward} transcribes from the code the forward pass of Figure~\ref{M-fig:pipeline} for one declared record.

\begin{algorithm}[!htbp]
  \caption{\fwname{} verification of a declared record.}
  \label{alg:forward}
  \DontPrintSemicolon
  \SetCommentSty{textit}
  \footnotesize
  \KwIn{$x$, $m$, optional image $I$, stratum $\mathcal T\in\{\Aset,\Rset\}$, threshold $t^{\star}$}
  \KwOut{verification report}
  \lForEach{axiom $k\in\mathcal T$}{availability $\alpha_k$, degree $v_k$ by Eq.~\eqref{M-eq:soft} ($-\ln$ of its truth for R5 to R9), truth $\sat_k=\exp(-v_k)$}
  $V_{\mathcal T}\leftarrow$ Eq.~\eqref{M-eq:pme}, or $\ln 2$ without available axioms;\quad $\feas\leftarrow\exp(-V_{\mathcal T})$\;
  \If{$I\neq\varnothing$}{
    \lForEach{axiom $j\in\AXset$}{availability $\alpha_j$, truth $\sat_j$}
    \lIf{$\gamma_{\AXset}=1$ and $I$ is a satellite crop ($\mathcal T=\Rset$)}{$\feas\leftarrow$ Eq.~\eqref{M-eq:strata}}
    \tcp{photographs: score $1-\feas_{\AXset}$, or 0 if $\gamma_{\AXset}=0$}
  }
  $V\leftarrow-\ln\feas$;\quad flagged $\leftarrow[V\ge t^{\star}]$;\quad $e_k\leftarrow\alpha_k v_k$;\quad $\hat p_{\mathrm{m}}$ and $c_k$ by Eq.~\eqref{M-eq:mono}\;
  report $\leftarrow$ degrees, truths, coverage, maps and the nearest clean training record\;
  \If{a specification record with at least one available constraint is flagged by its verdict or by an optional setting (the residual-learner alarm of rank-then-explain or the critical-field alert)}{
    diagnosis $\arg\max_{k:\,\alpha_k=1}v_k$ (secondary key $\arg\max_k c_k$);\quad field ranking by Eq.~\eqref{M-eq:abduction};\quad repair by Eq.~\eqref{M-eq:repair}\;
  }
  \Return{report}\;
\end{algorithm}
\providecommand{\oneci}[2]{\ensuremath{[#1,\penalty 200\ #2]}}%
\section{Protocol Details}
\label{app:protocol}

This appendix completes the setup of Section~\ref{M-sec:res-setup}: it lists the comparators and the configuration, states how the configuration was selected and which choices used data beyond the training split, and reports how the vision results change when the part detector's visibility threshold is re-selected on validation data alone.
Table~\ref{tab:comparators} orders the comparators by the knowledge they receive; each shares the partitions, the labels and the evaluation code of \fwname{}, with settings selected on validation data within a budget fixed before its runs, and one planned arm, the language model in its thinking mode, has no result: in a pilot run too few of its answers could be parsed, below the parse rate fixed before the runs. %
Table~\ref{tab:config} gives the configuration of every module of \fwname{}, of its optional settings and of the baseline detectors; bands, slopes, latents, visual heads, read-outs, operating thresholds and checkpoints are fitted or selected on data, and every other value was set once. The class bands of the generator oracle are fitted on all 107 aircraft, test and out-of-distribution aircraft included, and reach training through the generator labels and per-constraint targets (Section~\ref{M-sec:meth-data-gen}); Appendix~\ref{app:results-refine} refits the oracle without test or out-of-distribution aircraft.
\paragraph{Selection of the configuration.}
\fwname{}'s configuration was selected on validation data and on held-out training classes; during development, earlier configurations were also evaluated on the test splits. Validation data selected the configuration of both tabular strata, the largest-violation key, the abduction and repair rules and the optional settings, and held-out training classes also selected the key and the critical-field alert (Appendix~\ref{app:results-refine}). The rule that reads each constraint as held or violated has no tuned parameter. The part detector's visibility threshold was moved from 0.25 to 0.20 on a 240-image stratified subset including test and out-of-distribution images, by an informal trade-off, with no numeric criterion and no saved tuning output. The results of this paper use 0.20.

\paragraph{Visibility threshold re-selected on validation.}
Re-selected on 240 stratified validation images alone, under a criterion declared beforehand that keeps propeller detections on jets, which have no propeller, at or below 10\% of jet images, the threshold becomes 0.25, as in 197 of 200 resampled subsets: at 0.20 those detections reach 17\% on validation, against 7\% on the original subset. Weighing part visibility and those detections equally keeps 0.20. Without retraining, 0.25 turns 12.1\% of the part flags of training, validation and test images to not visible, 2,794 flags, 1,618 of them on the wing.

\paragraph{Vision module retrained at 0.25.}
Retrained at 0.25 with the same code, configuration and 8 seeds (even seeds on the RTX 5090, which trained the original module, and odd seeds on the RTX 4090; the even seeds alone give the same accuracy and pointing readings), the photograph module keeps its macro head accuracy (0.916 on test and 0.944 on unseen families, as at 0.20) and at most one collapsed head per seed on test. It also keeps its pointing lead over the network without a localization term ($+0.671$) and with permuted masks ($+0.414$), its lead over RRR, GradMask and Pukdee et al.\ on PASCAL-Part ($+0.568$, $+0.676$ and $+0.469$, every interval above zero) and the nine overall readings of the cross-modal stratum against Qwen3-VL-8B, InternVL3.5-8B and Qwen3-VL-32B on test, unseen families and the unseen class. Of the 21 readings by pair type, 4 change by one step under each of the two bootstrap streams of that comparison, 3 of them to the disadvantage of the stratum. The networks without a localization term or with permuted masks, RRR, GradMask and Pukdee et al.\ were trained with the flags at 0.20 and not retrained.
At 0.25 the cross-modal AUROC is 0.693 on test (0.692 at 0.20) and slightly lower on unseen families (0.716 against 0.733) and on the unseen class (0.827 against 0.844); the axiom on engine placement and tail configuration (AX4) is available on 0.784 of the test pairs scored for detection, instead of 0.956. Pointing on pseudo-masks is slightly lower on test (0.774 against 0.780) and on unseen families (0.778 against 0.784): there the seed-bootstrap and the image-and-seed bootstrap intervals include zero, while the intervals over paired seeds, \oneci{-0.010}{-0.001} and \oneci{-0.010}{-0.004}, do not. On PASCAL-Part human masks pointing is slightly lower with 3 runs (0.711 against 0.737, $-0.026$ \oneci{-0.051}{-0.002}) and lower over 8 seeds (0.708 against 0.729, $-0.020$ \oneci{-0.037}{-0.005}).
The auxiliary loss at weight 0.3, retrained at 0.25 with its 3 seeds, falls on PASCAL-Part as the axiom does (0.708 against 0.739, $-0.031$ \oneci{-0.055}{-0.005}). With both at 0.25 the two stay on par in localization on test pseudo-masks ($-0.009$), out of distribution ($-0.027$) and on PASCAL-Part ($+0.003$ \oneci{-0.028}{+0.032}), while the axiom keeps its out-of-distribution accuracy lead ($+0.016$ \oneci{+0.010}{+0.021}). With all 8 seeds of the axiom, the axiom's pointing on test pseudo-masks is slightly lower than that of the auxiliary loss at 0.25 ($-0.019$) and lower at 0.20 ($-0.026$). A first check against the auxiliary loss still at 0.20, added after the results, had shown the axiom lower on PASCAL-Part ($-0.028$); the like-for-like comparison attributes that difference to the threshold. %
\label{app:protocol-amendments}%

\begin{widetable}[tbp]
  \centering
  \caption{Comparators.}
  \label{tab:comparators}
  \scriptsize%
  \setlength{\tabcolsep}{4pt}%
  \begin{tabularx}{\textwidth}{@{}L{1.85cm}>{\raggedright\arraybackslash}XL{3.2cm}L{2.3cm}L{1.8cm}@{}}
    \toprule
    Knowledge & Comparators & Input & Output & Data \\
    \midrule
    None & Isolation Forest, Mahalanobis, MLP on normalized fields & declared fields & error score & records \\
     & LightGBM, XGBoost, CatBoost, TabPFN v2 on raw fields & declared fields, availability mask, one-hot categories & error score & records \\
    Crisp rules & train-fitted rule checker; rules with hierarchical pooling & band exceedances of the residuals & score, violated rules & records \\
    Classical diagnosis & Fellegi-Holt, first unsatisfiable core of Z3, MaxSMT & crisp constraints & violated set, minimum-cardinality localization & records \\
    Sizing statistics & Class I sizing regressions on maximum take-off mass, fitted on the training aircraft & declared fields & largest class-I residual, its field & records \\ %
    Residual learners & MLP, LightGBM, XGBoost, CatBoost, logistic regression, TabPFN v2, TabICL with its v2 checkpoint, local outlier factor, one-class SVM, autoencoder; class-standardized residual and its Ochiai ranking; SHAP attributions & physical residuals with availability & score, attribution & records \\
    Differentiable logic & DL2 method (network on declared fields, projected-gradient-descent (PGD) adversary); network with the DL2 loss and the per-constraint targets on the physical residuals & declared fields; physical residuals with availability & score, per-constraint probabilities & records \\
    Knowledge in words & Qwen3-8B with the constraints A1-A11 & anonymized record & verdict, score, first violated constraint & records \\
    Localization losses & auxiliary localization loss and the matched ablations of the grounding axiom & photograph, pseudo-masks in training & attribute heads, evidence maps & photographs \\
    Explanation supervision & RRR, GradMask, RES-L, Pukdee et al. & photograph, part masks in training & attribute heads, evidence maps & photographs, human masks \\
    Sequential pipeline & cross-modal stratum (vision trained alone, axioms on its posteriors); independent concept-bottleneck model with LightGBM & photograph or crop with the record & mismatch score & photographs, satellite pairs \\
    Late fusion & logistic regression, LightGBM, CatBoost, TabPFN on the record alone, on the record with the visual posteriors, and on these with the image embeddings or with the knowledge-base truths; best fusion; monotone read-out & record, visual posteriors, embeddings, knowledge-base truths & pair score & photographs, satellite pairs \\
    End-to-end fusion & AutoGluon-MM 1.6.3 (out of the box), TIP-style (re-implementation) & raw crop and declared fields & pair score & satellite pairs \\
    Vision-language models & Qwen3-VL-8B, InternVL3.5-8B, Qwen3-VL-32B (4-bit NormalFloat, NF4); Qwen3-VL-8B with the rules R1-R10 & photograph with the declared specification; crop with the declared vector and the rules & verdict, score & photographs, satellite pairs \\
    Imputers & MICE, missForest, nearest-neighbor imputation & record with missing fields & imputed values & records \\
    \midrule
    \textbf{\fwname{} (ours)} & specification and remote-sensing axioms with graded predicates, grounding axiom, cross-modal axioms & declared record with availability mask; optional photograph or satellite crop & verdict, per-constraint degrees (constraints respected and violated), evidence coverage and maps, largest-violation key, field localization, repair proposals, report & records, photographs, satellite pairs \\
    \bottomrule
  \end{tabularx}
\end{widetable}
\begin{table}[tbp]
\centering
\caption{Configuration.}
\label{tab:config}
\scriptsize%
\begin{tabularx}{\linewidth}{@{}L{0.15\linewidth}X@{}}
\toprule
Setting & Value \\
\midrule
Specification stratum & graded degrees in standard units of the clean training rows of each group, shrunk toward the fleet with weight $n_g/(n_g+5)$; in-band grading; hierarchical banks; truth $\exp(-v_k)$, degrees capped at $10^{4}$; masked mean verdict ($p=1$), $V=\ln 2$ without available axioms; 336 trainable parameters, 126 of them bounded latents; chosen on validation among eight candidates with eight seeds each \\ %
Remote-sensing stratum & six graded slots (R1, R2 for civil and military roles, R3, R4, R10) in standard units of the clean training rows; R5 to R9 as the negative logarithm of their truth; masked weighted mean verdict ($p=1$); in-band grading; flat banks; 57 trained parameters; chosen on validation among eight candidates \\ %
Training of the strata & AdamW~\cite{loshchilov2017decoupled}, learning rate $5\times10^{-3}$ multiplied by 5, that is $2.5\times10^{-2}$, for every trained parameter of a stratum (bands, tolerances, slopes, in-band weights, latents and the parameters of R5 and R6), weight decay 0, batch 256; at most 100 epochs, early stopping on validation AUROC with patience 15, best state restored; $\lamaux=0.5$; seeds 0-7 \\ %
Fusion of the strata & weighted mean between strata with $\lambda_{\Rset}=\lambda_{\AXset}=1$; masked $p$-mean error with $p=4$ inside the cross-modal stratum; Reichenbach implication (R7-R9, grounding axiom, photographic AX4, satellite AX5); axiom weight 1, except R9 and satellite AX5 (0.5) \\ %
\midrule
Second head & monotone logistic regression on the degrees, the availabilities and the verdict (on satellite pairs $V_{\Rset}$, the cross-modal truths and $1-\feas_{\AXset}$), inverse regularization in $\{10^{-4},\dots,10^{3}\}$ and, on specification records, the inclusion of the verdict selected on validation AUROC; operating thresholds: maximum of Youden's $J$~\cite{youden1950index} on validation \\ %
Abduction and repair & single-field log changes on a grid from $-3$ to 3 in steps of 0.005; smallest-change rule; feasibility threshold 0.9 ($V\leq-\ln 0.9\approx0.1054$) for both; repair on the $k_r=1$ top-ranked field by an augmented-Lagrangian method with Adam steps \\ %
Optional settings & rank-then-explain: residual MLP alarm at its validation Youden threshold; critical-field alert: undeclared fields required by an applicable axiom and declared by at least 95\% of the training records of the class, or of the engine type for classes with fewer than five training records; MICE imputation: chosen among nine configurations on validation \\ %
\midrule
Photographs & ResNet-50 with ImageNet weights up to stage 3 ($28\times28$ grid at 448~pixels), 25 sigmoid heads, class-activation maps; bounding box enlarged by 10\% per side, $448\times448$ pixels; AdamW, learning rate $3\times10^{-4}$ for heads and localization parameters and $3\times10^{-5}$ for the backbone, weight decay $10^{-4}$, gradient clipping at 5, batch 32, mixed precision; 20 fixed epochs without early stopping, checkpoint at the first epoch with the best validation macro head accuracy \\ %
Grounding & $\lambda_{\AVset}=1$; floor weight 0.1 and mass floor $S_0=1$; $\Loc_k$ with $\ell_k$ initialized at 0.5 and $\kappa_k$ at 10; visibility guard at 0.5 on the 0/1 visibility flag of each part (photographs: detector confidence at least 0.20, as in the pseudo-masks row; satellite crops: the annotated diamond) \\ %
Satellite crops & ResNet-50, ConvNeXt-T and ViT-B/16 (token maps), ImageNet weights, last stage; grounding region: the annotated diamond; aircraft box enlarged by 15\% per side, $224\times224$ pixels; batch 64; at most 20 epochs, early stopping on the validation macro head AUROC with patience 5; other settings as for photographs \\ %
Pseudo-masks & OWLv2 (base, patch 16, ensemble), boxes with score at least 0.15, up to 4 for engines and landing gear and 1 otherwise; SAM ViT-B on the kept boxes; visibility: detector confidence at least 0.20; dilation: maximum filter of radius 2 native pixels \\ %
Cross-modal axioms & photographs: AX2 Mach cap 0.62 and slope 25; AX3 band edges at $10^{4}$, $10^{5.05}$ and $10^{5.5}$~kg, slope 16 \\ %
\midrule
Baseline detectors & perceptron with two hidden layers of 64 units on the normalized fields and the mask (Adam, learning rate $10^{-3}$, weight decay $10^{-5}$, batch 256, at most 200 epochs, patience 20); Isolation Forest with 200 trees and Mahalanobis distance with Ledoit-Wolf covariance, both fitted on the clean training records \\ %
\bottomrule
\end{tabularx}
\end{table}

\providecommand{\twoci}[4]{\ensuremath{[#1,\penalty 200\ #2]\cdot\allowbreak[#3,\penalty 200\ #4]}}%
\providecommand{\oneci}[2]{\ensuremath{[#1,\penalty 200\ #2]}}%
\section{Additional Results and Figures}
\label{app:results}

\paragraph{The evaluation at a glance.}
\fwname{} is tested against 14 groups of comparators (Table~\ref{tab:comparators}), from knowledge-free detectors, residual learners and tabular foundation models to solvers, constraint losses and prompted language and vision-language models, on six benchmarks, four of them labeled independently of the generator oracle (Table~\ref{M-tab:evalmap}). Specification records carry two label families, generator and intervention labels, and are evaluated on three splits (test, unseen families and the unseen class), satellite pairs on two (test and unseen sites); each stratum of \fwname{} is trained with eight seeds (seeds 0 to 7, Table~\ref{tab:config}) and the satellite backbone grid with three seeds per backbone. Unless the text or a table note names another scheme, every interval on specification records rests on 5,000 cluster-bootstrap resamples at each of the two cluster levels (seed 20261002); the removals of Table~\ref{tab:loo} use 2,000, and the curve over the number of constraints in Appendix~\ref{app:results-ablation} resamples the subsets of each size and the seeds of each drawn subset (5,000 draws). Tables~\ref{tab:contrasts} to~\ref{tab:contrasts-4} print 168 contrasts with both intervals, and Table~\ref{tab:ablation} prints 41, of which 36 are not in Table~\ref{tab:contrasts-3}. %
Every primary diagnosis contrast of the largest-violation key keeps its reading in all 20 alternative bootstrap streams, except two on intervention labels, which keep it in 19 and 17 of the 20. So do 11 ablation contrasts, the unseen-class contrasts with half the fields hidden, also under each of 10 further bootstrap seeds of that analysis, and every primary localization and repair contrast except two on the record repair error. %

This appendix gives the complete levels behind Section~\ref{M-sec:res}, the audit tasks first: constraint verdicts, diagnosis, localization and repair, and missing evidence, then record coherence, records with images, grounding, falsification and the ablations; it ends with the selected components and the sensitivity analyses, including the generator oracles refitted without test or out-of-distribution aircraft (Appendix~\ref{app:results-refine}), and with the training checkpoints (Appendix~\ref{app:results-curves}).
Tables~\ref{tab:contrasts} to~\ref{tab:contrasts-4} give, for the comparisons of Sections~\ref{M-sec:res-diag} to~\ref{M-sec:res-missing} and~\ref{M-sec:res-falsif}, the difference with its intervals, written $[a,b]\cdot[c,d]$, [aircraft]~$\cdot$~[family] on specification records; those of Section~\ref{M-sec:res-grounding} are in Appendices~\ref{app:results-grounding} and~\ref{app:expsup}. %

\begin{table}[tbp]
  \centering
  \caption{Contrasts behind the body results: constraint diagnosis.}
  \label{tab:contrasts}
  \fontfamily{lmr}\fontsize{7.05}{8.0}\selectfont%
  \setlength{\tabcolsep}{0pt}%
  \ifdefined\tcciwidth\else\newlength{\tcciwidth}\fi
  \settowidth{\tcciwidth}{[\textminus 0.000,\,\textminus 0.000]\textperiodcentered[\textminus 0.000,\,\textminus 0.000]}%
  \begin{minipage}[t]{\dimexpr0.5\linewidth-4pt\relax}
  \begin{tabularx}{\linewidth}[t]{@{}>{\raggedright\arraybackslash}X@{\hspace{2pt}}r@{\hspace{2pt}}>{\raggedright\arraybackslash}p{\tcciwidth}@{}}
    \toprule
    Comparator & \multicolumn{2}{@{}l}{Difference and intervals} \\
    \midrule
    \multicolumn{3}{@{}>{\raggedright\arraybackslash}p{\linewidth}@{}}{\itshape Constraint diagnosis (top-1); \fwname{} returns the violated constraint from its own degrees, by the largest-violation key}\\
    \multicolumn{3}{@{}>{\raggedright\arraybackslash}p{\linewidth}@{}}{\hangindent=2pt\hangafter=0 Test}\\
    \hangindent=4pt\hangafter=0 \fwname{}'s second head & +0.002 & [\textminus 0.010,\,+0.018]\textperiodcentered\allowbreak[\textminus 0.013,\,+0.011]\\ %
    \hangindent=4pt\hangafter=0 Fellegi-Holt & +0.147 & [+0.068,\,+0.234]\textperiodcentered\allowbreak[+0.079,\,+0.298]\\ %
    \hangindent=4pt\hangafter=0 Z3 & +0.117 & [+0.040,\,+0.209]\textperiodcentered\allowbreak[+0.045,\,+0.256]\\ %
    \hangindent=4pt\hangafter=0 Rule diagnosis & +0.085 & [+0.027,\,+0.146]\textperiodcentered\allowbreak[+0.023,\,+0.150]\\ %
    \hangindent=4pt\hangafter=0 LightGBM-SHAP & +0.071 & [+0.025,\,+0.118]\textperiodcentered\allowbreak[+0.003,\,+0.098]\\ %
    \hangindent=4pt\hangafter=0 DL2 method & +0.499 & [+0.455,\,+0.539]\textperiodcentered\allowbreak[+0.428,\,+0.538]\\ %
    \hangindent=4pt\hangafter=0 Qwen3-8B\,\raisebox{2pt}{a} & +0.814 & [+0.759,\,+0.868]\textperiodcentered\allowbreak[+0.742,\,+0.875]\\ %
    \hangindent=4pt\hangafter=0 Class-standardized residual & +0.013 & [\textminus 0.020,\,+0.045]\textperiodcentered\allowbreak[\textminus 0.019,\,+0.041]\\ %
    \hangindent=4pt\hangafter=0 MLP-SHAP & +0.001 & [\textminus 0.043,\,+0.043]\textperiodcentered\allowbreak[\textminus 0.050,\,+0.028]\\ %
    \multicolumn{3}{@{}>{\raggedright\arraybackslash}p{\linewidth}@{}}{\hangindent=2pt\hangafter=0 Test, 81 records with one violated constraint}\\ %
    \hangindent=4pt\hangafter=0 Fellegi-Holt & +0.284 & [+0.186,\,+0.387]\textperiodcentered\allowbreak[+0.193,\,+0.476]\\ %
    \hangindent=4pt\hangafter=0 Z3 & +0.241 & [+0.152,\,+0.335]\textperiodcentered\allowbreak[+0.162,\,+0.398]\\ %
    \hangindent=4pt\hangafter=0 Rule diagnosis & +0.228 & [+0.132,\,+0.329]\textperiodcentered\allowbreak[+0.139,\,+0.391]\\ %
    \hangindent=4pt\hangafter=0 LightGBM-SHAP & +0.144 & [+0.070,\,+0.220]\textperiodcentered\allowbreak[+0.041,\,+0.195]\\ %
    \hangindent=4pt\hangafter=0 Class-standardized residual & +0.062 & [+0.016,\,+0.105]\textperiodcentered\allowbreak[+0.025,\,+0.104]\\ %
    \hangindent=4pt\hangafter=0 MLP-SHAP & +0.046 & [+0.005,\,+0.103]\textperiodcentered\allowbreak[+0.009,\,+0.074]\\ %
    \multicolumn{3}{@{}>{\raggedright\arraybackslash}p{\linewidth}@{}}{\hangindent=2pt\hangafter=0 Test, 140 records the crisp edits flag}\\ %
    \hangindent=4pt\hangafter=0 Fellegi-Holt & +0.083 & [\textminus 0.019,\,+0.192]\textperiodcentered\allowbreak[\textminus 0.016,\,+0.291]\\ %
    \hangindent=4pt\hangafter=0 Z3 & +0.050 & [\textminus 0.045,\,+0.162]\textperiodcentered\allowbreak[\textminus 0.041,\,+0.239]\\ %
    \hangindent=4pt\hangafter=0 Rule diagnosis & +0.014 & [\textminus 0.057,\,+0.087]\textperiodcentered\allowbreak[\textminus 0.053,\,+0.129]\\ %
    \multicolumn{3}{@{}>{\raggedright\arraybackslash}p{\linewidth}@{}}{\hangindent=2pt\hangafter=0 Unseen families}\\
    \hangindent=4pt\hangafter=0 DL2 method & +0.595 & [+0.544,\,+0.643]\textperiodcentered\allowbreak[+0.569,\,+0.656]\\ %
    \hangindent=4pt\hangafter=0 Qwen3-8B & +0.755 & [+0.691,\,+0.827]\textperiodcentered\allowbreak[+0.675,\,+0.817]\\ %
    \hangindent=4pt\hangafter=0 Class-standardized residual & +0.018 & [\textminus 0.036,\,+0.073]\textperiodcentered\allowbreak[\textminus 0.044,\,+0.100]\\ %
    \hangindent=4pt\hangafter=0 Rule diagnosis & +0.110 & [+0.024,\,+0.190]\textperiodcentered\allowbreak[0.000,\,+0.200]\\ %
    \multicolumn{3}{@{}>{\raggedright\arraybackslash}p{\linewidth}@{}}{\hangindent=2pt\hangafter=0 Unseen class}\\
    \hangindent=4pt\hangafter=0 \fwname{}'s second head & +0.119 & [+0.036,\,+0.219]\textperiodcentered\allowbreak[+0.037,\,+0.221]\\ %
    \hangindent=4pt\hangafter=0 Fellegi-Holt & +0.228 & [+0.152,\,+0.302]\textperiodcentered\allowbreak[+0.154,\,+0.310]\\ %
    \hangindent=4pt\hangafter=0 Z3 & +0.119 & [+0.083,\,+0.155]\textperiodcentered\allowbreak[+0.085,\,+0.156]\\ %
    \hangindent=4pt\hangafter=0 Rule diagnosis & +0.082 & [+0.046,\,+0.121]\textperiodcentered\allowbreak[+0.047,\,+0.124]\\ %
    \hangindent=4pt\hangafter=0 Class-standardized residual & +0.231 & [+0.151,\,+0.310]\textperiodcentered\allowbreak[+0.156,\,+0.303]\\ %
    \hangindent=4pt\hangafter=0 DL2 method & +0.599 & [+0.529,\,+0.672]\textperiodcentered\allowbreak[+0.529,\,+0.674]\\ %
    \hangindent=4pt\hangafter=0 Qwen3-8B & +0.734 & [+0.647,\,+0.816]\textperiodcentered\allowbreak[+0.646,\,+0.815]\\ %
    \hangindent=4pt\hangafter=0 LightGBM-SHAP & \textminus 0.032 & [\textminus 0.081,\,+0.017]\textperiodcentered\allowbreak[\textminus 0.081,\,+0.016]\\ %
    \hangindent=4pt\hangafter=0 MLP-SHAP & \textminus 0.073 & [\textminus 0.125,\,\textminus 0.023]\textperiodcentered\allowbreak[\textminus 0.124,\,\textminus 0.025]\\ %
    \multicolumn{3}{@{}>{\raggedright\arraybackslash}p{\linewidth}@{}}{\hangindent=2pt\hangafter=0 Unseen class, \fwname{}'s second head (its secondary key): the comparisons that the two keys read differently, a reading in plain words for either counting as different}\\ %
    \hangindent=4pt\hangafter=0 Z3 & \textminus 0.000 & [\textminus 0.085,\,+0.075]\textperiodcentered\allowbreak[\textminus 0.084,\,+0.072]\\ %
    \hangindent=4pt\hangafter=0 Rule diagnosis & \textminus 0.037 & [\textminus 0.143,\,+0.054]\textperiodcentered\allowbreak[\textminus 0.143,\,+0.051]\\ %
    \hangindent=4pt\hangafter=0 LightGBM-SHAP & \textminus 0.152 & [\textminus 0.283,\,\textminus 0.040]\textperiodcentered\allowbreak[\textminus 0.285,\,\textminus 0.043]\\ %
    \bottomrule
  \end{tabularx}
  \end{minipage}\hfill
  \begin{minipage}[t]{\dimexpr0.5\linewidth-4pt\relax}
  \begin{tabularx}{\linewidth}[t]{@{}>{\raggedright\arraybackslash}X@{\hspace{2pt}}r@{\hspace{2pt}}>{\raggedright\arraybackslash}p{\tcciwidth}@{}}
    \toprule
    Comparator & \multicolumn{2}{@{}l}{Difference and intervals} \\
    \midrule
    \multicolumn{3}{@{}>{\raggedright\arraybackslash}p{\linewidth}@{}}{\itshape Diagnosis on intervention labels (top-1), without retraining; \fwname{} returns the violated constraint by the same key}\\
    \multicolumn{3}{@{}>{\raggedright\arraybackslash}p{\linewidth}@{}}{\hangindent=2pt\hangafter=0 Test}\\
    \hangindent=4pt\hangafter=0 Fellegi-Holt & +0.180 & [+0.126,\,+0.243]\textperiodcentered\allowbreak[+0.135,\,+0.283]\\ %
    \hangindent=4pt\hangafter=0 Z3 & +0.146 & [+0.092,\,+0.214]\textperiodcentered\allowbreak[+0.090,\,+0.231]\\ %
    \hangindent=4pt\hangafter=0 Rule diagnosis & +0.132 & [+0.085,\,+0.181]\textperiodcentered\allowbreak[+0.079,\,+0.182]\\ %
    \hangindent=4pt\hangafter=0 MLP-SHAP & +0.074 & [+0.022,\,+0.119]\textperiodcentered\allowbreak[+0.005,\,+0.128]\\ %
    \hangindent=4pt\hangafter=0 Z3, records the crisp edits flag & +0.074 & [+0.001,\,+0.164]\textperiodcentered\allowbreak[+0.004,\,+0.215]\\ %
    \hangindent=4pt\hangafter=0 Rule diagnosis, same records & +0.054 & [\textminus 0.004,\,+0.115]\textperiodcentered\allowbreak[\textminus 0.000,\,+0.156]\\ %
    \hangindent=4pt\hangafter=0 LightGBM-SHAP & +0.058 & [\textminus 0.018,\,+0.129]\textperiodcentered\allowbreak[\textminus 0.066,\,+0.126]\\ %
    \hangindent=4pt\hangafter=0 Class-standardized residual & +0.096 & [+0.043,\,+0.145]\textperiodcentered\allowbreak[\textminus 0.014,\,+0.143]\\ %
    \multicolumn{3}{@{}>{\raggedright\arraybackslash}p{\linewidth}@{}}{\hangindent=2pt\hangafter=0 Test, \fwname{}'s second head (its secondary key): the comparisons that the two keys read differently, a reading in plain words for either counting as different}\\ %
    \hangindent=4pt\hangafter=0 MLP-SHAP & +0.065 & [+0.016,\,+0.108]\textperiodcentered\allowbreak[\textminus 0.002,\,+0.115]\\ %
    \hangindent=4pt\hangafter=0 Class-standardized residual & +0.087 & [+0.037,\,+0.131]\textperiodcentered\allowbreak[\textminus 0.016,\,+0.128]\\ %
    \hangindent=4pt\hangafter=0 Z3, records the crisp edits flag & +0.070 & [\textminus 0.003,\,+0.157]\textperiodcentered\allowbreak[\textminus 0.001,\,+0.209]\\ %
    \hangindent=4pt\hangafter=0 Rule diagnosis, same records & +0.050 & [\textminus 0.008,\,+0.110]\textperiodcentered\allowbreak[\textminus 0.006,\,+0.152]\\ %
    \multicolumn{3}{@{}>{\raggedright\arraybackslash}p{\linewidth}@{}}{\hangindent=2pt\hangafter=0 Unseen families}\\
    \hangindent=4pt\hangafter=0 \fwname{}'s second head & \textminus 0.019 & [\textminus 0.035,\,\textminus 0.006]\textperiodcentered\allowbreak[\textminus 0.046,\,\textminus 0.006]\\ %
    \hangindent=4pt\hangafter=0 MLP-SHAP & +0.101 & [+0.057,\,+0.150]\textperiodcentered\allowbreak[+0.057,\,+0.155]\\ %
    \hangindent=4pt\hangafter=0 LightGBM-SHAP & +0.024 & [\textminus 0.031,\,+0.097]\textperiodcentered\allowbreak[\textminus 0.033,\,+0.136]\\ %
    \hangindent=4pt\hangafter=0 Class-standardized residual & +0.068 & [\textminus 0.0032,\,+0.138]\textperiodcentered\allowbreak[+0.011,\,+0.174]\\ %
    \addlinespace[2pt]
    \multicolumn{3}{@{}>{\raggedright\arraybackslash}p{\linewidth}@{}}{\itshape Rank-then-explain, flagged and correctly diagnosed (5 seed pairs); \fwname{} returns the violated constraint behind the residual MLP's alarm}\\ %
    \multicolumn{3}{@{}>{\raggedright\arraybackslash}p{\linewidth}@{}}{\hangindent=2pt\hangafter=0 Test, against \fwname{}'s verdict with its own alarm and the same key (a declared sensitivity)}\\
    \hangindent=4pt\hangafter=0 Flagged and diagnosed & \textminus 0.028 & [\textminus 0.052,\,\textminus 0.006]\textperiodcentered\allowbreak[\textminus 0.064,\,\textminus 0.013]\\ %
    \hangindent=4pt\hangafter=0 False alarms, 32 controls & \textminus 0.094 & [\textminus 0.2125,\,\textminus 0.0125]\textperiodcentered\allowbreak[\textminus 0.256,\,0.000]\\ %
    \multicolumn{3}{@{}>{\raggedright\arraybackslash}p{\linewidth}@{}}{\hangindent=2pt\hangafter=0 Comparators behind the same alarm}\\
    \hangindent=4pt\hangafter=0 Fellegi-Holt, test & +0.132 & [+0.060,\,+0.211]\textperiodcentered\allowbreak[+0.066,\,+0.274]\\ %
    \hangindent=4pt\hangafter=0 Z3, test & +0.101 & [+0.031,\,+0.190]\textperiodcentered\allowbreak[+0.028,\,+0.231]\\ %
    \hangindent=4pt\hangafter=0 MLP-SHAP, test & \textminus 0.009 & [\textminus 0.047,\,+0.026]\textperiodcentered\allowbreak[\textminus 0.061,\,+0.021]\\ %
    \hangindent=4pt\hangafter=0 MLP-SHAP, unseen class & \textminus 0.041 & [\textminus 0.098,\,+0.013]\textperiodcentered\allowbreak[\textminus 0.098,\,+0.009]\\ %
    \bottomrule
  \end{tabularx}
  \end{minipage}
  \par\smallskip
  \begin{minipage}{\linewidth}\raggedright
    Difference: \fwname{}, in the setting named by the headings above the row, minus the comparator named by the row or by its heading (rank-then-explain minus the comparator in its block); for false alarms, repair errors and evasion, where lower is better, a negative difference favors that setting. Intervals: 95\% cluster-bootstrap percentiles at each of the two cluster levels, [aircraft]\textperiodcentered[family] on specification records, [annotated aircraft]\textperiodcentered[site] on satellite pairs and [image]\textperiodcentered[variant] on photographs. Generator-oracle labels unless intervention labels are named; native NaN: TabPFN v2 with native missing values.
    \par\raisebox{2pt}{a}\,Qwen3-8B: the language model given the constraints in words, on anonymized records.
  \end{minipage}
\end{table}

\begin{table}[tbp]
  \centering
  \caption{Contrasts behind the body results: field localization, repair and missing evidence.}
  \label{tab:contrasts-2}
  \fontfamily{lmr}\fontsize{7.05}{8.0}\selectfont%
  \setlength{\tabcolsep}{0pt}%
  \ifdefined\tcciwidth\else\newlength{\tcciwidth}\fi
  \settowidth{\tcciwidth}{[\textminus 0.000,\,\textminus 0.000]\textperiodcentered[\textminus 0.000,\,\textminus 0.000]}%
  \begin{minipage}[t]{\dimexpr0.5\linewidth-4pt\relax}
  \begin{tabularx}{\linewidth}[t]{@{}>{\raggedright\arraybackslash}X@{\hspace{2pt}}r@{\hspace{2pt}}>{\raggedright\arraybackslash}p{\tcciwidth}@{}}
    \toprule
    Comparator & \multicolumn{2}{@{}l}{Difference and intervals} \\
    \midrule
    \multicolumn{3}{@{}>{\raggedright\arraybackslash}p{\linewidth}@{}}{\itshape Field localization (exact top-1); \fwname{} returns the altered field by abduction, the single-field change that restores consistency}\\
    \multicolumn{3}{@{}>{\raggedright\arraybackslash}p{\linewidth}@{}}{\hangindent=2pt\hangafter=0 Test}\\
    \hangindent=4pt\hangafter=0 Minimum cardinality & +0.375 & [+0.286,\,+0.470]\textperiodcentered\allowbreak[+0.293,\,+0.529]\\ %
    \hangindent=4pt\hangafter=0 Ochiai ranking & +0.339 & [+0.234,\,+0.450]\textperiodcentered\allowbreak[+0.246,\,+0.413]\\ %
    \hangindent=4pt\hangafter=0 Largest class-I residual & +0.145 & [+0.065,\,+0.232]\textperiodcentered\allowbreak[+0.071,\,+0.203]\\ %
    \hangindent=4pt\hangafter=0 Same comparator, \fwname{} with the largest-feasibility rule\,\raisebox{2pt}{a} & +0.095 & [\textminus 0.006,\,+0.192]\textperiodcentered\allowbreak[\textminus 0.020,\,+0.176]\\ %
    \hangindent=4pt\hangafter=0 Oracle tie-break & +0.113 & [\textminus 0.033,\,+0.280]\textperiodcentered\allowbreak[\textminus 0.028,\,+0.454]\\ %
    \multicolumn{3}{@{}>{\raggedright\arraybackslash}p{\linewidth}@{}}{\hangindent=2pt\hangafter=0 Unseen families}\\
    \hangindent=4pt\hangafter=0 Minimum cardinality & +0.359 & [+0.241,\,+0.485]\textperiodcentered\allowbreak[+0.241,\,+0.496]\\ %
    \hangindent=4pt\hangafter=0 Ochiai ranking & +0.179 & [+0.041,\,+0.305]\textperiodcentered\allowbreak[+0.029,\,+0.302]\\ %
    \hangindent=4pt\hangafter=0 Oracle tie-break & +0.065 & [\textminus 0.125,\,+0.274]\textperiodcentered\allowbreak[\textminus 0.089,\,+0.370]\\ %
    \hangindent=4pt\hangafter=0 Largest class-I residual & +0.043 & [\textminus 0.041,\,+0.136]\textperiodcentered\allowbreak[\textminus 0.018,\,+0.141]\\ %
    \multicolumn{3}{@{}>{\raggedright\arraybackslash}p{\linewidth}@{}}{\hangindent=2pt\hangafter=0 Unseen class}\\
    \hangindent=4pt\hangafter=0 Minimum cardinality & +0.360 & [+0.254,\,+0.478]\textperiodcentered\allowbreak[+0.257,\,+0.480]\\ %
    \hangindent=4pt\hangafter=0 Largest class-I residual & +0.526 & [+0.394,\,+0.653]\textperiodcentered\allowbreak[+0.380,\,+0.646]\\ %
    \hangindent=4pt\hangafter=0 Ochiai ranking & +0.344 & [+0.220,\,+0.463]\textperiodcentered\allowbreak[+0.225,\,+0.460]\\ %
    \hangindent=4pt\hangafter=0 Oracle tie-break & +0.163 & [\textminus 0.002,\,+0.335]\textperiodcentered\allowbreak[\textminus 0.003,\,+0.351]\\ %
    \multicolumn{3}{@{}>{\raggedright\arraybackslash}p{\linewidth}@{}}{\hangindent=2pt\hangafter=0 Intervention labels, without retraining, test}\\
    \hangindent=4pt\hangafter=0 Minimum cardinality & +0.287 & [+0.220,\,+0.359]\textperiodcentered\allowbreak[+0.238,\,+0.411]\\ %
    \hangindent=4pt\hangafter=0 Same, records the crisp edits flag & +0.248 & [+0.153,\,+0.338]\textperiodcentered\allowbreak[+0.164,\,+0.406]\\ %
    \hangindent=4pt\hangafter=0 Oracle tie-break & +0.140 & [+0.041,\,+0.248]\textperiodcentered\allowbreak[+0.041,\,+0.361]\\ %
    \hangindent=4pt\hangafter=0 Largest class-I residual & +0.140 & [+0.073,\,+0.204]\textperiodcentered\allowbreak[+0.068,\,+0.200]\\ %
    \addlinespace[2pt]
    \multicolumn{3}{@{}>{\raggedright\arraybackslash}p{\linewidth}@{}}{\itshape Repair (log-space error, test); \fwname{} returns a repair proposal from its own abduction and verdict}\\
    \multicolumn{3}{@{}>{\raggedright\arraybackslash}p{\linewidth}@{}}{\hangindent=2pt\hangafter=0 Error on the altered field}\\
    \hangindent=4pt\hangafter=0 No repair & \textminus 0.233 & [\textminus 0.262,\,\textminus 0.204]\textperiodcentered\allowbreak[\textminus 0.272,\,\textminus 0.214]\\ %
    \hangindent=4pt\hangafter=0 MICE, class-I localization & +0.013 & [\textminus 0.022,\,+0.047]\textperiodcentered\allowbreak[\textminus 0.033,\,+0.044]\\ %
    \hangindent=4pt\hangafter=0 Regression, class-I localization & \textminus 0.027 & [\textminus 0.065,\,+0.008]\textperiodcentered\allowbreak[\textminus 0.089,\,+0.004]\\ %
    \hangindent=4pt\hangafter=0 missForest, class-I localization & \textminus 0.028 & [\textminus 0.068,\,+0.011]\textperiodcentered\allowbreak[\textminus 0.085,\,+0.007]\\ %
    \hangindent=4pt\hangafter=0 PiShield-style ordering, class-I localization & \textminus 0.085 & [\textminus 0.119,\,\textminus 0.049]\textperiodcentered\allowbreak[\textminus 0.117,\,\textminus 0.049]\\ %
    \multicolumn{3}{@{}>{\raggedright\arraybackslash}p{\linewidth}@{}}{\hangindent=2pt\hangafter=0 Error on the whole record, class-I localization}\\
    \hangindent=4pt\hangafter=0 Nearest neighbor & \textminus 0.106 & [\textminus 0.171,\,\textminus 0.046]\textperiodcentered\allowbreak[\textminus 0.210,\,\textminus 0.050]\\ %
    \hangindent=4pt\hangafter=0 Regression & \textminus 0.082 & [\textminus 0.142,\,\textminus 0.018]\textperiodcentered\allowbreak[\textminus 0.186,\,\textminus 0.027]\\ %
    \hangindent=4pt\hangafter=0 MICE & +0.039 & [\textminus 0.015,\,+0.099]\textperiodcentered\allowbreak[\textminus 0.013,\,+0.112]\\ %
    \hangindent=4pt\hangafter=0 missForest & \textminus 0.037 & [\textminus 0.102,\,+0.029]\textperiodcentered\allowbreak[\textminus 0.133,\,+0.021]\\ %
    \hangindent=4pt\hangafter=0 PiShield-style ordering & \textminus 0.063 & [\textminus 0.120,\,+0.001]\textperiodcentered\allowbreak[\textminus 0.128,\,\textminus 0.013]\\ %
    \multicolumn{3}{@{}>{\raggedright\arraybackslash}p{\linewidth}@{}}{\hangindent=2pt\hangafter=0 Error on the whole record, other comparators}\\
    \hangindent=4pt\hangafter=0 No repair & \textminus 0.166 & [\textminus 0.217,\,\textminus 0.107]\textperiodcentered\allowbreak[\textminus 0.215,\,\textminus 0.110]\\ %
    \hangindent=4pt\hangafter=0 Fellegi-Holt & \textminus 0.114 & [\textminus 0.176,\,\textminus 0.042]\textperiodcentered\allowbreak[\textminus 0.163,\,\textminus 0.047]\\ %
    \hangindent=4pt\hangafter=0 MaxSMT & \textminus 0.108 & [\textminus 0.170,\,\textminus 0.035]\textperiodcentered\allowbreak[\textminus 0.175,\,\textminus 0.042]\\ %
    \hangindent=4pt\hangafter=0 MICE, Fellegi-Holt localization & \textminus 0.151 & [\textminus 0.378,\,\textminus 0.001]\textperiodcentered\allowbreak[\textminus 0.497,\,+0.024]\\ %
    \bottomrule
  \end{tabularx}
  \end{minipage}\hfill
  \begin{minipage}[t]{\dimexpr0.5\linewidth-4pt\relax}
  \begin{tabularx}{\linewidth}[t]{@{}>{\raggedright\arraybackslash}X@{\hspace{2pt}}r@{\hspace{2pt}}>{\raggedright\arraybackslash}p{\tcciwidth}@{}}
    \toprule
    Comparator & \multicolumn{2}{@{}l}{Difference and intervals} \\
    \midrule
    \multicolumn{3}{@{}>{\raggedright\arraybackslash}p{\linewidth}@{}}{\itshape Fields hidden at random (AUROC); \fwname{} returns its verdict on the available constraints and lists the unchecked ones}\\
    \multicolumn{3}{@{}>{\raggedright\arraybackslash}p{\linewidth}@{}}{\hangindent=2pt\hangafter=0 Unseen class, half of the fields hidden}\\
    \hangindent=4pt\hangafter=0 Residual MLP & +0.0410 & [+0.0085,\,+0.0738]\textperiodcentered\allowbreak[+0.0074,\,+0.0734]\\ %
    \hangindent=4pt\hangafter=0 TabPFN v2, native NaN & +0.0382 & [+0.0047,\,+0.0736]\textperiodcentered\allowbreak[+0.0060,\,+0.0735]\\ %
    \hangindent=4pt\hangafter=0 LightGBM trained with masks & +0.0113 & [\textminus 0.0366,\,+0.0655]\textperiodcentered\allowbreak[\textminus 0.0387,\,+0.0704]\\ %
    \multicolumn{3}{@{}>{\raggedright\arraybackslash}p{\linewidth}@{}}{\hangindent=2pt\hangafter=0 10\% of the fields hidden}\\ %
    \hangindent=4pt\hangafter=0 MLP trained with masks, test & \textminus 0.0425 & [\textminus 0.0705,\,\textminus 0.0168]\textperiodcentered\allowbreak[\textminus 0.0627,\,\textminus 0.0038]\\ %
    \hangindent=4pt\hangafter=0 Residual MLP, unseen families & \textminus 0.0925 & [\textminus 0.1266,\,\textminus 0.0635]\textperiodcentered\allowbreak[\textminus 0.1320,\,\textminus 0.0507]\\ %
    \addlinespace[2pt]
    \multicolumn{3}{@{}>{\raggedright\arraybackslash}p{\linewidth}@{}}{\itshape Satellite images hidden or from another sensor (AUROC), test\,\raisebox{2pt}{b}; \fwname{} returns the fused verdict, the record verdict bit for bit when the image is hidden}\\
    \multicolumn{3}{@{}>{\raggedright\arraybackslash}p{\linewidth}@{}}{\hangindent=2pt\hangafter=0 30\% of the images hidden}\\ %
    \hangindent=4pt\hangafter=0 CatBoost on knowledge-base truths & \textminus 0.031 & [\textminus 0.043,\,\textminus 0.020]\textperiodcentered\allowbreak[\textminus 0.042,\,\textminus 0.015]\\ %
    \hangindent=4pt\hangafter=0 Read-out trained with modality dropout & \textminus 0.026 & [\textminus 0.036,\,\textminus 0.018]\textperiodcentered\allowbreak[\textminus 0.036,\,\textminus 0.012]\\ %
    \multicolumn{3}{@{}>{\raggedright\arraybackslash}p{\linewidth}@{}}{\hangindent=2pt\hangafter=0 Visual stratum trained on ground photographs, scoring satellite crops}\\
    \hangindent=4pt\hangafter=0 Record verdict alone & \textminus 0.071 & [\textminus 0.094,\,\textminus 0.052]\textperiodcentered\allowbreak[\textminus 0.103,\,\textminus 0.039]\\ %
    \bottomrule
  \end{tabularx}
  \end{minipage}
  \par\smallskip
  \begin{minipage}{\linewidth}\raggedright
    Difference, intervals and notation as in Table~\ref{tab:contrasts}.
    \par\raisebox{2pt}{a}\,Abduction with the largest-feasibility rule, a sensitivity not selected on validation; with the nearest threshold of the selected rule, both intervals of the difference from the class-I residual stay above zero.
    \par\raisebox{2pt}{b}\,The comparators were not retrained or retuned, and those on knowledge-base truths keep the truths of the saturating-envelope variant of the remote-sensing stratum they were fitted on (Appendix~\ref{app:results-xmodal}).
  \end{minipage}
\end{table}

\begin{table}[tbp]
  \centering
  \caption{Contrasts behind the body results: record coherence.}
  \label{tab:contrasts-3}
  \fontfamily{lmr}\fontsize{7.05}{8.0}\selectfont%
  \setlength{\tabcolsep}{0pt}%
  \ifdefined\tcciwidth\else\newlength{\tcciwidth}\fi
  \settowidth{\tcciwidth}{[\textminus 0.000,\,\textminus 0.000]\textperiodcentered[\textminus 0.000,\,\textminus 0.000]}%
  \begin{minipage}[t]{\dimexpr0.5\linewidth-4pt\relax}
  \begin{tabularx}{\linewidth}[t]{@{}>{\raggedright\arraybackslash}X@{\hspace{2pt}}r@{\hspace{2pt}}>{\raggedright\arraybackslash}p{\tcciwidth}@{}}
    \toprule
    Comparator & \multicolumn{2}{@{}l}{Difference and intervals} \\
    \midrule
    \multicolumn{3}{@{}>{\raggedright\arraybackslash}p{\linewidth}@{}}{\itshape Record coherence (AUROC); \fwname{} returns its verdict, each flag with the violated constraint, field and repair}\\
    \multicolumn{3}{@{}>{\raggedright\arraybackslash}p{\linewidth}@{}}{\hangindent=2pt\hangafter=0 Test, learners on the raw fields}\\
    \hangindent=4pt\hangafter=0 CatBoost & +0.187 & [+0.114,\,+0.259]\textperiodcentered\allowbreak[+0.094,\,+0.335]\\ %
    \hangindent=4pt\hangafter=0 TabPFN v2 & +0.197 & [+0.118,\,+0.266]\textperiodcentered\allowbreak[+0.105,\,+0.320]\\ %
    \hangindent=4pt\hangafter=0 XGBoost & +0.205 & [+0.119,\,+0.280]\textperiodcentered\allowbreak[+0.109,\,+0.334]\\ %
    \hangindent=4pt\hangafter=0 LightGBM & +0.231 & [+0.136,\,+0.315]\textperiodcentered\allowbreak[+0.119,\,+0.395]\\ %
    \multicolumn{3}{@{}>{\raggedright\arraybackslash}p{\linewidth}@{}}{\hangindent=2pt\hangafter=0 Test, other detectors}\\
    \hangindent=4pt\hangafter=0 Rules with pooling & +0.075 & [+0.039,\,+0.115]\textperiodcentered\allowbreak[+0.041,\,+0.132]\\ %
    \hangindent=4pt\hangafter=0 Train-fitted rule checker & +0.151 & [+0.101,\,+0.172]\textperiodcentered\allowbreak[+0.083,\,+0.165]\\ %
    \hangindent=4pt\hangafter=0 Mahalanobis & +0.186 & [+0.132,\,+0.241]\textperiodcentered\allowbreak[+0.136,\,+0.269]\\ %
    \hangindent=4pt\hangafter=0 DL2 method & +0.276 & [+0.223,\,+0.322]\textperiodcentered\allowbreak[+0.224,\,+0.353]\\ %
    \hangindent=4pt\hangafter=0 Qwen3-8B & +0.359 & [+0.275,\,+0.437]\textperiodcentered\allowbreak[+0.296,\,+0.427]\\ %
    \multicolumn{3}{@{}>{\raggedright\arraybackslash}p{\linewidth}@{}}{\hangindent=2pt\hangafter=0 Test, learners on the residuals}\\
    \hangindent=4pt\hangafter=0 MLP & \textminus 0.010 & [\textminus 0.022,\,+0.002]\textperiodcentered\allowbreak[\textminus 0.017,\,+0.002]\\ %
    \hangindent=4pt\hangafter=0 TabPFN v2 & +0.021 & [\textminus 0.008,\,+0.069]\textperiodcentered\allowbreak[\textminus 0.011,\,+0.037]\\ %
    \hangindent=4pt\hangafter=0 One-class SVM & +0.022 & [+0.0003,\,+0.039]\textperiodcentered\allowbreak[\textminus 0.0001,\,+0.043]\\ %
    \hangindent=4pt\hangafter=0 LightGBM & +0.023 & [+0.0003,\,+0.046]\textperiodcentered\allowbreak[\textminus 0.0004,\,+0.034]\\ %
    \multicolumn{3}{@{}>{\raggedright\arraybackslash}p{\linewidth}@{}}{\hangindent=2pt\hangafter=0 Unseen families}\\
    \hangindent=4pt\hangafter=0 Train-fitted rule checker & +0.090 & [+0.043,\,+0.121]\textperiodcentered\allowbreak[+0.014,\,+0.109]\\ %
    \hangindent=4pt\hangafter=0 DL2 method & +0.296 & [+0.244,\,+0.345]\textperiodcentered\allowbreak[+0.241,\,+0.395]\\ %
    \hangindent=4pt\hangafter=0 Qwen3-8B & +0.254 & [+0.204,\,+0.305]\textperiodcentered\allowbreak[+0.188,\,+0.332]\\ %
    \hangindent=4pt\hangafter=0 MLP & \textminus 0.008 & [\textminus 0.019,\,+0.0002]\textperiodcentered\allowbreak[\textminus 0.017,\,+0.001]\\ %
    \hangindent=4pt\hangafter=0 TabPFN v2, fields and residuals & \textminus 0.009 & [\textminus 0.021,\,+0.0001]\textperiodcentered\allowbreak[\textminus 0.019,\,+0.0002]\\ %
    \multicolumn{3}{@{}>{\raggedright\arraybackslash}p{\linewidth}@{}}{\hangindent=2pt\hangafter=0 Unseen class, other detectors}\\
    \hangindent=4pt\hangafter=0 DL2 method & +0.355 & [+0.312,\,+0.394]\textperiodcentered\allowbreak[+0.326,\,+0.383]\\ %
    \hangindent=4pt\hangafter=0 Qwen3-8B & +0.266 & [+0.181,\,+0.356]\textperiodcentered\allowbreak[+0.176,\,+0.347]\\ %
    \multicolumn{3}{@{}>{\raggedright\arraybackslash}p{\linewidth}@{}}{\hangindent=2pt\hangafter=0 Unseen class, learners on the residuals}\\
    \hangindent=4pt\hangafter=0 MLP & \textminus 0.045 & [\textminus 0.074,\,\textminus 0.005]\textperiodcentered\allowbreak[\textminus 0.075,\,\textminus 0.004]\\ %
    \hangindent=4pt\hangafter=0 TabPFN v2 & \textminus 0.050 & [\textminus 0.080,\,\textminus 0.015]\textperiodcentered\allowbreak[\textminus 0.083,\,\textminus 0.013]\\ %
    \hangindent=4pt\hangafter=0 TabPFN v2, native NaN & \textminus 0.051 & [\textminus 0.075,\,\textminus 0.021]\textperiodcentered\allowbreak[\textminus 0.077,\,\textminus 0.019]\\ %
    \hangindent=4pt\hangafter=0 TabICL & \textminus 0.046 & [\textminus 0.078,\,\textminus 0.011]\textperiodcentered\allowbreak[\textminus 0.080,\,\textminus 0.009]\\ %
    \hangindent=4pt\hangafter=0 TabPFN v2, fields and residuals & \textminus 0.044 & [\textminus 0.075,\,\textminus 0.007]\textperiodcentered\allowbreak[\textminus 0.079,\,\textminus 0.005]\\ %
    \addlinespace[2pt]
    \multicolumn{3}{@{}>{\raggedright\arraybackslash}p{\linewidth}@{}}{\itshape \fwname{} and the comparators retrained on intervention labels (AUROC), test; \fwname{} (0.807) returns its verdict with the per-constraint degrees}\\
    \hangindent=4pt\hangafter=0 Residual MLP & \textminus 0.053 & [\textminus 0.084,\,\textminus 0.028]\textperiodcentered\allowbreak[\textminus 0.105,\,\textminus 0.033]\\ %
    \hangindent=4pt\hangafter=0 TabPFN v2, native NaN\,\raisebox{2pt}{a} & \textminus 0.078 & [\textminus 0.114,\,\textminus 0.047]\textperiodcentered\allowbreak[\textminus 0.119,\,\textminus 0.062]\\ %
    \hangindent=4pt\hangafter=0 TabPFN v2, imputed\,\raisebox{2pt}{a} & \textminus 0.081 & [\textminus 0.121,\,\textminus 0.048]\textperiodcentered\allowbreak[\textminus 0.145,\,\textminus 0.060]\\ %
    \bottomrule
  \end{tabularx}
  \end{minipage}\hfill
  \begin{minipage}[t]{\dimexpr0.5\linewidth-4pt\relax}
  \begin{tabularx}{\linewidth}[t]{@{}>{\raggedright\arraybackslash}X@{\hspace{2pt}}r@{\hspace{2pt}}>{\raggedright\arraybackslash}p{\tcciwidth}@{}}
    \toprule
    Comparator & \multicolumn{2}{@{}l}{Difference and intervals} \\
    \midrule
    \multicolumn{3}{@{}>{\raggedright\arraybackslash}p{\linewidth}@{}}{\itshape False alarms on the 32 controls at a TPR of 0.90, test\,\raisebox{2pt}{b}; each alarm of \fwname{} names the constraint that raised it or, with the alert, the undeclared critical field}\\ %
    \multicolumn{3}{@{}>{\raggedright\arraybackslash}p{\linewidth}@{}}{\hangindent=2pt\hangafter=0 \fwname{}'s verdict}\\
    \hangindent=4pt\hangafter=0 Mahalanobis & \textminus 0.570 & [\textminus 0.844,\,\textminus 0.384]\textperiodcentered\allowbreak[\textminus 0.792,\,\textminus 0.385]\\ %
    \hangindent=4pt\hangafter=0 Rules with pooling & \textminus 0.244 & [\textminus 0.500,\,\textminus 0.031]\textperiodcentered\allowbreak[\textminus 0.620,\,\textminus 0.056]\\ %
    \hangindent=4pt\hangafter=0 Z3 edit count & \textminus 0.244 & [\textminus 0.491,\,\textminus 0.062]\textperiodcentered\allowbreak[\textminus 0.544,\,\textminus 0.093]\\ %
    \hangindent=4pt\hangafter=0 Fellegi-Holt & \textminus 0.244 & [\textminus 0.482,\,\textminus 0.008]\textperiodcentered\allowbreak[\textminus 0.569,\,\textminus 0.050]\\ %
    \hangindent=4pt\hangafter=0 Residual MLP & +0.078 & [0.000,\,+0.188]\textperiodcentered\allowbreak[0.000,\,+0.113]\\ %
    \multicolumn{3}{@{}>{\raggedright\arraybackslash}p{\linewidth}@{}}{\hangindent=2pt\hangafter=0 \fwname{} with the critical-field alert\,\raisebox{2pt}{c}}\\
    \hangindent=4pt\hangafter=0 \fwname{}'s verdict & +0.062 & [0.000,\,+0.188]\textperiodcentered\allowbreak[0.000,\,+0.250]\\ %
    \hangindent=4pt\hangafter=0 Mahalanobis & \textminus 0.508 & [\textminus 0.812,\,\textminus 0.306]\textperiodcentered\allowbreak[\textminus 0.767,\,\textminus 0.264]\\ %
    \hangindent=4pt\hangafter=0 Z3 edit count & \textminus 0.181 & [\textminus 0.444,\,\textminus 0.006]\textperiodcentered\allowbreak[\textminus 0.492,\,\textminus 0.009]\\ %
    \hangindent=4pt\hangafter=0 Rules with pooling & \textminus 0.181 & [\textminus 0.444,\,+0.062]\textperiodcentered\allowbreak[\textminus 0.545,\,+0.111]\\ %
    \hangindent=4pt\hangafter=0 Fellegi-Holt & \textminus 0.181 & [\textminus 0.437,\,+0.033]\textperiodcentered\allowbreak[\textminus 0.486,\,+0.008]\\ %
    \hangindent=4pt\hangafter=0 Residual MLP & +0.141 & [0.000,\,+0.301]\textperiodcentered\allowbreak[0.000,\,+0.278]\\ %
    \addlinespace[2pt]
    \multicolumn{3}{@{}>{\raggedright\arraybackslash}p{\linewidth}@{}}{\itshape Satellite pairs (AUROC)\,\raisebox{2pt}{d}; \fwname{} returns the fused verdict and names the cross-modal axiom the image contradicts}\\
    \multicolumn{3}{@{}>{\raggedright\arraybackslash}p{\linewidth}@{}}{\hangindent=2pt\hangafter=0 \fwname{}'s fused verdict, ResNet-50}\\
    \hangindent=4pt\hangafter=0 AutoGluon-MM, test & +0.165 & [+0.137,\,+0.194]\textperiodcentered\allowbreak[+0.132,\,+0.197]\\ %
    \hangindent=4pt\hangafter=0 AutoGluon-MM, unseen sites & +0.110 & [+0.090,\,+0.132]\textperiodcentered\allowbreak[+0.098,\,+0.144]\\ %
    \hangindent=4pt\hangafter=0 TIP-style, test & +0.067 & [+0.046,\,+0.086]\textperiodcentered\allowbreak[+0.055,\,+0.081]\\ %
    \hangindent=4pt\hangafter=0 TIP-style, unseen sites & +0.048 & [+0.032,\,+0.064]\textperiodcentered\allowbreak[+0.035,\,+0.086]\\ %
    \multicolumn{3}{@{}>{\raggedright\arraybackslash}p{\linewidth}@{}}{\hangindent=2pt\hangafter=0 \fwname{}'s multimodal verifier and Qwen3-VL-8B given the rules}\\
    \hangindent=4pt\hangafter=0 ConvNeXt-T, test & +0.247 & [+0.208,\,+0.286]\textperiodcentered\allowbreak[+0.201,\,+0.283]\\ %
    \hangindent=4pt\hangafter=0 ConvNeXt-T, unseen sites & +0.324 & [+0.287,\,+0.362]\textperiodcentered\allowbreak[+0.200,\,+0.356]\\ %
    \hangindent=4pt\hangafter=0 ResNet-50, test & +0.241 & [+0.202,\,+0.280]\textperiodcentered\allowbreak[+0.196,\,+0.279]\\ %
    \hangindent=4pt\hangafter=0 ViT-B/16, test & +0.242 & [+0.204,\,+0.281]\textperiodcentered\allowbreak[+0.197,\,+0.277]\\ %
    \multicolumn{3}{@{}>{\raggedright\arraybackslash}p{\linewidth}@{}}{\hangindent=2pt\hangafter=0 \fwname{}'s fused verdict, test image swaps}\\
    \hangindent=4pt\hangafter=0 Late fusion trained with swaps & \textminus 0.012 & [\textminus 0.026,\,\textminus 0.0033]\textperiodcentered\allowbreak[\textminus 0.018,\,\textminus 0.004]\\ %
    \multicolumn{3}{@{}>{\raggedright\arraybackslash}p{\linewidth}@{}}{\hangindent=2pt\hangafter=0 Record verdict alone and Qwen3-VL-8B given the rules}\\
    \hangindent=4pt\hangafter=0 Test & +0.048 & [\textminus 0.009,\,+0.103]\textperiodcentered\allowbreak[\textminus 0.009,\,+0.094]\\ %
    \addlinespace[2pt]
    \multicolumn{3}{@{}>{\raggedright\arraybackslash}p{\linewidth}@{}}{\itshape Photographs (AUROC), cross-modal stratum; \fwname{} returns the contradicted axiom with the part map of its evidence}\\
    \hangindent=4pt\hangafter=0 Qwen3-VL-8B, test & +0.025 & [\textminus 0.010,\,+0.058]\textperiodcentered\allowbreak[\textminus 0.046,\,+0.098]\\ %
    \hangindent=4pt\hangafter=0 InternVL3.5-8B, test & +0.032 & [\textminus 0.002,\,+0.067]\textperiodcentered\allowbreak[\textminus 0.033,\,+0.086]\\ %
    \hangindent=4pt\hangafter=0 Qwen3-VL-32B, test & \textminus 0.020 & [\textminus 0.055,\,+0.016]\textperiodcentered\allowbreak[\textminus 0.073,\,+0.042]\\ %
    \hangindent=4pt\hangafter=0 Qwen3-VL-32B, unseen families & \textminus 0.094 & [\textminus 0.131,\,\textminus 0.058]\textperiodcentered\allowbreak[\textminus 0.154,\,\textminus 0.056]\\ %
    \bottomrule
  \end{tabularx}
  \end{minipage}
  \par\smallskip
  \begin{minipage}{\linewidth}\raggedright
    Difference, intervals and notation as in Table~\ref{tab:contrasts}; Table~\ref{tab:ablation} gives the ablation and the same residuals in other models.
    \par\raisebox{2pt}{a}\,Its preprocessing statistics include the scored rows.
    \par\raisebox{2pt}{b}\,Ties at the threshold interpolated; counting every tie, \fwname{} with the critical-field alert differs by \textminus 0.852 from the pooled rules [\textminus 0.992,\,+0.062]\textperiodcentered[\textminus 0.955,\,+0.111], Fellegi-Holt [\textminus 0.992,\,+0.031]\textperiodcentered[\textminus 0.955,\,0.000] and the Z3 count [\textminus 0.992,\,\textminus 0.039]\textperiodcentered[\textminus 0.955,\,\textminus 0.067].
    \par\raisebox{2pt}{c}\,\fwname{}'s optional critical-field alert: a record is flagged when the verdict or the alert fires (Table~\ref{tab:spec}, note g).
    \par\raisebox{2pt}{d}\,The comparators were not retrained or retuned, and those on knowledge-base truths keep the truths of the saturating-envelope variant of the remote-sensing stratum they were fitted on (Appendix~\ref{app:results-xmodal}).
  \end{minipage}
\end{table}

\begin{table}[tbp]
  \centering
  \caption{Contrasts behind the body results: deliberate falsification.}
  \label{tab:contrasts-4}
  \fontfamily{lmr}\fontsize{7.05}{8.0}\selectfont%
  \setlength{\tabcolsep}{0pt}%
  \ifdefined\tcciwidth\else\newlength{\tcciwidth}\fi
  \settowidth{\tcciwidth}{[\textminus 0.000,\,\textminus 0.000]\textperiodcentered[\textminus 0.000,\,\textminus 0.000]}%
  \begin{minipage}[t]{\dimexpr0.5\linewidth-4pt\relax}
  \begin{tabularx}{\linewidth}[t]{@{}>{\raggedright\arraybackslash}X@{\hspace{2pt}}r@{\hspace{2pt}}>{\raggedright\arraybackslash}p{\tcciwidth}@{}}
    \toprule
    Comparator & \multicolumn{2}{@{}l}{Difference and intervals} \\
    \midrule
    \multicolumn{3}{@{}>{\raggedright\arraybackslash}p{\linewidth}@{}}{\itshape Black-box attacks (evasion); \fwname{} returns its verdict, each flag with the field found by abduction and a repair proposal}\\
    \multicolumn{3}{@{}>{\raggedright\arraybackslash}p{\linewidth}@{}}{\hangindent=2pt\hangafter=0 Cover-ups, up to five further edits}\\
    \hangindent=4pt\hangafter=0 Residual MLP & \textminus 0.064 & [\textminus 0.073,\,\textminus 0.055]\textperiodcentered\allowbreak[\textminus 0.075,\,\textminus 0.053]\\ %
    \hangindent=4pt\hangafter=0 Residual TabPFN v2 & \textminus 0.059 & [\textminus 0.071,\,\textminus 0.046]\textperiodcentered\allowbreak[\textminus 0.072,\,\textminus 0.043]\\ %
    \hangindent=4pt\hangafter=0 Residual one-class SVM & \textminus 0.062 & [\textminus 0.077,\,\textminus 0.048]\textperiodcentered\allowbreak[\textminus 0.081,\,\textminus 0.043]\\ %
    \hangindent=4pt\hangafter=0 Raw-field LightGBM & \textminus 0.065 & [\textminus 0.075,\,\textminus 0.055]\textperiodcentered\allowbreak[\textminus 0.076,\,\textminus 0.055]\\ %
    \hangindent=4pt\hangafter=0 Isolation Forest & \textminus 0.049 & [\textminus 0.064,\,\textminus 0.035]\textperiodcentered\allowbreak[\textminus 0.064,\,\textminus 0.028]\\ %
    \hangindent=4pt\hangafter=0 Rules with pooling & +0.004 & [\textminus 0.004,\,+0.013]\textperiodcentered\allowbreak[\textminus 0.005,\,+0.013]\\ %
    \hangindent=4pt\hangafter=0 Mahalanobis & +0.010 & [\textminus 0.0033,\,+0.025]\textperiodcentered\allowbreak[\textminus 0.005,\,+0.028]\\ %
    \multicolumn{3}{@{}>{\raggedright\arraybackslash}p{\linewidth}@{}}{\hangindent=2pt\hangafter=0 Bare lies}\\
    \hangindent=4pt\hangafter=0 Rules with pooling & \textminus 0.113 & [\textminus 0.143,\,\textminus 0.084]\textperiodcentered\allowbreak[\textminus 0.149,\,\textminus 0.064]\\ %
    \hangindent=4pt\hangafter=0 Raw-field LightGBM & \textminus 0.169 & [\textminus 0.208,\,\textminus 0.130]\textperiodcentered\allowbreak[\textminus 0.205,\,\textminus 0.121]\\ %
    \hangindent=4pt\hangafter=0 Residual MLP & \textminus 0.016 & [\textminus 0.052,\,+0.023]\textperiodcentered\allowbreak[\textminus 0.064,\,+0.052]\\ %
    \hangindent=4pt\hangafter=0 Residual LightGBM & +0.052 & [+0.009,\,+0.096]\textperiodcentered\allowbreak[+0.005,\,+0.115]\\ %
    \hangindent=4pt\hangafter=0 Residual TabPFN v2 & +0.105 & [+0.058,\,+0.155]\textperiodcentered\allowbreak[+0.056,\,+0.176]\\ %
    \bottomrule
  \end{tabularx}
  \end{minipage}\hfill
  \begin{minipage}[t]{\dimexpr0.5\linewidth-4pt\relax}
  \begin{tabularx}{\linewidth}[t]{@{}>{\raggedright\arraybackslash}X@{\hspace{2pt}}r@{\hspace{2pt}}>{\raggedright\arraybackslash}p{\tcciwidth}@{}}
    \toprule
    Comparator & \multicolumn{2}{@{}l}{Difference and intervals} \\
    \midrule
    \multicolumn{3}{@{}>{\raggedright\arraybackslash}p{\linewidth}@{}}{\itshape Omissions of up to two fields (evasion); \fwname{} returns its verdict and lists each undeclared field with the constraints it leaves unchecked}\\
    \hangindent=4pt\hangafter=0 Residual MLP & +0.026 & [+0.014,\,+0.039]\textperiodcentered\allowbreak[+0.012,\,+0.044]\\ %
    \hangindent=4pt\hangafter=0 Residual one-class SVM & +0.143 & [+0.103,\,+0.184]\textperiodcentered\allowbreak[+0.103,\,+0.200]\\ %
    \hangindent=4pt\hangafter=0 Mahalanobis & +0.119 & [+0.071,\,+0.178]\textperiodcentered\allowbreak[+0.069,\,+0.196]\\ %
    \addlinespace[2pt]
    \multicolumn{3}{@{}>{\raggedright\arraybackslash}p{\linewidth}@{}}{\itshape Omissions of up to two fields, attacker aware of the alert, \fwname{} with the critical-field alert\,\raisebox{2pt}{a}; each flag names the undeclared critical field or the violated constraint}\\
    \multicolumn{3}{@{}>{\raggedright\arraybackslash}p{\linewidth}@{}}{\hangindent=2pt\hangafter=0 Comparators with the same alert, all splits}\\
    \hangindent=4pt\hangafter=0 Residual MLP & \textminus 0.027 & [\textminus 0.064,\,+0.014]\textperiodcentered\allowbreak[\textminus 0.073,\,+0.046]\\ %
    \hangindent=4pt\hangafter=0 Residual one-class SVM & +0.019 & [\textminus 0.035,\,+0.073]\textperiodcentered\allowbreak[\textminus 0.058,\,+0.104]\\ %
    \multicolumn{3}{@{}>{\raggedright\arraybackslash}p{\linewidth}@{}}{\hangindent=2pt\hangafter=0 Residual MLP with the same alert, per split}\\
    \hangindent=4pt\hangafter=0 Test & \textminus 0.077 & [\textminus 0.117,\,\textminus 0.029]\textperiodcentered\allowbreak[\textminus 0.119,\,+0.016]\\ %
    \hangindent=4pt\hangafter=0 Unseen families & \textminus 0.074 & [\textminus 0.103,\,\textminus 0.043]\textperiodcentered\allowbreak[\textminus 0.099,\,\textminus 0.008]\\ %
    \hangindent=4pt\hangafter=0 Unseen class & +0.125 & [+0.060,\,+0.197]\textperiodcentered\allowbreak[+0.060,\,+0.197]\\ %
    \bottomrule
  \end{tabularx}
  \end{minipage}
  \par\smallskip
  \begin{minipage}{\linewidth}\raggedright
    Difference, intervals and notation as in Table~\ref{tab:contrasts}.
    \par\raisebox{2pt}{a}\,\fwname{}'s optional critical-field alert: a record is flagged when the verdict or the alert fires (Table~\ref{tab:spec}, note g).
  \end{minipage}
\end{table}

\subsection{Constraint Verdicts}
\label{app:results-verdicts}

\fwname{}'s constraint verdicts reach a macro-F1 of 0.795 on generator-label test and, without retraining, 0.575 on intervention-label test (Table~\ref{tab:verdicts}), from a rule fixed by the semantics (a degree above $\ln 2$) with no tuned parameter; comparator thresholds were chosen on validation rows only. %
Against the network trained with the DL2 loss on the same residuals and per-constraint targets, their macro-F1 is higher in all six split and label-family cells, with every interval above zero: on test by $+0.134$ \twoci{+0.086}{+0.242}{+0.099}{+0.264} with generator and $+0.149$ \twoci{+0.090}{+0.268}{+0.095}{+0.276} with intervention labels, on unseen families by $+0.243$ and $+0.250$ and on the unseen class by $+0.164$ and $+0.121$. In all six cells they also find every labeled violation of a corrupted record more often (on generator-label test 0.934 against 0.522, $+0.413$ \twoci{+0.367}{+0.460}{+0.371}{+0.510}) and reach a higher micro recall (0.959 against 0.635 on generator-label test). Their micro precision is on par on test and lower on the unseen class (0.458 against 0.577 on generator and 0.517 against 0.639 on intervention labels), and their exact-set agreement is higher on unseen families (0.721 against 0.508 and 0.373 against 0.221) and on par on test and the unseen class. Over the 8 seeds, \fwname{}'s generator-label test macro-F1 spans 0.767 to 0.808 (sd 0.013) and its exact-set agreement 0.446 to 0.549, the network's macro-F1 0.442 to 0.838; the intervals average the seeds. The network reports fewer false violations on clean-labeled and control records on test and on the unseen class (0.020 against 0.074 on generator-label test, 0.041 against 0.125 on the unseen class), and the crisp edits are on par in macro-F1 on generator-label test (exact-set agreement, macro-F1 and false-violation rate 0.569, 0.829 and 0.036 against 0.504, 0.795 and 0.074; macro-F1 $-0.034$ \twoci{-0.114}{+0.075}{-0.132}{+0.064}) and below the constraint verdicts in macro-F1 on intervention-label unseen families (0.569 against 0.694, $+0.125$ \twoci{+0.053}{+0.187}{+0.035}{+0.186}); beyond the verdicts, \fwname{} also returns the altered field and a repair from the same degrees. %

\begin{table}[tbp]
  \centering
  \caption{Constraint verdicts against per-constraint labels on test.}
  \label{tab:verdicts}
  \scriptsize
  \setlength{\tabcolsep}{4pt}
  \begin{threeparttable}
  \begin{tabular}{@{}>{\raggedright\arraybackslash}p{8.2cm}cc@{}}
    \toprule
    Reading & Generator-oracle labels & Intervention labels \\
    \midrule
    \multicolumn{3}{@{}l}{\textit{Differentiable logic: a probability per constraint}} \\
    Network with the DL2 loss on the same residuals & 0.474/0.661/0.020 & 0.222/0.426/0.019 \\
    \quad same network, thresholds tuned on validation & 0.436/0.746/0.056 & 0.218/0.541/0.056 \\
    DL2 method on the declared fields & 0.106/0.425/0.140 & 0.048/0.259/0.141 \\
    \quad same method, thresholds tuned on validation & 0.053/0.484/0.216 & 0.015/0.319/0.214 \\
    \multicolumn{3}{@{}l}{\textit{Crisp rules and edits: a crisp violated set}} \\
    Crisp edits (Fellegi-Holt, Z3), small classes pooled & 0.569/0.829/0.036 & 0.212/0.505/0.036 \\
    Rule checker, per-class bands, small classes not pooled & 0.103/0.691/0.243 & 0.049/0.494/0.252 \\
    \multicolumn{3}{@{}l}{\textit{Residual statistics: a z-score per residual}} \\
    Class-standardized residual, threshold on specification validation & 0.446/0.714/0.069 & 0.194/0.480/0.078 \\
    \quad same residual, threshold on intervention validation & - & 0.138/0.560/0.174 \\
    \multicolumn{3}{@{}l}{\textit{Knowledge in words: a judgment with the constraints it lists}} \\
    Language model given the constraints in words & 0.113/0.153/0.140 & - \\
    \quad same model, read only when it calls a record incoherent & 0.235/0.065/0.000 & - \\
    \midrule
    \textbf{\fwname{}}, graded degree and held or violated verdict per constraint & \textbf{0.504/0.795/0.074} & \textbf{0.283/0.575/0.084} \\
    \bottomrule
  \end{tabular}
  \begin{tablenotes}[flushleft]\scriptsize
    \item[] Test split. Generator-oracle labels follow the axioms' relations and were the per-constraint training targets; intervention labels, set from the edit alone, are independent of them. Each cell: exact-set agreement / macro-F1 / false-violation rate (lower is better); - not measured. Intervention labels mark every available constraint whose scope contains a field moved beyond tolerance; their false-violation rate is computed on unaltered records and their noisy-but-coherent copies; on records edited only within tolerance, which these labels also count as unaltered, it is 0.119 for \fwname{}, 0.063 for the DL2-loss network and 0.087 for the crisp edits. The ablation trained with the binary record label only, without per-constraint targets, reaches 0.081/0.634/0.194 with generator-oracle labels (Section~\ref{M-sec:res-knowledge}). %
  \end{tablenotes}
  \end{threeparttable}
\end{table}

\begin{table}[tbp]
  \centering
  \caption{Diagnosis, localization and repair.}
  \label{tab:diag}
  \scriptsize%
  \begin{threeparttable}
  \begin{tabular*}{\linewidth}{@{\extracolsep{\fill}}lccc@{}}
    \toprule
    \textit{Constraint diagnosis, top-1 on corrupted records} & Test & Unseen families & Unseen class \\
    \midrule
    MLP-SHAP on residuals & 0.942 & 0.895 & 0.970 \\ %
    Class-standardized residual & 0.929 & 0.955 & 0.667 \\ %
    LightGBM-SHAP on residuals & 0.871 & 0.856 & 0.930 \\ %
    Rule diagnosis, rules ranked by exceedance & 0.857 & 0.863 & 0.815 \\ %
    Z3, first unsatisfiable core & 0.825 & 0.799 & 0.778 \\ %
    Fellegi-Holt, complete violated set & 0.795 & 0.799 & 0.670 \\ %
    DL2 method, declared fields & 0.443 & 0.377 & 0.298 \\ %
    Language model, constraints in words & 0.128 & 0.218 & 0.163 \\ %
    Chance & 0.177 & 0.166 & 0.152 \\ %
    \fwname{}, second head (secondary key) & $0.940 \pm 0.005$ & 0.973 & $0.778 \pm 0.064$ \\ %
    \textbf{\fwname{} (ours), largest-violation key} & \textbf{0.942} & \textbf{0.973} & $\mathbf{0.897 \pm 0.007}$ \\ %
    \midrule
    \textit{Field localization, top-1 on single-field corruptions} & Test & Unseen families & Unseen class \\
    \midrule
    Minimum cardinality, oracle tie-break & 0.629 & 0.670 & 0.481 \\ %
    Minimum cardinality, random tie-break & 0.367 & 0.377 & 0.284 \\ %
    Minimum cardinality, worst tie-break & 0.210 & 0.205 & 0.156 \\ %
    Largest class-I residual & 0.597 & 0.693 & 0.119 \\ %
    Class-standardized residual with Ochiai & 0.403 & 0.557 & 0.300 \\ %
    Chance & 0.077 & 0.076 & 0.079 \\ %
    \textbf{\fwname{} (ours), abduction} & $\mathbf{0.742 \pm 0.011}$ & $\mathbf{0.736 \pm 0.010}$ & $\mathbf{0.644 \pm 0.014}$ \\ %
    \midrule
    \textit{Repair on test records, lower is better} & Altered field & Record & Clean records changed \\
    \midrule
    MICE, class-I localization & 0.120 & 0.161 & 0.917 \\ %
    MICE, class-I localization and detection gate & 0.155 & 0.162 & 0.271 \\ %
    Class-I regression, class-I localization & 0.160 & 0.282 & 1.000 \\ %
    missForest, class-I localization & 0.161 & 0.238 & 1.000 \\ %
    MICE, Fellegi-Holt localization & 0.211 & 0.352 & 0.312 \\ %
    Class-I PiShield-style ordering & 0.218 & 0.264 & 0.250 \\ %
    Fellegi-Holt, all optimal solutions & 0.241 & 0.314 & 0.250 \\ %
    No repair & 0.366 & 0.366 & 0.000 \\ %
    \textbf{\fwname{} (ours), repair proposal} & $\mathbf{0.133 \pm 0.002}$ & $\mathbf{0.200 \pm 0.002}$ & $\mathbf{0.562 \pm 0.033}$ \\ %
    \midrule
    \textit{Flagged and correctly diagnosed, corrupted records} & Test & Unseen families & Unseen class \\
    \midrule
    \multicolumn{4}{@{}l}{Residual-MLP alarm, explained by (with \fwname{}'s key: its optional rank-then-explain setting)} \\
    \hspace{1em}MLP-SHAP & 0.910 & 0.858 & 0.890 \\ %
    \hspace{1em}Class-standardized residual & 0.895 & 0.878 & 0.661 \\ %
    \hspace{1em}Rule diagnosis, rules ranked by exceedance & 0.832 & 0.806 & 0.791 \\ %
    \hspace{1em}Z3 & 0.800 & 0.743 & 0.759 \\ %
    \hspace{1em}Fellegi-Holt & 0.770 & 0.738 & 0.648 \\ %
    \hspace{1em}\fwname{}, second head (secondary key) & $0.900 \pm 0.011$ & $0.875 \pm 0.015$ & $0.732 \pm 0.053$ \\ %
    \hspace{1em}\textbf{\fwname{} (ours), largest-violation key} & $\mathbf{0.901 \pm 0.013}$ & $\mathbf{0.875 \pm 0.015}$ & $\mathbf{0.849 \pm 0.043}$ \\ %
    \addlinespace[3pt]
    \multicolumn{4}{@{}l}{\fwname{}'s own alarm, explained by its largest-violation key} \\
    \hspace{1em}\fwname{} with the critical-field alert & 0.929 & $0.925 \pm 0.016$ & $0.871 \pm 0.007$ \\ %
    \hspace{1em}\textbf{\fwname{}'s verdict (ours)} & \textbf{0.929} & $\mathbf{0.925 \pm 0.016}$ & $\mathbf{0.864 \pm 0.008}$ \\ %
    \bottomrule
  \end{tabular*}
  \begin{tablenotes}\scriptsize
    \item Generator-oracle labels, train prior; comparators as in Table~\ref{tab:comparators}; top-1 rates count tied candidates at their expected value under random tie-breaking. Rows of \fwname{}: mean $\pm$ sd over 8 seeds (in the last block 5 seed pairs, the seeds 0 to 4 for which MLP-SHAP was computed), without $\pm$ where every seed agrees; comparator rows: mean over their seeds. Every row of \fwname{} reads the same degrees, which also give the held or violated verdict per constraint and the report. %
    \item Largest-violation key: the available constraint with the largest violation degree, selected on held-out training classes. Second head: the monotone read-out on the degrees, kept as the secondary key. SHAP: grouped exact Shapley values of the residual model, one group per constraint. Oracle and worst tie-break: ties among the optimal localizations broken toward and away from the altered field. Repair errors, on single-field corruptions: log-space error of the altered field, plus the change of the correct fields for the record; clean records changed: share of the 48 clean and control test records. Class-I PiShield-style ordering: sequential projection onto the crisp constraints, fields in ascending order of their class-I residual. Detection gate: imputation only on the records that the class-I residual flags, a residual outside its 95\% prediction interval. Residual-MLP alarm and own alarm: the residual MLP or the verdict at its validation threshold; with the alert, the verdict or \fwname{}'s optional critical-field alert. %
  \end{tablenotes}
  \end{threeparttable}
\end{table}

\subsection{Diagnosis, Localization and Repair}
\label{app:results-diag}

From the same degrees, \fwname{}'s largest-violation key names the violated constraint in 0.942 of corrupted test records and in $0.897 \pm 0.007$ on the unseen class, abduction finds the altered field in $0.742 \pm 0.011$ of single-field test corruptions, and the repair proposal lowers the log-space error on the altered field from 0.366 to 0.133 (Table~\ref{tab:diag}, the generator-label levels behind Figure~\ref{M-fig:diagnosis}b and c; Table~\ref{tab:verdicts} gives the test levels of panel a). %
On unseen families the key, 0.973 on each seed, is above Fellegi-Holt ($+0.174$ \twoci{+0.051}{+0.304}{+0.074}{+0.365}), Z3 ($+0.174$ \twoci{+0.047}{+0.311}{+0.101}{+0.286}), MLP-SHAP ($+0.077$ \twoci{+0.020}{+0.143}{+0.018}{+0.189}) and LightGBM-SHAP ($+0.117$ \twoci{+0.052}{+0.184}{+0.055}{+0.163}) and on par with the class-standardized residual ($+0.018$); these unseen-family contrasts were not checked across the 20 alternative bootstrap streams. %
Class-I localization names the field of the largest residual of the Class I statistical sizing regressions on maximum take-off mass fitted on the training aircraft (the Class I sizing-regression residual of Section~\ref{M-sec:res-local}). %
MLP-SHAP, the attribution of a second model on the residuals, ranks the violated constraint first more often on the unseen class (0.970), and MICE at the field of the class-I residual is on par with the repair on both errors (0.120 and 0.161); \fwname{}'s key, abduction and repair come with a held or violated verdict per constraint and the report. %
With 8 seed pairs instead of 5, the optional rank-then-explain setting stays above Fellegi-Holt ($+0.130$) and Z3 ($+0.100$) behind the same alarm on test, on par with the class-standardized residual and below \fwname{}'s verdict with its own alarm and the same key ($-0.028$); strict abstention, which gives no diagnosis when no available degree exceeds $\ln 2$, withdraws the diagnosis of 3 to 5 records per seed, none of them correctly diagnosed. %

\subsection{Missing Evidence}
\label{app:results-missing}

Masked quantification makes every verdict of \fwname{} a function of the available evidence alone: Figure~\ref{M-fig:missing} in Section~\ref{M-sec:res-missing} shows the exact reduction on specification records and satellite pairs and the unseen class with half the fields hidden, where the masked verdict (0.577) is ahead of the residual MLP and TabPFN v2 with native missing values and on par with LightGBM trained with masks (Table~\ref{tab:contrasts-2}). With 10\% of the fields hidden in distribution, the MLP trained with masks ranks test records higher ($-0.0425$) and the residual MLP unseen-family records ($-0.0925$); both read every field, imputed or masked in training, while the masked verdict reads the declared evidence alone and states which constraints remain checkable.
With no field hidden beyond those the records already leave undeclared, the masked verdict is ahead of MICE imputation of every missing field at scoring, by $+0.051$ \twoci{+0.031}{+0.063}{+0.031}{+0.062} on test and $+0.035$ \twoci{+0.011}{+0.062}{+0.012}{+0.064} on unseen families, and ahead of global means at scoring on every split ($+0.147$, $+0.121$ and $+0.164$). With fields hidden on unseen families, MICE at scoring ranks higher, with differences of $-0.055$ \twoci{-0.081}{-0.029}{-0.088}{-0.030}, $-0.103$ \twoci{-0.130}{-0.078}{-0.139}{-0.084} and $-0.113$ \twoci{-0.140}{-0.087}{-0.155}{-0.096} with 10, 20 and 30\% of the fields hidden; MICE fills the values that the masked verdict leaves unread, and the report lists the constraints that remain checkable. With MICE imputation enabled, an optional setting of \fwname{} in place of exact masking (Section~\ref{M-sec:res-missing}), the verdict is on par with the MLP trained with masks given the same imputer, $+0.0089$ \twoci{-0.0039}{+0.0230}{-0.0037}{+0.0349} on test with 10\% of the fields hidden. %

\subsection{Record Coherence}
\label{app:results-detect}

\fwname{}'s record verdict reaches a test AUROC of $0.979 \pm 0.002$ on generator labels, from 0.937 on cruise to 1.000 on geometric corruptions by corruption type, each type scored against the 48 clean-labeled test records (16 clean records and their 32 noisy-but-coherent copies; means of 8 seeds), and each of its flags names the violated constraint, the altered field and a repair proposal. %
Table~\ref{tab:spec} groups the detectors of Section~\ref{M-sec:res-detect} by the knowledge they receive, with what each returns. On test the record verdict is on par with TabPFN v2 and close to the residual MLP, LightGBM and the one-class SVM on the residuals and ahead of the learners on the raw fields, the rule checkers, the Mahalanobis distance, the DL2 method and the language model (Table~\ref{tab:contrasts-3}). On the unseen class five learners on the residuals rank higher in record-level detection (residual MLP 0.978, TabPFN v2 up to 0.984); there, with \fwname{}'s optional rank-then-explain setting, the residual MLP raises the alarm and \fwname{} explains it with the violated constraint, the altered field and a repair, flagging and correctly diagnosing $0.849 \pm 0.043$ of the corrupted unseen-class records (Table~\ref{tab:diag}). %

Against intervention labels the record verdict reaches 0.701, 0.750 and 0.719 on test, unseen families and the unseen class, without retraining, while the generator oracle itself, with its bands on all 107 aircraft, reaches 0.676, 0.709 and 0.714 on the same labels. These labels count as incoherent the 613 altered rows that break no constraint under those bands, and as coherent the edits within tolerance, some of which do break one. With only the clean records and their noisy copies as negatives, the verdict reaches 0.840, 0.847 and 0.811. On the altered rows that break no constraint, scored against all clean rows, it stays below chance (0.394, 0.449 and 0.439), as the comparators that use the physics do on the pooled held-out splits. %

\begin{table}[tbp]
  \centering
  \caption{Detection.}
  \label{tab:spec}\vspace{3pt}
  \scriptsize%
  \renewcommand{\TPTtagStyle}[1]{\smash{\font\TPTmarkfont=ecrm0700 at 7.2pt\TPTmarkfont #1}}%
  \begin{threeparttable}
  \begin{tabular}{@{}lcccc@{}}
    \toprule
    & \multicolumn{3}{c}{AUROC} & \\
    \cmidrule(lr){2-4}
    Detector & Test & Unseen families & Unseen class & FPR, controls \\
    \midrule
    \multicolumn{5}{@{}l}{\textit{Learners on the physical residuals: score; the network with the DL2 loss also a probability per constraint}} \\
    MLP & $0.990 \pm 0.003$ & $0.999 \pm 0.001$ & $0.978 \pm 0.011$ & 0.008 \\ %
    TabPFN v2 & $0.958 \pm 0.006$ & $0.999 \pm 0.001$ & $0.983 \pm 0.006$ & - \\ %
    TabPFN v2, native missing values & $0.957 \pm 0.009$ & $0.999 \pm 0.001$ & $0.984 \pm 0.006$ & - \\ %
    TabPFN v2, raw fields and residuals\tnote{a} & $0.965 \pm 0.007$ & $1.000 \pm 0.000$ & $0.978 \pm 0.006$ & - \\ %
    TabICL & $0.953 \pm 0.003$ & $0.996 \pm 0.001$ & $0.979 \pm 0.005$ & - \\ %
    LightGBM & $0.957 \pm 0.004$ & $0.993 \pm 0.002$ & $0.904 \pm 0.006$ & - \\ %
    XGBoost, CatBoost and logistic regression & 0.893-0.958 & 0.988-0.992 & 0.876-0.930 & - \\ %
    Learners on raw fields and residuals\tnote{a} & 0.897-0.955 & 0.975-0.997 & 0.881-0.905 & - \\ %
    Network with the DL2 loss & $0.965 \pm 0.021$ & - & $0.921 \pm 0.018$ & - \\ %
    One-class SVM, label-free & 0.958 & 0.969 & 0.901 & - \\ %
    Local outlier factor and autoencoder, label-free & 0.893-0.940 & 0.919-0.952 & 0.462-0.917 & - \\ %
    \midrule
    \multicolumn{5}{@{}l}{\textit{Crisp rules on the knowledge base: score and crisp violated set}} \\
    Rules with hierarchical pooling\tnote{c} & 0.904 & 0.933 & 0.924\tnote{b} & 0.330\tnote{d} \\ %
    Z3 violated-edit count\tnote{c} & 0.867 & 0.868 & 0.825 & 0.330\tnote{d} \\ %
    Fellegi-Holt\tnote{c} & 0.856 & 0.867 & 0.757 & 0.330\tnote{d} \\ %
    Rule checker, train-fitted bands & 0.828 & 0.901 & 0.924\tnote{b} & - \\ %
    \midrule
    \multicolumn{5}{@{}l}{\textit{Knowledge in other forms: score with a probability per constraint (DL2 method) or a judgment (language model)}} \\
    DL2 method, declared fields & $0.703 \pm 0.030$ & - & $0.579 \pm 0.027$ & - \\ %
    Language model (Qwen3-8B), constraints in words & 0.620 & 0.737 & 0.667 & - \\ %
    \midrule
    \multicolumn{5}{@{}l}{\textit{Knowledge-free: error score}} \\
    Mahalanobis distance & 0.793 & 0.739 & 0.616 & 0.656 \\ %
    Learners on raw fields\tnote{a} & 0.748-0.792 & 0.834-0.863 & 0.505-0.641 & - \\ %
    MLP on normalized fields & $0.711 \pm 0.040$ & $0.713 \pm 0.037$ & $0.537 \pm 0.021$ & - \\ %
    Isolation Forest & $0.576 \pm 0.022$ & $0.542 \pm 0.016$ & $0.525 \pm 0.028$ & - \\ %
    \midrule
    \multicolumn{5}{@{}l}{\textit{\fwname{}, Logic Tensor Network: score, verdict per constraint, key, altered field, repair and report}} \\
    Saturating-envelope ablation\tnote{e} & $0.821 \pm 0.008$ & $0.883 \pm 0.005$ & $0.898 \pm 0.004$ & 0.363 \\ %
    Textbook prior (sensitivity)\tnote{f} & $0.977 \pm 0.003$ & $0.988 \pm 0.001$ & $0.929 \pm 0.002$ & 0.062 \\ %
    \fwname{}, second head & $0.976 \pm 0.002$ & $0.993 \pm 0.002$ & $0.822 \pm 0.043$ & - \\ %
    \fwname{} with the critical-field alert\tnote{g} & - & - & - & 0.148 \\ %
    \textbf{\fwname{}'s verdict (ours)} & $\mathbf{0.979 \pm 0.002}$ & $\mathbf{0.991 \pm 0.002}$ & $\mathbf{0.934 \pm 0.001}$ & \textbf{0.086} \\ %
    \bottomrule
  \end{tabular}
  \begin{tablenotes}\scriptsize
    \item AUROC on generator-oracle labels, mean $\pm$ sd over seeds where a method has several; a range gives the lowest and highest mean of the detectors of its row. FPR, controls: false-positive rate on the 32 noisy-but-coherent test controls at a true-positive rate of 0.90, mean over seeds. A cell marked - is not reported; in the alert row, the AUROC is the verdict's (note g). %
    \item[a] Raw fields: the declared fields with the availability mask and one-hot categories. The learners on raw fields are LightGBM, XGBoost, CatBoost and TabPFN v2; the learners on raw fields and residuals are LightGBM, XGBoost and CatBoost. %
    \item[b] Band-width exceedances. With log-ratio exceedances the rule checker scores 0.748 and the rules with hierarchical pooling 0.720; validation holds no business jet to choose between the two normalizations. %
    \item[c] Pooling candidates designed after a diagnosis on test; validation supports them independently. Fellegi-Holt (fewest fields to change, ties broken by the smallest change in log space) and the Z3 count (number of violated edits) use the same pooled bands; their scores do not depend on the exceedance normalization of note b. %
    \item[d] Every record that no edit flags scores 0, and these ties straddle the threshold: the printed value interpolates them, and it is 1.000 when every tie counts as a flag; every other row has the same value under both conventions. %
    \item[e] Clamped envelope truths in place of the graded predicates that carry the residual magnitude. %
    \item[f] Handbook envelopes in place of the bands fitted on the training fleet; its class mapping was written after the per-split audit; values of the verdict alone, without the alert. %
    \item[g] \fwname{}'s optional critical-field alert: a record is flagged when the verdict or the alert fires. The alert leaves the score unchanged, so the AUROC columns are the verdict's. Among the genuine test records and their copies, the alert fires only on one aircraft of a training class, on its record and on its two noisy-but-coherent copies; 5 of the 6 false alerts on the 46 genuine held-out records fall on the 19 records of the unseen business-jet class, whose alert pool is the turbofan training records. %
  \end{tablenotes}
  \end{threeparttable}
\end{table}\begin{table}[tbp]
  \centering
  \caption{Ablation study.}
  \label{tab:ablation}
  \fontfamily{lmr}\fontsize{7.05}{8.0}\selectfont%
  \setlength{\tabcolsep}{4.5pt}%
  \ifdefined\ablw\else\newlength{\ablw}\newlength{\ablmk}\newlength{\ablt}\fi
  \settowidth{\ablmk}{\textperiodcentered}%
  \settowidth{\ablw}{[\textminus 0.0000347,\,\textminus 0.0000008]}%
  \settowidth{\ablt}{[\textminus 0.0050,\,+0.0035]}\ifdim\ablt>\ablw\setlength{\ablw}{\ablt}\fi
  \settowidth{\ablt}{[\textminus 0.0064,\,+0.0014]}\ifdim\ablt>\ablw\setlength{\ablw}{\ablt}\fi
  \settowidth{\ablt}{[\textminus 0.0068,\,+0.0002]}\ifdim\ablt>\ablw\setlength{\ablw}{\ablt}\fi
  \settowidth{\ablt}{0.979\quad \textminus 0.0000158\,b}\ifdim\ablt>\ablw\setlength{\ablw}{\ablt}\fi
  \settowidth{\ablt}{0.991\quad \textminus 0.0000032\,b}\ifdim\ablt>\ablw\setlength{\ablw}{\ablt}\fi
  \addtolength{\ablw}{\ablmk}%
  \newcommand*{\ablc}[5]{\makebox[\linewidth][s]{#1\hfill#2\makebox[\ablmk][r]{#3}}\newline #4\makebox[\ablmk][l]{\textperiodcentered}\newline #5\hspace*{\ablmk}}%
  \newcommand*{\ablb}[1]{\makebox[\linewidth][l]{\textbf{#1}}}%
  \newcommand*{\ablx}{\makebox[\linewidth][s]{-\hfill -\hspace*{\ablmk}}}%
  \newcommand*{\ablh}{\makebox[\linewidth][s]{AUROC\hfill Difference\hspace*{\ablmk}}}%
  \begin{threeparttable}
  \begin{tabularx}{\linewidth}{@{}>{\raggedright\arraybackslash}X*{3}{>{\raggedleft\arraybackslash}p{\ablw}}@{}}
    \toprule
     & \multicolumn{1}{c}{Test} & \multicolumn{1}{c}{Unseen families} & \multicolumn{1}{c}{Unseen class} \\
    \cmidrule(lr){2-2}\cmidrule(lr){3-3}\cmidrule(l){4-4}
    Ablation or model & \ablh & \ablh & \ablh \\
    \midrule
    \multicolumn{4}{@{}l}{\itshape Ablations of \fwname{} on specification records, generator-oracle labels, train prior; intervals [aircraft]\textperiodcentered[family]} \\
    Saturating envelopes (bounded envelope truths) & \ablc{0.821}{+0.158}{}{[+0.073,\,+0.223]}{[+0.022,\,+0.196]} & \ablc{0.883}{+0.108}{}{[+0.031,\,+0.163]}{[+0.014,\,+0.123]} & \ablc{0.898}{+0.035}{}{[+0.0031,\,+0.062]}{[+0.005,\,+0.067]} \\ %
    Generalized mean, \textit{p}\,=\,4, instead of the mean degree & \ablc{0.983}{\textminus 0.004}{}{[\textminus 0.020,\,+0.013]}{[\textminus 0.008,\,+0.004]} & \ablc{0.980}{+0.011}{}{[\textminus 0.004,\,+0.041]}{[\textminus 0.004,\,+0.065]} & \ablc{0.940}{\textminus 0.006}{}{[\textminus 0.016,\,+0.006]}{[\textminus 0.015,\,+0.006]} \\ %
    Without in-band grading & \ablc{0.970}{+0.009}{}{[+0.002,\,+0.017]}{[+0.0003,\,+0.015]} & \ablc{0.977}{+0.014}{}{[+0.0012,\,+0.039]}{[+0.0010,\,+0.052]} & \ablc{0.908}{+0.025}{}{[+0.014,\,+0.036]}{[+0.014,\,+0.036]} \\ %
    Flat banks instead of hierarchical banks & \ablc{0.965}{+0.015}{}{[+0.004,\,+0.026]}{[+0.0031,\,+0.023]} & \ablc{0.960}{+0.031}{}{[+0.006,\,+0.065]}{[+0.009,\,+0.090]} & \ablc{0.915}{+0.019}{}{[+0.008,\,+0.030]}{[+0.008,\,+0.031]} \\ %
    Binary label alone, without per-constraint targets & \ablc{0.964}{+0.015}{}{[+0.005,\,+0.026]}{[+0.0028,\,+0.023]} & \ablc{0.974}{+0.017}{}{[+0.0019,\,+0.046]}{[+0.0028,\,+0.060]} & \ablc{0.921}{+0.012}{}{[+0.0024,\,+0.023]}{[+0.0025,\,+0.024]} \\ %
    Textbook prior instead of the train prior & \ablc{0.977}{+0.002}{}{[\textminus 0.004,\,+0.009]}{[\textminus 0.009,\,+0.007]} & \ablc{0.988}{+0.003}{}{[\textminus 0.003,\,+0.016]}{[\textminus 0.0019,\,+0.010]} & \ablc{0.929}{+0.004}{}{[\textminus 0.018,\,+0.027]}{[\textminus 0.018,\,+0.025]} \\ %
    Knowledge base at its initial values, untrained & \ablc{0.907}{+0.073}{}{[+0.045,\,+0.091]}{[+0.035,\,+0.089]} & \ablc{0.930}{+0.061}{}{[+0.015,\,+0.106]}{[+0.009,\,+0.118]} & \ablc{0.845}{+0.088}{}{[+0.058,\,+0.109]}{[+0.056,\,+0.111]} \\ %
    Neural predicates, learned monotone functions & \ablc{0.984}{\textminus 0.005}{}{[\textminus 0.022,\,+0.010]}{[\textminus 0.011,\,+0.0018]} & \ablc{0.991}{+0.0004}{}{[\textminus 0.008,\,+0.012]}{[\textminus 0.0028,\,+0.017]} & \ablc{0.958}{\textminus 0.024}{}{[\textminus 0.043,\,\textminus 0.0009]}{[\textminus 0.044,\,\textminus 0.0018]} \\ %
    \addlinespace[2pt]
    \multicolumn{4}{@{}l}{\itshape Learners on the same physical residuals, generator-oracle labels; intervals [aircraft]\textperiodcentered[family]} \\
    Monotone model on the residual magnitudes, without LTN & \ablc{0.966}{+0.013}{}{[\textminus 0.0019,\,+0.029]}{[\textminus 0.007,\,+0.023]} & \ablc{0.993}{\textminus 0.002}{}{[\textminus 0.013,\,+0.006]}{[\textminus 0.011,\,+0.006]} & \ablc{0.921}{+0.012}{}{[\textminus 0.019,\,+0.042]}{[\textminus 0.023,\,+0.042]} \\ %
    Network with the DL2 loss on the residuals & \ablc{0.965}{+0.014}{}{[+0.0028,\,+0.026]}{[+0.004,\,+0.025]} & \ablx & \ablc{0.921}{+0.013}{}{[\textminus 0.010,\,+0.039]}{[\textminus 0.011,\,+0.039]} \\ %
    TabPFN v2 & \ablc{0.958}{+0.021}{}{[\textminus 0.008,\,+0.069]}{[\textminus 0.011,\,+0.037]} & \ablc{0.999}{\textminus 0.008}{}{[\textminus 0.020,\,+0.0003]}{[\textminus 0.017,\,+0.0002]} & \ablc{0.983}{\textminus 0.050}{}{[\textminus 0.080,\,\textminus 0.015]}{[\textminus 0.083,\,\textminus 0.013]} \\ %
    MLP & \ablc{0.990}{\textminus 0.010}{}{[\textminus 0.022,\,+0.002]}{[\textminus 0.017,\,+0.002]} & \ablc{0.999}{\textminus 0.0076}{}{[\textminus 0.0185,\,+0.0002]}{[\textminus 0.0170,\,+0.0014]} & \ablc{0.978}{\textminus 0.045}{}{[\textminus 0.074,\,\textminus 0.005]}{[\textminus 0.075,\,\textminus 0.004]} \\ %
    \midrule
    \textbf{\fwname{}'s verdict (ours), full model} & \ablb{0.979} & \ablb{0.991} & \ablb{0.934} \\ %
    \midrule
    \multicolumn{4}{@{}l}{\itshape Satellite pairs, ablations of \fwname{}'s fused verdict; intervals [annotated aircraft]\textperiodcentered[site]} \\
     & \multicolumn{1}{c}{Test} & \multicolumn{1}{c}{Unseen sites} & \\
    \cmidrule(lr){2-2}\cmidrule(lr){3-3}
    Single masked quantifier, 30\% of images withheld & \ablc{-}{\textminus 0.002}{}{[\textminus 0.0064,\,+0.0014]}{[\textminus 0.0050,\,+0.0035]} & \ablc{-}{\textminus 0.008}{}{[\textminus 0.0117,\,\textminus 0.0046]}{[\textminus 0.0086,\,\textminus 0.0045]} &  \\ %
    Single masked quantifier, all images & \ablc{0.995}{\textminus 0.004}{}{[\textminus 0.0097,\,\textminus 0.0002]}{[\textminus 0.0068,\,+0.0002]} & \ablc{0.992}{\textminus 0.013}{}{[\textminus 0.0201,\,\textminus 0.0073]}{[\textminus 0.0147,\,\textminus 0.0058]} &  \\ %
    Kleene-Dienes implication, $\max(1-a,b)$~\cite{vankrieken2022analyzing}, instead of Reichenbach & \ablc{0.991}{\textminus 0.0000032}{}{[\textminus 0.0000137,\,0]}{[\textminus 0.0000118,\,0]} & \ablc{0.979}{\textminus 0.0000158}{}{[\textminus 0.0000347,\,\textminus 0.0000008]}{[\textminus 0.0000338,\,0]} &  \\ %
    \midrule
    \textbf{\fwname{}'s fused verdict (ours), all images} & \ablb{0.991} & \ablb{0.979} &  \\ %
    \bottomrule
  \end{tabularx}
  \begin{tablenotes}\fontfamily{lmr}\fontsize{7.05}{8.0}\selectfont
    \item AUROC: mean over 8 seeds, one model for the knowledge base at its initial values and for the monotone model; on satellite pairs nine runs, three backbones with three seeds each; with 30\% of images withheld, differences only, from the three ResNet-50 runs with 20 random image masks each, masks that differ from those behind the 30\% levels of Section~\ref{M-sec:res-missing}. Difference: \fwname{}'s verdict, on satellite pairs its fused verdict on the same images, minus the row, so a positive difference favors \fwname{}; under it, its 95\% cluster-bootstrap intervals at the two levels of the block heading, one per line; - not reported. Values near 0 are printed with more decimals, seven in the Kleene-Dienes cells. %
    \item The ablations with \textit{p}\,=\,4, without in-band grading and with flat banks are trained models, and the textbook prior is a sensitivity. The knowledge base at its initial values starts from the train prior and is a single deterministic model. Other detectors are in Table~\ref{tab:spec}. %
  \end{tablenotes}
  \end{threeparttable}
\end{table}

\subsection{Records with Images}
\label{app:results-xmodal}

Without an image, \fwname{}'s fused verdict equals the record verdict bit for bit in 9 of 9 satellite cells, against 0 of 207 learned method and cell rows, and its second head with ConvNeXt-T (8 seeds) reaches $0.998 \pm 0.001$ on test and $0.991 \pm 0.003$ on unseen sites over all pair types. %
On the satellite pairs of Section~\ref{M-sec:res-xmodal}, the best learned comparators rank the pairs higher after learning from labeled pairs: CatBoost on knowledge-base truths on test swaps, TabPFN on visual posteriors and image embeddings on unseen-site swaps, and late fusions on knowledge-base truths over all pair types on unseen sites (nine runs of three backbones). Learned heads also rank the pairs higher in the contrasts of Tables~\ref{tab:contrasts-2} and~\ref{tab:contrasts-3}, where \fwname{}'s fused verdict differs by $-0.012$ from the late fusion trained with swaps on test swaps and by $-0.031$ from CatBoost on knowledge-base truths with 30\% of the images hidden. The fused verdict, instead, aggregates \fwname{}'s strata without labels and names the contradicted axiom; on photographs, Qwen3-VL-32B ranks unseen-family pairs higher than the cross-modal stratum ($-0.094$). %
With 30\% of satellite images hidden, the fused verdict (0.947) keeps the exact reduction; the read-out trained with modality dropout ranks the pairs at 0.973 without keeping it. When a visual stratum trained on ground photographs scores satellite crops, the record verdict alone (0.796) is the fallback across sensors without labels: the fused verdict reaches 0.724 and an out-of-domain abstention restores the record level, while the fusion selected with satellite labels reaches 0.786 and a refit trained on satellite labels, a reference with privileged access to them, reaches 0.979. %
Fitted on specification records and applied without satellite training through the shared axioms R2 to R4 (Appendix~\ref{app:data-benchmarks}), the knowledge base reaches $0.767 \pm 0.007$ on satellite test annotations and $0.764 \pm 0.008$ on unseen sites, against 0.796 and 0.812 for the remote-sensing record verdict trained on these annotations; fitted on satellite annotations and applied to specification records, it reaches $0.541 \pm 0.003$ (8 seeds, no interval); the military roles, absent from the records, keep their prior bands. %
Comparators on satellite pairs were not retrained or retuned, and those that read knowledge-base truths (CatBoost and TabPFN on knowledge-base truths, the best fusion and the fusion selected without satellite labels) keep the truths of the saturating-envelope variant of the remote-sensing stratum they were fitted on (the specification-record row of this ablation is in Table~\ref{tab:ablation}). \fwname{}'s graded remote-sensing stratum was selected on validation among eight candidate configurations; it trains with per-constraint targets that the end-to-end, fusion and vision-language comparators do not receive, and its validation split serves both early stopping and the choice of configuration. %

\subsection{Grounding}
\label{app:results-grounding}

On the FGVC pseudo-masks the grounding axiom raises test pointing over the network without a localization term on every backbone, by $+0.664$ with ConvNeXt-T, $+0.699$ with ResNet-50 at 512 pixels, $+0.615$ with ViT token-CAM and $+0.677$ with ResNet-50 at 448 pixels (8 against 3 seeds), each with an exact seed-bootstrap interval above zero. Figure~\ref{fig:partgrounding} gives its pointing levels against the controls (panel a, over its 8 seeds on FGVC-Aircraft and its 3 runs at weight 1 on PASCAL-Part; the location prior, a reference without learning, always points to the grid cell that the part covers most often in other images: the training images for FGVC-Aircraft and the other half of the images for PASCAL-Part) and its comparison with the localization and explanation-supervision methods (panel b, 3 runs per method at the weights selected on validation, where the axiom points to the part in 0.790 of test pseudo-masks). At the same loss weight the auxiliary loss puts about 4.3 times the axiom's localization weight on each conjunct (Section~\ref{M-sec:meth-vision}), so the two are compared at their weights selected on validation (Appendix~\ref{app:expsup}). %

\begin{widefigure}[tbp]
  \centering
  \includegraphics{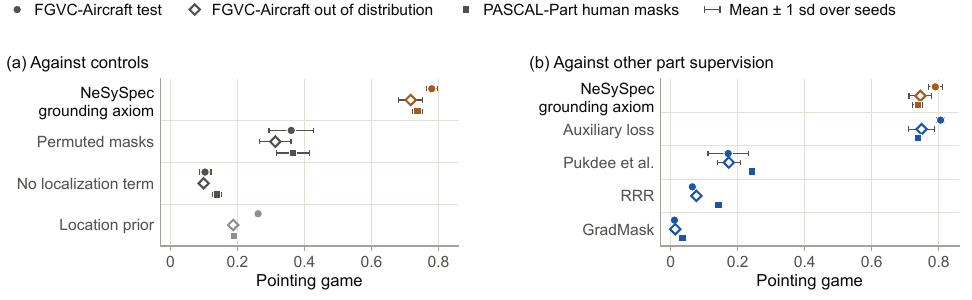}
  \caption{Grounding of visual parts.}
  \label{fig:partgrounding}
\end{widefigure}

On satellite crops (ConvNeXt-T, seed intervals over 8 paired seeds) the axiom keeps attribute and head AUROC on par with the network without a localization term. %

\subsection{Deliberate Falsification}
\label{app:results-indep}

\paragraph{Protocol.}
Each of the 736 falsification problems inflates one capacity of a genuine held-out record, the range, the cruise speed, the maximum take-off mass or the ceiling, by 5, 10, 20 or 40\% (46 records: the 16 clean test records, 11 of unseen families and 19 of the unseen class). A bare lie submits the inflated record; a cover-up also edits up to five further declared fields within the ranges of the fleet, found by a black-box search that reads only the detector's score; an omission withholds up to two declared fields other than the inflated one. A lie evades a detector when the submitted record scores below the detector's validation threshold, and evasion is the share of evading lies among the problems whose genuine record the detector accepts; the smallest number of further edits that covers a flagged lie is censored at six. %

\fwname{} with the critical-field alert, an optional setting, flags a record when the verdict or the alert fires. %
Under black-box falsification, a lie flagged by \fwname{}'s record verdict is covered less often than one flagged by the residual MLP: with up to five further edits 0.936 of cover-ups pass against 1.000 (Table~\ref{tab:contrasts-4}), and 0.247 of its flagged bare lies cannot be covered, against 0.000. %
The smallest cover of one of its flagged lies has 2.69 further edited fields on average (2.68 with the critical-field alert), against 1.01 to 1.63 for the learned detectors (on the residuals, MLP 1.38, LightGBM 1.37 and TabPFN v2 1.63) and more for the Mahalanobis distance and the pooled rules (3.45 and 3.62; point estimates, censored at six). On the 72 problems that the verdict and the residual MLP flag on every seed, 0.413 of the verdict's flagged lies cannot be covered, against 0.000, with 3.76 edits per cover against 1.58. %
On omissions of up to two fields and on bare lies some comparators are evaded less often (Table~\ref{tab:contrasts-4}: on omissions the residual MLP $+0.026$, the residual one-class SVM $+0.143$ and the Mahalanobis distance $+0.119$; on bare lies residual LightGBM $+0.052$ and residual TabPFN v2 $+0.105$). A withheld field masks the axioms that read it (Proposition~\ref{M-prop:reduction}); with the optional critical-field alert, which names the undeclared field and the constraints it leaves unchecked, \fwname{} is evaded less than the residual MLP with the same alert on unseen families ($-0.074$) and more on the unseen class ($+0.125$). %
On this fleet a rule without the knowledge base (an alert on any undeclared field that the record's class, or its engine type for classes with fewer than five training records, declares in at least 95\% of the training records, whether or not an applicable axiom reads it) gives the alert's protection with the same false alerts on the genuine records; \fwname{}'s report also names the constraints that each withheld field leaves unchecked. %

\subsection{Ablation and Other Forms of the Same Knowledge}
\label{app:results-ablation}

Training the knowledge base through Real Logic is ahead of the same knowledge base at its initial values on every split ($+0.073$, $+0.061$ and $+0.088$), and the full model is ahead of the saturating-envelope ablation and of the ablation with flat banks on every split; Table~\ref{tab:ablation} gives each ablation with its difference and both intervals, Section~\ref{M-sec:res-knowledge} states the reading, and the textbook prior, a sensitivity, is on par with the full model on test and the unseen class and close to it on unseen families. %
On satellite pairs, a single masked quantifier in place of the strata ranks unseen-site pairs higher (differences $-0.008$ with 30\% of the images withheld and $-0.013$ with all images) and is close to the fused verdict on test, while \fwname{}'s strata keep one verdict per evidence source, so the report says whether the record or the image raised the flag. %

\paragraph{The knowledge encoded.}
To measure what each constraint contributes, \fwname{} was retrained with subsets of its 12 specification constraints. A removed constraint is unavailable for every record in training and scoring, while the record labels of training stay unchanged. The record verdict was then scored on intervention labels, which no constraint defines, and on generator labels, which the same constraints define, so that a removal hides the corruptions of its constraint by construction (Figure~\ref{M-fig:knowledge} in Section~\ref{M-sec:res-knowledge}, all panels on intervention labels; Table~\ref{tab:loo}). %
Over 96 subsets of 1 to 11 constraints, all 12 of sizes 1 and 11 and 8 random ones of each size from 2 to 10, with 3 seeds each, and the full knowledge base, the intervention-label AUROC averaged over test, unseen families and the unseen class has a mean per size that rises from one size to the next, from 0.535 with one constraint to 0.723 with all twelve (Spearman 0.93 \oneci{0.90}{0.95}), by steps whose intervals include zero at 5 of the 11 steps. The returns diminish: a logarithmic trend fits better than a linear one ($R^2$ 0.87 against 0.84; on the unseen class alone the two fit equally), the gain is about 0.028 per constraint up to six and 0.008 beyond, and the twelfth still adds a small gain (0.009 \oneci{0.003}{0.015}). %
The AUROC grows linearly with the share of edits that alter a field read by some active constraint ($R^2$ 0.88, against 0.74 for a logarithmic trend), a share that goes from 0.22 with one constraint to all edits with twelve, while on the edits of fields the constraints read it stays between 0.70 and 0.75 at every size: the gain follows the coverage of the edits. The residual MLP on the same subsets gains from the same knowledge (0.540 to 0.706), with no trend in its gap to \fwname{}. %
Removing one constraint from the full knowledge base lowers the intervention-label AUROC for A5 and A2, slightly for A4, A7 and A8 and for A4 with A4P together, and leaves it on par for A1, A4P, A9 and A10 and close for A3, A6 and A11; removing A1 lowers it on the unseen class and raises it slightly on test, and no removal raises it averaged over the three splits. Stall speed, engine count and fuselage length, the fields that only A6, A9 and A11 read, are never edited alone in the intervention records, and stall speed is never edited at all; A3, A4P and A10 read only fields that other constraints also read, as do A2, A4, A7 and A8, whose removal still lowers the AUROC, slightly for A4, A7 and A8. On generator labels the removal of A1, A2, A4, A4P, A5 or A8 lowers the AUROC, that of A7 or A10 slightly, and that of A3 raises it slightly (0.001); abduction is lower without A1 (0.187), A5 (0.077), A3, A4, A4P, A8 and A11, slightly lower without A2, A7 and A10, and on par without A6 and A9. In Figure~\ref{M-fig:knowledge}c the darker line gives the aircraft interval and the lighter the family interval, with filled points where both exclude zero. %

\begin{table}[tbp]
  \centering
  \caption{Specification constraints removed from the knowledge base.}
  \label{tab:loo}
  \scriptsize
  \setlength{\tabcolsep}{3pt}
  \begin{threeparttable}
  \begin{tabular}{@{}lc>{\centering\arraybackslash}p{2.35cm}c>{\centering\arraybackslash}p{2.35cm}c>{\centering\arraybackslash}p{2.35cm}@{}}
    \toprule
    & \multicolumn{2}{c}{Detection AUROC, intervention labels} & \multicolumn{2}{c}{Detection AUROC, generator labels} & \multicolumn{2}{c}{Abduction top-1, intervention labels} \\
    \cmidrule(lr){2-3}\cmidrule(lr){4-5}\cmidrule(l){6-7}
    Removed & Drop & Intervals & Drop & Intervals & Drop & Intervals \\
    \midrule
    None (level) & 0.723 & & 0.968 & & 0.542 & \\ %
    A5 & +0.033 & [+0.016,\,+0.053]\newline [+0.017,\,+0.060] & +0.059 & [+0.042,\,+0.079]\newline [+0.043,\,+0.078] & +0.077 & [+0.054,\,+0.100]\newline [+0.049,\,+0.098] \\ %
    A2 & +0.022 & [+0.011,\,+0.034]\newline [+0.003,\,+0.035] & +0.041 & [+0.031,\,+0.053]\newline [+0.032,\,+0.056] & +0.035 & [\textminus 0.000,\,+0.071]\newline [+0.006,\,+0.075] \\ %
    A1 & +0.017 & [\textminus 0.004,\,+0.039]\newline [\textminus 0.006,\,+0.038] & +0.136 & [+0.120,\,+0.150]\newline [+0.119,\,+0.153] & +0.187 & [+0.161,\,+0.214]\newline [+0.164,\,+0.212] \\ %
    A4 and A4P & +0.014 & [+0.003,\,+0.026]\newline [+0.001,\,+0.029] & +0.032 & [+0.021,\,+0.044]\newline [+0.022,\,+0.049] & +0.052 & [+0.024,\,+0.078]\newline [+0.023,\,+0.081] \\ %
    A4 & +0.013 & [+0.004,\,+0.023]\newline [+0.003,\,+0.025] & +0.019 & [+0.011,\,+0.029]\newline [+0.009,\,+0.030] & +0.034 & [+0.015,\,+0.055]\newline [+0.011,\,+0.054] \\ %
    A8 & +0.007 & [+0.000,\,+0.014]\newline [\textminus 0.002,\,+0.015] & +0.022 & [+0.013,\,+0.032]\newline [+0.011,\,+0.030] & +0.030 & [+0.012,\,+0.050]\newline [+0.014,\,+0.052] \\ %
    A10 & +0.003 & [\textminus 0.003,\,+0.008]\newline [\textminus 0.005,\,+0.007] & +0.007 & [+0.002,\,+0.013]\newline [+0.002,\,+0.013] & +0.017 & [+0.001,\,+0.035]\newline [+0.000,\,+0.033] \\ %
    A7 & +0.002 & [+0.001,\,+0.004]\newline [+0.001,\,+0.004] & +0.002 & [+0.001,\,+0.004]\newline [+0.001,\,+0.004] & +0.011 & [+0.003,\,+0.021]\newline [+0.000,\,+0.021] \\ %
    A9 & +0.001 & [\textminus 0.003,\,+0.005]\newline [\textminus 0.004,\,+0.005] & \textminus 0.003 & [\textminus 0.008,\,+0.003]\newline [\textminus 0.008,\,+0.003] & +0.001 & [\textminus 0.006,\,+0.009]\newline [\textminus 0.008,\,+0.008] \\ %
    A6 & +0.000 & [\textminus 0.000,\,+0.001]\newline [\textminus 0.000,\,+0.001] & \textminus 0.000 & [\textminus 0.001,\,+0.000]\newline [\textminus 0.001,\,+0.000] & \textminus 0.000 & [\textminus 0.006,\,+0.005]\newline [\textminus 0.006,\,+0.005] \\ %
    A11 & \textminus 0.000 & [\textminus 0.001,\,+0.001]\newline [\textminus 0.001,\,+0.002] & \textminus 0.001 & [\textminus 0.002,\,+0.001]\newline [\textminus 0.002,\,+0.000] & +0.032 & [+0.021,\,+0.043]\newline [+0.021,\,+0.044] \\ %
    A4P & \textminus 0.001 & [\textminus 0.006,\,+0.005]\newline [\textminus 0.006,\,+0.007] & +0.017 & [+0.010,\,+0.024]\newline [+0.011,\,+0.028] & +0.024 & [+0.005,\,+0.044]\newline [+0.006,\,+0.045] \\ %
    A3 & \textminus 0.001 & [\textminus 0.002,\,+0.001]\newline [\textminus 0.002,\,+0.001] & \textminus 0.001 & [\textminus 0.003,\,\textminus 0.000]\newline [\textminus 0.002,\,\textminus 0.000] & +0.010 & [+0.004,\,+0.017]\newline [+0.004,\,+0.015] \\ %
    \bottomrule
  \end{tabular}
  \begin{tablenotes}[flushleft]\scriptsize
    \item[] Drop: the full knowledge base minus the knowledge base without the constraint, 8 seeds each, paired on the same records; detection averaged over test, unseen families and the unseen class, abduction over the 695 single-field intervention records of the three splits. Intervals: [aircraft] above [family], 95\% cluster-bootstrap percentiles, 2,000 draws. Generator labels are defined by the same constraints, so a removal hides the corruptions of its constraint by construction. %
  \end{tablenotes}
  \end{threeparttable}
\end{table}

\subsection{Selected Components and Sensitivity Analyses}
\label{app:results-refine}

\paragraph{Selected components.}
The diagnosis key was selected on held-out training classes: with each of the six training classes left out in turn, the verifier retrained and every statistic, prior and head recomputed without it, the largest-violation key ranked a violated constraint first in 0.802 of the held-out corrupted records (mean over the six classes), against 0.630 for the second head, with a higher mean on each held-out class (differences from 0.006 to 0.432); on validation records of the seen classes the two keys differ by 0.0008, and the second head is kept as the secondary key. %
\fwname{}'s optional critical-field alert was selected because, on validation, it lowered evasion by an attacker who knows the alert from 0.949 to 0.739 and fired on 2 of the 17 genuine validation records, and because, with each training class held out in turn, it fired on 0.028 of that class's genuine records (mean over six classes; limit 0.25; largest rate 0.091, against 0.263 on the unseen class). %
Validation selected the abduction rule, the smallest-change rule of Table~\ref{tab:config}, and its feasibility threshold 0.9, also chosen in 7 of 8 leave-one-seed-out selections. With threshold 0.8, abduction reaches 0.750 on single-field test corruptions, its readings against minimum-cardinality localization unchanged and still ahead of the Class I residual on test. The largest-feasibility rule of Table~\ref{tab:contrasts-2}, which ranks every field by the lowest verdict it reaches, gives 0.692 on test, 0.747 on unseen families and 0.660 on the unseen class, on par with the Class I residual on test. On intervention labels every reading holds under both alternatives except the margin over the oracle tie-break ($+0.112$ and $+0.095$). %
For a class without training aircraft \fwname{} reads fleet statistics; five alternative reference statistics (engine-type or declared-mass pools with the union or weighted mean of class bands, and the nearest training class), compared on held-out training classes, met no adoption rule, and on the business jets each flags every clean record (1.000 against 0.322 with fleet statistics) with a detection AUROC of 0.836 to 0.906 against 0.934; under the mass pools and the nearest class a withheld field changes the reference class of 16.0, 21.7 and 1.8\% of masked held-out records. %

\paragraph{Sensitivity analyses.}
On generator labels, every band alternative that places more of \fwname{}'s learned bands inside the training fleet range widened by 12\%, $[0.88\min,\,1.12\max]$, lowers validation detection by more than 0.005: the learned bands act as tolerances learned from the fleet, 45.7\% of them with both edges inside that range. The alternatives project the band edges onto that range after every step, or pull them toward it with a penalty of weight 0.01, 0.1 or 1.0 (eight seeds each). Retrained on intervention labels, the four alternatives hold 68.5, 86.7, 96.5 and 100\% of the 46 counted bands per seed (those with a training fleet range) inside the range and stay within the validation margin; the projection reaches 0.804 on test, against 0.807 for \fwname{} retrained on the same labels. The penalty of weight 0.01, the alternative selected on validation, keeps its share at 45.7\% and was not adopted; it reaches on generator-label test a detection AUROC of 0.977, on par with 0.979, the same diagnosis top-1 (0.942), a localization top-1 of 0.728 against 0.742 and a false-positive rate on the controls of 0.082 against 0.086, and on the unseen class a detection AUROC slightly lower, 0.932 against 0.934; with weight 1.0 (53.8\% inside) unseen-class detection is lower, 0.924 against 0.934. These alternatives are sensitivity analyses of \fwname{}. %

\paragraph{Oracle bands fitted on the training aircraft.}
With the class bands of the generator oracle fitted on the 44 clean training aircraft alone (business jets, which have no training aircraft, on the band of all training aircraft), and the generator, its seed and the family splits unchanged, the benchmark holds 1,239 instances instead of 1,312. These bands flag 38 of the 63 clean held-out aircraft, 15 of the 16 in test, mostly on A1, A9 and A11; controls must stay clean under the oracle, so the test controls fall from 32 to 4. \fwname{} was retrained with its selected configuration on 8 seeds, as were the residual MLP and TabPFN~v2 with their full selection protocol and the DL2-loss network with its selected hyperparameters, whose search was not repeated; the crisp edits and minimum-cardinality localization were recomputed, and the rule diagnosis, TabPFN~v2 on intervention labels and the alternative bootstrap streams were not rerun (Tables~\ref{tab:trainoracle} and~\ref{tab:trainoracle-ci}). The validation rule of the crisp edits now selects per-class bands without pooling (validation AUROC 0.919 against 0.915), the bands of the new label oracle itself, with which the crisp edits lead \fwname{} on generator labels in key and macro-F1 on every split; the tables keep the pooled bands of the paper, and on intervention labels \fwname{} is on par with the unpooled crisp edits in macro-F1 on test ($-0.006$) and ahead on the other two splits. %
On intervention labels, whose records are the same under both oracles, \fwname{} stays ahead on every split of the DL2-loss network in macro-F1 ($+0.045$ \twoci{+0.012}{+0.142}{+0.012}{+0.163} on test), of Fellegi-Holt and Z3 in diagnosis ($+0.126$ and $+0.092$ on test) and of minimum-cardinality localization with random tie-break ($+0.226$ on test); on test it is on par with the oracle tie-break ($+0.079$) and with the crisp edits in macro-F1 ($-0.016$), and the DL2-loss network has the higher exact-set agreement ($-0.149$) and fewer false violations. %
\fwname{}'s own test levels there are lower, with more false violations, every paired interval excluding zero: macro-F1 0.488 against 0.575, exact-set agreement 0.070 against 0.283, false-violation rate 0.219 against 0.084, key 0.732 against 0.786 and abduction 0.532 against 0.593, while detection is on par (0.695 against 0.701). On unseen families key, abduction and detection are on par and macro-F1 comparable (0.666 against 0.694), exact-set agreement is lower (0.255 against 0.373) and the false-violation rate slightly higher (0.088 against 0.028); on the unseen class the verdicts, abduction and detection are on par and the key is lower (0.743 against 0.778). Trained on the flags of these bands through its per-constraint targets, \fwname{} learns narrower bands where the oracle's bands narrow. Over the 46 counted bands, the median ratio of its learned widths to those learned under the paper's oracle is 0.903 (seed-bootstrap interval \oneci{0.897}{0.911}), and 64\% of the bands narrow, with a moderate rank correlation with the oracle's width ratios over 45 class bands (Spearman 0.43). The narrowing is largest for the thrust bands of A9 and the passenger-density bands of A11 (median 0.71), which carry the largest rises in false violations on the 16 clean test aircraft ($+0.35$ and $+0.33$; 0.061 to 0.208 over all constraints), followed by A2 ($+0.26$), whose learned bands shift. Seven of its eight seeds reached their best epoch by epoch 22, against 80 to 100 for the paper's seeds; the seed whose best epoch is 93 shows a similar narrowing on the bands of A3 and the thrust bands of A9 (0.78 and 0.74) and a smaller one on the passenger-density bands of A11 (0.78). %
On generator labels the records differ between the two benchmarks: the held-out corrupted records keep the violations their clean origin already has under the narrower bands, and the clean record itself stays labeled clean. There the residual MLP is slightly higher in detection on test (0.950 against 0.910, $-0.040$ \twoci{-0.075}{-0.010}{-0.071}{+0.001}) and on unseen families ($-0.013$), as is TabPFN~v2 on unseen families ($-0.020$), and stays slightly ahead on the unseen class (0.986 against 0.959), where TabPFN~v2 is now on par ($-0.008$). The key reaches 0.944 on test and 0.931 on unseen families (0.973 before), slightly ahead of Fellegi-Holt and Z3 on test ($+0.102$) and on par with them on unseen families ($+0.094$); on the unseen class, where the pooled crisp bands equal those of the label oracle, Fellegi-Holt and Z3 lead ($-0.094$). The verdicts stay ahead of the DL2-loss network in macro-F1 on every split ($+0.288$ on test), are on par with the crisp edits on test and unseen families and below them on the unseen class ($-0.125$), with more false violations than the crisp edits on every split. Abduction is now ahead of the oracle tie-break on test ($+0.199$ \twoci{+0.042}{+0.370}{+0.052}{+0.528}) and stays ahead of random tie-breaking on every split. The repair still lowers the error on the altered field on every split (0.126 against 0.236 without repair on test); as the verifier now rates almost every clean and control test record as violating, its target is reached on 0.036 of single-field test records (0.555 before) and it proposes a change for 0.988 of the clean and control test records (0.588 on unseen families and 0.707 on the unseen class); the repair rule re-selected on the new validation split (feasibility threshold 0.8 instead of 0.9) reaches the target on 0.190 and changes 0.894 of them. Bands fitted on all 107 aircraft cover the range of the fleet the records describe, which bands fitted on 44 aircraft do not. %

\paragraph{Oracle fitted on training and validation aircraft, tolerances chosen out of family.}
Fitted on the 61 clean training and validation aircraft, the oracle had each of its three tolerances set, by a rule declared before any result, to the smallest value of a grid that flags at most 5\% of the clean aircraft when their family is held out in turn: the range margin, the relative widening of each class band ($[0.88\min,\,1.12\max]$ at 0.12), becomes 0.50 instead of 0.12 (2 of 61 aircraft flagged, against 24 at 0.12); the equality factor, which multiplies the largest absolute log-residual of the clean records of each class for the identities A1, A5 and A6, becomes 1.25 instead of 1.15 (3 of 60, at the limit of the rule); and the mass-balance margin of A7, subtracted from the smallest clean mass-balance margin, stays at 0.05; with these values 5 of the 61 aircraft are flagged by some tolerance. This oracle flags 14 of the 46 clean test and out-of-distribution aircraft: 11 of the 16 in test, 10 of them narrowbodies on A1, whose class tolerance rests on 3 aircraft of 2 families, 3 business jets on A9, judged by the band fitted on all 61 aircraft, and none of the unseen families; its wider bands let the generator escalate further (median magnitude 0.8 against 0.4), and the benchmark holds 1,260 instances with 10 test controls. %
Retrained on its labels as above (Tables~\ref{tab:trainoracle} and~\ref{tab:trainoracle-ci}), \fwname{} on intervention labels keeps its false violations at the paper's level on test (0.073 against 0.084, on par), with fewer on the unseen class (0.087 against 0.128), and its exact-set agreement on par on test and unseen families and slightly higher on the unseen class (0.206 against 0.155); its macro-F1 is lower on test (0.537 against 0.575) and unseen families (0.636 against 0.694), its abduction on test (0.491 against 0.593) and unseen families (0.508 against 0.594), and its key on the unseen class (0.735 against 0.778) and slightly on test (0.738 against 0.786), while detection is on par on test and unseen families and close on the unseen class (0.711 against 0.719). %
Against the comparators on intervention labels it stays ahead on every split of Fellegi-Holt and Z3 in diagnosis and of minimum-cardinality localization with random tie-break, ahead of the DL2-loss network in macro-F1 on unseen families, slightly on test ($+0.037$) and on par on the unseen class, with the network keeping fewer false violations on test and the unseen class. In macro-F1 it is on par with the crisp edits on test and unseen families and slightly ahead on the unseen class ($+0.029$), and, against the bands with a 25\% margin, small classes pooled, that the validation rule of the crisp edits now selects (validation AUROC 0.943 against 0.942), on par on test and ahead on the other two splits. In abduction it is on par with the oracle tie-break on every split, and against the Class I residual on par on test, lower on unseen families ($-0.096$) and ahead on the unseen class. %
On generator labels, with larger corruptions, \fwname{} reaches a detection AUROC of 0.976 on test, slightly below the residual MLP (0.994), and on the unseen class slightly below the residual MLP and TabPFN~v2 (0.988 against 0.999 for both); its macro-F1 is on par with the DL2-loss network on test and the unseen class and ahead on unseen families, ahead of the crisp edits on test, slightly ahead on unseen families and below them on the unseen class; its key, 0.955 on test, is ahead of Fellegi-Holt on every split and of Z3 on test and unseen families, and on par with Z3 on the unseen class; abduction is on par with the oracle tie-break on every split and ahead of random tie-breaking; the repair lowers the error on the altered field on every split, reaches its target on 0.570 of single-field test records and changes 0.418 of the clean and control test records; and 109 of the 154 corrupted test records carry a violation their clean origin already has. Under these tolerances the bands widen, the training corruptions grow and training stops earlier (best epochs 15 to 27, against 80 to 100), so their effects on the key and abduction levels are not separated. %

\begin{table}[tbp]
  \centering
  \caption{Generator oracles: levels of \fwname{} and its comparators.}
  \label{tab:trainoracle}
  \scriptsize
  \setlength{\tabcolsep}{4pt}
  \begin{threeparttable}
  \begin{tabular}{@{}>{\raggedright\arraybackslash}p{7.4cm}ccc@{}}
    \toprule
    & \multicolumn{3}{c}{Oracle fitted on} \\
    \cmidrule(l){2-4}
    Reading & 107 aircraft & 44 aircraft & 61 aircraft \\
    \midrule
    \multicolumn{4}{@{}l}{\textit{Benchmark}} \\
    Instances & 1,312 & 1,239 & 1,260 \\ %
    Test controls & 32 & 4 & 10 \\ %
    Clean test aircraft flagged by the oracle, of 16 & 0 & 15 & 11 \\ %
    \multicolumn{4}{@{}l}{\textit{Intervention labels, same records: \fwname{}}} \\
    Constraint verdicts, macro-F1, test & 0.575 & 0.488 & 0.537 \\ %
    \quad unseen families & 0.694 & 0.666 & 0.636 \\ %
    \quad unseen class & 0.431 & 0.444 & 0.400 \\ %
    Constraint verdicts, exact-set, test & 0.283 & 0.070 & 0.223 \\ %
    \quad unseen families & 0.373 & 0.255 & 0.348 \\ %
    \quad unseen class & 0.155 & 0.153 & 0.206 \\ %
    Constraint verdicts, false-violation rate, test & 0.084 & 0.219 & 0.073 \\ %
    \quad unseen families & 0.028 & 0.088 & 0.015 \\ %
    \quad unseen class & 0.128 & 0.120 & 0.087 \\ %
    Largest-violation key, top-1, test & 0.786 & 0.732 & 0.738 \\ %
    \quad unseen families & 0.828 & 0.808 & 0.795 \\ %
    \quad unseen class & 0.778 & 0.743 & 0.735 \\ %
    Abduction, top-1, test & 0.593 & 0.532 & 0.491 \\ %
    \quad unseen families & 0.594 & 0.580 & 0.508 \\ %
    \quad unseen class & 0.465 & 0.453 & 0.441 \\ %
    Detection AUROC, test & 0.701 & 0.695 & 0.696 \\ %
    \quad unseen families & 0.750 & 0.747 & 0.737 \\ %
    \quad unseen class & 0.719 & 0.716 & 0.711 \\ %
    \multicolumn{4}{@{}l}{\textit{Intervention labels, test: \fwname{} minus the comparator}} \\
    DL2-loss network, macro-F1 & +0.149 & +0.045 & +0.037 \\ %
    DL2-loss network, exact-set & +0.061 & \textminus 0.149 & \textminus 0.058 \\ %
    DL2-loss network, false-violation rate & +0.065 & +0.194 & +0.059 \\ %
    Crisp edits, small classes pooled, macro-F1 & +0.071 & \textminus 0.016 & +0.032 \\ %
    Fellegi-Holt, key top-1 & +0.180 & +0.126 & +0.132 \\ %
    Z3, key top-1 & +0.146 & +0.092 & +0.098 \\ %
    Minimum cardinality, random tie-break, abduction & +0.287 & +0.226 & +0.186 \\ %
    Minimum cardinality, oracle tie-break, abduction & +0.140 & +0.079 & +0.038 \\ %
    Class I sizing-regression residual, abduction & +0.140 & +0.079 & +0.038 \\ %
    \multicolumn{4}{@{}l}{\textit{Generator labels, test and unseen class: different records under each oracle}} \\
    Detection AUROC, residual MLP & 0.990 & 0.950 & 0.994 \\ %
    Detection AUROC, TabPFN~v2 on the residuals & 0.958 & 0.932 & 0.983 \\ %
    \textbf{Detection AUROC, \fwname{}} & \textbf{0.979} & \textbf{0.910} & \textbf{0.976} \\ %
    Detection AUROC, unseen class, residual MLP & 0.978 & 0.986 & 0.999 \\ %
    Detection AUROC, unseen class, TabPFN~v2 on the residuals & 0.983 & 0.967 & 0.999 \\ %
    \textbf{Detection AUROC, unseen class, \fwname{}} & \textbf{0.934} & \textbf{0.959} & \textbf{0.988} \\ %
    Constraint verdicts, DL2-loss network & 0.474/0.661/0.020 & 0.145/0.456/0.022 & 0.431/0.760/0.012 \\ %
    Constraint verdicts, crisp edits, small classes pooled & 0.569/0.829/0.036 & 0.421/0.664/0.030 & 0.256/0.662/0.059 \\ %
    \textbf{Constraint verdicts, \fwname{}} & \textbf{0.504/0.795/0.074} & \textbf{0.176/0.745/0.187} & \textbf{0.360/0.773/0.059} \\ %
    Key top-1, chance & 0.177 & 0.373 & 0.243 \\ %
    Key top-1, Fellegi-Holt & 0.795 & 0.842 & 0.733 \\ %
    Key top-1, Z3 & 0.825 & 0.842 & 0.822 \\ %
    \textbf{Key top-1, \fwname{}} & \textbf{0.942} & \textbf{0.944} & \textbf{0.955} \\ %
    Abduction top-1, minimum cardinality, random tie-break & 0.367 & 0.192 & 0.423 \\ %
    Abduction top-1, minimum cardinality, oracle tie-break & 0.629 & 0.421 & 0.590 \\ %
    \textbf{Abduction top-1, \fwname{}} & \textbf{0.742} & \textbf{0.620} & \textbf{0.706} \\ %
    Repair, log error on the altered field, no repair & 0.366 & 0.236 & 0.511 \\ %
    \textbf{Repair, log error on the altered field, \fwname{}} & \textbf{0.133} & \textbf{0.126} & \textbf{0.169} \\ %
    Repair, target reached on single-field records, \fwname{} & 0.555 & 0.036 & 0.570 \\ %
    Repair, clean and control records changed, \fwname{} & 0.562 & 0.988 & 0.418 \\ %
    \bottomrule
  \end{tabular}
  \begin{tablenotes}[flushleft]\scriptsize
    \item[] Oracle fitted on all 107 aircraft (the paper), on the 44 training aircraft with the paper's tolerances, and on the 61 training and validation aircraft with tolerances chosen by leaving one family out (range margin 0.50, equality factor 1.25, mass-balance margin 0.05). Learned methods: means over 8 seeds. Intervention labels, set from the edit alone, are the same records under every oracle; generator labels differ between the three benchmarks. Constraint verdicts on generator labels: exact-set agreement / macro-F1 / false-violation rate. On intervention labels only \fwname{} and the DL2-loss network change; the crisp edits, Fellegi-Holt, Z3, minimum cardinality and the Class I residual are the same under every oracle, and since the oracle tie-break and the Class I residual reach the same intervention-label test level (0.453), their differences from \fwname{} coincide. Intervals of the intervention-label rows in Table~\ref{tab:trainoracle-ci}; those of the differences under the oracle on 107 aircraft in the DL2-loss macro-F1, the key and abduction are in Tables~\ref{tab:contrasts} to~\ref{tab:contrasts-4} and Appendix~\ref{app:results-verdicts}. %
  \end{tablenotes}
  \end{threeparttable}
\end{table}
\begin{table}[tbp]
  \centering
  \caption{Generator oracles fitted without test or out-of-distribution aircraft: paired changes and differences with their intervals.}
  \label{tab:trainoracle-ci}
  \scriptsize
  \setlength{\tabcolsep}{3pt}
  \begin{threeparttable}
  \begin{tabular}{@{}>{\raggedright\arraybackslash}p{4.6cm}cccc@{}}
    \toprule
    & \multicolumn{2}{c}{Oracle on 44 aircraft} & \multicolumn{2}{c}{Oracle on 61 aircraft} \\
    \cmidrule(lr){2-3}\cmidrule(l){4-5}
    Reading & Value & Intervals & Value & Intervals \\
    \midrule
    \multicolumn{5}{@{}l}{\textit{Intervention labels, same records: \fwname{}, change from the oracle on 107 aircraft}} \\
    Constraint verdicts, macro-F1, test & \textminus 0.087 & [\textminus 0.173,\,\textminus 0.022]\textperiodcentered[\textminus 0.146,\,\textminus 0.029] & \textminus 0.039 & [\textminus 0.109,\,\textminus 0.007]\textperiodcentered[\textminus 0.098,\,\textminus 0.0034] \\ %
    \quad unseen families & \textminus 0.028 & [\textminus 0.054,\,+0.017]\textperiodcentered[\textminus 0.061,\,+0.003] & \textminus 0.058 & [\textminus 0.105,\,\textminus 0.022]\textperiodcentered[\textminus 0.116,\,\textminus 0.026] \\ %
    \quad unseen class & +0.013 & [\textminus 0.010,\,+0.041]\textperiodcentered[\textminus 0.011,\,+0.039] & \textminus 0.032 & [\textminus 0.059,\,+0.003]\textperiodcentered[\textminus 0.060,\,+0.002] \\ %
    Constraint verdicts, exact-set, test & \textminus 0.213 & [\textminus 0.309,\,\textminus 0.125]\textperiodcentered[\textminus 0.279,\,\textminus 0.117] & \textminus 0.060 & [\textminus 0.164,\,+0.032]\textperiodcentered[\textminus 0.187,\,+0.043] \\ %
    \quad unseen families & \textminus 0.118 & [\textminus 0.226,\,\textminus 0.005]\textperiodcentered[\textminus 0.188,\,\textminus 0.008] & \textminus 0.026 & [\textminus 0.083,\,+0.027]\textperiodcentered[\textminus 0.108,\,+0.032] \\ %
    \quad unseen class & \textminus 0.002 & [\textminus 0.040,\,+0.044]\textperiodcentered[\textminus 0.041,\,+0.041] & +0.050 & [+0.001,\,+0.109]\textperiodcentered[\textminus 0.007,\,+0.118] \\ %
    Constraint verdicts, false-violation rate, test & +0.135 & [+0.092,\,+0.178]\textperiodcentered[+0.068,\,+0.168] & \textminus 0.011 & [\textminus 0.044,\,+0.025]\textperiodcentered[\textminus 0.049,\,+0.030] \\ %
    \quad unseen families & +0.061 & [+0.011,\,+0.111]\textperiodcentered[\textminus 0.011,\,+0.099] & \textminus 0.013 & [\textminus 0.032,\,+0.002]\textperiodcentered[\textminus 0.042,\,+0.003] \\ %
    \quad unseen class & \textminus 0.007 & [\textminus 0.024,\,+0.010]\textperiodcentered[\textminus 0.024,\,+0.010] & \textminus 0.041 & [\textminus 0.059,\,\textminus 0.023]\textperiodcentered[\textminus 0.061,\,\textminus 0.021] \\ %
    Largest-violation key, top-1, test & \textminus 0.054 & [\textminus 0.094,\,\textminus 0.019]\textperiodcentered[\textminus 0.109,\,\textminus 0.023] & \textminus 0.048 & [\textminus 0.091,\,\textminus 0.007]\textperiodcentered[\textminus 0.101,\,\textminus 0.002] \\ %
    \quad unseen families & \textminus 0.020 & [\textminus 0.059,\,+0.016]\textperiodcentered[\textminus 0.055,\,+0.045] & \textminus 0.033 & [\textminus 0.063,\,+0.004]\textperiodcentered[\textminus 0.060,\,+0.035] \\ %
    \quad unseen class & \textminus 0.035 & [\textminus 0.066,\,\textminus 0.004]\textperiodcentered[\textminus 0.065,\,\textminus 0.009] & \textminus 0.043 & [\textminus 0.072,\,\textminus 0.018]\textperiodcentered[\textminus 0.070,\,\textminus 0.017] \\ %
    Abduction, top-1, test & \textminus 0.061 & [\textminus 0.101,\,\textminus 0.024]\textperiodcentered[\textminus 0.102,\,\textminus 0.030] & \textminus 0.102 & [\textminus 0.151,\,\textminus 0.059]\textperiodcentered[\textminus 0.173,\,\textminus 0.063] \\ %
    \quad unseen families & \textminus 0.014 & [\textminus 0.048,\,+0.021]\textperiodcentered[\textminus 0.036,\,+0.043] & \textminus 0.086 & [\textminus 0.136,\,\textminus 0.043]\textperiodcentered[\textminus 0.129,\,\textminus 0.060] \\ %
    \quad unseen class & \textminus 0.012 & [\textminus 0.045,\,+0.025]\textperiodcentered[\textminus 0.044,\,+0.025] & \textminus 0.024 & [\textminus 0.077,\,+0.033]\textperiodcentered[\textminus 0.077,\,+0.035] \\ %
    Detection AUROC, test & \textminus 0.006 & [\textminus 0.022,\,+0.008]\textperiodcentered[\textminus 0.026,\,+0.007] & \textminus 0.005 & [\textminus 0.022,\,+0.011]\textperiodcentered[\textminus 0.030,\,+0.009] \\ %
    \quad unseen families & \textminus 0.003 & [\textminus 0.031,\,+0.025]\textperiodcentered[\textminus 0.025,\,+0.030] & \textminus 0.013 & [\textminus 0.035,\,+0.005]\textperiodcentered[\textminus 0.043,\,+0.017] \\ %
    \quad unseen class & \textminus 0.003 & [\textminus 0.016,\,+0.005]\textperiodcentered[\textminus 0.015,\,+0.005] & \textminus 0.008 & [\textminus 0.020,\,\textminus 0.000]\textperiodcentered[\textminus 0.020,\,+0.000] \\ %
    \multicolumn{5}{@{}l}{\textit{Intervention labels: \fwname{} minus the comparator under each oracle}} \\
    DL2-loss network, macro-F1, test & +0.045 & [+0.012,\,+0.142]\textperiodcentered[+0.012,\,+0.163] & +0.037 & [+0.002,\,+0.113]\textperiodcentered[+0.010,\,+0.143] \\ %
    \quad unseen families & +0.220 & [+0.184,\,+0.277]\textperiodcentered[+0.209,\,+0.267] & +0.129 & [+0.094,\,+0.168]\textperiodcentered[+0.101,\,+0.172] \\ %
    \quad unseen class & +0.141 & [+0.106,\,+0.185]\textperiodcentered[+0.111,\,+0.182] & +0.000 & [\textminus 0.029,\,+0.040]\textperiodcentered[\textminus 0.028,\,+0.038] \\ %
    DL2-loss network, exact-set, test & \textminus 0.149 & [\textminus 0.203,\,\textminus 0.095]\textperiodcentered[\textminus 0.207,\,\textminus 0.077] & \textminus 0.058 & [\textminus 0.148,\,+0.026]\textperiodcentered[\textminus 0.199,\,+0.050] \\ %
    \quad unseen families & +0.041 & [\textminus 0.065,\,+0.143]\textperiodcentered[\textminus 0.016,\,+0.134] & +0.081 & [+0.023,\,+0.133]\textperiodcentered[+0.033,\,+0.152] \\ %
    \quad unseen class & \textminus 0.093 & [\textminus 0.166,\,\textminus 0.016]\textperiodcentered[\textminus 0.163,\,\textminus 0.013] & \textminus 0.033 & [\textminus 0.101,\,+0.037]\textperiodcentered[\textminus 0.105,\,+0.034] \\ %
    DL2-loss network, false-violation rate, test & +0.194 & [+0.147,\,+0.242]\textperiodcentered[+0.126,\,+0.228] & +0.059 & [+0.022,\,+0.100]\textperiodcentered[+0.021,\,+0.101] \\ %
    \quad unseen families & +0.070 & [+0.031,\,+0.112]\textperiodcentered[+0.017,\,+0.102] & +0.010 & [\textminus 0.004,\,+0.028]\textperiodcentered[\textminus 0.009,\,+0.032] \\ %
    \quad unseen class & +0.102 & [+0.052,\,+0.152]\textperiodcentered[+0.053,\,+0.152] & +0.064 & [+0.020,\,+0.109]\textperiodcentered[+0.023,\,+0.112] \\ %
    Crisp edits, small classes pooled, macro-F1, test & \textminus 0.016 & [\textminus 0.061,\,+0.053]\textperiodcentered[\textminus 0.072,\,+0.045] & +0.032 & [\textminus 0.025,\,+0.098]\textperiodcentered[\textminus 0.036,\,+0.088] \\ %
    \quad unseen families & +0.097 & [+0.025,\,+0.178]\textperiodcentered[+0.006,\,+0.177] & +0.067 & [\textminus 0.007,\,+0.142]\textperiodcentered[\textminus 0.039,\,+0.148] \\ %
    \quad unseen class & +0.073 & [+0.033,\,+0.103]\textperiodcentered[+0.035,\,+0.103] & +0.029 & [+0.001,\,+0.057]\textperiodcentered[+0.002,\,+0.056] \\ %
    Fellegi-Holt, key top-1, test & +0.126 & [+0.086,\,+0.170]\textperiodcentered[+0.097,\,+0.193] & +0.132 & [+0.084,\,+0.186]\textperiodcentered[+0.092,\,+0.227] \\ %
    \quad unseen families & +0.183 & [+0.114,\,+0.256]\textperiodcentered[+0.139,\,+0.280] & +0.170 & [+0.103,\,+0.244]\textperiodcentered[+0.118,\,+0.268] \\ %
    \quad unseen class & +0.231 & [+0.173,\,+0.292]\textperiodcentered[+0.175,\,+0.298] & +0.223 & [+0.159,\,+0.289]\textperiodcentered[+0.161,\,+0.296] \\ %
    Z3, key top-1, test & +0.092 & [+0.047,\,+0.153]\textperiodcentered[+0.038,\,+0.164] & +0.098 & [+0.046,\,+0.168]\textperiodcentered[+0.040,\,+0.187] \\ %
    \quad unseen families & +0.179 & [+0.101,\,+0.279]\textperiodcentered[+0.129,\,+0.250] & +0.166 & [+0.093,\,+0.266]\textperiodcentered[+0.101,\,+0.264] \\ %
    \quad unseen class & +0.145 & [+0.107,\,+0.186]\textperiodcentered[+0.107,\,+0.186] & +0.137 & [+0.094,\,+0.181]\textperiodcentered[+0.095,\,+0.183] \\ %
    Minimum cardinality, random tie-break, abduction, test & +0.226 & [+0.160,\,+0.295]\textperiodcentered[+0.176,\,+0.351] & +0.186 & [+0.119,\,+0.256]\textperiodcentered[+0.125,\,+0.302] \\ %
    \quad unseen families & +0.269 & [+0.164,\,+0.389]\textperiodcentered[+0.181,\,+0.471] & +0.197 & [+0.063,\,+0.336]\textperiodcentered[+0.066,\,+0.402] \\ %
    \quad unseen class & +0.288 & [+0.215,\,+0.359]\textperiodcentered[+0.219,\,+0.364] & +0.276 & [+0.185,\,+0.363]\textperiodcentered[+0.191,\,+0.371] \\ %
    Minimum cardinality, oracle tie-break, abduction, test & +0.079 & [\textminus 0.013,\,+0.185]\textperiodcentered[\textminus 0.017,\,+0.299] & +0.038 & [\textminus 0.052,\,+0.142]\textperiodcentered[\textminus 0.048,\,+0.248] \\ %
    \quad unseen families & +0.120 & [\textminus 0.033,\,+0.290]\textperiodcentered[\textminus 0.003,\,+0.440] & +0.048 & [\textminus 0.123,\,+0.230]\textperiodcentered[\textminus 0.084,\,+0.350] \\ %
    \quad unseen class & +0.145 & [+0.015,\,+0.273]\textperiodcentered[+0.022,\,+0.280] & +0.133 & [\textminus 0.013,\,+0.277]\textperiodcentered[\textminus 0.003,\,+0.286] \\ %
    Class I sizing-regression residual, abduction, test & +0.079 & [+0.022,\,+0.131]\textperiodcentered[+0.016,\,+0.115] & +0.038 & [\textminus 0.016,\,+0.095]\textperiodcentered[\textminus 0.020,\,+0.064] \\ %
    \quad unseen families & \textminus 0.024 & [\textminus 0.098,\,+0.049]\textperiodcentered[\textminus 0.071,\,+0.098] & \textminus 0.096 & [\textminus 0.171,\,\textminus 0.033]\textperiodcentered[\textminus 0.144,\,\textminus 0.022] \\ %
    \quad unseen class & +0.362 & [+0.291,\,+0.433]\textperiodcentered[+0.292,\,+0.432] & +0.350 & [+0.282,\,+0.424]\textperiodcentered[+0.283,\,+0.423] \\ %
    \bottomrule
  \end{tabular}
  \begin{tablenotes}[flushleft]\scriptsize
    \item[] Intervals: [aircraft]\textperiodcentered[family], 95\% cluster-bootstrap percentiles; the first panel gives the paired change of \fwname{} on the same records, the second the difference at each oracle (Table~\ref{tab:trainoracle} gives the levels and the differences under the oracle on 107 aircraft). %
  \end{tablenotes}
  \end{threeparttable}
\end{table}

\FloatBarrier
\subsection{Training Checkpoints}
\label{app:results-curves}

Figure~\ref{M-fig:dynamics} in Section~\ref{M-sec:res-knowledge} shows the training curves of the axioms; two readings complete it. At the frozen checkpoint, restored at the best validation epoch, the mean truth of each specification constraint is lower on corrupted than on clean-labeled records in 46 of 46 split and constraint pairs with available evidence, by as little as $1.0\times10^{-5}$ on the unseen class. %
With the train prior (8 seeds), the best epoch, selected by early stopping on the validation AUROC (Table~\ref{tab:config}), reaches the declared cap of 100 epochs in three of eight seeds. 
\FloatBarrier%
\section{Explanation-Supervision Comparison}
\label{app:expsup}

This appendix gives the results and the protocol of the comparison with explanation supervision summarized in Section~\ref{M-sec:res-grounding}. %
Under the protocol described at the end of this appendix, the grounding axiom localizes the annotated parts more than RRR~\cite{ross2017right}, GradMask~\cite{simpson2019gradmask} and Pukdee et al.~\cite{pukdee2023learning}, with every cluster-bootstrap interval of the difference above zero on the FGVC pseudo-masks, on PASCAL-Part human masks (levels in Figure~\ref{fig:partgrounding}b) and on two further human-mask datasets. Against RRR, GradMask and Pukdee et al., in this order, its pointing exceeds theirs by $+0.725$ \oneci{+0.700}{+0.750}, $+0.779$ and $+0.618$ on the FGVC test pseudo-masks, by $+0.594$ \oneci{+0.559}{+0.630}, $+0.702$ and $+0.496$ on PASCAL-Part, by $+0.523$ \oneci{+0.475}{+0.571}, $+0.631$ and $+0.420$ on PartImageNet and by $+0.483$ \oneci{+0.359}{+0.595}, $+0.636$ and $+0.446$ on the 40 items of UDA-Part, with wider intervals on the last two; Pukdee et al.\ is run at the calibration of the protocol, which flattened their teacher toward the uniform prediction. The intervals of all 27 differences, nine mask configurations by three baselines, lie above zero, also under an independent bootstrap seed. %
Across the weights 0.3, 1 and 3 the axiom keeps at most one collapsed head, a head whose F1 is at most 0.05, against up to 5.7 on average for the auxiliary loss. At weight 1 it keeps an out-of-distribution accuracy lead of $+0.016$ \oneci{+0.012}{+0.020} over the auxiliary loss at its validation weight 0.3; test accuracy, by contrast, is degenerate (a constant predictor scores 0.926). %
\paragraph{Protocol.}
A first set of runs of RRR, GradMask, RES~\cite{gao2022res} (its RES-L variant, adapted to these part masks) and Pukdee et al.\ used $\lambda=1$ and penalties coupled across a batch through batch normalization~\cite{ioffe2015batch} in training mode. The comparison retrains RRR, GradMask and Pukdee et al.\ with per-example penalties, the normalization constants held fixed, selects $\lambda\in\{0.1,1,10\}$ times a calibration constant on validation with one seed per value, and evaluates three seeds at the selected value. RES-L is not retrained: on these masks its adapted objective does not reward localization, so its pointing does not measure what RES can localize. A second arm of Pukdee et al., added after the results, measures its calibration constant at the first projection step instead of at initialization; it avoids the flattened teacher and weights the constraint about three orders of magnitude less. It points to the part in 0.179 of PASCAL-Part human masks, where the axiom's lead over it is $+0.558$ \oneci{+0.524}{+0.591}, and the intervals of the axiom's differences from it also lie above zero on the FGVC pseudo-masks and on every other human-mask set. Holding the normalization constants fixed makes each penalty per-example and leaves the network in inference mode unconstrained by it, a gap that a diagnostic measures for RRR. %

\FloatBarrier 


\begin{thebibliography}{10}

\bibitem{wienand2014detecting}
Dominik Wienand and Heiko Paulheim.
\newblock Detecting incorrect numerical data in {DBpedia}.
\newblock In {\em The Semantic Web: Trends and Challenges ({ESWC} 2014)},
  volume 8465 of {\em Lecture Notes in Computer Science}, pages 504--518.
  Springer, 2014.

\bibitem{sun2020openap}
Junzi Sun, Jacco~M. Hoekstra, and Joost Ellerbroek.
\newblock {OpenAP}: An open-source aircraft performance model for air
  transportation studies and simulations.
\newblock {\em Aerospace}, 7(8):104, 2020.

\bibitem{lichtenthaler2004technology}
Eckhard Lichtenthaler.
\newblock Technology intelligence processes in leading {European} and {North
  American} multinationals.
\newblock {\em R\&D Management}, 34(2):121--135, 2004.

\bibitem{zhang2016technology}
Yi~Zhang, Douglas K.~R. Robinson, Alan~L. Porter, Donghua Zhu, Guangquan Zhang,
  and Jie Lu.
\newblock Technology roadmapping for competitive technical intelligence.
\newblock {\em Technological Forecasting and Social Change}, 110:175--186,
  2016.

\bibitem{glassman2012intelligence}
Michael Glassman and Min~Ju Kang.
\newblock Intelligence in the internet age: The emergence and evolution of open
  source intelligence ({OSINT}).
\newblock {\em Computers in Human Behavior}, 28(2):673--682, 2012.

\bibitem{pastorgalindo2020osint}
Javier Pastor-Galindo, Pantaleone Nespoli, F{\'e}lix G{\'o}mez~M{\'a}rmol, and Gregorio
  Mart{\'i}nez~P{\'e}rez.
\newblock The not yet exploited goldmine of {OSINT}: Opportunities, open
  challenges and future trends.
\newblock {\em IEEE Access}, 8:10282--10304, 2020.

\bibitem{maji2013fine}
Subhransu Maji, Esa Rahtu, Juho Kannala, Matthew Blaschko, and Andrea Vedaldi.
\newblock Fine-grained visual classification of aircraft.
\newblock arXiv:1306.5151, 2013.

\bibitem{shermeyer2021rareplanes}
Jacob Shermeyer, Thomas Hossler, Adam Van~Etten, Daniel Hogan, Ryan Lewis, and
  Daeil Kim.
\newblock {RarePlanes}: Synthetic data takes flight.
\newblock In {\em Proceedings of the IEEE Winter Conference on Applications of
  Computer Vision (WACV 2021)}, pages 207--217, 2021.

\bibitem{han2022adbench}
Songqiao Han, Xiyang Hu, Hailiang Huang, Minqi Jiang, and Yue Zhao.
\newblock {ADBench}: Anomaly detection benchmark.
\newblock In {\em Advances in Neural Information Processing Systems 35 (NeurIPS
  2022), Datasets and Benchmarks Track}, pages 32142--32159, 2022.

\bibitem{hollmann2025accurate}
Noah Hollmann, Samuel M{\"u}ller, Lennart Purucker, Arjun Krishnakumar, Max
  K{\"o}rfer, Shi~Bin Hoo, Robin~Tibor Schirrmeister, and Frank Hutter.
\newblock Accurate predictions on small data with a tabular foundation model.
\newblock {\em Nature}, 637(8045):319--326, 2025.

\bibitem{tsai2025anollm}
Che-Ping Tsai, Ganyu Teng, Phillip Wallis, and Wei Ding.
\newblock {AnoLLM}: Large language models for tabular anomaly detection.
\newblock In {\em Proceedings of the 13th International Conference on Learning
  Representations (ICLR 2025)}, pages 7562--7592, 2025.

\bibitem{lundberg2017unified}
Scott~M. Lundberg and Su-In Lee.
\newblock A unified approach to interpreting model predictions.
\newblock In {\em Advances in Neural Information Processing Systems 30 (NIPS
  2017)}, pages 4765--4774, 2017.

\bibitem{fellegi1976systematic}
I.~P. Fellegi and D.~Holt.
\newblock A systematic approach to automatic edit and imputation.
\newblock {\em Journal of the American Statistical Association},
  71(353):17--35, 1976.

\bibitem{reiter1987theory}
Raymond Reiter.
\newblock A theory of diagnosis from first principles.
\newblock {\em Artificial Intelligence}, 32(1):57--95, 1987.

\bibitem{davilagarcez2023neurosymbolic}
Artur d'Avila Garcez and Lu{\'i}s~C. Lamb.
\newblock Neurosymbolic {AI}: The 3rd wave.
\newblock {\em Artificial Intelligence Review}, 56(11):12387--12406, 2023.

\bibitem{vonrueden2021informed}
Laura von Rueden, Sebastian Mayer, Katharina Beckh, Bogdan Georgiev, Sven
  Giesselbach, Raoul Heese, Birgit Kirsch, Julius Pfrommer, Annika Pick,
  Rajkumar Ramamurthy, Michal Walczak, Jochen Garcke, Christian Bauckhage, and
  Jannis Schuecker.
\newblock Informed machine learning: {A} taxonomy and survey of integrating
  prior knowledge into learning systems.
\newblock {\em IEEE Transactions on Knowledge and Data Engineering},
  35(1):614--633, 2023.

\bibitem{acharya2025integrating}
Kamal Acharya, Iman Sharifi, Mehul Lad, Liang Sun, and Houbing Song.
\newblock Integrating neurosymbolic {AI} in advanced air mobility: A
  comprehensive survey.
\newblock In {\em Proceedings of the Thirty-Fourth International Joint
  Conference on Artificial Intelligence (IJCAI 2025), Survey Track}, pages 10362--10370, 2025.

\bibitem{mondal2025logic}
Semanto Mondal, Antonino Ferraro, Fabiano Pecorelli, and Giuseppe De~Pietro.
\newblock A {Logic Tensor Network}-based neurosymbolic framework for
  explainable diabetes prediction.
\newblock {\em Applied Sciences}, 15(21):11806, 2025.

\bibitem{badreddine2022ltn}
Samy Badreddine, Artur d'Avila Garcez, Luciano Serafini, and Michael Spranger.
\newblock {Logic Tensor Networks}.
\newblock {\em Artificial Intelligence}, 303:103649, 2022.

\bibitem{fischer2019dl2}
Marc Fischer, Mislav Balunovi{\'c}, Dana Drachsler-Cohen, Timon Gehr, Ce~Zhang, and
  Martin Vechev.
\newblock {DL}2: Training and querying neural networks with logic.
\newblock In {\em
  Proceedings of the 36th International Conference on Machine Learning (ICML 2019)},
  volume~97 of {\em Proceedings of Machine Learning Research}, pages
  1931--1941. PMLR, 2019.

\bibitem{vankrieken2022analyzing}
Emile van Krieken, Erman Acar, and Frank van Harmelen.
\newblock Analyzing differentiable fuzzy logic operators.
\newblock {\em Artificial Intelligence}, 302:103602, 2022.

\bibitem{kang2023physics}
Ruiyuan Kang, Tingting Mu, Panagiotis Liatsis, and Dimitrios Kyritsis.
\newblock Physics-driven {ML}-based modelling for correcting inverse
  estimation.
\newblock In {\em Advances in Neural Information Processing Systems 36
  ({NeurIPS} 2023)}, pages 15639--15659, 2023.

\bibitem{vanderloo2021data}
Mark P.~J. van~der Loo and Edwin de~Jonge.
\newblock Data validation infrastructure for {R}.
\newblock {\em Journal of Statistical Software}, 97(10):1--31, 2021.

\bibitem{ni2024automatic}
Wei Ni, Xiaoye Miao, Xiangyu Zhao, Yangyang Wu, Shuwei Liang, and Jianwei Yin.
\newblock Automatic data repair: Are we ready to deploy?
\newblock {\em Proceedings of the {VLDB} Endowment}, 17(10):2617--2630, 2024.

\bibitem{rekatsinas2017holoclean}
Theodoros Rekatsinas, Xu~Chu, Ihab~F. Ilyas, and Christopher R{\'e}.
\newblock {HoloClean}: Holistic data repairs with probabilistic inference.
\newblock {\em Proceedings of the VLDB Endowment}, 10(11):1190--1201, 2017.

\bibitem{mahdavi2020baran}
Mohammad Mahdavi and Ziawasch Abedjan.
\newblock {Baran}: Effective error correction via a unified context
  representation and transfer learning.
\newblock {\em Proceedings of the {VLDB} Endowment}, 13(12):1948--1961, 2020.

\bibitem{demoura2008z3}
Leonardo de~Moura and Nikolaj Bj{\o}rner.
\newblock {Z3}: An efficient {SMT} solver.
\newblock In {\em Tools and Algorithms for the Construction and Analysis of
  Systems ({TACAS} 2008)}, volume 4963 of {\em Lecture Notes in Computer
  Science}, pages 337--340. Springer, 2008.

\bibitem{bjorner2015nuz}
Nikolaj Bj{\o}rner, Anh-Dung Phan, and Lars Fleckenstein.
\newblock {$\nu$Z}: {A}n optimizing {SMT} solver.
\newblock In {\em Tools and Algorithms for the Construction and Analysis of
  Systems ({TACAS} 2015)}, volume 9035 of {\em Lecture Notes in Computer
  Science}, pages 194--199. Springer, 2015.

\bibitem{eastman2009automatic}
C.~Eastman, Jae-min Lee, Yeon-suk Jeong, and Jin-kook Lee.
\newblock Automatic rule-based checking of building designs.
\newblock {\em Automation in Construction}, 18(8):1011--1033, 2009.

\bibitem{iversen2026leveraging}
Odin Iversen and Lizhen Huang.
\newblock Leveraging large language models for {BIM}-based automated compliance
  checking.
\newblock {\em Automation in Construction}, 182:106707, 2026.

\bibitem{schelter2018automating}
Sebastian Schelter, Dustin Lange, Philipp Schmidt, Meltem Celikel, Felix
  Biessmann, and Andreas Grafberger.
\newblock Automating large-scale data quality verification.
\newblock {\em Proceedings of the {VLDB} Endowment}, 11(12):1781--1794, 2018.

\bibitem{chen2026ai}
Xin Chen and Kai Cheng.
\newblock {AI}-assisted auditing review of mechanical engineering drawings: A
  {Perception-Cognition-Collaboration} reference architecture and its
  implementation perspectives.
\newblock {\em International Journal of Computer Integrated Manufacturing},
  pages 1--22, 2026. \newblock Published online, doi:10.1080/0951192X.2026.2714117.

\bibitem{jung2018combining}
Daniel Jung, Kok~Yew Ng, Erik Frisk, and Mattias Krysander.
\newblock Combining model-based diagnosis and data-driven anomaly classifiers
  for fault isolation.
\newblock {\em Control Engineering Practice}, 80:146--156, 2018.

\bibitem{mohammadi2025consistency}
Arman Mohammadi, Mattias Krysander, and Daniel Jung.
\newblock Consistency-based diagnosis using data-driven residuals and limited
  training data.
\newblock {\em Control Engineering Practice}, 159:106283, 2025.

\bibitem{bohne2025domain}
Tim Bohne, Anne-Kathrin~Patricia Windler, and Martin Atzmueller.
\newblock A domain-agnostic neuro-symbolic architecture for multimodal
  human-in-the-loop anomaly detection and complex fault diagnosis.
\newblock {\em {IEEE} Access}, 13:210201--210236, 2025.

\bibitem{huang2026integrating}
Haoliang Huang, Zihuang Cai, Zhuo Tang, Yifan Liu, Chen Tian, Kenli Li, and
  Changjian Chen.
\newblock Integrating data validation with large language models for
  regulation-guided tabular anomaly detection.
\newblock In {\em Proceedings of the 64th Annual Meeting of the Association for
  Computational Linguistics (Volume 1: Long Papers)}, pages 6559--6581.
  Association for Computational Linguistics, 2026.

\bibitem{yan2024gidcl}
Mengyi Yan, Yaoshu Wang, Yue Wang, Xiaoye Miao, and Jianxin Li.
\newblock {GIDCL}: A graph-enhanced interpretable data cleaning framework with
  large language models.
\newblock {\em Proceedings of the {ACM} on Management of Data},
  2(6):236:1--236:29, 2024.

\bibitem{ni2025zeroed}
Wei Ni, Kaihang Zhang, Xiaoye Miao, Xiangyu Zhao, Yangyang Wu, Yaoshu Wang, and
  Jianwei Yin.
\newblock {ZeroED}: Hybrid zero-shot error detection through large language
  model reasoning.
\newblock In {\em 2025 {IEEE} 41st International Conference on Data Engineering
  ({ICDE})}, pages 3126--3139, 2025.

\bibitem{donadello2017logic}
Ivan Donadello, Luciano Serafini, and Artur d'Avila Garcez.
\newblock {Logic Tensor Networks} for semantic image interpretation.
\newblock In {\em Proceedings of the Twenty-Sixth International Joint
  Conference on Artificial Intelligence (IJCAI 2017)}, pages 1596--1602, 2017.

\bibitem{hogea2024logiclstm}
Eduard Hogea, Darian~M. Onchi{\c{s}}, Ruqiang Yan, and Zheng Zhou.
\newblock {LogicLSTM}: Logically-driven long short-term memory model for fault
  diagnosis in gearboxes.
\newblock {\em Journal of Manufacturing Systems}, 77:892--902, 2024.

\bibitem{hogea2026rule}
Eduard Hogea, Darian~M. Onchi{\c{s}}, Ruqiang Yan, and Zheng Zhou.
\newblock Rule guided transformers for dynamic knowledge adaptation in rotating
  machinery fault diagnosis.
\newblock {\em Advanced Engineering Informatics}, 72:104444, 2026.

\bibitem{manhaeve2021neural}
Robin Manhaeve, Sebastijan Duman{\v{c}}i{\'c}, Angelika Kimmig, Thomas
  Demeester, and Luc De~Raedt.
\newblock Neural probabilistic logic programming in {DeepProbLog}.
\newblock {\em Artificial Intelligence}, 298:103504, 2021.

\bibitem{li2023scallop}
Ziyang Li, Jiani Huang, and Mayur Naik.
\newblock Scallop: A language for neurosymbolic programming.
\newblock {\em Proceedings of the ACM on Programming Languages},
  7(PLDI):1463--1487, 2023.

\bibitem{yang2020neurasp}
Zhun Yang, Adam Ishay, and Joohyung Lee.
\newblock {NeurASP}: Embracing neural networks into answer set programming.
\newblock In {\em Proceedings of the Twenty-Ninth International Joint
  Conference on Artificial Intelligence ({IJCAI} 2020)}, pages 1755--1762,
  2020.

\bibitem{hu2025efficient}
Wen-Chao Hu, Wang-Zhou Dai, Yuan Jiang, and Zhi-Hua Zhou.
\newblock Efficient rectification of neuro-symbolic reasoning inconsistencies
  by abductive reflection.
\newblock In {\em Proceedings of the {AAAI} Conference on Artificial
  Intelligence}, volume~39, pages 17333--17341, 2025.

\bibitem{xu2018semantic}
Jingyi Xu, Zilu Zhang, Tal Friedman, Yitao Liang, and Guy {Van den Broeck}.
\newblock A semantic loss function for deep learning with symbolic knowledge.
\newblock In {\em Proceedings of the 35th International Conference on Machine
  Learning (ICML 2018)}, volume~80 of {\em Proceedings of Machine Learning
  Research}, pages 5502--5511. PMLR, 2018.

\bibitem{stoian2024pishield}
Mihaela~C{\u{a}}t{\u{a}}lina Stoian, Alex Tatomir, Thomas Lukasiewicz, and
  Eleonora Giunchiglia.
\newblock {PiShield}: A {PyTorch} package for learning with requirements.
\newblock In {\em Proceedings of the Thirty-Third International Joint
  Conference on Artificial Intelligence (IJCAI 2024), Demo Track}, pages
  8805--8809, 2024.

\bibitem{daniele2019kenn}
Alessandro Daniele and Luciano Serafini.
\newblock Knowledge enhanced neural networks.
\newblock In {\em {PRICAI} 2019: Trends in Artificial Intelligence}, volume 11670 of {\em Lecture Notes in Computer Science}, pages
  542--554. Springer, 2019.

\bibitem{ciravegna2023logic}
Gabriele Ciravegna, Pietro Barbiero, Francesco Giannini, Marco Gori, Pietro
  Li{\`o}, Marco Maggini, and Stefano Melacci.
\newblock Logic explained networks.
\newblock {\em Artificial Intelligence}, 314:103822, 2023.

\bibitem{bach2017hinge}
Stephen~H. Bach, Matthias Broecheler, Bert Huang, and Lise Getoor.
\newblock Hinge-loss {M}arkov random fields and probabilistic soft logic.
\newblock {\em Journal of Machine Learning Research}, 18(109):1--67, 2017.

\bibitem{hsia2025neurosymbolic}
Yung-Shen Hsia, Fang Yu, and Jie-Hong~Roland Jiang.
\newblock Neuro-symbolic compliance: Integrating {LLMs} and {SMT} solvers for
  automated financial legal analysis.
\newblock In {\em 2025 2nd {IEEE/ACM} International Conference on {AI}-powered
  Software ({AIware})}, pages 1--10. IEEE, 2025.

\bibitem{yang2025qwen3}
An~Yang et~al.
\newblock {Qwen3} technical report.
\newblock arXiv:2505.09388, 2025.

\bibitem{bai2025qwen3vl}
Shuai Bai et~al.
\newblock {Qwen3-VL} technical report.
\newblock arXiv:2511.21631, 2025.

\bibitem{doris2025designqa}
Anna~C. Doris, Daniele Grandi, Ryan Tomich, Md~Ferdous Alam, Mohammadmehdi
  Ataei, Hyunmin Cheong, and Faez Ahmed.
\newblock {DesignQA}: A multimodal benchmark for evaluating large language
  models' understanding of engineering documentation.
\newblock {\em Journal of Computing and Information Science in Engineering},
  25(2):021009, 2025.

\bibitem{barbara2023neurosymbolic}
Vito Barbara, Massimo Guarascio, Nicola Leone, Giuseppe Manco, Alessandro
  Quarta, Francesco Ricca, and Ettore Ritacco.
\newblock Neuro-symbolic {AI} for compliance checking of electrical control
  panels.
\newblock {\em Theory and Practice of Logic Programming}, 23(4):748--764, 2023.

\bibitem{kim2026visual}
Jisoo Kim, TaeYoon Kwack, Jinwoo Jang, and Honguk Woo.
\newblock Visual compliance via executable safety rule entailment.
\newblock arXiv:2609.18328, 2026.
\newblock Accepted to {EMNLP} 2026.

\bibitem{ross2017right}
Andrew~Slavin Ross, Michael~C. Hughes, and Finale Doshi-Velez.
\newblock Right for the right reasons: Training differentiable models by
  constraining their explanations.
\newblock In {\em Proceedings of the Twenty-Sixth International Joint
  Conference on Artificial Intelligence (IJCAI 2017)}, pages 2662--2670, 2017.

\bibitem{simpson2019gradmask}
Becks Simpson, Francis Dutil, Yoshua Bengio, and Joseph~Paul Cohen.
\newblock {GradMask}: Reduce overfitting by regularizing saliency.
\newblock arXiv:1904.07478, 2019.

\bibitem{gao2022res}
Yuyang Gao, Tong~Steven Sun, Guangji Bai, Siyi Gu, Sungsoo~Ray Hong, and Liang
  Zhao.
\newblock {RES}: A robust framework for guiding visual explanation.
\newblock In {\em Proceedings of the 28th ACM SIGKDD Conference on Knowledge
  Discovery and Data Mining}, pages 432--442. ACM, 2022.

\bibitem{rao2023studying}
Sukrut Rao, Moritz B{\"o}hle, Amin Parchami-Araghi, and Bernt Schiele.
\newblock Studying how to efficiently and effectively guide models with
  explanations.
\newblock In {\em Proceedings of the IEEE/CVF International Conference on Computer Vision
  (ICCV 2023)}, pages 1922--1933, 2023.

\bibitem{pukdee2023learning}
Rattana Pukdee, Dylan Sam, J.~Zico Kolter, Maria-Florina Balcan, and Pradeep
  Ravikumar.
\newblock Learning with explanation constraints.
\newblock In {\em Advances in Neural Information Processing Systems 36 (NeurIPS
  2023)}, pages 49883--49926, 2023.

\bibitem{du2024tip}
Siyi Du, Shaoming Zheng, Yinsong Wang, Wenjia Bai, Declan~P. O'Regan, and Chen
  Qin.
\newblock {TIP}: Tabular-image pre-training for multimodal classification with
  incomplete data.
\newblock In {\em Computer Vision -- {ECCV} 2024}, volume 15073 of {\em Lecture
  Notes in Computer Science}, pages 478--496. Springer, 2024.

\bibitem{tang2024autogluon}
Zhiqiang Tang, Haoyang Fang, Su~Zhou, Taojiannan Yang, Zihan Zhong, Cuixiong
  Hu, Katrin Kirchhoff, and George Karypis.
\newblock {AutoGluon-Multimodal} ({AutoMM}): Supercharging multimodal {AutoML}
  with foundation models.
\newblock In {\em Proceedings of the Third International Conference on
  Automated Machine Learning (AutoML 2024)}, volume 256 of {\em Proceedings of
  Machine Learning Research}, pages 15/1--35. PMLR, 2024.

\bibitem{stoian2024realistic}
Mihaela~C{\u{a}}t{\u{a}}lina Stoian, Salijona Dyrmishi, Maxime Cordy, Thomas
  Lukasiewicz, and Eleonora Giunchiglia.
\newblock How realistic is your synthetic data? {C}onstraining deep generative
  models for tabular data.
\newblock In {\em Proceedings of the 12th International Conference on Learning
  Representations (ICLR 2024)}, pages 32167--32204, 2024.

\bibitem{radtke2022expert}
Maximilian-Peter Radtke and J{\"u}rgen Bock.
\newblock Expert knowledge induced {Logic Tensor Networks}: A bearing fault
  diagnosis case study.
\newblock {\em PHM Society European Conference}, 7(1):421--431, 2022.

\bibitem{gaikwad2026neurosymbolic}
Devashish Gaikwad, Wil M.~P. van~der Aalst, and Gyunam Park.
\newblock Neuro-symbolic process anomaly detection.
\newblock In {\em Advanced Information Systems Engineering ({CAiSE} 2026)},
  volume 16559 of {\em Lecture Notes in Computer Science}, pages 240--255.
  Springer, 2026.

\bibitem{gao2024going}
Yuyang Gao, Siyi Gu, Junji Jiang, Sungsoo~Ray Hong, Dazhou Yu, and Liang Zhao.
\newblock Going beyond {XAI}: A systematic survey for explanation-guided
  learning.
\newblock {\em ACM Computing Surveys}, 56(7):188:1--188:39, 2024.

\bibitem{koh2020concept}
Pang~Wei Koh, Thao Nguyen, Yew~Siang Tang, Stephen Mussmann, Emma Pierson, Been
  Kim, and Percy Liang.
\newblock Concept bottleneck models.
\newblock In {\em Proceedings of the 37th International Conference on Machine
  Learning (ICML 2020)}, volume 119 of {\em Proceedings of Machine Learning
  Research}, pages 5338--5348. PMLR, 2020.

\bibitem{wang2025internvl35}
Weiyun Wang et~al.
\newblock {InternVL3.5}: Advancing open-source multimodal models in
  versatility, reasoning, and efficiency.
\newblock arXiv:2508.18265, 2025.

\bibitem{carraro2024ltntorch}
Tommaso Carraro, Luciano Serafini, and Fabio Aiolli.
\newblock {LTNtorch}: {PyTorch} implementation of {Logic Tensor Networks}.
\newblock arXiv:2409.16045, 2024.

\bibitem{raymer2018aircraft}
Daniel Raymer.
\newblock {\em Aircraft Design: A Conceptual Approach}.
\newblock American Institute of Aeronautics and Astronautics, Inc., sixth
  edition, 2018.

\bibitem{loshchilov2017decoupled}
Ilya Loshchilov and Frank Hutter.
\newblock Decoupled weight decay regularization.
\newblock In {\em Proceedings of the 7th International Conference on Learning
  Representations (ICLR 2019)}, 2019.
\newblock OpenReview Bkg6RiCqY7.

\bibitem{he2016resnet}
Kaiming He, Xiangyu Zhang, Shaoqing Ren, and Jian Sun.
\newblock Deep residual learning for image recognition.
\newblock In {\em Proceedings of the IEEE Conference on Computer Vision and
  Pattern Recognition (CVPR 2016)}, pages 770--778, 2016.

\bibitem{deng2009imagenet}
Jia Deng, Wei Dong, Richard Socher, Li-Jia Li, Kai Li, and Li~Fei-Fei.
\newblock {ImageNet}: A large-scale hierarchical image database.
\newblock In {\em Proceedings of the IEEE Conference on Computer Vision and Pattern
  Recognition (CVPR 2009)}, pages 248--255, 2009.

\bibitem{zhou2016learning}
Bolei Zhou, Aditya Khosla, Agata Lapedriza, Aude Oliva, and Antonio Torralba.
\newblock Learning deep features for discriminative localization.
\newblock In {\em Proceedings of the IEEE Conference on Computer Vision and
  Pattern Recognition (CVPR 2016)}, pages 2921--2929, 2016.

\bibitem{minderer2023scaling}
Matthias Minderer, Alexey Gritsenko, and Neil Houlsby.
\newblock Scaling open-vocabulary object detection.
\newblock In {\em Advances in Neural Information Processing Systems 36 (NeurIPS
  2023)}, pages 72983--73007, 2023.

\bibitem{kirillov2023segment}
Alexander Kirillov, Eric Mintun, Nikhila Ravi, Hanzi Mao, Chloe Rolland, Laura
  Gustafson, Tete Xiao, Spencer Whitehead, Alexander~C. Berg, Wan-Yen Lo, Piotr
  Doll{\'a}r, and Ross Girshick.
\newblock Segment anything.
\newblock In {\em Proceedings of the IEEE/CVF International Conference on
  Computer Vision (ICCV 2023)}, pages 3992--4003, 2023.

\bibitem{youden1950index}
W.~J. Youden.
\newblock Index for rating diagnostic tests.
\newblock {\em Cancer}, 3(1):32--35, 1950.

\bibitem{byrd1995limited}
Richard~H. Byrd, Peihuang Lu, Jorge Nocedal, and Ciyou Zhu.
\newblock A limited memory algorithm for bound constrained optimization.
\newblock {\em SIAM Journal on Scientific Computing}, 16(5):1190--1208, 1995.

\bibitem{guidotti2024counterfactual}
Riccardo Guidotti.
\newblock Counterfactual explanations and how to find them: Literature review
  and benchmarking.
\newblock {\em Data Mining and Knowledge Discovery}, 38(5):2770--2824, 2024.

\bibitem{nocedal2006numerical}
Jorge Nocedal and Stephen~J. Wright.
\newblock {\em Numerical Optimization}.
\newblock Springer Series in Operations Research and Financial Engineering.
  Springer, New York, second edition, 2006.

\bibitem{kingma2014adam}
Diederik~P. Kingma and Jimmy Ba.
\newblock Adam: A method for stochastic optimization.
\newblock In {\em Proceedings of the 3rd International Conference on Learning
  Representations (ICLR 2015)}, San Diego, 2015.
\newblock arXiv:1412.6980.

\bibitem{vanbuuren2011mice}
Stef van Buuren and Karin Groothuis-Oudshoorn.
\newblock {mice}: Multivariate imputation by chained equations in {R}.
\newblock {\em Journal of Statistical Software}, 45(3):1--67, 2011.

\bibitem{chen2014detect}
Xianjie Chen, Roozbeh Mottaghi, Xiaobai Liu, Sanja Fidler, Raquel Urtasun, and
  Alan Yuille.
\newblock Detect what you can: Detecting and representing objects using
  holistic models and body parts.
\newblock In {\em Proceedings of the IEEE Conference on Computer Vision and Pattern
  Recognition (CVPR 2014)}, pages 1979--1986, 2014.

\bibitem{he2022partimagenet}
Ju~He, Shuo Yang, Shaokang Yang, Adam Kortylewski, Xiaoding Yuan, Jie-Neng
  Chen, Shuai Liu, Cheng Yang, Qihang Yu, and Alan Yuille.
\newblock {PartImageNet}: A large, high-quality dataset of parts.
\newblock In {\em Computer Vision -- {ECCV} 2022}, volume 13668 of {\em Lecture
  Notes in Computer Science}, pages 128--145. Springer, 2022.

\bibitem{liu2022udapart}
Qing Liu, Adam Kortylewski, Zhishuai Zhang, Zizhang Li, Mengqi Guo, Qihao Liu,
  Xiaoding Yuan, Jiteng Mu, Weichao Qiu, and Alan Yuille.
\newblock Learning part segmentation through unsupervised domain adaptation
  from synthetic vehicles.
\newblock In {\em Proceedings of the IEEE/CVF Conference on Computer Vision and
  Pattern Recognition (CVPR 2022)}, pages 19118--19129, 2022.

\bibitem{field2007bootstrapping}
C.~A. Field and A.~H. Welsh.
\newblock Bootstrapping clustered data.
\newblock {\em Journal of the Royal Statistical Society Series B: Statistical
  Methodology}, 69(3):369--390, 2007.

\bibitem{ke2017lightgbm}
Guolin Ke, Qi~Meng, Thomas Finley, Taifeng Wang, Wei Chen, Weidong Ma, Qiwei
  Ye, and Tie-Yan Liu.
\newblock {LightGBM}: A highly efficient gradient boosting decision tree.
\newblock In {\em Advances in Neural Information Processing Systems 30 (NIPS
  2017)}, pages 3146--3154, 2017.

\bibitem{zhang2018top}
Jianming Zhang, Sarah~Adel Bargal, Zhe Lin, Jonathan Brandt, Xiaohui Shen, and
  Stan Sclaroff.
\newblock Top-down neural attention by excitation backprop.
\newblock {\em International Journal of Computer Vision}, 126(10):1084--1102,
  2018.

\bibitem{liu2022convnet}
Zhuang Liu, Hanzi Mao, Chao-Yuan Wu, Christoph Feichtenhofer, Trevor Darrell,
  and Saining Xie.
\newblock A {ConvNet} for the 2020s.
\newblock In {\em Proceedings of the IEEE/CVF Conference on Computer Vision and
  Pattern Recognition (CVPR 2022)}, pages 11966--11976, 2022.

\bibitem{rubin1976inference}
Donald~B. Rubin.
\newblock Inference and missing data.
\newblock {\em Biometrika}, 63(3):581--592, 1976.

\bibitem{neverova2016moddrop}
Natalia Neverova, Christian Wolf, Graham Taylor, and Florian Nebout.
\newblock {ModDrop}: Adaptive multi-modal gesture recognition.
\newblock {\em IEEE Transactions on Pattern Analysis and Machine Intelligence},
  38(8):1692--1706, 2016.

\bibitem{brundage2018malicious}
Miles Brundage et~al.
\newblock The malicious use of artificial intelligence: Forecasting,
  prevention, and mitigation.
\newblock arXiv:1802.07228, 2018.

\bibitem{ioffe2015batch}
Sergey Ioffe and Christian Szegedy.
\newblock Batch normalization: Accelerating deep network training by reducing
  internal covariate shift.
\newblock In {\em Proceedings of the 32nd International Conference on Machine
  Learning (ICML 2015)}, volume~37 of {\em Proceedings of Machine Learning
  Research}, pages 448--456. PMLR, 2015.

\end{thebibliography}
\end{document}